\documentclass{article}
\PassOptionsToPackage{numbers, compress}{natbib}
\usepackage[final,main]{neurips_2026}

\usepackage{times}

\usepackage{nicefrac}       %
\usepackage{microtype}      %

\usepackage[utf8]{inputenc}
\usepackage[T1]{fontenc}
\usepackage{microtype}
\usepackage{graphicx}
\usepackage{subfigure}
\usepackage{booktabs}

\usepackage[colorlinks=true,citecolor=blue,linkcolor=blue]{hyperref}

\usepackage{url}
\usepackage{nicefrac}
\usepackage{listings}
\usepackage{mathtools}
\usepackage{amsfonts}
\usepackage{amssymb}
\usepackage{amsmath}
\usepackage{amsthm}
\usepackage{docmute}
\usepackage[acronym, nomain]{glossaries}
\usepackage{indentfirst}
\usepackage{bm}
\usepackage{here}
\usepackage{booktabs}       %
\usepackage{cancel}
\usepackage{multirow}
\usepackage{wrapfig}
\usepackage{lipsum}
\usepackage{bbm}
\usepackage{thm-restate}

\usepackage{algorithm}
\usepackage{algpseudocode}
\usepackage{enumitem}

\newtheorem{theorem}{Theorem}
\newtheorem{proposition}{Proposition}
\newtheorem{corollary}{Corollary}

\newtheorem{definition}{Definition}

\newtheorem{remark}{Remark}

\DeclarePairedDelimiterX{\KL}[2]{\mathrm{KL}[}{]}{#1\;\delimsize\|\;#2}

\DeclarePairedDelimiterX\braket[2]{\langle}{\rangle}{#1 \delimsize\vert #2}

\newcommand{\calx}{\mathcal{X}}
\newcommand{\calg}{\mathcal{G}}

\newcommand{\Ex}{\mathbb{E}}
\newcommand{\calo}{\mathcal{O}}

\newcommand{\calw}{\mathcal{W}}

\newcommand{\calc}{\mathcal{C}}

\newcommand{\caly}{\mathcal{Y}}

\newcommand{\R}{\mathbb{R}}

\newcommand{\MC}{\operatorname{MC}}
\newcommand{\Lip}{\operatorname{Lip}}

\newcommand{\UU}{\mathrm{UU}}
\newcommand{\PU}{\mathrm{PU}}
\newcommand{\DU}{\mathrm{DU}}
\newcommand{\Pconf}{\mathrm{Pconf}}
\newcommand{\SU}{\mathrm{SU}}
\newcommand{\SD}{\mathrm{SD}}
\newcommand{\Pcomp}{\mathrm{Pcomp}}

\newcommand{\Sconf}{\mathrm{Sconf}}
\newcommand{\WSL}{\mathrm{WSL}}

\usepackage{multirow}

\usepackage[utf8]{inputenc} %
\usepackage[T1]{fontenc}    %
\usepackage{hyperref}       %
\usepackage{url}            %
\usepackage{booktabs}
\usepackage{multirow}       %
\usepackage{amsfonts}       %
\usepackage{nicefrac}       %
\usepackage{microtype}      %
\usepackage{lipsum}
\usepackage{fancyhdr}       %
\usepackage{graphicx}       %
\graphicspath{{media/}}     %
\usepackage{mathtools}
\usepackage{CJKutf8}
\usepackage{mathtools}

\mathtoolsset{showonlyrefs=true}
\usepackage{xcolor}

\title{Unified Approach for Weakly Supervised Multicalibration}
\author{%
Futoshi Futami\\
The University of Osaka / RIKEN AIP / The University of Tokyo\\
\texttt{futami.futoshi.es@osaka-u.ac.jp}
\AND
Takashi Ishida\\
The University of Tokyo / RIKEN AIP\\
\texttt{ishi@k.u-tokyo.ac.jp}
}
\date{}

\begin{document}
\maketitle

\begin{abstract}
Multicalibration requires predicted scores to agree with label probabilities across
rich families of subgroups and score-dependent tests, but existing methods require
clean input--label pairs for evaluation and post-processing. This assumption fails
in weakly supervised learning (WSL) regimes---including positive-unlabeled,
unlabeled-unlabeled, and positive-confidence learning---where clean labels are
costly or unavailable even though reliable uncertainty estimates may be crucial.
We address this gap by developing estimators of multicalibration error and post-hoc
correction methods for WSL settings in which clean input--label pairs are unavailable.
We propose a unified framework for estimating and correcting multicalibration
under weak supervision by combining contamination-matrix risk rewrites with
witness-based calibration constraints, yielding corrected multicalibration moments
with finite-sample guarantees. We further propose weak-label multicalibration boost (WLMC), a generic post-hoc recalibration algorithm under weak supervision.
Finally, we conduct experiments across multiple weak-supervision settings to
evaluate multicalibration behavior and offer empirical insight into uncertainty estimation under weak supervision.
\end{abstract}

\section{Introduction}
\label{sec_intro}
Reliable uncertainty estimates are now an indispensable requirement in high-stakes machine learning applications. In domains such as medicine, science, and decision support, a predictor should not only be accurate on average but also output scores that can be trusted as probabilities. Calibration formalizes this requirement by asking predicted probabilities to agree with outcome frequencies \citep{Dawid1982,foster1998asymptotic}. Multicalibration strengthens the requirement by enforcing calibration simultaneously across a rich family of subpopulations and score tests \citep{johnson18a,gopalan2022low,pmlr-v202-globus-harris23a}. This viewpoint has made calibration useful far beyond aggregate reliability: it is used to reason about subgroup fairness, distributional reliability, and downstream decisions based on predicted probabilities \citep{dwork2021outcome}.

Despite this progress, existing multicalibration methods still face a basic limitation. To the best of our knowledge, they assume access to complete input-label pairs both when computing calibration metrics and when running post-processing. In some applications, however, clean labels are expensive to obtain or unavailable at scale, while only weaker supervision signals are observed. Such situations have motivated weakly supervised learning (WSL), especially for classification, where predictors can be learned even without complete label information. For example, in positive-unlabeled (PU) learning~\citep{duplessis2014analysis,plessis2015convex,kiryo2017nnpu}, only positive and unlabeled data may be available; other representative examples include unlabeled-unlabeled (UU) learning \citep{lu2018on,lu21c} and positive-confidence (Pconf) learning \citep{ishida2018}. 

In these settings, weak supervision creates a basic asymmetry: it can provide enough information to train an accurate classifier, while withholding the clean label-frequency information needed to evaluate or correct calibration directly. This is problematic because WSL is often used where clean labels are costly, delayed, or unavailable---including medicine and decision support---yet learned scores set thresholds, prioritize cases, and guide downstream actions. Thus, the settings in which reliable uncertainty estimates are most needed are often the same settings in which multicalibration and post-processing must be performed using only weakly supervised data. Since calibration is defined through label frequencies, and multicalibration through subgroup- and score-dependent residual moments, standard metrics and post-processing methods cannot be applied without modification.

Only limited work addresses this obstacle. Most closely, \citet{kiryo2026estimating} extend binned expected calibration error (ECE) to PU learning. 
Their estimator is derived directly for the PU binned-ECE setting by applying Bayes' rule within each score bin. This result is important, but it addresses only PU binned ECE: it does not cover multicalibration over rich witness classes, does not give a post-processing method, and does not address other weak-observation models. More broadly, weakly supervised learning contains many observation models beyond PU learning (see Appendix~\ref{app:additional-settings}), and deriving a separate metric and correction algorithm for each model would fragment the theory, complicate comparison and model selection, and obscure the relationships among settings.

This motivates a unified viewpoint. A natural starting point is the unified risk-analysis view of weakly supervised learning \citep{ChiangSugiyama2025}. There, weak supervision is described by a contamination matrix, or more generally an operator, that maps inaccessible clean distributions to observable weak distributions; target clean risks are then recovered by a decontamination rewrite. A key advantage is that, once the observable marginal information and the contamination mechanism are specified, estimators whose expectations equal the target losses can be derived in a systematic way.

In this paper, we use the WSL risk-rewrite principle in a different statistical role. 
Decontamination identities in WSL are most commonly used to construct empirical risks for training classifiers from weak observations. 
Here, the predictor is fixed, and the target is not a training objective but a calibration residual that measures whether predicted probabilities agree with labels after restricting attention to a subgroup and a score-dependent test. 
This shift matters for multicalibration: the goal is not to optimize a single corrected risk, but to audit a family of residuals and control the worst violation over the witness class. 
For each fixed witness, the residual can be made observable by the same decontamination machinery as an ordinary loss; the remaining statistical question is whether the corrected empirical residuals can be controlled uniformly and used to drive post-hoc correction.

Our contributions are as follows. 
First, we introduce a witness-level decontamination reduction for multicalibration (Theorem~\ref{thm:generic-wsl-mc}), which treats calibration constraints as signed residual auditing functionals rather than as training losses. 
This yields corrected PU, UU, and Pconf estimators (Corollaries~\ref{prop:pu}, \ref{prop:uu}, and~\ref{prop:pconf}) and finite-sample uniform guarantees for weak multicalibration estimates. 
Second, we use these moments in weak-label multicalibration boost (WLMC), a weak-label post-processing algorithm, and prove its finite-sample guarantee in Theorem~\ref{thm:pu-wlmc-boost}. Finally, to the best of our knowledge, we give the first real-data study of weak-label multicalibration, using weak labels for both estimation and post-hoc correction. The experiments yield three takeaways: weak multicalibration estimates can track oracle estimates; weakly trained predictors are often less calibrated; and weak labels can improve post-hoc calibration even for clean-label predictors, suggesting a lower-cost route for improving multicalibration when clean calibration labels are expensive.

The paper is organized as follows. Section~\ref{sec:related} discusses related work, Section~\ref{sec:prelim} introduces multicalibration and weak supervision, Section~\ref{sec:wsl-mc} introduces our corrected metrics, Section~\ref{sec:algo} proposes the post-processing method, and Section~\ref{sec:experiments} reports the empirical study.

\vspace{-1mm}
\section{\texorpdfstring{Related work}{Related work}}
\label{sec:related}
\vspace{-1.5mm}
Multicalibration was introduced by \citet{johnson18a} as a subgroup fairness notion, and has since been studied through efficient algorithms, boosting, and statistical guarantees, and empirical evaluations
\citep{gopalan2022low,pmlr-v202-globus-harris23a,gopalan2024computationally,hansen2024multicalibration}.
The broader calibration literature studies classical forecasting calibration, neural-network calibration, smooth or IPM-style metrics, trainable calibration objectives, scalable testing, and adaptive evaluation
\citep{Dawid1982,foster1998asymptotic,guo2017calibration,blasiok2023unify,widmann2021calibration,marx2023calibration,hu2024testing,NEURIPS2024_futami}.
Existing multicalibration and decision-calibration style guarantees, however, are developed under clean labels. Our work is complementary: we study how to estimate and correct calibration witnesses when only weak labels are available. 

Our technical starting point is the decontamination view of WSL, which represents weak-supervision models through contamination matrices and rewrites clean risks from weak data~\citep{ChiangSugiyama2025}. 
These rewrites are most commonly used to build empirical objectives for classifier learning. 
We instead apply the same principle to witness-based calibration residuals and post-hoc correction. 
The closest prior calibration-estimation work is PU-ECE~\citep{kiryo2026estimating}, which estimates binned ECE by applying Bayes' rule within each score bin; in our framework, this bin correction corresponds to using a constant subgroup function equal to one and a binning function for the score weight. 
Our framework gives a different abstraction: each calibration constraint is a signed residual moment indexed by a witness consisting of a subgroup function and a score-weighting function, and the weak-observation model enters only through the decontamination operator. 
Our results cover subgroup-dependent witnesses, UU and Pconf supervision, uniform convergence over subgroup and score-test classes, and post-hoc correction. Together, these results provide a first step toward multicalibration guarantees in WSL.

A related line of work studies uncertainty quantification when clean labels are unavailable or corrupted.
Conformal prediction has been extended to weak supervision and analyzed under label noise, while PU-ECE studies calibration-error estimation in positive-unlabeled learning
\citep{maxime2024,Bat2024,kiryo2026estimating}.
These works show that uncertainty guarantees must account for the label observation process.
Our work follows this motivation, but focuses on multicalibration and post-hoc correction in WSL.

\section{Preliminaries}
\label{sec:prelim}
\subsection{Multicalibration for binary classification}
\label{sec:prelim-mc}
Let $\calx$ and $\caly$ denote the input and label spaces. Throughout, we focus on binary classification, $\caly=\{0,1\}$. Let $P(Y,X)$ be the joint distribution over $\calx\times \caly$, and write the class priors as $\pi_+=P(Y=1)$ and $\pi_-=P(Y=0)$. We denote the class-conditional distributions by $P_+=P(X\mid Y=1)$ and $P_-=P(X\mid Y=0)$, and the regression function by $r(x)=P(Y=1\mid X=x)$. Following notation used in weakly supervised learning~\citep{ChiangSugiyama2025}, we also use the compact joint-measure vector
\begin{equation}
\label{eq:joint-measure-vector}
P:=(P(X,Y=1),P(X,Y=0))^\top=(\pi_+P_+,\pi_-P_-)^\top .
\end{equation}
A predictor is a measurable function $f:\calx\to[0,1]$, intended to estimate the conditional class probability $f^\star(x):=\Ex[Y\mid X=x]$. 

Let $\calc\subseteq\{c:\calx\to[0,1]\}$ be a class of subgroup functions, and let $\calw\subseteq\{w:[0,1]\to\R:\|w\|_\infty\le 1\}$ be a class of score-dependent test functions. We use the following witness-based definition of multicalibration from \citet{gopalan2022low}.

\begin{definition}[Multicalibration]
For a predictor $f$, define its $(\calc,\calw)$-multicalibration error by
\begin{equation}
\label{eq:def-mc}
\MC_{\calc,\calw}(f)
:=
\sup_{c\in\calc,\,w\in\calw}
\left|\Ex\big[c(X)w(f(X))(Y-f(X))\big]\right|.
\end{equation}
We say that $f$ is $(\calc,\calw,\varepsilon)$-multicalibrated if $\MC_{\calc,\calw}(f)\le \varepsilon$.
\end{definition}

This definition subsumes familiar calibration criteria by specifying the witness classes $\calc$ and $\calw$. With $\calc=\{1\}$, taking $\calw$ to be all measurable maps $[0,1]\to[-1,1]$ recovers ECE, while $\calw=\Lip_1([0,1],[-1,1])$ gives smooth calibration error \citep{blasiok2023unify}. Binned ECE is obtained by fixing $\delta\in(0,1]$, setting $J=\lceil 1/\delta\rceil$, $I_j=[(j-1)\delta,j\delta)$ for $j<J$ and $I_J=[(J-1)\delta,1]$, and using $\calw_\delta=\{p\mapsto \mathbf 1\{p\in I_j\}:j=1,\dots,J\}$; with a nontrivial $\calc$, this is the usual binned form of multicalibration \citep{gupta2020,johnson18a}. The class $\calc$ specifies the groups or tests on which calibration must hold. Indicator classes $\calc=\{\mathbf 1_G:G\in\calg\}$ give subgroup fairness notions \citep{johnson18a,kearns18a}; richer choices such as decision trees, or linear thresholds require the calibration residual to be indistinguishable from zero against a broad test class, connecting multicalibration and multiaccuracy%
~\citep{kim2019multiaccuracy,dwork2021outcome}. Given independent and identically distributed (i.i.d.) samples $S=\{(X_i,Y_i)\}_{i=1}^n$ from $P(X,Y)$, the empirical analogue is
\begin{equation}
\label{eq:def-emp-clean-mc}
\widehat{\MC}_{\calc,\calw}(f;S)
:=
\sup_{c\in\calc,\,w\in\calw}
\Big|\frac1n\sum_{i=1}^n c(X_i)w(f(X_i))(Y_i-f(X_i))\Big|.
\end{equation}

\subsection{Weakly supervised learning and unified risk rewrites}
\label{sec:prelim-wsl}
WSL replaces clean input--label pairs $(X,Y)$ with weaker observations. In PU learning~\citep{duplessis2014analysis,plessis2015convex,kiryo2017nnpu,kiryo2026estimating}, the sample consists of independent positive and unlabeled examples,
\begin{equation}
\label{eq:pu-sample-model-prelim}
S_+=\{X_i^+\}_{i=1}^{n_+},
\overset{\mathrm{i.i.d.}}{\sim} P_+,
\qquad
S_u=\{X_j^u\}_{j=1}^{n_u}\overset{\mathrm{i.i.d.}}{\sim} P_X,
\qquad
P_X=\pi_+P_+ + \pi_-P_- .
\end{equation}
For any binary loss $\ell:[0,1]\times \caly\to\mathbb{R}$, the clean risk admits the PU rewrite
\begin{align}
R_\ell(f)
=\Ex_{P(X,Y)}[\ell(f(X),Y)]&=
\pi_+\Ex_{P_+}[\ell(f(X),1)] + \pi_-\Ex_{P_-}[\ell(f(X),0)] \notag\\
&=
\pi_+\Ex_{P_+}[\ell(f(X),1)-\ell(f(X),0)] + \Ex_{P_X}[\ell(f(X),0)].
\label{eq:pu-risk-rewrite-prelim}
\end{align}
This yields the corrected empirical estimator
\begin{equation}
\label{eq:pu-risk-estimator-prelim}
\widehat R_\ell^{\PU}(f)=
\frac{\pi_+}{n_+}\sum_{i=1}^{n_+}\{\ell(f(X_i^+),1)-\ell(f(X_i^+),0)\}
+\frac1{n_u}\sum_{j=1}^{n_u}\ell(f(X_j^u),0).
\end{equation}
This example highlights the central WSL principle: rather than recovering the hidden label for each instance, we rewrite the target expectation in terms of quantities observable under the weak sampling scheme. Throughout our PU analysis, we treat $\pi_+$ as known or estimated in a preprocessing step, as is standard in PU learning~\citep{duplessis2014analysis,plessis2015convex,kiryo2017nnpu,kiryo2026estimating}.%

The unified framework of \citet{ChiangSugiyama2025} expresses such rewrites systematically in matrix form. Set $B:=(P_+,P_-)^\top$, $\Pi:=\mathrm{diag}(\pi_+,\pi_-)$, and $P=\Pi B=(\pi_+P_+,\pi_-P_-)^\top$.  

For a binary loss, define $L_\ell(x;f):=(\pi_+\ell(f(x),1),\pi_-\ell(f(x),0))^\top$, so that
\begin{equation}
\label{eq:loss-vector}
R_\ell(f)=\int_{\calx}L_\ell(x;f)^\top\,dB(x).
\end{equation}
A weak-observation model specifies an observable source distribution $\bar P$ and a decontamination rule
\begin{equation}
\label{eq:decontam}
\bar P=M_{\mathrm{corr}}B,
\qquad
B=M_{\mathrm{corr}}^{\dagger}\bar P,
\end{equation}
where $M_{\mathrm{corr}}^{\dagger}$ may be an inverse matrix or an $x$-dependent operator. This relation gives the risk rewrite
\begin{equation}
\label{eq:generic-risk-rewrite}
R_\ell(f)=\int_{\calx}L_\ell(x;f)^\top M_{\mathrm{corr}}^{\dagger}\,d\bar P(x).
\end{equation}
For PU, the identities used below are
\begin{align}
\label{eq:def-M-PU-main}
&\bar P_{\PU}:=(P_+,P_X)^\top=M_{\PU}B,
\qquad
M_{\PU}:=\begin{psmallmatrix}1&0\\ \pi_+&\pi_-\end{psmallmatrix},
\\
\label{eq:def-M-PU-inv-main}
&M_{\PU}^{\dagger}=M_{\PU}^{-1}=\begin{psmallmatrix}1&0\\ -\pi_+/\pi_-&1/\pi_-\end{psmallmatrix},
\qquad
B=M_{\PU}^{\dagger}\bar P_{\PU}.
\end{align}
Plugging \eqref{eq:def-M-PU-inv-main} into \eqref{eq:generic-risk-rewrite}, we recover Eqs.~\eqref{eq:pu-risk-rewrite-prelim} and \eqref{eq:pu-risk-estimator-prelim}. Other weak-observation models use the same template with a different $M_{\mathrm{corr}}$; Appendix~\ref{app:additional-settings} lists the variants used later.

\section{Weakly supervised multicalibration}
\label{sec:wsl-mc}
Multicalibration under weak supervision faces two difficulties: the residual
\(Y-f(X)\) is label-dependent although labels are observed only indirectly, and the error is a supremum over subgroup- and score-dependent witnesses rather than a single risk. Our key perspective is to regard each multicalibration witness moment as a signed target risk, not as a loss to be minimized for training but as an auditing functional for a fixed probabilistic predictor.
This fixed-witness view lets us apply the same
decontamination principle as in weakly supervised risk estimation; the remaining
task is to control the corrected moments uniformly over the witness class and
use them for post-processing.

We develop this idea in two steps. Section~\ref{sec:general-witness} gives the
generic decontamination rewrite, and
Section~\ref{sec:examples} instantiates it for PU, UU, and Pconf learning. We
state and prove the finite-sample guarantee in detail for PU in the main text;
analogous guarantees for UU and Pconf follow under their corresponding envelope
and sampling assumptions.

\subsection{Unified rewrite of multicalibration in the WSL setting}
\label{sec:general-witness}
Section~\ref{sec:prelim-wsl} showed how clean risks can be rewritten under weak supervision. We now lift the same algebra from ordinary losses to multicalibration witnesses.

For any fixed $(c,w)\in\calc\times\calw$, define
\begin{equation}
\label{eq:def-phi-cw}
\phi_f^{c,w}(x,y):=c(x)w(f(x))(y-f(x)),
\qquad
R_{c,w}(f):=\Ex\big[\phi_f^{c,w}(X,Y)\big].
\end{equation}
Then \eqref{eq:def-mc} becomes $\MC_{\calc,\calw}(f)=\sup_{c,w}|R_{c,w}(f)|$. Similarly to Eq.~\eqref{eq:loss-vector}, conditioning on $Y$ gives
\begin{equation}
\label{eq:def-L-cw}
L_{c,w}(x;f)
:=
\begin{pmatrix}
\pi_+ c(x)w(f(x))(1-f(x))\\[1mm]
-\pi_- c(x)w(f(x))f(x)
\end{pmatrix},
\qquad
R_{c,w}(f)=\int L_{c,w}(x;f)^\top\,dB(x).
\end{equation}
Similarly to Eq.~\eqref{eq:generic-risk-rewrite}, we obtain the following rewrite of the multicalibration moment.
\begin{theorem}[Generic WSL rewrite for multicalibration]
\label{thm:generic-wsl-mc}
Assume that the weak-supervision model admits a decontamination matrix or operator $M_{\mathrm{corr}}^{\dagger}$ such that $B=M_{\mathrm{corr}}^{\dagger}\bar P$ as a signed-measure identity, and that the corrected integrand is integrable. For any fixed witness $(c,w)$, define
\begin{equation}
\label{eq:generic-wsl-cw}
R_{c,w}^{\WSL}(f):=\int_{\calx} L_{c,w}(x;f)^\top M_{\mathrm{corr}}^{\dagger}\,d\bar P(x).
\end{equation}
Then $R_{c,w}(f)=R_{c,w}^{\WSL}(f)$ holds. Consequently, multicalibration admits the representation
\begin{equation}
\label{eq:def-mc-wsl}
\MC_{\calc,\calw}(f)
=
\sup_{c\in\calc,\,w\in\calw}
\left|R_{c,w}^{\WSL}(f)\right|,
\end{equation}
\end{theorem}
The proof is given in Appendix~\ref{app:proofs-main}.  For a fixed witness, the equality follows from the same decontamination principle used in weakly supervised risk estimation. The difference is its statistical role: instead of constructing an empirical objective for training a classifier, we use the rewrite to make calibration residual moments observable for a fixed probabilistic predictor. Multicalibration then requires controlling not one corrected risk, but a witness-indexed signed residual process over \(C\times W\). The finite-sample results below show that these corrected moments can be controlled uniformly and therefore used as weak-label auditors for post-hoc correction.

Let $\widehat R^{\WSL}_{c,w}(f;S)$ denote the finite-sample estimator of $R_{c,w}^{\WSL}(f)$ on a weak dataset $S$. The empirical estimator is obtained by replacing the relevant expectations with the corresponding source-wise empirical averages for the data-generation mechanism. Section~\ref{sec:examples} writes these estimators explicitly. %

\subsection{Three concrete weak-supervision examples}
\label{sec:examples}
We instantiate Theorem~\ref{thm:generic-wsl-mc} in three binary settings: PU, UU, and Pconf. These examples show how the multicalibration moment is corrected under weak supervision. Other weak-supervision specifications, including SU, DU, SD, pairwise comparison, soft-label, and Sconf variants, are recorded in  Appendix~\ref{app:additional-settings}.
To keep the main text focused, we state the finite-sample guarantee for PU; analogous arguments for UU and Pconf are provided in Appendix~\ref{app:gen}. 

\paragraph{PU learning.}
We reuse the PU sample model and matrices from Section~\ref{sec:prelim-wsl}, namely Eq.~\eqref{eq:pu-sample-model-prelim} and Eq.~\eqref{eq:def-M-PU-main}--\eqref{eq:def-M-PU-inv-main}.  Substituting $M_{\PU}^{\dagger}$ into the generic rewrite in Eq.~\eqref{eq:generic-wsl-cw} gives the following expression.
\begin{corollary}
\label{prop:pu}
Under the PU observation model, for every fixed $(c,w)\in\calc\times\calw$,
\begin{equation}
\label{eq:Rcw-PU}
R_{c,w}^{\PU}(f)
=\pi_+\Ex_{P_+}[c(X)w(f(X))]-\Ex_{P_X}[c(X)w(f(X))f(X)].
\end{equation}
Given $\{X_i^+\}_{i=1}^{n_+}\overset{\mathrm{i.i.d.}}{\sim} P_+$ and $\{X_j^u\}_{j=1}^{n_u}\overset{\mathrm{i.i.d.}}{\sim} P_X$, the corresponding finite-sample estimator is
\begin{equation}
\label{eq:Rhat-PU}
\widehat R_{c,w}^{\PU}(f)
=\frac{\pi_+}{n_+}\sum_{i=1}^{n_+}c(X_i^+)w(f(X_i^+))
-\frac1{n_u}\sum_{j=1}^{n_u}c(X_j^u)w(f(X_j^u))f(X_j^u).
\end{equation}
The multicalibration error estimator is
 $\widehat\MC^{\PU}_{\calc,\calw}(f):=\sup_{c,w}|\widehat R_{c,w}^{\PU}(f)|$.
\end{corollary}
The following proposition controls the finite-sample deviation of this estimator.
\begin{proposition}[Uniform convergence for the PU weak estimate]
\label{prop:pu-gen}
Assume that $f$ is fixed independently of the PU samples.  Define $(P_1,n_1,\pi_1):=(P_+,n_+,\pi_+)$ and $(P_2,n_2,\pi_2):=(P_X,n_u,1)$.  Then for any $\delta\in(0,1]$ and $\eta\in(0,1]$, with probability at least $1-\delta$,
\begin{equation}
\label{eq:pu-gen-main}
\bigl|\widehat\MC^{\PU}_{\calc,\calw}(f)-\MC_{\calc,\calw}(f)\bigr|
\le
2\sum_{s=1}^2\pi_s
\Big(\mathfrak{R}_{n_s}(\calc;P_s)+\eta+
\sqrt{\frac{2\log N_\infty(\eta,\calw)+\log(4/\delta)}{n_s}}\Big),
\end{equation}
where $
\mathfrak{R}_n(\calc;P):=\Ex_{S\sim P^n,\,\sigma}\left[\sup_{c\in\calc}\frac1n\sum_{i=1}^n \sigma_i c(X_i)\right]
$ is the expected Rademacher complexity of \(\calc\) under \(P\), $\sigma_i\in\{-1,+1\}$ are independent Rademacher random variables, and $N_\infty(\eta,\calw)=N(\eta,\calw,\|\cdot\|_\infty)$ is the covering number of $\calw$ under the sup norm. 
\end{proposition}

This PU bound shows that both positive and unlabeled samples are needed to control finite-sample estimation error. It also yields sample-complexity rates once $\calc$ and $\calw$ are specified: for example, when $n=n_+=n_u$ and \(\calc\) and \(\calw\) are finite, as in standard binning-based MC, \(\bigl|\widehat\MC^{\PU}_{\calc,\calw}(f)-\MC_{\calc,\calw}(f)\bigr|=\calo_p(\sqrt{\log(2|\calc||\calw|)/n})\). This logarithmic dependence is important: even though the estimator takes a worst case over subgroup and score witnesses, the price of uniformly estimating corrected moments grows logarithmically with the sizes of group and scoring classes. Hence, once enough samples are available from each observed source, the weak-label estimator remains efficient. The proof and bounds for other choices of $\calc$ and $\calw$ are given in Appendix~\ref{app:pu-gen}.

\paragraph{UU learning.}
\label{sec:uu}
We now move to the more surprising case where \emph{no} labeled sample is available. 
In UU learning~\citep{lu2018on,lu21c}, one observes two unlabeled sources with different mixture proportions
\[
P_{U_1}=(1-\gamma_1)P_+ + \gamma_1 P_-,
\qquad
P_{U_2}=\gamma_2 P_+ + (1-\gamma_2)P_-,
\qquad
\Delta:=1-\gamma_1-\gamma_2\neq 0.
\]
We assume throughout this subsection that the mixture proportions $\gamma_1,\gamma_2\in[0,1]$ are fixed and known a priori.
Thus PU and UU are both mutually contaminated-distribution (MCD) settings: the observed sources are fixed mixtures of the two class-conditionals, with no $x$-dependent correction. Equivalently,
$M_{\UU}=\begin{psmallmatrix}
1-\gamma_1 & \gamma_1\\
\gamma_2 & 1-\gamma_2
\end{psmallmatrix}$,
$\bar P=\begin{psmallmatrix}
P_{U_1}\\
P_{U_2}
\end{psmallmatrix}
=M_{\UU}B$,
with inverse
$M_{\UU}^{-1}=\frac{1}{\Delta}
\begin{psmallmatrix}
1-\gamma_2 & -\gamma_1\\
-\gamma_2 & 1-\gamma_1
\end{psmallmatrix}$.
This uses the same decontamination idea as PU, except that both sources are unlabeled mixtures. Indeed, PU is recovered by setting $(\gamma_1,\gamma_2)=(0,\pi_+)$ with $(P_{U_1},P_{U_2})=(P_+,P_X)$. The key point of the UU case is that, although no sample is labeled, the two mixtures are sufficient to estimate the multicalibration error. Substituting $M_{\UU}^{-1}$ into the generic rewrite gives the following corrected moment and empirical estimator; the finite-sample bound is given in Appendix~\ref{app:gen}.
\begin{corollary}
\label{prop:uu}
Under the UU observation model, for every fixed $(c,w)\in\calc\times\calw$, 
\begin{align}
R_{c,w}^{\UU}(f)
&=\sum_{k=1}^2\Ex_{P_{U_k}}[\alpha_k(f(X))c(X)w(f(X))],
\label{eq:Rcw-UU-main}
\end{align}
with $\alpha_1(v):=\frac{(1-\gamma_2)\pi_+(1-v)+\gamma_2\pi_- v}{\Delta},
\qquad
\alpha_2(v):=\frac{-\gamma_1\pi_+(1-v)-(1-\gamma_1)\pi_- v}{\Delta}$. 
Given samples $\{X_i^{U_1}\}_{i=1}^{n_1}\overset{\mathrm{i.i.d.}}{\sim} P_{U_1}$ and $\{X_j^{U_2}\}_{j=1}^{n_2}\overset{\mathrm{i.i.d.}}{\sim} P_{U_2}$, the corresponding estimator is
\begin{align}
\widehat R_{c,w}^{\UU}(f)
&=\sum_{k=1}^2\frac1{n_k}\sum_{i=1}^{n_k}
\alpha_k(f(X_i^{U_k}))c(X_i^{U_k})w(f(X_i^{U_k})).
\label{eq:Rhat-UU}
\end{align}
\end{corollary}

\paragraph{Pconf learning.}
Unlike the PU and UU cases, which use fixed mixture proportions, Pconf~\citep{ishida2018} learning considers a soft-label setting in which each positive example carries an $x$-dependent confidence.  We observe i.i.d. pairs $(X_i,r_i)$ with $X_i\sim P_+$ and $r_i=r(X_i)=P(Y=1\mid X_i)$. We treat $\pi_+$ as known or pre-estimated and assume the standard support condition $P_-\ll P_+$, so that the negative class can be recovered from positive-confidence data. With $B=(P_+,P_-)^\top$, define the duplicated observable source $\bar P_{\Pconf}(dx):=\begin{psmallmatrix}P_+(dx)\\ P_+(dx)\end{psmallmatrix}$. Bayes' rule gives $P_-(dx)=\frac{\pi_+}{\pi_-}\frac{1-r(x)}{r(x)}P_+(dx)$, and hence the $x$-dependent decontamination rule
$B(dx)=M_{\Pconf}^{\dagger}(x)\bar P_{\Pconf}(dx)$ and 
$M_{\Pconf}^{\dagger}(x)=\begin{psmallmatrix}1&0\\ 0&\frac{\pi_+}{\pi_-}\frac{1-r(x)}{r(x)}\end{psmallmatrix}$.
The density-ratio derivation is deferred to Appendix~\ref{app:confidence-full}. Substituting this operator into the generic witness rewrite gives the following corrected moment and estimator.
\begin{corollary}
\label{prop:pconf}
For every fixed $(c,w)\in\calc\times\calw$,
\begin{equation}
\label{eq:Rcw-Pconf}
R_{c,w}^{\Pconf}(f)
=\pi_+\Ex_{P_+}\!\left[c(X)w(f(X))\left(1-f(X)r(X)^{-1}\right)\right].
\end{equation}
With Pconf data $\{(X_i,r_i)\}_{i=1}^n$ where $X_i\sim P_+$ and $r_i=r(X_i)$,
\begin{equation}
\label{eq:Rhat-Pconf}
\widehat R_{c,w}^{\Pconf}(f)
=\frac{\pi_+}{n}\sum_{i=1}^n c(X_i)w(f(X_i))\left(1-f(X_i)r_i^{-1}\right).
\end{equation}
\end{corollary}

The finite-sample estimation error is analyzed in Appendix~\ref{app:pconf-gen}.%

We assume known nuisance parameters, such as $\pi_+$ for PU and $\gamma_1,\gamma_2$ for UU. Appendix~\ref{app:nuisance-parameters} discusses how their estimation errors affect the multicalibration estimation error.

\section{Post-processing methods}
\label{sec:algo}
This section develops post-processing methods based on the corrected moments in Section~\ref{sec:wsl-mc}: WLMC, a weak-supervision multicalibration booster, and weak versions of Platt scaling and temperature scaling that use the same risk rewrites.

\subsection{WLMC: weak-label multicalibration boost}
\label{sec:wlmc-boost}
Algorithm~\ref{alg:wsl-mc-boosting} is our proposed post-processing algorithm. It is a boosting-style recalibration method, in the spirit of \citet{johnson18a,gopalan2022low}: at each iteration, it searches for a direction along which multicalibration is violated and moves the predictor in that direction. The search ranges over \((c,w)\in\calc\times\calw\). {The main difference from existing methods is that they assume independent samples at each iteration, i.e., fresh samples, whereas our method reuses the same samples throughout the algorithm. Nevertheless, as shown by the following theorem, the sample complexity remains the same.}

\begin{algorithm}[t]
\small
\caption{Finite-sample weak-label multicalibration boost (WLMC)}
\label{alg:wsl-mc-boosting}
\begin{algorithmic}
\State \textbf{Input:} \(\varepsilon>0\), fixed weak calibration samples \(S\), step size \(\lambda\), initial predictor \(f_0\), classes \(\calc,\calw\), estimator \(\widehat R^{\WSL}\)
\State Set \(t=0\)
\While{\(\exists(\tau,c,w)\in\{\pm1\}\times\calc\times\calw\) such that \(\tau\widehat R^{\WSL}_{c,w}(f_t;S)>3\varepsilon/4\)}
    \State Choose \((\tau_t,c_t,w_t)\in\{\pm1\}\times\calc\times\calw\) satisfying the while condition
    \State \(\forall x\in\calx,\quad f_{t+1}(x)=\Pi_{[0,1]}\!\left[f_t(x)+\lambda\tau_t c_t(x)w_t(f_t(x))\right]\)
    \Statex \hspace{\algorithmicindent} where \(\Pi_{[0,1]}[z]:=\min\{1,\max\{0,z\}\}\) is the projection of \(z\) onto \([0,1]\).
    \State \(t\gets t+1\)
\EndWhile
\State \textbf{Output:} \(f_t\)
\end{algorithmic}
\end{algorithm}

\begin{theorem}[(Informal) PU guarantee for WLMC]
\label{thm:pu-wlmc-boost}
Let \(0<\varepsilon\le1\) and \(0<\delta<1\). Assume the PU setting of Proposition~\ref{prop:pu-gen} with known class prior \(\pi_+\), nonempty finite $\calc$ and $\calw$, and \(n_+=n_u=n\).\footnote{The equal-size assumption is made only for ease of presentation; see Appendix~\ref{app:proof-pu-wlmc-boost} for the unequal-source-size case.}
Assume that \(f_0,\calc,\calw\) are fixed independently of the calibration samples.
Let \(D_0=\Ex[(f_0(X)-f^\star(X))^2]\) and \(T=\lceil 4D_0/\varepsilon^2\rceil\). Draw \(n\) i.i.d. samples from each of \(P_+\) and \(P_X\), independently across sources, and run Algorithm~\ref{alg:wsl-mc-boosting}.
For a universal constant \(C>0\), if \(n\ge C\log((2|\calc||\calw|)^{T+2}/\delta)/\varepsilon^{2}\), then Algorithm~\ref{alg:wsl-mc-boosting} with
\(\lambda=\varepsilon/2\) stops after at most \(T\) successful updates and returns a \((\calc,\calw,\varepsilon)\)-multicalibrated hypothesis, both with probability at least \(1-\delta\).
\end{theorem}
Therefore, \(n\) is the total number of samples from each of the positive and unlabeled sources. Since \(D_0\le1\), a sufficient total budget per source is \(\calo(\log(2|\calc||\calw|)/\varepsilon^{4}+\log(1/\delta)/\varepsilon^{2})\). The \(\varepsilon^{-4}\) dependence matches that of existing methods~\citep{johnson18a,gopalan2022low}. UU- and Pconf-based WLMC variants also require the same \(O(\varepsilon^{-4})\) sample complexity; see Appendix~\ref{app:proof-boosting} with the formal statement. For comparison with prior work, see Section~\ref{sec:related}.

\subsection{Platt scaling and temperature scaling}
\label{sec:weak-postprocessing-baselines}
\vspace{-2mm}
Recent empirical work on multicalibration~\citep{hansen2024multicalibration} reports that classical post-processing methods, such as Platt scaling~\citep{platt1999probabilistic} and temperature scaling~\citep{guo2017calibration}, can perform competitively with boosting-based methods that come with theoretical guarantees. 
We therefore implement weak-supervision analogues of these methods. 
Both keep the base predictor fixed and apply a simple transformation to its output logit. 
Let \(\sigma(t)=(1+e^{-t})^{-1}\), \(\text{logit}=\sigma^{-1}\), and \(z_f(x)=\text{logit}(f(x))\). 
Temperature scaling uses \(g_\beta(x)=\sigma(z_f(x)/\beta)\) with parameter \(\beta>0\), while Platt scaling uses \(g_{a,b}(x)=\sigma(a z_f(x)+b)\) with parameters \((a,b)\in\mathbb R^2\). 
With clean labels, both are fitted by minimizing held-out negative log-likelihood. 
Under weak supervision, we keep the same parametric families but replace the clean held-out NLL with corrected weak NLL objectives derived from the risk rewrites in Section~\ref{sec:prelim-wsl}. 
Appendix~\ref{app:weak-postprocessing-objectives} lists the concrete corrected NLL objectives for the three weak-supervision models.

\vspace{-1mm}
\section{Experiments}\label{sec:experiments}
\vspace{-1.5mm}
We empirically study multicalibration under weak supervision on toy and real datasets.
Beyond validating the finite-sample estimators in Section~\ref{sec:wsl-mc}, the real-data experiments are meant to test whether weak labels can be used as a practical calibration resource: can they estimate multicalibration reliably, and can they support post-hoc correction without requiring an additional clean calibration set? To the best of our knowledge, this is the first numerical study of multicalibration on real data in which the calibration metric and post-processing step are both implemented from weak labels. 
All main-text results use finite subgroup and prediction-bin collections $\calc$ and $\calw$; experimental details are deferred to Appendices~\ref{app:toy-data-settings}--\ref{app:calib-hparams}. For brevity, we refer to multicalibration error as {\it MC} in this section. 

\subsection{Toy data: Verification of finite-sample estimation error}
\vspace{-1.5mm}
\label{sec:exp-toy}
\begin{figure*}[t]
    \centering
    \newlength{\mainexpfigheight}
    \setlength{\mainexpfigheight}{0.25\textwidth}
    \begin{minipage}[c]{0.34\textwidth}
        \centering
        \includegraphics[height=\mainexpfigheight,keepaspectratio]{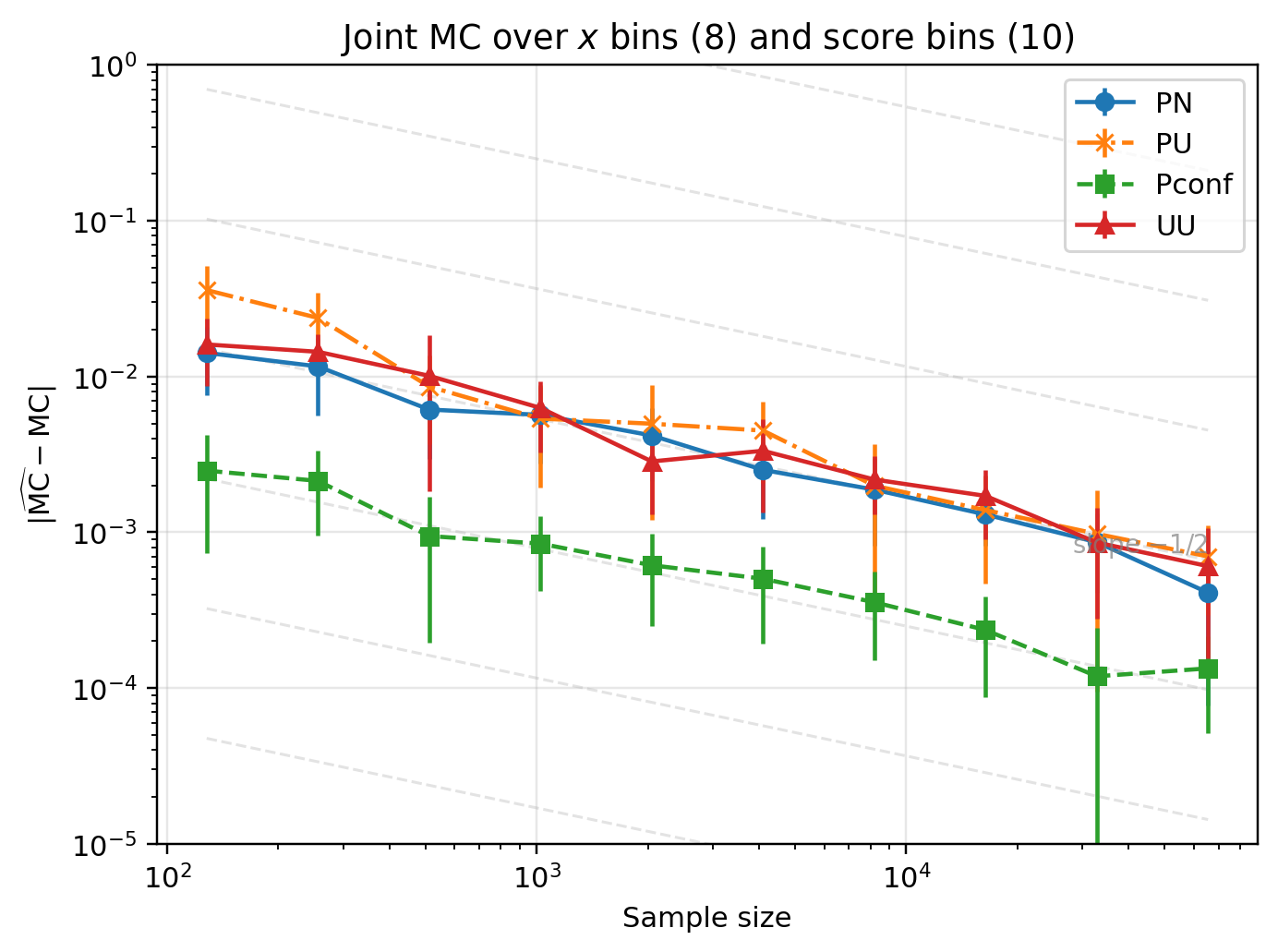}
    \end{minipage}
    \hfill
    \begin{minipage}[c]{0.62\textwidth}
        \centering
        \includegraphics[height=\dimexpr\mainexpfigheight-0.8pt\relax,keepaspectratio]{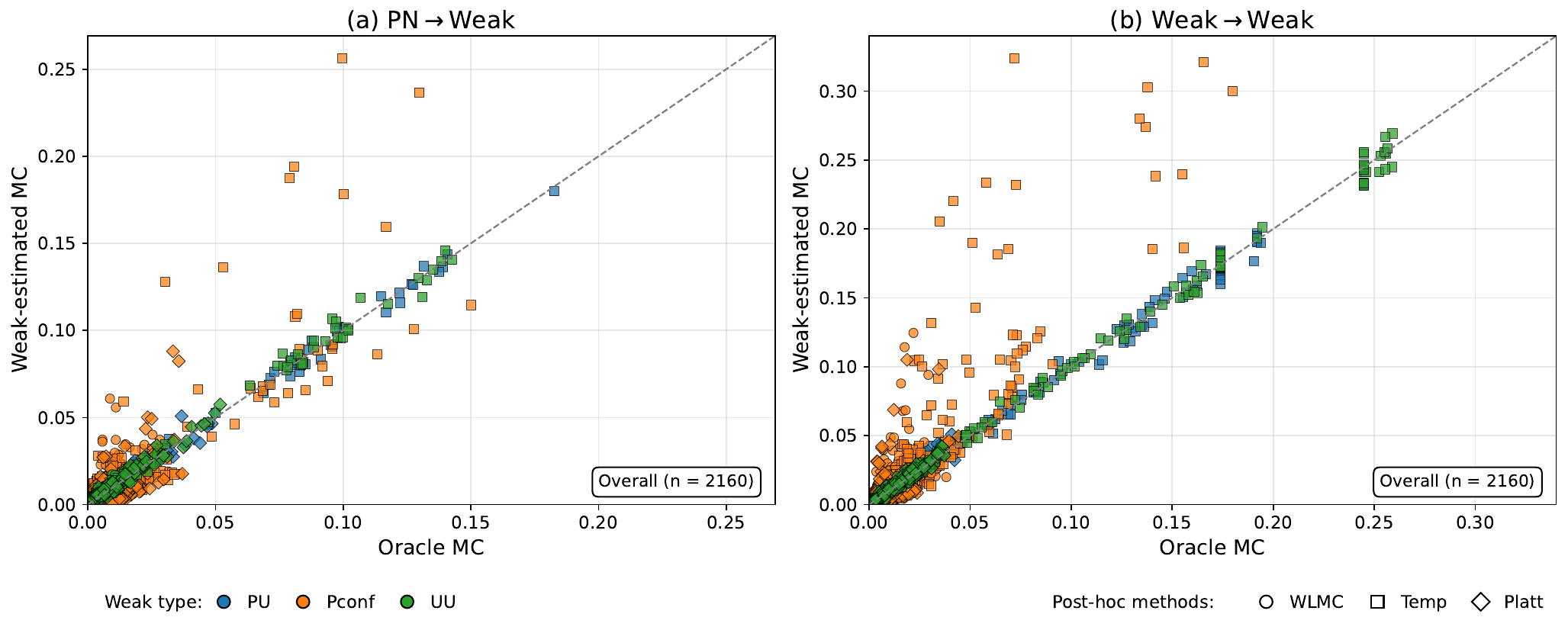}
    \end{minipage}
    \caption{Left: toy-data verification of MC estimators. Middle/right: oracle MC (x-axis; PN labels) versus weak-estimated MC (y-axis) for \((\mathrm{PN},\mathrm{weak})\) and \((\mathrm{weak},\mathrm{weak})\) on tabular data. The dashed diagonal denotes exact agreement.}
    \label{fig:toy-and-tabular-mc}
    \label{fig:Toy_convergence}
    \label{fig:tabular-oracle-vs-weak-mc}
\end{figure*}

We first evaluate the finite-sample estimation error of multicalibration metrics under the WSL settings in Section~\ref{sec:wsl-mc}. We generate a toy binary classification problem with a known
weak-observation mechanism. We plot the absolute error between each population metric and its estimate as a
function of sample size. This experiment verifies the convergence behavior
predicted by our analysis in Section~\ref{sec:wsl-mc}. All metrics
should exhibit an $n^{-1/2}$ rate on a log--log plot. The empirical curves in the left panel of Figure~\ref{fig:toy-and-tabular-mc} follow this reference rate.

\subsection{Real data experiments}
Following \citet{hansen2024multicalibration}, we study MC under weak supervision on real datasets. 
For a baseline predictor $f_0$, we vary the data for base training and post-hoc correction: $(\mathrm{base},\mathrm{post})
\in
\{(\mathrm{PN},\mathrm{PN}),(\mathrm{PN},\mathrm{weak}),(\mathrm{weak},\mathrm{weak})\}$. PN denotes fully supervised view with complete input--label pairs, while weak denotes a PU, UU, or Pconf. The weak-observation type for base learning and post-processing is matched when both are weak. We call estimates computed with clean labels \emph{oracle estimates (MC)}, and those computed with weak data \emph{weak estimates (MC)}. Weak post-processing uses only the corresponding weak validation criterion for selection; oracle labels are not used for weak fitting or selection. Protocol details are deferred to Appendices~\ref{app:dataset-subgroups}--\ref{app:calib-hparams}. Within each benchmark, PN and weak post-processing use the same correction split; tabular data uses correction fraction $0.4$, while CelebA and CivilComments use $0.2$. Thus PN and weak post-processing receive comparable correction data, while the weak setting is more stringent because clean labels are withheld.

We summarize the main empirical results in three findings; comprehensive results across supervision settings and post-processing methods are reported in Appendices~\ref{app:tabular-fixed} and~\ref{app:large-model-supp}. \textbf{F1.} Weak estimates can proxy oracle estimates when the weak-observation model is well specified.  \textbf{F2.} Weak-data post-processing can improve PN-trained predictors, suggesting weak labels can serve as possible lower-clean-label-cost calibration signals.  \textbf{F3.} Weakly trained predictors are less calibrated, but their larger initial violations make weak-data post-processing especially useful.  Together, these findings suggest that weak labels may reduce the need for clean calibration labels while still enabling effective multicalibration correction.

\begin{figure*}[t]
    \centering
    \includegraphics[width=\dimexpr0.98\textwidth-0.8pt\relax]{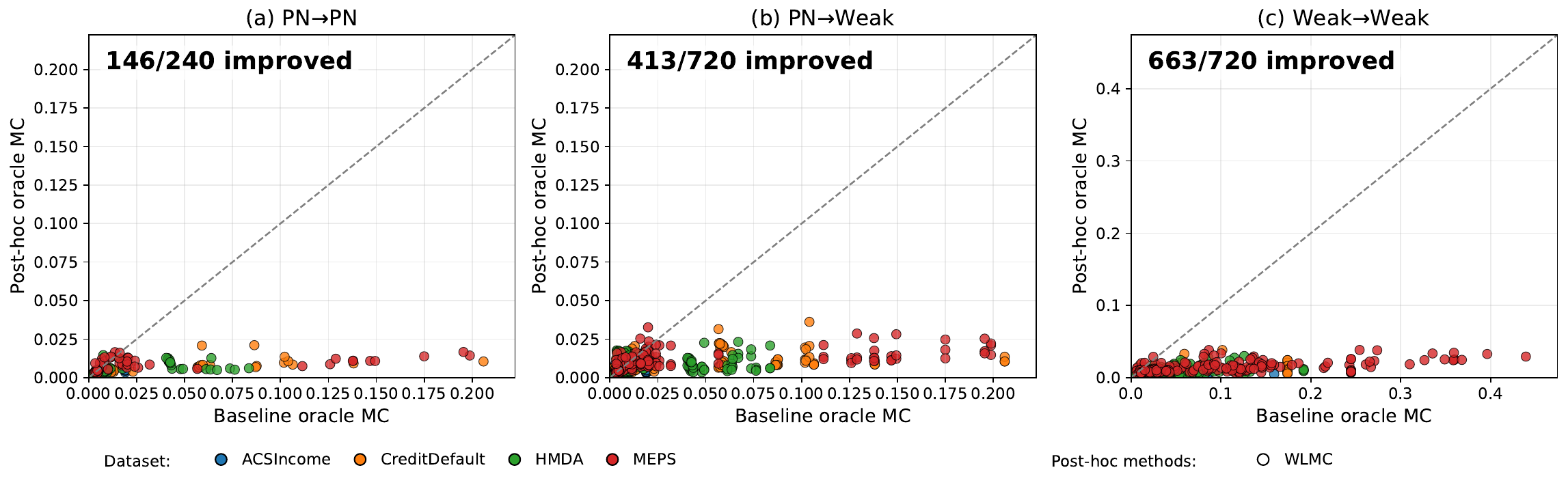}
    \caption{Oracle test MC before and after WLMC.}
    \label{fig:tabular-wlmc-correction}
    \par\vspace{\floatsep}
    \centering
    \includegraphics[width=\dimexpr0.98\textwidth-0.8pt\relax]{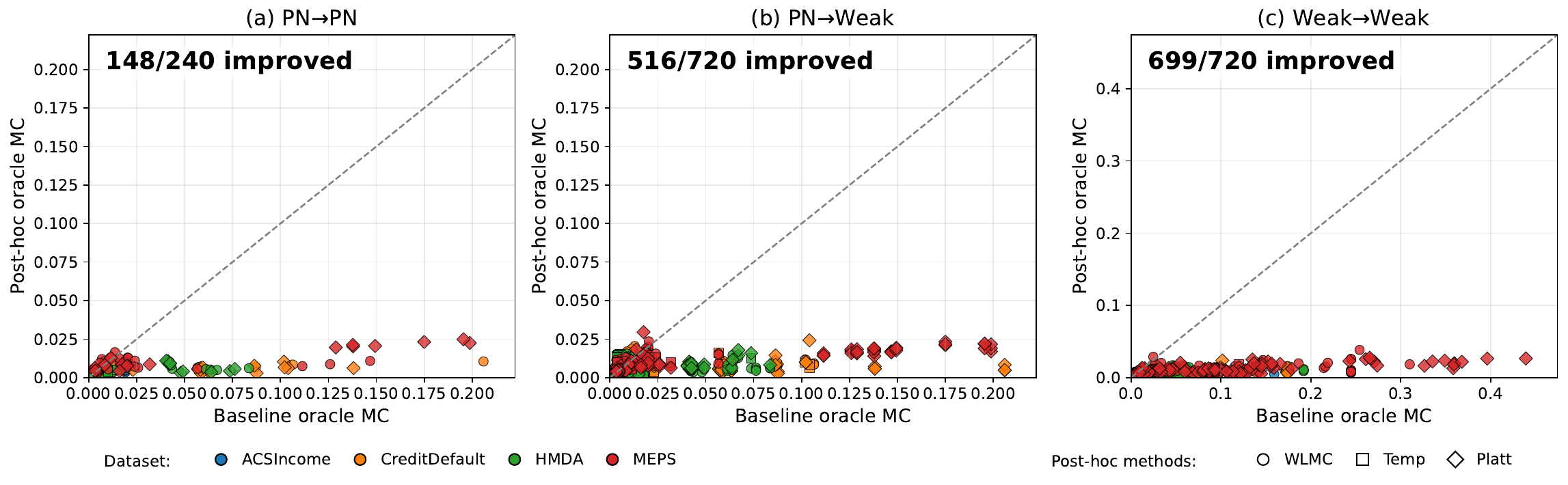}
    \caption{Oracle test MC before and after correction by a method selected from WLMC, temperature scaling, and Platt scaling using the smallest oracle validation MC.}
    \label{fig:tabular-posthoc-mc-bestofthree}
\end{figure*}

\paragraph{Tabular data.} We used four tabular datasets---ACS Income, UCI Credit Default, HMDA, and MEPS (see Appendix~\ref{app:dataset-subgroups} for details)---and six base predictors (decision tree, random forest, SVM, naive Bayes, logistic regression, MLP). We first examine whether multicalibration error can be estimated reliably from weak labels.
The middle and right panels of Figure~\ref{fig:toy-and-tabular-mc} compare oracle MC and weak MC across tabular datasets, baselines, weak-label types, and post-hoc correction methods. The diagonal $y=x$ denotes exact agreement: points closer to this line indicate that weak MC more accurately approximates oracle MC. The same plots with correlation coefficients are provided in Appendix~\ref{app:tabular-fixed}.
Overall, weak MC is well aligned with oracle MC.
PU and UU points are mostly concentrated near the diagonal, suggesting that weak labels can provide useful MC estimates when the weak-observation model is well matched.
Pconf, however, shows several deviations.
This is likely because Pconf confidence values are generated from conditional probabilities estimated by a separate probabilistic model, so the estimation error can be inherited in the confidence information itself.  In contrast, the toy experiment uses exact confidences from the data-generating model, where Pconf follows the predicted convergence trend.%

We next evaluate whether weak labels are also useful for post-hoc correction. Figures~\ref{fig:tabular-wlmc-correction} and~\ref{fig:tabular-posthoc-mc-bestofthree} compare oracle MC on held-out test data before and after correction under $(\mathrm{PN},\mathrm{PN})$, $(\mathrm{PN},\mathrm{weak})$, and $(\mathrm{weak},\mathrm{weak})$. Each panel pools experimental runs: \(240=10\) seeds \(\times\,4\) datasets \(\times\,6\) base predictors for \((\mathrm{PN},\mathrm{PN})\); the \((\mathrm{PN},\mathrm{weak})\) and \((\mathrm{weak},\mathrm{weak})\) panels have \(720\) runs each by additionally pooling PU, UU, and Pconf. Points below the diagonal indicate improvement, and the number in each panel reports how many runs reduce oracle MC. Figure~\ref{fig:tabular-wlmc-correction} shows results for our proposed WLMC, whereas Figure~\ref{fig:tabular-posthoc-mc-bestofthree} shows results for the method selected from WLMC, temperature scaling, and Platt scaling using the smallest oracle validation MC. In $(\mathrm{PN},\mathrm{PN})$, WLMC uses the boosting-style updates of Algorithm~\ref{alg:wsl-mc-boosting} with clean labels to estimate calibration violations. Figure~\ref{fig:tabular-posthoc-mc-bestofthree} uses oracle MC for selection.

The left panels of Figures~\ref{fig:tabular-wlmc-correction} and~\ref{fig:tabular-posthoc-mc-bestofthree} show that correction in $(\mathrm{PN},\mathrm{PN})$ does not always improve MC, even with label-based boosting updates or validation selection among the three methods. This is partly because PN-trained baselines often already have small violations. The middle panel of Figure~\ref{fig:tabular-wlmc-correction} shows that WLMC can improve PN-trained predictors without additional clean labels. The right panel shows that WLMC can also reduce MC when both base learning and correction use weak data. Weakly trained predictors typically start with larger MC than PN-trained predictors, as also shown in Appendix~\ref{app:tabular-fixed}, giving post-processing more room to improve. Figure~\ref{fig:tabular-posthoc-mc-bestofthree} shows the same overall tendency under its selection protocol: weak data can support post-hoc calibration of both PN-trained and weakly trained predictors.

\begin{figure*}[t]
    \centering
    \includegraphics[width=0.98\textwidth]{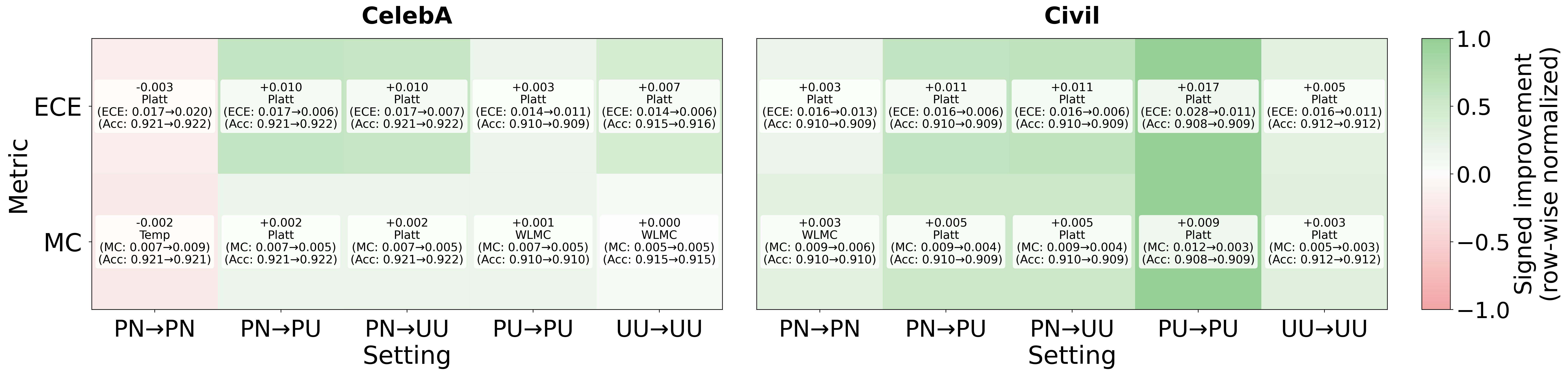}
    \caption{ECE and MC improvements on CelebA and CivilComments.}
    \label{fig:large-model-heatmap}
\end{figure*}

\paragraph{Image and text experiments with larger models.}
Finally, we test whether the same phenomena appear for larger neural models, using CelebA~\citep{liu2015deep} with a ViT-based image classifier~\citep{dosovitskiy2020image} and CivilComments~\citep{borkan2019nuanced} with a BERT-based text classifier~\citep{sanh2019distilbert}.
For each setting, we freeze the baseline predictor and compare WLMC, temperature scaling, and Platt scaling. Experimental details are reported in Appendices~\ref{app:calib-hparams} and~\ref{app:large-model-supp}; Appendix~\ref{app:large-model-supp} also reports a CelebA-ResNet~\citep{he2016deep} experiment.

Figure~\ref{fig:large-model-heatmap} reports, for each metric, the method with the largest improvement together with the change in accuracy. Weak-label post-processing remains effective beyond tabular data.
ECE and MC improve even in the $(\mathrm{PN},\mathrm{weak})$ setting, and the ECE improvement is sometimes larger than in $(\mathrm{PN},\mathrm{PN})$, despite using weaker supervision for correction.
The $(\mathrm{weak},\mathrm{weak})$ settings also improve consistently, although their baselines tend to have larger ECE and MC compared with PN baselines.

Across methods, Platt scaling is strong, especially for ECE, while WLMC is competitive for MC and performs comparably to Platt scaling in several CelebA settings. This is consistent with clean-label reports that Platt scaling can be a strong calibration baseline; we do not aim to outperform Platt scaling in every case, but instead show that weak-label post-processing, including WLMC, can achieve comparable calibration improvements under weak supervision.
Temperature scaling is less reliable in some CelebA configurations, suggesting that a global temperature may be insufficient for subgroup-level errors.
Importantly, none of the displayed methods causes a large accuracy drop, indicating that weak-label post-processing can improve calibration without sacrificing predictive performance.

\vspace{-1mm}
\section{Conclusion and limitations}
\label{sec_conclusion}
\vspace{-1mm}
In this work, we developed a unified framework for multicalibration under weak supervision by expressing calibration witnesses through the decontamination principle used in WSL. 
This yields corrected estimators for PU, UU, and Pconf data, together with finite-sample guarantees, and enables post-hoc correction using weak labels rather than clean calibration labels. 
Empirically, to the best of our knowledge, this is the first numerical study of weak-label multicalibration on real-data benchmarks with controlled weak-observation views, covering estimation and post-hoc correction. The results show that weak estimates can track oracle multicalibration, weak training tends to increase calibration violations, and weak labels can still improve post-hoc calibration, suggesting reduced reliance on clean calibration labels. 
The main limitations are the need for reliable class priors and mixture proportions and the binary-label focus of our methods. Improving robustness to estimated nuisance parameters and extending the framework to multiclass settings are important directions for future work.
\section*{\texorpdfstring{Acknowledgements}{Acknowledgements}}
FF was supported by JST BOOST Program Grant Number JPMJBY24G8. This work was supported by Japan Science and Technology Agency (JST) as part of Adopting Sustainable Partnerships for Innovative Research Ecosystem (ASPIRE), Grant Number JPMJAP25B1.

\newpage
\bibliography{main}

@article{
ChiangSugiyama2025,
title={Unified Risk Analysis for Weakly Supervised Learning},
author={Chao-Kai Chiang and Masashi Sugiyama},
journal={Transactions on Machine Learning Research},
issn={2835-8856},
year={2025},
}

@inproceedings{
lu2018on,
title={On the Minimal Supervision for Training Any Binary Classifier from Only Unlabeled Data},
author={Nan Lu and Gang Niu and Aditya K. Menon and Masashi Sugiyama},
booktitle={International Conference on Learning Representations},
year={2019},
}

@InProceedings{lu21c,
  title = 	 {Binary Classification from Multiple Unlabeled Datasets via Surrogate Set Classification},
  author =       {Lu, Nan and Lei, Shida and Niu, Gang and Sato, Issei and Sugiyama, Masashi},
  booktitle = 	 {Proceedings of the 38th International Conference on Machine Learning},
  pages = 	 {7134--7144},
  year = 	 {2021},
  editor = 	 {Meila, Marina and Zhang, Tong},
  volume = 	 {139},
  series = 	 {Proceedings of Machine Learning Research},
  month = 	 {18--24 Jul},
  publisher =    {PMLR},
}

@article{
kiryo2026estimating,
title={Estimating Expected Calibration Error for Positive-Unlabeled Learning},
author={Ryuichi Kiryo and Futoshi Futami and Masashi Sugiyama},
journal={Transactions on Machine Learning Research},
issn={2835-8856},
year={2026},
}

@inproceedings{duplessis2014analysis,
 author = {du Plessis, Marthinus C. and Niu, Gang and Sugiyama, Masashi},
 booktitle = {Advances in Neural Information Processing Systems},
 editor = {Z. Ghahramani and M. Welling and C. Cortes and N. Lawrence and K.Q. Weinberger},
 pages = {},
 publisher = {Curran Associates, Inc.},
 title = {Analysis of Learning from Positive and Unlabeled Data},
 volume = {27},
 year = {2014}
}

@inproceedings{kiryo2017nnpu,
 author = {Kiryo, Ryuichi and Niu, Gang and du Plessis, Marthinus C and Sugiyama, Masashi},
 booktitle = {Advances in Neural Information Processing Systems},
 editor = {I. Guyon and U. Von Luxburg and S. Bengio and H. Wallach and R. Fergus and S. Vishwanathan and R. Garnett},
 pages = {},
 publisher = {Curran Associates, Inc.},
 title = {Positive-Unlabeled Learning with Non-Negative Risk Estimator},
 volume = {30},
 year = {2017}
}

@InProceedings{plessis2015convex,
  title = 	 {Convex Formulation for Learning from Positive and Unlabeled Data},
  author = 	 {Plessis, Marthinus Du and Niu, Gang and Sugiyama, Masashi},
  booktitle = 	 {Proceedings of the 32nd International Conference on Machine Learning},
  pages = 	 {1386--1394},
  year = 	 {2015},
  editor = 	 {Bach, Francis and Blei, David},
  volume = 	 {37},
  series = 	 {Proceedings of Machine Learning Research},
  address = 	 {Lille, France},
  month = 	 {07--09 Jul},
  publisher =    {PMLR},
}

@inproceedings{ishida2018,
 author = {Ishida, Takashi and Niu, Gang and Sugiyama, Masashi},
 booktitle = {Advances in Neural Information Processing Systems},
 editor = {S. Bengio and H. Wallach and H. Larochelle and K. Grauman and N. Cesa-Bianchi and R. Garnett},
 pages = {},
 publisher = {Curran Associates, Inc.},
 title = {Binary Classification from Positive-Confidence Data},
 volume = {31},
 year = {2018}
}

@inproceedings{dwork2021outcome,
  title={Outcome indistinguishability},
  author={Dwork, Cynthia and Kim, Michael P and Reingold, Omer and Rothblum, Guy N and Yona, Gal},
  booktitle={Proceedings of the 53rd Annual ACM SIGACT Symposium on Theory of Computing},
  pages={1095--1108},
  year={2021}
}

@InProceedings{johnson18a,
  title = 	 {Multicalibration: Calibration for the ({C}omputationally-Identifiable) Masses},
  author =       {Hebert-Johnson, Ursula and Kim, Michael and Reingold, Omer and Rothblum, Guy},
  booktitle = 	 {Proceedings of the 35th International Conference on Machine Learning},
  pages = 	 {1939--1948},
  year = 	 {2018},
  editor = 	 {Dy, Jennifer and Krause, Andreas},
  volume = 	 {80},
  series = 	 {Proceedings of Machine Learning Research},
  month = 	 {10--15 Jul},
  publisher =    {PMLR},
}

@article{bartlett2002rademacher,
  title={Rademacher and gaussian complexities: Risk bounds and structural results},
  author={Bartlett, Peter L and Mendelson, Shahar},
  journal={Journal of machine learning research},
  volume={3},
  number={Nov},
  pages={463--482},
  year={2002}
}

@inproceedings{
hu2024testing,
title={Testing Calibration in Nearly-Linear Time},
author={Lunjia Hu and Arun Jambulapati and Kevin Tian and Chutong Yang},
booktitle={The Thirty-eighth Annual Conference on Neural Information Processing Systems},
year={2024}
}

@inproceedings{gopalan2024computationally,
  title={On computationally efficient multi-class calibration},
  author={Gopalan, Parikshit and Hu, Lunjia and Rothblum, Guy N},
  booktitle={The Thirty Seventh Annual Conference on Learning Theory},
  pages={1983--2026},
  year={2024},
  organization={PMLR}
}

@article{platt1999probabilistic,
  title={Probabilistic outputs for support vector machines and comparisons to regularized likelihood methods},
  author={Platt, John and others},
  journal={Advances in large margin classifiers},
  volume={10},
  number={3},
  pages={61--74},
  year={1999},
  publisher={Cambridge, MA}
}

@inproceedings{NEURIPS2024_futami,
 author = {Futami, Futoshi and Fujisawa, Masahiro},
 booktitle = {Advances in Neural Information Processing Systems},
 editor = {A. Globerson and L. Mackey and D. Belgrave and A. Fan and U. Paquet and J. Tomczak and C. Zhang},
 pages = {84246--84297},
 publisher = {Curran Associates, Inc.},
 title = {Information-theoretic Generalization Analysis for Expected Calibration Error},
 volume = {37},
 year = {2024}
}

@InProceedings{pmlr-v202-globus-harris23a,
  title = 	 {Multicalibration as Boosting for Regression},
  author =       {Globus-Harris, Ira and Harrison, Declan and Kearns, Michael and Roth, Aaron and Sorrell, Jessica},
  booktitle = 	 {Proceedings of the 40th International Conference on Machine Learning},
  pages = 	 {11459--11492},
  year = 	 {2023},
  editor = 	 {Krause, Andreas and Brunskill, Emma and Cho, Kyunghyun and Engelhardt, Barbara and Sabato, Sivan and Scarlett, Jonathan},
  volume = 	 {202},
  series = 	 {Proceedings of Machine Learning Research},
  month = 	 {23--29 Jul},
  publisher =    {PMLR},
}

@inproceedings{he2016deep,
  title={Deep residual learning for image recognition},
  author={He, Kaiming and Zhang, Xiangyu and Ren, Shaoqing and Sun, Jian},
  booktitle={Proceedings of the IEEE conference on computer vision and pattern recognition},
  pages={770--778},
  year={2016}
}

@article{sanh2019distilbert,
  title={DistilBERT, a distilled version of BERT: smaller, faster, cheaper and lighter},
  author={Sanh, Victor and Debut, Lysandre and Chaumond, Julien and Wolf, Thomas},
  journal={arXiv preprint arXiv:1910.01108},
  year={2019}
}

@article{dosovitskiy2020image,
  title={An image is worth 16x16 words: Transformers for image recognition at scale},
  author={Dosovitskiy, Alexey and Beyer, Lucas and Kolesnikov, Alexander and Weissenborn, Dirk and Zhai, Xiaohua and Unterthiner, Thomas and Dehghani, Mostafa and Minderer, Matthias and Heigold, Georg and Gelly, Sylvain and others},
  journal={arXiv preprint arXiv:2010.11929},
  year={2020}
}

@inproceedings{borkan2019nuanced,
  title={Nuanced metrics for measuring unintended bias with real data for text classification},
  author={Borkan, Daniel and Dixon, Lucas and Sorensen, Jeffrey and Thain, Nithum and Vasserman, Lucy},
  booktitle={Companion proceedings of the 2019 world wide web conference},
  pages={491--500},
  year={2019}
}

@inproceedings{liu2015deep,
  title={Deep learning face attributes in the wild},
  author={Liu, Ziwei and Luo, Ping and Wang, Xiaogang and Tang, Xiaoou},
  booktitle={Proceedings of the IEEE international conference on computer vision},
  pages={3730--3738},
  year={2015}
}

@inproceedings{sharma2021fair,
  title={Fair-n: Fair and robust neural networks for structured data},
  author={Sharma, Shubham and Gee, Alan H and Paydarfar, David and Ghosh, Joydeep},
  booktitle={Proceedings of the 2021 AAAI/ACM Conference on AI, Ethics, and Society},
  pages={946--955},
  year={2021}
}

@article{cooper2023variance,
  title={Variance, self-consistency, and arbitrariness in fair classification},
  author={Cooper, A Feder and Barocas, Solon and De Sa, Christopher and Sen, Siddhartha},
  journal={arXiv preprint arXiv:2301.11562},
  pages={1--84},
  year={2023}
}

@misc{default_of_credit_card_clients_350,
  author       = {Yeh, I-Cheng},
  title        = {{Default of Credit Card Clients}},
  year         = {2009},
  howpublished = {UCI Machine Learning Repository},
  note         = {{DOI}: https://doi.org/10.24432/C55S3H}
}

@article{ding2021retiring,
  title={Retiring adult: New datasets for fair machine learning},
  author={Ding, Frances and Hardt, Moritz and Miller, John and Schmidt, Ludwig},
  journal={Advances in neural information processing systems},
  volume={34},
  pages={6478--6490},
  year={2021}
}

@inproceedings{guo2017calibration,
  title={On calibration of modern neural networks},
  author={Guo, C. and Pleiss, G. and Sun, Y. and Weinberger, K. Q},
  booktitle={International conference on machine learning},
  pages={1321--1330},
  year={2017}
}

@article{marx2023calibration,
  title={Calibration by distribution matching: Trainable kernel calibration metrics},
  author={Marx, Charlie and Zalouk, Sofian and Ermon, Stefano},
  journal={Advances in Neural Information Processing Systems},
  volume={36},
  pages={25910--25928},
  year={2023}
}

@inproceedings{
widmann2021calibration,
title={Calibration tests beyond classification},
author={D. Widmann and F. Lindsten and D. Zachariah},
booktitle={International Conference on Learning Representations},
year={2021},
}

@book{Mohri,
	edition = {2},
	title = {Foundations of {Machine} {Learning}},
	isbn = {978-0-262-03940-6},
	publisher = {The MIT Press},
	author = {Mohri, Mehryar and Rostamizadeh, Afshin and Talwalkar, Ameet},
	year = {2018},
}

@article{Dawid1982,
author = {A. P. Dawid},
title = {The Well-Calibrated {B}ayesian},
journal = {Journal of the American Statistical Association},
volume = {77},
number = {379},
pages = {605--610},
year = {1982},
publisher = {Taylor & Francis},
doi = {10.1080/01621459.1982.10477856},
}

@InProceedings{kearns18a,
  title = 	 {Preventing Fairness Gerrymandering: Auditing and Learning for Subgroup Fairness},
  author =       {Kearns, Michael and Neel, Seth and Roth, Aaron and Wu, Zhiwei Steven},
  booktitle = 	 {Proceedings of the 35th International Conference on Machine Learning},
  pages = 	 {2564--2572},
  year = 	 {2018},
  editor = 	 {Dy, Jennifer and Krause, Andreas},
  volume = 	 {80},
  series = 	 {Proceedings of Machine Learning Research},
  month = 	 {10--15 Jul},
  publisher =    {PMLR},
}

@inproceedings{kim2019multiaccuracy,
  title={Multiaccuracy: Black-box post-processing for fairness in classification},
  author={Kim, Michael P and Ghorbani, Amirata and Zou, James},
  booktitle={Proceedings of the 2019 AAAI/ACM Conference on AI, Ethics, and Society},
  pages={247--254},
  year={2019}
}

@article{hansen2024multicalibration,
  title={When is multicalibration post-processing necessary?},
  author={Hansen, Dutch and Devic, Siddartha and Nakkiran, Preetum and Sharan, Vatsal},
  journal={Advances in Neural Information Processing Systems},
  volume={37},
  pages={38383--38455},
  year={2024}
}

@inproceedings{
futami2026smooth,
title={Smooth Calibration Error: Uniform Convergence and Functional Gradient Analysis},
author={Futoshi Futami and Atsushi Nitanda},
booktitle={The Fourteenth International Conference on Learning Representations},
year={2026},
}

@article{Bat2024,
  author  = {Bat-Sheva Einbinder and Shai Feldman and Stephen Bates and Anastasios N. Angelopoulos and Asaf Gendler and Yaniv Romano},
  title   = {Label Noise Robustness of Conformal Prediction},
  journal = {Journal of Machine Learning Research},
  year    = {2024},
  volume  = {25},
  number  = {328},
  pages   = {1--66},
}

@article{maxime2024,
  author  = {Maxime Cauchois and Suyash Gupta and Alnur Ali and John C. Duchi},
  title   = {Predictive Inference with Weak Supervision},
  journal = {Journal of Machine Learning Research},
  year    = {2024},
  volume  = {25},
  number  = {118},
  pages   = {1--45},
}

@inproceedings{gopalan2022low,
  title={Low-degree multicalibration},
  author={Gopalan, Parikshit and Kim, Michael P and Singhal, Mihir A and Zhao, Shengjia},
  booktitle={Conference on Learning Theory},
  pages={3193--3234},
  year={2022},
  organization={PMLR}
}

@article{gupta2020,
  title={Distribution-free binary classification: prediction sets, confidence intervals and calibration},
  author={Gupta, C. and Podkopaev, A. and Ramdas, A.},
  journal={Advances in Neural Information Processing Systems},
  volume={33},
  pages={3711--3723},
  year={2020}
}

@inproceedings{blasiok2023unify,
author = {B\l{}asiok, J. and Gopalan, P. and Hu, L. and Nakkiran, P.},
title = {A Unifying Theory of Distance from Calibration},
year = {2023},
booktitle = {Proceedings of the 55th Annual ACM Symposium on Theory of Computing},
pages = {1727–1740},
numpages = {14}
}

@article{foster1998asymptotic,
  title={Asymptotic calibration},
  author={Foster, Dean P and Vohra, Rakesh V},
  journal={Biometrika},
  volume={85},
  number={2},
  pages={379--390},
  year={1998},
  publisher={Oxford University Press}
}

@InProceedings{rossellini2025cutoff,
  title = 	 {Can a calibration metric be both testable and actionable?},
  author =       {Rossellini, Raphael and Soloff, Jake A. and Barber, Rina Foygel and Ren, Zhimei and Willett, Rebecca},
  booktitle = 	 {Proceedings of Thirty Eighth Conference on Learning Theory},
  pages = 	 {4937--4972},
  year = 	 {2025},
  editor = 	 {Haghtalab, Nika and Moitra, Ankur},
  volume = 	 {291},
  series = 	 {Proceedings of Machine Learning Research},
  month = 	 {30 Jun--04 Jul},
  publisher =    {PMLR}
}
\bibliographystyle{icml2024}

\newpage

\appendix

\section{Proofs for the generic rewrite statements}
\label{app:proofs-main}

\subsection{Review of the risk rewrite}
Starting from Eq.~\eqref{eq:loss-vector}, write
\[
R_\ell(f)=\int_{\calx}L_\ell(x;f)^\top\,dB(x).
\]
Substituting the decontamination identity $B=M_{\mathrm{corr}}^{\dagger}\bar P$ from Eq.~\eqref{eq:decontam} gives
\[
R_\ell(f)=\int_{\calx}L_\ell(x;f)^\top M_{\mathrm{corr}}^{\dagger}\,d\bar P(x),
\]
which is exactly Eq.~\eqref{eq:generic-risk-rewrite}.

The same substitution also covers the $x$-dependent case that appears in confidence-based weak supervision. Namely, whenever the clean conditional vector can be recovered through an $x$-dependent decontamination operator
\begin{equation}
\label{eq:operator-decontam}
B(dx)=M_{\mathrm{corr}}^{\dagger}(x)\,\bar P(dx),
\end{equation}
we obtain the pointwise rewrite
\begin{equation}
\label{eq:operator-mc}
R_\ell(f)=\int_{\calx}L_\ell(x;f)^\top M_{\mathrm{corr}}^{\dagger}(x)\,d\bar P(x).
\end{equation}
Thus the algebra of the contamination-matrix case remains unchanged except that the correction coefficients are now functions of $x$. In the terminology of \citet{ChiangSugiyama2025}, this is the regime of confidence-based settings, where observable confidences provide the weights defining the contamination matrix. Section~\ref{sec:examples} treats Pconf as the representative binary example, and Appendix~\ref{app:additional-settings} records other instances such as posterior-confidence and Sconf supervision.

\subsection{Proof of Theorem~\ref{thm:generic-wsl-mc}}
 By definition,
\[
R_{c,w}(f)=\Ex[c(X)w(f(X))(Y-f(X))].
\]
Fix $(c,w)$. By conditioning on \(Y\), we obtain
\begin{align*}
R_{c,w}(f)
&=
\Ex[\phi_f^{c,w}(X,Y)] \\
&=
\Ex[\phi_f^{c,w}(X,Y)\mathbf 1\{Y=1\}]
+
\Ex[\phi_f^{c,w}(X,Y)\mathbf 1\{Y=0\}] \\
&=
P(Y=1)\Ex[\phi_f^{c,w}(X,1)\mid Y=1]
+
P(Y=0)\Ex[\phi_f^{c,w}(X,0)\mid Y=0] \\
&=
\pi_+ \Ex_{P_+}[\phi_f^{c,w}(X,1)]
+
\pi_- \Ex_{P_-}[\phi_f^{c,w}(X,0)].
\end{align*}

For a fixed witness \(\phi_f\), define the corresponding vector of label-resolved terms
\[
L_{c,w}(x;f)
:=
\begin{pmatrix}
\pi_+ \phi_f(x,1)\\
\pi_- \phi_f(x,0)
\end{pmatrix}.
\]
Consequently, the witness moment is
\[
R_{c,w}(f)
=
\int_{\mathcal X} L_{c,w}(x;f)^\top\, dB(x).
\]
Substituting the decontamination identity $B=M_{\mathrm{corr}}^{\dagger}\bar P$ from Eq.~\eqref{eq:decontam} gives
\[
R_{c,w}(f)=\int_{\calx}L_{c,w}(x;f)^\top M_{\mathrm{corr}}^{\dagger}\,d\bar P(x).
\]
Taking the supremum over $(c,w)$ yields Eq.~\eqref{eq:def-mc-wsl}.

\section{Additional weak-supervision settings}
\label{app:additional-settings}
This appendix collects the binary weak-supervision estimators induced by the contamination-based rewrite. Following the request of the main text, we omit multiclass cases and organize the appendix by weak-supervision model. In every case, the derivation follows the same template: first rewrite a fixed witness
\[
R_\phi(f)
:=
\pi_+\Ex_{P_+}[\phi_f(X,1)]
+
\pi_-\Ex_{P_-}[\phi_f(X,0)],
\]
under the observable weak distributions, and then substitute the multicalibration witness
\[
\phi_f^{c,w}(x,y)=c(x)w(f(x))(y-f(x)).
\]
For later use, note that
\[
\phi_f^{c,w}(x,1)=c(x)w(f(x))(1-f(x)),
\qquad
\phi_f^{c,w}(x,0)=-c(x)w(f(x))f(x).
\]

\subsection{MCD scenarios: UU template and special cases}
\label{app:mcd-full}
We begin with the matrix-contamination settings. This subsection gives the derivation underlying the concise UU discussion in Section~\ref{sec:uu}.

\paragraph{Data model.} In the general UU setting we observe two unlabeled samples $X_1^{U_1},\dots,X_{n_1}^{U_1}\sim P_{U_1}$ and $X_1^{U_2},\dots,X_{n_2}^{U_2}\sim P_{U_2}$.  Write $\theta_k=P(Y=1\mid U_k)$, so that $P_{U_k}=\theta_kP_++(1-\theta_k)P_-$.  We use $B=(P_+,P_-)^\top$ and $L_\phi(x)=(\pi_+\phi_f(x,1),\pi_-\phi_f(x,0))^\top$.  Then
\begin{equation}
\label{eq:def-M-UU-app}
\bar P=
\begin{pmatrix}P_{U_1}\\ P_{U_2}\end{pmatrix}
=
M_{\UU}B,
\qquad
M_{\UU}=
\begin{pmatrix}
\theta_1 & 1-\theta_1\\
\theta_2 & 1-\theta_2
\end{pmatrix},
\end{equation}
with $\Delta:=\theta_1-\theta_2\neq 0$. Hence
\begin{equation}
\label{eq:def-M-UU-inv-app}
M_{\UU}^{-1}
=
\frac{1}{\Delta}
\begin{pmatrix}
1-\theta_2 & -(1-\theta_1)\\
-\theta_2 & \theta_1
\end{pmatrix}.
\end{equation}
Multiplying $L_\phi(x)^\top$ by $M_{\UU}^{-1}$ gives the corrected witness row vector
\begin{align}
L_\phi(x)^\top M_{\UU}^{-1}
&=
\frac{1}{\Delta}
\Bigl(
(1-\theta_2)\pi_+\phi_f(x,1)-\theta_2\pi_-\phi_f(x,0),
\nonumber\\[-1mm]
&\hspace{28mm}
-(1-\theta_1)\pi_+\phi_f(x,1)+\theta_1\pi_-\phi_f(x,0)
\Bigr).
\label{eq:Lphi-UU-app}
\end{align}
Substituting Eq.~\eqref{eq:Lphi-UU-app} into the generic rewrite yields
\begin{align}
R_\phi^{\UU}(f)
&=
\Ex_{P_{U_1}}\!\left[
\frac{(1-\theta_2)\pi_+}{\Delta}\phi_f(X,1)
-
\frac{\theta_2\pi_-}{\Delta}\phi_f(X,0)
\right]
\nonumber\\
&\quad+
\Ex_{P_{U_2}}\!\left[
-\frac{(1-\theta_1)\pi_+}{\Delta}\phi_f(X,1)
+
\frac{\theta_1\pi_-}{\Delta}\phi_f(X,0)
\right].
\label{eq:Rphi-UU-app}
\end{align}
For the multicalibration witness, define
\[
\alpha_1(v):=\frac{(1-\theta_2)\pi_+(1-v)+\theta_2\pi_-v}{\Delta},
\qquad
\alpha_2(v):=\frac{-(1-\theta_1)\pi_+(1-v)-\theta_1\pi_-v}{\Delta}.
\]
With this notation,
\begin{align}
R_{c,w}^{\UU}(f)
&=
\sum_{k=1}^2\Ex_{P_{U_k}}\!\left[
\alpha_k(f(X))c(X)w(f(X))
\right],
\label{eq:Rcw-UU-app}
\\
\widehat R_{c,w}^{\UU}(f)
&:=
\sum_{k=1}^2\frac{1}{n_k}\sum_{i=1}^{n_k}
\alpha_k(f(X_i^{U_k}))c(X_i^{U_k})w(f(X_i^{U_k})).
\label{eq:Rhat-UU-app}
\end{align}
The remaining MCD-type estimators are obtained by substituting the appropriate positive-class proportions into Eq.~\eqref{eq:Rphi-UU-app}--\eqref{eq:Rhat-UU-app}. Unless stated otherwise, these MCD-style rewrites assume that the class prior $\pi_+$ and the weak-source mixing parameters appearing in the model are known or have been pre-estimated. We write the resulting fixed-witness identities explicitly below. In the similar/dissimilar settings, $P_{\tilde S}$ and $P_{\tilde D}$ denote the observable one-point marginals induced by similar and dissimilar pair supervision, namely
\[
P_{\tilde S}=\frac{\pi_+^2}{\pi_+^2+\pi_-^2}P_+ + \frac{\pi_-^2}{\pi_+^2+\pi_-^2}P_-,
\qquad
P_{\tilde D}=\frac12 P_+ + \frac12 P_-.
\]

\paragraph{PU learning.}
\emph{Observed data.} We observe a positively labeled sample $X_1^+,\dots,X_{n_+}^+\sim P_+$ and an unlabeled sample $X_1^u,\dots,X_{n_u}^u\sim P_X=\pi_+P_+ + \pi_-P_-$. We assume the class prior $\pi_+$ is known or pre-estimated. This is the UU specialization $(\theta_1,\theta_2)=(1,\pi_+)$.
Set $\theta_1=1$ and $\theta_2=\pi_+$, so that
\[
\bar P=
\begin{pmatrix}P_+\\ P_X\end{pmatrix},
\qquad
M_{\PU}=
\begin{pmatrix}
1 & 0\\
\pi_+ & \pi_-
\end{pmatrix},
\qquad
M_{\PU}^{-1}=
\begin{pmatrix}
1 & 0\\
-\pi_+/\pi_- & 1/\pi_-
\end{pmatrix}.
\]
Hence the PU fixed-witness rewrite is
\begin{align}
R_\phi^{\PU}(f)
&=
\pi_+\Ex_{P_+}[\phi_f(X,1)-\phi_f(X,0)]
+
\Ex_{P_X}[\phi_f(X,0)].
\label{eq:Rphi-PU-app}
\end{align}
Since $\phi_f^{c,w}(x,1)-\phi_f^{c,w}(x,0)=c(x)w(f(x))$, it follows that
\begin{align}
R_{c,w}^{\PU}(f)
&=
\pi_+\Ex_{P_+}[c(X)w(f(X))]
-
\Ex_{P_X}[c(X)w(f(X))f(X)],
\label{eq:Rcw-PU-app}
\\
\widehat R_{c,w}^{\PU}(f)
&:=
\pi_+\frac{1}{n_+}\sum_{i=1}^{n_+}c(X_i^+)w(f(X_i^+))
-
\frac{1}{n_u}\sum_{j=1}^{n_u}c(X_j^u)w(f(X_j^u))f(X_j^u).
\label{eq:Rhat-PU-app}
\end{align}

\paragraph{SU learning.}
\emph{Observed data.} We observe i.i.d. similar pairs $(X_i^s,X_i^{s\prime})$ known to share the same label, together with unlabeled points $X_1^u,\dots,X_{n_u}^u\sim P_X$. Let $X_1^{\tilde S},\dots,X_{n_s}^{\tilde S}\sim P_{\tilde S}$ denote the induced one-point marginal of the similar-pair sample. The rewrite below is written in terms of this observable marginal.
Set $\theta_1=\pi_+^2/(\pi_+^2+\pi_-^2)$ and $\theta_2=\pi_+$. Then
\begin{align}
R_\phi^{\SU}(f)
&=
\Ex_{P_{\tilde S}}\!\left[
\frac{\pi_+^2+\pi_-^2}{\pi_+-\pi_-}\bigl(\phi_f(X,1)-\phi_f(X,0)\bigr)
\right]
\nonumber\\
&\quad+
\Ex_{P_X}\!\left[
-\frac{\pi_-}{\pi_+-\pi_-}\phi_f(X,1)
+
\frac{\pi_+}{\pi_+-\pi_-}\phi_f(X,0)
\right].
\label{eq:Rphi-SU-app}
\end{align}
Substituting $\phi_f^{c,w}$ gives
\begin{align}
R_{c,w}^{\SU}(f)
&=
\frac{\pi_+^2+\pi_-^2}{\pi_+-\pi_-}\Ex_{P_{\tilde S}}[c(X)w(f(X))]
-
\Ex_{P_X}\!\left[c(X)w(f(X))\left(\frac{\pi_-}{\pi_+-\pi_-}+f(X)\right)\right],
\label{eq:Rcw-SU-app}
\\
\widehat R_{c,w}^{\SU}(f)
&:=
\frac{\pi_+^2+\pi_-^2}{\pi_+-\pi_-}\frac{1}{n_s}\sum_{i=1}^{n_s}c(X_i^{\tilde S})w(f(X_i^{\tilde S}))
\nonumber\\
&\quad-
\frac{1}{n_u}\sum_{j=1}^{n_u}c(X_j^u)w(f(X_j^u))\left(\frac{\pi_-}{\pi_+-\pi_-}+f(X_j^u)\right).
\label{eq:Rhat-SU-app}
\end{align}

\paragraph{DU learning.}
\emph{Observed data.} We observe i.i.d. dissimilar pairs $(X_i^d,X_i^{d\prime})$ known to have different labels, together with unlabeled points $X_1^u,\dots,X_{n_u}^u\sim P_X$. Let $X_1^{\tilde D},\dots,X_{n_d}^{\tilde D}\sim P_{\tilde D}$ denote the induced one-point marginal of the dissimilar-pair sample. The rewrite below is written in terms of this observable marginal.
Set $\theta_1=1/2$ and $\theta_2=\pi_+$. Then
\begin{align}
R_\phi^{\DU}(f)
&=
\Ex_{P_{\tilde D}}\!\left[
\frac{2\pi_+\pi_-}{\pi_--\pi_+}\bigl(\phi_f(X,1)-\phi_f(X,0)\bigr)
\right]
\nonumber\\
&\quad+
\Ex_{P_X}\!\left[
-\frac{\pi_+}{\pi_--\pi_+}\phi_f(X,1)
+
\frac{\pi_-}{\pi_--\pi_+}\phi_f(X,0)
\right].
\label{eq:Rphi-DU-app}
\end{align}
Therefore the DU witness moment is
\begin{align}
R_{c,w}^{\DU}(f)
&=
-\frac{2\pi_+\pi_-}{\pi_+-\pi_-}\Ex_{P_{\tilde D}}[c(X)w(f(X))]
+
\Ex_{P_X}\!\left[c(X)w(f(X))\left(\frac{\pi_+}{\pi_+-\pi_-}-f(X)\right)\right],
\label{eq:Rcw-DU-app}
\\
\widehat R_{c,w}^{\DU}(f)
&:=
-\frac{2\pi_+\pi_-}{\pi_+-\pi_-}\frac{1}{n_d}\sum_{i=1}^{n_d}c(X_i^{\tilde D})w(f(X_i^{\tilde D}))
\nonumber\\
&\quad+
\frac{1}{n_u}\sum_{j=1}^{n_u}c(X_j^u)w(f(X_j^u))\left(\frac{\pi_+}{\pi_+-\pi_-}-f(X_j^u)\right).
\label{eq:Rhat-DU-app}
\end{align}

\paragraph{SD learning.}
\emph{Observed data.} We observe both i.i.d. similar pairs $(X_i^s,X_i^{s\prime})$ and i.i.d. dissimilar pairs $(X_j^d,X_j^{d\prime})$. Let $X_1^{\tilde S},\dots,X_{n_s}^{\tilde S}\sim P_{\tilde S}$ and $X_1^{\tilde D},\dots,X_{n_d}^{\tilde D}\sim P_{\tilde D}$ denote the corresponding one-point marginals used below.
Set $\theta_1=\pi_+^2/(\pi_+^2+\pi_-^2)$ and $\theta_2=1/2$. Then
\begin{align}
R_\phi^{\SD}(f)
&=
(\pi_+^2+\pi_-^2)\Ex_{P_{\tilde S}}\!\left[
\frac{\pi_+}{\pi_+-\pi_-}\phi_f(X,1)
-
\frac{\pi_-}{\pi_+-\pi_-}\phi_f(X,0)
\right]
\nonumber\\
&\quad+
2\pi_+\pi_-\Ex_{P_{\tilde D}}\!\left[
-\frac{\pi_-}{\pi_+-\pi_-}\phi_f(X,1)
+
\frac{\pi_+}{\pi_+-\pi_-}\phi_f(X,0)
\right].
\label{eq:Rphi-SD-app}
\end{align}
Therefore the SD witness moment is
\begin{align}
R_{c,w}^{\SD}(f)
&=
(\pi_+^2+\pi_-^2)\Ex_{P_{\tilde S}}\!\left[c(X)w(f(X))\left(\frac{\pi_+}{\pi_+-\pi_-}-f(X)\right)\right]
\nonumber\\
&\quad-
2\pi_+\pi_-\Ex_{P_{\tilde D}}\!\left[c(X)w(f(X))\left(\frac{\pi_-}{\pi_+-\pi_-}+f(X)\right)\right],
\label{eq:Rcw-SD-app}
\\
\widehat R_{c,w}^{\SD}(f)
&:=
(\pi_+^2+\pi_-^2)\frac{1}{n_s}\sum_{i=1}^{n_s}c(X_i^{\tilde S})w(f(X_i^{\tilde S}))\left(\frac{\pi_+}{\pi_+-\pi_-}-f(X_i^{\tilde S})\right)
\nonumber\\
&\quad-
2\pi_+\pi_-\frac{1}{n_d}\sum_{j=1}^{n_d}c(X_j^{\tilde D})w(f(X_j^{\tilde D}))\left(\frac{\pi_-}{\pi_+-\pi_-}+f(X_j^{\tilde D})\right).
\label{eq:Rhat-SD-app}
\end{align}

\subsection{Pairwise comparison learning}
\label{app:pcomp-full}
\emph{Observed data.} We observe comparison pairs carrying a superior/inferior outcome rather than a clean binary label. Let $X_1^{\mathrm{sup}},\dots,X_{n_{\mathrm{sup}}}^{\mathrm{sup}}\sim P_{\mathrm{Sup}}$ and $X_1^{\mathrm{inf}},\dots,X_{n_{\mathrm{inf}}}^{\mathrm{inf}}\sim P_{\mathrm{Inf}}$ denote the induced one-point marginals associated with superior and inferior events, respectively. The comparison model has marginals
\[
P_{\mathrm{Sup}}=\frac{\pi_+P_++\pi_-^2P_-}{1-\pi_+\pi_-},
\qquad
P_{\mathrm{Inf}}=\frac{\pi_+^2P_++\pi_-P_-}{1-\pi_+\pi_-}.
\]
The fixed-witness rewrite uses these marginals and the known class prior.
The unified-risk rewrite is expressed through the corrected losses
\[
\bar\ell_{\mathrm{Sup}}=\ell_1-\pi_+\ell_0,
\qquad
\bar\ell_{\mathrm{Inf}}=-\pi_-\ell_1+\ell_0,
\]
where $\Ex_{\mathrm{Sup}}$ and $\Ex_{\mathrm{Inf}}$ denote expectations over $P_{\mathrm{Sup}}$ and $P_{\mathrm{Inf}}$. Replacing $(\ell_1,\ell_0)$ by $(\phi_f(\cdot,1),\phi_f(\cdot,0))$ yields
\begin{align}
R_\phi^{\Pcomp}(f)
&=
\Ex_{\mathrm{Sup}}[\phi_f(X,1)-\pi_+\phi_f(X,0)]
+
\Ex_{\mathrm{Inf}}[-\pi_-\phi_f(X,1)+\phi_f(X,0)].
\label{eq:Rphi-Pcomp-app}
\end{align}
Thus the comparison-learning witness moment is
\begin{align}
R_{c,w}^{\Pcomp}(f)
&=
\Ex_{\mathrm{Sup}}[c(X)w(f(X))(1-\pi_-f(X))]
-
\Ex_{\mathrm{Inf}}[c(X)w(f(X))(\pi_-+\pi_+f(X))],
\label{eq:Rcw-Pcomp-app}
\\
\widehat R_{c,w}^{\Pcomp}(f)
&:=
\frac{1}{n_{\mathrm{sup}}}\sum_{i=1}^{n_{\mathrm{sup}}}c(X_i^{\mathrm{sup}})w(f(X_i^{\mathrm{sup}}))(1-\pi_-f(X_i^{\mathrm{sup}}))
\nonumber\\
&\quad-
\frac{1}{n_{\mathrm{inf}}}\sum_{j=1}^{n_{\mathrm{inf}}}c(X_j^{\mathrm{inf}})w(f(X_j^{\mathrm{inf}}))(\pi_-+\pi_+f(X_j^{\mathrm{inf}})).
\label{eq:Rhat-Pcomp-app}
\end{align}

\subsection{Confidence-based binary settings}
\label{app:confidence-full}
We next collect binary settings whose decontamination rule depends on observed confidences.

\paragraph{Posterior-confidence learning.}
\emph{Observed data.} We observe i.i.d. pairs $(X_i,r_i)$ with $X_i\sim P_X$ and $r_i=r(X_i)=P(Y=1\mid X_i)$. In contrast to Pconf, the sample is drawn from the population marginal rather than from $P_+$. Conditioning on $X$ directly gives
\begin{align}
R_\phi^{\mathrm{PC}}(f)
&=
\Ex_X\big[r(X)\phi_f(X,1)+(1-r(X))\phi_f(X,0)\big].
\label{eq:Rphi-PC-app}
\end{align}
Specializing to $\phi_f^{c,w}$ yields
\begin{align}
R_{c,w}^{\mathrm{PC}}(f)
&=
\Ex_X\big[c(X)w(f(X))(r(X)-f(X))\big],
\label{eq:Rcw-PC-app}
\\
\widehat R_{c,w}^{\mathrm{PC}}(f)
&:=
\frac{1}{n}\sum_{i=1}^n c(X_i)w(f(X_i))(r_i-f(X_i)).
\label{eq:Rhat-PC-app}
\end{align}
Thus this supervision model replaces the inaccessible Bernoulli outcome $Y$ by the observed posterior mean $r(X)$.

\paragraph{Pconf learning.}
\emph{Observed data.} We observe i.i.d. positive-confidence pairs $(X_i,r_i)$ with $X_i\sim P_+$ and $r_i=r(X_i)=P(Y=1\mid X_i)$. We also assume the class prior $\pi_+$ is known or pre-estimated, because the exact witness moment keeps the original population scaling. The exact rewrite requires $P_-\ll P_+$. In the $B=(P_+,P_-)^\top$ notation of Section~\ref{sec:prelim-wsl}, this corresponds to the $x$-dependent decontamination operator $M_{\Pconf}^{\dagger}(x)=\mathrm{diag}\{1,(\pi_+/\pi_-)(1-r(x))/r(x)\}$ acting on the duplicated positive source $\bar P_{\Pconf}=(P_+,P_+)^\top$. Bayes' rule implies the density-ratio identity
\begin{equation}
\label{eq:pconf-density-app}
\pi_-P_-(dx)
=
\pi_+\frac{1-r(x)}{r(x)}P_+(dx).
\end{equation}
Substituting Eq.~\eqref{eq:pconf-density-app} into the label-resolved expansion of $R_\phi(f)$ gives
\begin{align}
R_\phi^{\Pconf}(f)
&=
\pi_+\Ex_{P_+}\!\left[\phi_f(X,1)+\frac{1-r(X)}{r(X)}\phi_f(X,0)\right].
\label{eq:Rphi-Pconf-app}
\end{align}
For the multicalibration witness,
\begin{align}
R_{c,w}^{\Pconf}(f)
&=
\pi_+\Ex_{P_+}\!\left[c(X)w(f(X))\left(1-\frac{f(X)}{r(X)}\right)\right],
\label{eq:Rcw-Pconf-app}
\\
\widehat R_{c,w}^{\Pconf}(f)
&:=
\frac{\pi_+}{n}\sum_{i=1}^n c(X_i)w(f(X_i))\left(1-\frac{f(X_i)}{r_i}\right).
\label{eq:Rhat-Pconf-app}
\end{align}
This is the exact fixed-witness estimator used in the main text.

\paragraph{Sconf learning.}
\emph{Observed data.} We observe i.i.d. triples
\[
\begin{gathered}
(X_1,X_1',r_1),\dots,(X_{n_{\mathrm{sc}}},X_{n_{\mathrm{sc}}}',r_{n_{\mathrm{sc}}}),\\
(X_i,X_i')\sim P_{XX'},\qquad
r_i=r(X_i,X_i'):=P(Y_i=Y_i'\mid X_i,X_i').
\end{gathered}
\]
Each pair consists of independent draws $(X,Y),(X',Y')$ from $P$, so $P_{XX'}=P_X\otimes P_X$. The class prior $\pi_+$ is known and $\pi_+\ne\pi_-$. In contrast to Pconf, the side information is attached to a pair rather than to a single positive example. The unified-risk rewrite gives
\begin{align}
R_\phi^{\Sconf}(f)
&=
\Ex_{(X,X',r)}\Biggl[
\frac{r-\pi_-}{\pi_+-\pi_-}\frac{\phi_f(X,1)+\phi_f(X',1)}{2}
+
\frac{\pi_+-r}{\pi_+-\pi_-}\frac{\phi_f(X,0)+\phi_f(X',0)}{2}
\Biggr].
\label{eq:Rphi-Sconf-app}
\end{align}
Substituting $\phi_f^{c,w}$ yields
\begin{align}
R_{c,w}^{\Sconf}(f)
&=
\Ex_{(X,X',r)}\Biggl[
\frac{r-\pi_-}{\pi_+-\pi_-}
\frac{c(X)w(f(X))(1-f(X))+c(X')w(f(X'))(1-f(X'))}{2}
\nonumber\\
&\hspace{1.2cm}-
\frac{\pi_+-r}{\pi_+-\pi_-}
\frac{c(X)w(f(X))f(X)+c(X')w(f(X'))f(X')}{2}
\Biggr],
\label{eq:Rcw-Sconf-app}
\\
\widehat R_{c,w}^{\Sconf}(f)
&:=
\frac{1}{n_{\mathrm{sc}}}\sum_{i=1}^{n_{\mathrm{sc}}}
\Biggl[
\frac{r_i-\pi_-}{\pi_+-\pi_-}
\frac{c(X_i)w(f(X_i))(1-f(X_i))+c(X_i')w(f(X_i'))(1-f(X_i'))}{2}
\nonumber\\
&\hspace{1.2cm}-
\frac{\pi_+-r_i}{\pi_+-\pi_-}
\frac{c(X_i)w(f(X_i))f(X_i)+c(X_i')w(f(X_i'))f(X_i')}{2}
\Biggr].
\label{eq:Rhat-Sconf-app}
\end{align}

\paragraph{Summary.}
Every binary weak-supervision model admitting the stated decontamination identity gives a fixed-witness estimator by the same substitution principle: rewrite $R_\phi(f)$ using the decontamination rule of the model, then insert the witness $\phi_f^{c,w}$. The resulting empirical multicalibration criterion is always
\[
\widehat{\MC}_{\calc,\calw}^{\,\mathrm{WSL}}(f)
:=
\sup_{c\in\calc,\,w\in\calw}
\bigl|\widehat R_{c,w}^{\mathrm{WSL}}(f)\bigr|,
\]
with $\widehat R_{c,w}^{\mathrm{WSL}}(f)$ chosen from Eq.~\eqref{eq:Rhat-UU-app}, Eq.~\eqref{eq:Rhat-PU-app}, \eqref{eq:Rhat-SU-app}, \eqref{eq:Rhat-DU-app}, \eqref{eq:Rhat-SD-app}, \eqref{eq:Rhat-Pcomp-app}, \eqref{eq:Rhat-PC-app}, \eqref{eq:Rhat-Pconf-app}, or \eqref{eq:Rhat-Sconf-app} according to the observation model.

\section{Proofs of the uniform convergence}
\label{app:gen}
Throughout this appendix, \(f,\calc,\calw\) are fixed independently of the weak samples, \(0<\delta<1\), and the correction uses the exact weak-observation parameters. We write
\[
\mathfrak{R}_n(\calc;P):=\Ex_{S\sim P^n,\,\sigma}\left[\sup_{c\in\calc}\frac1n\sum_{i=1}^n \sigma_i c(X_i)\right]
\]
for the expected Rademacher complexity of \(\calc\) under \(P\), where \(\sigma_i\in\{-1,+1\}\) are independent Rademacher variables.

\subsection{Uniform convergence for UU}
For the UU estimator in Eq.~\eqref{eq:Rhat-UU}, define
\[
\MC^{\UU}_{\calc,\calw}(f)
:=
\sup_{c\in\calc,\,w\in\calw}|R_{c,w}^{\UU}(f)|,
\qquad
\widehat{\MC}^{\UU}_{\calc,\calw}(f)
:=
\sup_{c\in\calc,\,w\in\calw}|\widehat R_{c,w}^{\UU}(f)|.
\]
Define the source envelopes
\[
A_1:=\sup_{v\in[0,1]}|\alpha_1(v)|,
\qquad
A_2:=\sup_{v\in[0,1]}|\alpha_2(v)|.
\]

\begin{proposition}[Uniform convergence for the UU weak estimate]
\label{prop:uu-gen}
Assume \(|c(x)|\le 1\) for all \(c\in\calc\) and \(|w(v)|\le 1\) for all \(w\in\calw\). Then for any \(\eta\in(0,1]\), with probability at least \(1-\delta\),
\[
\left|\widehat{\MC}^{\UU}_{\calc,\calw}(f)-\MC^{\UU}_{\calc,\calw}(f)\right|
\le
2\sum_{k=1}^2 A_k\left(\mathfrak{R}_{n_k}(\calc;P_{U_k})+\eta+\sqrt{\frac{2\log N_\infty(\eta,\calw)+\log(4/\delta)}{n_k}}\right).
\]
\end{proposition}
\begin{proof}[Proof of Proposition~\ref{prop:uu-gen}]
Fix \(\eta\in(0,1]\), and let
\(\widetilde{\calw}\subseteq\calw\) be a minimal \(\eta\)-net of
\(\calw\) under \(\|\cdot\|_\infty\), so that
\[
|\widetilde{\calw}|=N_\infty(\eta,\calw).
\]
For \(k\in\{1,2\}\) and \(w\in\widetilde{\calw}\), define the
source-specific classes
\[
\mathcal G_{k,w}(f)
:=
\left\{
x\mapsto \alpha_k(f(x))w(f(x))c(x): c\in\calc
\right\},
\qquad
\mathcal G_k(f)
:=
\bigcup_{w\in\widetilde{\calw}}\mathcal G_{k,w}(f).
\]
By construction, every function in \(\mathcal G_k(f)\) is bounded in
absolute value by \(A_k\).

By symmetrization and bounded-differences concentration, with
probability at least \(1-\delta\), simultaneously for \(k=1,2\),
\[
\sup_{g\in\mathcal G_k(f)}
\left|
\frac1{n_k}\sum_{i=1}^{n_k}g(X_i^{U_k})
-
\Ex_{P_{U_k}}[g(X)]
\right|
\le
2\mathfrak{R}_{n_k}(\mathcal G_k(f);P_{U_k})
+
A_k\sqrt{\frac{2\log(4/\delta)}{n_k}}.
\]

We next bound the Rademacher complexity of \(\mathcal G_k(f)\) by the
complexity of \(\calc\) and the size of the weight net.  For a fixed
\(w\in\widetilde{\calw}\) and a fixed sample
\(x_1,\ldots,x_{n_k}\), the maps
\[
\psi_i(t):=\alpha_k(f(x_i))w(f(x_i))t
\]
satisfy \(\psi_i(0)=0\) and are \(A_k\)-Lipschitz.  Therefore the
contraction inequality gives
\[
\widehat{\mathfrak{R}}_{n_k}(\mathcal G_{k,w}(f))
\le
A_k\widehat{\mathfrak{R}}_{n_k}(\calc).
\]

We use the following finite-union consequence of the sub-Gaussian
maximal inequality.  Conditional on a sample \(x_1,\ldots,x_m\), let
\(\mathcal F=\bigcup_{\ell=1}^M\mathcal F_\ell\) and suppose that
\(|h(x_i)|\le \lambda\) for all \(h\in\mathcal F\) and all \(i\).
Define
\[
Z_\ell
:=
\sup_{h\in\mathcal F_\ell}
\frac1m\sum_{i=1}^m\sigma_i h(x_i),
\qquad
\ell=1,\ldots,M .
\]
Then \(Z_\ell-\Ex_\sigma Z_\ell\) is
\(\lambda/\sqrt m\)-sub-Gaussian by the bounded-difference inequality.
Therefore, the sub-Gaussian maximal inequality gives
\[
\widehat{\mathfrak R}_m(\mathcal F)
=
\Ex_\sigma\max_{\ell\le M} Z_\ell
\le
\max_{\ell\le M}\widehat{\mathfrak R}_m(\mathcal F_\ell)
+
\lambda\sqrt{\frac{2\log M}{m}} .
\]
Applying this with
\[
\mathcal F=\mathcal G_k(f),
\qquad
\mathcal F_\ell=\mathcal G_{k,w_\ell}(f),
\qquad
M=|\widetilde{\calw}|=N_\infty(\eta,\calw),
\qquad
\lambda=A_k,
\]
and then taking expectations over the \(U_k\)-sample, yields
\[
\mathfrak{R}_{n_k}(\mathcal G_k(f);P_{U_k})
\le
A_k\mathfrak{R}_{n_k}(\calc;P_{U_k})
+
A_k\sqrt{\frac{2\log N_\infty(\eta,\calw)}{n_k}}.
\]
For the explicit constants, apply symmetrization and bounded-differences
concentration to each \(\mathcal G_{k,w}(f)\), and take a union bound over
\(k\in\{1,2\}\) and \(w\in\widetilde{\calw}\). The contraction bound gives,
with probability at least \(1-\delta\),
\begin{align}
\label{eq:uu-net-uniform}
\sup_{\substack{c\in\calc\\ w\in\widetilde{\calw}}}
\left|\widehat R_{c,w}^{\UU}(f)-R_{c,w}^{\UU}(f)\right|
&\le 2\sum_{k=1}^2 A_k
\left(\mathfrak{R}_{n_k}(\calc;P_{U_k})
+\sqrt{\frac{2\log N_\infty(\eta,\calw)+\log(4/\delta)}{n_k}}\right).
\end{align}

It remains to pass from the net \(\widetilde{\calw}\) to the full class
\(\calw\).  For any \(w\in\calw\), choose
\(\widetilde w\in\widetilde{\calw}\) such that
\(\|w-\widetilde w\|_\infty\le\eta\).  Then, for each
\(k\in\{1,2\}\), uniformly over \(c\in\calc\),
\[
\left|
\Ex_{P_{U_k}}
\left[
\alpha_k(f(X))c(X)
\{w(f(X))-\widetilde w(f(X))\}
\right]
\right|
\le
A_k\eta,
\]
and the same bound holds for the empirical average,
\[
\left|
\frac1{n_k}\sum_{i=1}^{n_k}
\alpha_k(f(X_i^{U_k}))c(X_i^{U_k})
\{w(f(X_i^{U_k}))-\widetilde w(f(X_i^{U_k}))\}
\right|
\le
A_k\eta.
\]
Thus the approximation error from replacing \(w\) by
\(\widetilde w\) is at most \(2A_k\eta\) for source \(k\).  Therefore,
\[
\sup_{c\in\calc,\,w\in\calw}
\left|
\widehat R_{c,w}^{\UU}(f)-R_{c,w}^{\UU}(f)
\right|
\le
2\sum_{k=1}^2 A_k
\left(
\mathfrak{R}_{n_k}(\calc;P_{U_k})
+
\eta
+
\sqrt{\frac{2\log N_\infty(\eta,\calw)+\log(4/\delta)}{n_k}}
\right).
\]
Finally,
\[
\left|
\widehat{\MC}^{\UU}_{\calc,\calw}(f)
-
\MC^{\UU}_{\calc,\calw}(f)
\right|
\le
\sup_{c\in\calc,\,w\in\calw}
\left|
\widehat R_{c,w}^{\UU}(f)-R_{c,w}^{\UU}(f)
\right|,
\]
which completes the proof.
\end{proof}

\subsection{Uniform convergence for PU}
\label{app:pu-gen}
This section proves Proposition~\ref{prop:pu-gen}. The proof is written directly for the two PU sources, rather than by introducing the composed class \(x\mapsto c(x)w(f(x))\).

\begin{proof}[Proof of Proposition~\ref{prop:pu-gen}]
Fix \(\eta\in(0,1]\), and let \(\widetilde{\calw}\subseteq\calw\) be a minimal \(\eta\)-net of \(\calw\) under \(\|\cdot\|_\infty\), so that
\[
|\widetilde{\calw}|=N_\infty(\eta,\calw).
\]
For \(w\in\widetilde{\calw}\), define the source-specific classes
\[
\mathcal G_{+,w}(f)
:=
\left\{
x\mapsto \pi_+\,w(f(x))c(x): c\in\calc
\right\},
\qquad
\mathcal G_{u,w}(f)
:=
\left\{
x\mapsto f(x)w(f(x))c(x): c\in\calc
\right\},
\]
and set
\[
\mathcal G_+(f):=\bigcup_{w\in\widetilde{\calw}}\mathcal G_{+,w}(f),
\qquad
\mathcal G_u(f):=\bigcup_{w\in\widetilde{\calw}}\mathcal G_{u,w}(f).
\]
Every function in \(\mathcal G_+(f)\) is bounded by \(\pi_+\) in absolute value, and every function in \(\mathcal G_u(f)\) is bounded by \(1\) in absolute value.

By symmetrization and bounded-differences concentration, with probability at least \(1-\delta\), simultaneously for the positive and unlabeled samples,
\[
\sup_{g\in\mathcal G_+(f)}
\left|
\frac1{n_+}\sum_{i=1}^{n_+}g(X_i^+)-\Ex_{P_+}[g(X)]
\right|
\le
2\mathfrak{R}_{n_+}(\mathcal G_+(f);P_+)
+
\pi_+\sqrt{\frac{2\log(4/\delta)}{n_+}},
\]
and, for the unlabeled source,
\[
\sup_{g\in\mathcal G_u(f)}
\left|
\frac1{n_u}\sum_{j=1}^{n_u}g(X_j^u)-\Ex_{P_X}[g(X)]
\right|
\le
2\mathfrak{R}_{n_u}(\mathcal G_u(f);P_X)
+
\sqrt{\frac{2\log(4/\delta)}{n_u}}.
\]
We next bound the two Rademacher complexities by the complexity of
\(\calc\) and the size of the weight net.  For a fixed
\(w\in\widetilde{\calw}\) and a fixed positive sample
\(x_1,\ldots,x_{n_+}\), the maps
\[
\psi_i(t):=\pi_+\,w(f(x_i))\,t
\]
satisfy \(\psi_i(0)=0\) and are \(\pi_+\)-Lipschitz. Therefore the
contraction inequality gives
\[
\widehat{\mathfrak{R}}_{n_+}(\mathcal G_{+,w}(f))
\le
\pi_+\,\widehat{\mathfrak{R}}_{n_+}(\calc).
\]

We use the following finite-union consequence of the sub-Gaussian
maximal inequality.  Conditional on a sample \(x_1,\ldots,x_m\), let
\(\mathcal F=\bigcup_{\ell=1}^M\mathcal F_\ell\) and suppose that
\(|h(x_i)|\le \lambda\) for all \(h\in\mathcal F\) and all \(i\).
Define
\[
Z_\ell
:=
\sup_{h\in\mathcal F_\ell}
\frac1m\sum_{i=1}^m\sigma_i h(x_i),
\qquad \ell=1,\ldots,M .
\]
Then \(Z_\ell-\Ex_\sigma Z_\ell\) is
\(\lambda/\sqrt m\)-sub-Gaussian by the bounded-difference inequality.
Therefore, the sub-Gaussian maximal inequality gives
\[
\widehat{\mathfrak R}_m(\mathcal F)
=
\Ex_\sigma\max_{\ell\le M} Z_\ell
\le
\max_{\ell\le M}\widehat{\mathfrak R}_m(\mathcal F_\ell)
+
\lambda\sqrt{\frac{2\log M}{m}} .
\]
Applying this with \(M=|\widetilde{\calw}|=N_\infty(\eta,\calw)\) and
\(\lambda=\pi_+\), and then taking expectations over the positive
sample, yields
\[
\mathfrak{R}_{n_+}(\mathcal G_+(f);P_+)
\le
\pi_+\,\mathfrak{R}_{n_+}(\calc;P_+)
+
\pi_+\sqrt{\frac{2\log N_\infty(\eta,\calw)}{n_+}}.
\]

Similarly, for the unlabeled sample, the maps
\[
\psi_i(t):=f(x_i)w(f(x_i))\,t
\]
satisfy \(\psi_i(0)=0\) and are \(1\)-Lipschitz, because
\(0\le f\le1\) and \(\|w\|_\infty\le1\). Hence, by the contraction
inequality and the same finite-union argument with
\(M=|\widetilde{\calw}|=N_\infty(\eta,\calw)\) and \(\lambda=1\),
\[
\mathfrak{R}_{n_u}(\mathcal G_u(f);P_X)
\le
\mathfrak{R}_{n_u}(\calc;P_X)
+
\sqrt{\frac{2\log N_\infty(\eta,\calw)}{n_u}}.
\]
The same argument as in Eq.~\eqref{eq:uu-net-uniform}, with
\((P_{U_1},n_1,A_1)=(P_+,n_+,\pi_+)\) and
\((P_{U_2},n_2,A_2)=(P_X,n_u,1)\), gives the corresponding uniform bound
for the PU estimator with probability at least \(1-\delta\).

It remains to pass from \(\widetilde{\calw}\) to \(\calw\). For any \(w\in\calw\), choose \(\widetilde w\in\widetilde{\calw}\) such that \(\|w-\widetilde w\|_\infty\le\eta\). Then, uniformly over \(c\in\calc\),
\[
\left|
\pi_+\,\Ex_{P_+}
\left[c(X)\{w(f(X))-\widetilde w(f(X))\}\right]
\right|
\le
\pi_+\eta,
\]
and the same bound holds for the empirical positive average. Likewise,
\[
\left|
\Ex_{P_X}
\left[c(X)\{w(f(X))-\widetilde w(f(X))\}f(X)\right]
\right|
\le
\eta,
\]
and the same bound holds for the empirical unlabeled average. Absorbing these approximation errors yields the bound in Eq.~\eqref{eq:pu-gen-main}. Finally,
\[
\left|
\widehat{\MC}^{\PU}_{\calc,\calw}(f)
-
\MC^{\PU}_{\calc,\calw}(f)
\right|
\le
\sup_{c\in\calc,\,w\in\calw}
\left|
\widehat R_{c,w}^{\PU}(f)-R_{c,w}^{\PU}(f)
\right|,
\]
which completes the proof.
\end{proof}

\paragraph{Example rates.}
The same bound yields explicit rates for other common witness choices.  If $n=n_+=n_u$, $\calw$ is finite, and $\calc$ has VC dimension $d$, the VC/Rademacher bound gives
\[
\bigl|\widehat\MC^{\PU}_{\calc,\calw}(f)-\MC_{\calc,\calw}(f)\bigr|
=O_p\!\left(\sqrt{\frac{d+\log(2|\calw|)}{n}}\right).
\]
If $\calc=\{1\}$ and $\calw=\mathrm{Lip}_1([0,1],[-1,1])$, corresponding to SmoothCE, the preceding single-scale covering bound is not sharp. Indeed, since $\log N_\infty(\eta,\calw)=O(1/\eta)$, optimizing Eq.~\eqref{eq:pu-gen-main} over one scale gives only $O_p(n^{-1/3})$.  For SmoothCE we instead use Dudley's chaining bound for the one-dimensional Lipschitz witness class.  Define
\[
\mathcal G_1:=\{z\mapsto w(z):w\in\calw\},
\qquad
\mathcal G_2:=\{z\mapsto z w(z):w\in\calw\}.
\]
Then $\mathcal G_1\subset \mathrm{Lip}_1([0,1],[-1,1])$. Moreover, for $w\in\calw$, the map $z\mapsto z w(z)$ is bounded by one and is $2$-Lipschitz, since
\[
|z w(z)-z'w(z')|
\le |z-z'|\,|w(z)|+z'|w(z)-w(z')|
\le 2|z-z'|.
\]
Hence, for $k=1,2$, there is a universal constant $C$ such that
\[
\log N_\infty(u,\mathcal G_k)\le \frac{C}{u},
\qquad u\in(0,1].
\]
For a positive sample $X_1^+,\ldots,X_m^+\overset{\mathrm{i.i.d.}}{\sim}P_+$ and an unlabeled sample $X_1^u,\ldots,X_m^u\overset{\mathrm{i.i.d.}}{\sim}P_X$, write
\[
\widehat\Ex_{P_+,m}[\varphi(X)]
:=\frac1m\sum_{i=1}^m \varphi(X_i^+),
\qquad
\widehat\Ex_{P_X,m}[\varphi(X)]
:=\frac1m\sum_{i=1}^m \varphi(X_i^u).
\]
Dudley's entropy integral, with the empirical $L_2$ entropy bounded by the sup-norm entropy, gives, for each $k=1,2$,
\begin{align}
&\Ex_{(X_i^+)\sim P_+^m}
\sup_{g\in\mathcal G_k}
\left|
\widehat\Ex_{P_+,m}[g(f(X))]-\Ex_{P_+}[g(f(X))]
\right|\\
&\le
2\mathfrak R_m(\mathcal G_k\circ f;P_+)
\le
\frac{C}{\sqrt m}
\int_0^1\sqrt{\log N_\infty(u,\mathcal G_k)}\,du
\le
\frac{C}{\sqrt m},
\end{align}
and similarly
\[
\Ex_{(X_i^u)\sim P_X^m}
\sup_{g\in\mathcal G_k}
\left|
\widehat\Ex_{P_X,m}[g(f(X))]-\Ex_{P_X}[g(f(X))]
\right|
\le
2\mathfrak R_m(\mathcal G_k\circ f;P_X)
\le
\frac{C}{\sqrt m}.
\]
Here we used $\int_0^1u^{-1/2}du<\infty$. Since all functions in $\mathcal G_k$ are uniformly bounded by one, bounded-difference concentration yields, with probability at least $1-\delta$,
\[
\sup_{g\in\mathcal G_k}
\left|
\widehat\Ex_{P_+,m}[g(f(X))]-\Ex_{P_+}[g(f(X))]
\right|
\le
C\sqrt{\frac{1+\log(1/\delta)}{m}},
\]
and the same bound holds with $P_+$ and $\widehat\Ex_{P_+,m}$ replaced by $P_X$ and $\widehat\Ex_{P_X,m}$.
For the PU residual with $\calc=\{1\}$,
\begin{align}
&\widehat R_{1,w}^{\PU}(f)-R_{1,w}^{\PU}(f)\\
&=
\pi_+
\Bigl\{
\widehat\Ex_{P_+,n_+}[w(f(X))]-\Ex_{P_+}[w(f(X))]
\Bigr\}
-
\Bigl\{
\widehat\Ex_{P_X,n_u}[f(X)w(f(X))]-\Ex_{P_X}[f(X)w(f(X))]
\Bigr\}.
\end{align}
Therefore,
\begin{align}
&\sup_{w\in\calw}\left|\widehat R_{1,w}^{\PU}(f)-R_{1,w}^{\PU}(f)\right|\\
&\le
\pi_+
\sup_{g\in\mathcal G_1}
\left|
\widehat\Ex_{P_+,n_+}[g(f(X))]-\Ex_{P_+}[g(f(X))]
\right|
+
\sup_{g\in\mathcal G_2}
\left|
\widehat\Ex_{P_X,n_u}[g(f(X))]-\Ex_{P_X}[g(f(X))]
\right|.
\end{align}
Together with
\[
\left|
\widehat{\MC}^{\PU}_{\{1\},\calw}(f)-\MC_{\{1\},\calw}(f)
\right|
\le
\sup_{w\in\calw}\left|\widehat R_{1,w}^{\PU}(f)-R_{1,w}^{\PU}(f)\right|,
\]
a union bound over the positive and unlabeled empirical processes gives, with probability at least $1-\delta$,
\[
\left|
\widehat{\MC}^{\PU}_{\{1\},\calw}(f)-\MC_{\{1\},\calw}(f)
\right|
\le
C\left\{
\pi_+\sqrt{\frac{1+\log(2/\delta)}{n_+}}
+
\sqrt{\frac{1+\log(2/\delta)}{n_u}}
\right\}.
\]
In particular, if $n_+=n_u=n$, then
\[
\bigl|\widehat\MC^{\PU}_{\{1\},\calw}(f)-\MC_{\{1\},\calw}(f)\bigr|=O_p(n^{-1/2}).
\]
This result matches the order of smooth CE in \citet{futami2026smooth}.

\begin{remark}[VC-type dimension]
\label{rem:known-pi-vc-dim}
We use the phrase ``VC dimension'' in the usual shorthand sense.
When \(\calc\) is a class of indicator functions, this is the ordinary
VC dimension.  For a general \([0,1]\)-valued subgroup class \(\calc\),
the corresponding assumption should be interpreted as finite
VC-subgraph dimension, equivalently finite pseudo-dimension up to
standard convention-dependent constants.  Under this condition the
standard Rademacher bound gives
\[
\mathfrak R_n(\calc;P)=O\!\left(\sqrt{\frac{d}{n}}\right),
\]
which is the complexity term used in the example rates.
\end{remark}

\subsection{Uniform convergence for Pconf}
\label{app:pconf-gen}
This section states and proves the Pconf uniform convergence bound. The setup matches \citet{ishida2018}: the Pconf sample consists of i.i.d. draws $(X_i,r_i)$ with $X_i\sim P_+$ and $r_i=r(X_i)=P(Y=1\mid X_i)$.

\begin{proposition}[Uniform convergence for the Pconf weak estimate]
\label{prop:pconf-gen}
Define the population and empirical Pconf weak-estimate errors by
\[
\MC^{\Pconf}_{\calc,\calw}(f):=\sup_{c\in\calc,\,w\in\calw}|R_{c,w}^{\Pconf}(f)|,
\qquad
\widehat{\MC}^{\Pconf}_{\calc,\calw}(f):=\sup_{c\in\calc,\,w\in\calw}|\widehat R_{c,w}^{\Pconf}(f)|.
\]
Assume $f,\calc,\calw$ are fixed independently of the Pconf sample, $|c(x)|\le 1$ for all $c\in\calc$, $|w(v)|\le 1$ for all $w\in\calw$, and there exists $C_r>0$ such that $r(x)\ge C_r$ almost surely under $P_+$. Then for any $\eta\in(0,1]$, with probability at least $1-\delta$,
\[
\left|\widehat{\MC}^{\Pconf}_{\calc,\calw}(f)-\MC^{\Pconf}_{\calc,\calw}(f)\right|
\le
C\,\pi_+\!\left(1+\frac{1}{C_r}\right)
\left(\mathfrak{R}_{n}(\calc;P_+)+\eta+\sqrt{\frac{2\log N_\infty(\eta,\calw)+\log(4/\delta)}{n}}\right).
\]
\end{proposition}

\begin{proof}
Fix $\eta\in(0,1]$ and let $\widetilde{\calw}\subseteq\calw$ be a minimal $\eta$-net of $\calw$ under $\|\cdot\|_\infty$, so $|\widetilde{\calw}|=N_\infty(\eta,\calw)$. Define
\[
\mathcal G_w(f):=\left\{x\mapsto \pi_+\left(1-\frac{f(x)}{r(x)}\right)w(f(x))c(x): c\in\calc\right\},
\qquad
\mathcal G(f):=\bigcup_{w\in\widetilde{\calw}}\mathcal G_w(f).
\]
By Corollary~\ref{prop:pconf},
\[
\sup_{c\in\calc,\,w\in\widetilde{\calw}}\left|\widehat R_{c,w}^{\Pconf}(f)-R_{c,w}^{\Pconf}(f)\right|
=
\sup_{g\in\mathcal G(f)}\left|\frac1n\sum_{i=1}^n g(X_i)-\Ex_{P_+}[g(X)]\right|.
\]
Because $0\le f\le 1$, $|w|\le 1$, and $r(x)\ge C_r$ almost surely, every $g\in\mathcal G(f)$ satisfies
\[
|g(x)|\le B:=\pi_+\left(1+\frac{1}{C_r}\right)
\qquad P_+\text{-a.s.}
\]
Standard symmetrization and McDiarmid's inequality therefore give, with probability at least $1-\delta$,
\[
\sup_{g\in\mathcal G(f)}\left|\frac1n\sum_{i=1}^n g(X_i)-\Ex_{P_+}[g(X)]\right|
\le
2\mathfrak{R}_n(\mathcal G(f);P_+)+B\sqrt{\frac{2\log(4/\delta)}{n}}.
\]
For a fixed $w\in\widetilde{\calw}$, the maps
\[
\psi_i(t):=\pi_+\left(1-\frac{f(x_i)}{r(x_i)}\right)w(f(x_i))t
\]
satisfy $\psi_i(0)=0$ and are $B$-Lipschitz, so the contraction inequality gives $\widehat{\mathfrak{R}}_n(\mathcal G_w(f))\le B\widehat{\mathfrak{R}}_n(\calc)$. Taking expectations and applying the finite-union lemma over $|\widetilde{\calw}|=N_\infty(\eta,\calw)$ classes yields
\[
\mathfrak{R}_n(\mathcal G(f);P_+)
\le
B\mathfrak{R}_n(\calc;P_+)+B\sqrt{\frac{2\log N_\infty(\eta,\calw)}{n}}.
\]
It remains to move from the net back to the full class. For any $w\in\calw$, choose $\tilde w\in\widetilde{\calw}$ with $\|w-\tilde w\|_\infty\le \eta$. Then
\[
\sup_c\left|\Ex_{P_+}\!\left[\pi_+\left(1-\frac{f(X)}{r(X)}\right)c(X)(w(f(X))-\tilde w(f(X)))\right]\right|\le B\eta,
\]
and the same estimate holds for the empirical average. Absorbing these approximation errors proves the stated bound.
\end{proof}

\paragraph{Clipped Pconf estimator.}
\label{app:clipped-pconf}
The Pconf rewrite in Corollary~\ref{prop:pconf} is exact but can have large variance when some confidences are small. A standard stabilization is to choose a clipping threshold $\tau\in(0,1]$ and define
\[
\widehat R_{c,w}^{\Pconf,\tau}(f)
:=
\frac{\pi_+}{n}\sum_{i=1}^n c(X_i)w(f(X_i))\left(1-\frac{f(X_i)}{\tau\vee r_i}\right).
\]
Inside Algorithm~\ref{alg:wsl-mc-boosting}, one may replace $\widehat R^{\Pconf}$ by $\widehat R^{\Pconf,\tau}$. Unless clipping leaves the confidences unchanged, the resulting bias must also be controlled.

\section{\texorpdfstring{Analysis of nuisance parameters}{Analysis of nuisance parameters}}
\label{app:nuisance-parameters}

The main paper assumes that the nuisance parameters used in the weak-supervision corrections are known. Here we analyze how errors in their estimates affect multicalibration estimation in the PU, UU, and Pconf settings. Assume \(0<\pi_+<1\), fix a predictor \(f:\calx\to[0,1]\), and let \(\calc\) and \(\calw\) be nonempty.

\subsection{PU: estimated class prior}
\label{app:nuisance-pu}
Given an estimate \(\widehat\pi_+\in[0,1]\) of the class prior \(\pi_+\), define
\[
e_\pi:=|\widehat\pi_+-\pi_+|,
\qquad \widehat\pi_-:=1-\widehat\pi_+.
\]
The population residual using this estimate is
\[
R^{\PU}_{c,w}(f;\widehat\pi_+)
:=\widehat\pi_+\Ex_{P_+}[c(X)w(f(X))]
-\Ex_{P_X}[c(X)w(f(X))f(X)].
\]

\begin{proposition}[PU nuisance error]
\label{prop:nuisance-pu}
For every \((c,w)\in\calc\times\calw\),
\[
R^{\PU}_{c,w}(f;\widehat\pi_+)-R_{c,w}(f)
=(\widehat\pi_+-\pi_+)\Ex_{P_+}[c(X)w(f(X))].
\]
Consequently,
\begin{equation}
\label{eq:nuisance-pu-bias}
\sup_{c\in\calc,\,w\in\calw}
\left|R^{\PU}_{c,w}(f;\widehat\pi_+)-R_{c,w}(f)\right|
\le e_\pi.
\end{equation}
\end{proposition}
\begin{proof}
Since \(P_X=\pi_+P_++\pi_-P_-\),
\begin{align*}
R_{c,w}(f)
={}&\pi_+\Ex_{P_+}[c(X)w(f(X))]
-\pi_+\Ex_{P_+}[c(X)w(f(X))f(X)]\\
&-\pi_-\Ex_{P_-}[c(X)w(f(X))f(X)]\\
={}&\pi_+\Ex_{P_+}[c(X)w(f(X))]
-\Ex_{P_X}[c(X)w(f(X))f(X)].
\end{align*}
Subtracting this identity from the definition of
\(R^{\PU}_{c,w}(f;\widehat\pi_+)\) gives the bias identity. Finally,
\[
\left|\Ex_{P_+}[c(X)w(f(X))]\right|
\le\Ex_{P_+}[|c(X)w(f(X))|]\le1
\]
implies Eq.~\eqref{eq:nuisance-pu-bias}.
\end{proof}

\begin{corollary}[PU nuisance and sampling errors]
\label{cor:nuisance-pu-sampling}
Assume that \(\calc\) and \(\calw\) are nonempty and finite, and let
\(\{X_i^+\}_{i=1}^{n_+}\overset{\mathrm{i.i.d.}}{\sim}P_+\) and
\(\{X_j^u\}_{j=1}^{n_u}\overset{\mathrm{i.i.d.}}{\sim}P_X\)
be independent samples. Assume that \(f,\calc,\calw,\widehat\pi_+\)
are fixed independently of these samples, and define
\[
\widehat R^{\PU}_{c,w}(f;\widehat\pi_+)
:={} \frac{\widehat\pi_+}{n_+}\sum_{i=1}^{n_+}
c(X_i^+)w(f(X_i^+))
 -\frac1{n_u}\sum_{j=1}^{n_u}
c(X_j^u)w(f(X_j^u))f(X_j^u).
\]
For \(\delta\in(0,1)\), with probability at least \(1-\delta\),
\begin{equation}
\label{eq:nuisance-pu-sampling}
\sup_{c\in\calc,\,w\in\calw}
\left|\widehat R^{\PU}_{c,w}(f;\widehat\pi_+)-R_{c,w}(f)\right|
\le{} e_\pi
+\widehat\pi_+\sqrt{\frac{2\log(4|\calc||\calw|/\delta)}{n_+}}
 +\sqrt{\frac{2\log(4|\calc||\calw|/\delta)}{n_u}}.
\end{equation}
\end{corollary}
\begin{proof}
Since \(|c\,w(f)|\le1\) and \(|c\,w(f)f|\le1\), the summands from
both sources lie in \([-1,1]\). Applying Hoeffding's inequality to each
witness and each source, and taking a union bound over
\(2|\calc||\calw|\) events, gives, with probability at least \(1-\delta\),
\begin{align*}
\left|\frac1{n_+}\sum_{i=1}^{n_+}c(X_i^+)w(f(X_i^+))
-\Ex_{P_+}[c(X)w(f(X))]\right|
&\le\sqrt{\frac{2\log(4|\calc||\calw|/\delta)}{n_+}},\\
\left|\frac1{n_u}\sum_{j=1}^{n_u}c(X_j^u)w(f(X_j^u))f(X_j^u)
-\Ex_{P_X}[c(X)w(f(X))f(X)]\right|
&\le\sqrt{\frac{2\log(4|\calc||\calw|/\delta)}{n_u}}
\end{align*}
simultaneously for all \((c,w)\in\calc\times\calw\).
Multiplying the first bound by \(\widehat\pi_+\), adding the second bound,
and then adding the bias \(e_\pi\) from Proposition~\ref{prop:nuisance-pu}
proves Eq.~\eqref{eq:nuisance-pu-sampling}.
\end{proof}

\subsection{UU: estimated class prior and mixture proportions}
\label{app:nuisance-uu}
The true observation distributions are
\[
P_{U_1}=(1-\gamma_1)P_++\gamma_1P_-,
\qquad
P_{U_2}=\gamma_2P_++(1-\gamma_2)P_-,
\]
where \(\gamma_1,\gamma_2\in[0,1]\) and
\(\Delta:=1-\gamma_1-\gamma_2\ne0\).
Given an estimate \(\widehat\pi_+\in[0,1]\) of the class prior \(\pi_+\)
and estimates \(\widehat\gamma_1,\widehat\gamma_2\in[0,1]\) of the mixture
proportions \(\gamma_1,\gamma_2\), let
\[
\widehat\gamma:=(\widehat\gamma_1,\widehat\gamma_2),
\qquad \widehat\Delta:=1-\widehat\gamma_1-\widehat\gamma_2\ne0,
\]
and define the estimation errors by
\[
\delta_\pi:=\widehat\pi_+-\pi_+,
\qquad \delta_k:=\widehat\gamma_k-\gamma_k\quad(k=1,2),
\qquad e_\gamma:=|\delta_1|+|\delta_2|.
\]
For \(v\in[0,1]\), the correction coefficients using these estimates are
\begin{equation}
\label{eq:nuisance-uu-coefficients}
\widehat\alpha_1(v)
 :=\frac{(1-\widehat\gamma_2)\widehat\pi_+(1-v)
+\widehat\gamma_2\widehat\pi_-v}{\widehat\Delta},
\qquad
\widehat\alpha_2(v)
 :=-\frac{\widehat\gamma_1\widehat\pi_+(1-v)
+(1-\widehat\gamma_1)\widehat\pi_-v}{\widehat\Delta}.
\end{equation}
Under the true observation distributions, the population residual using these coefficients is
\begin{equation}
\label{eq:nuisance-uu-residual}
R^{\UU}_{c,w}(f;\widehat\pi_+,\widehat\gamma)
:={} \Ex_{P_{U_1}}[c(X)w(f(X))\widehat\alpha_1(f(X))]
 +\Ex_{P_{U_2}}[c(X)w(f(X))\widehat\alpha_2(f(X))].
\end{equation}

\begin{proposition}[UU nuisance error]
\label{prop:nuisance-uu}
Under the assumptions above,
\begin{equation}
\label{eq:nuisance-uu-bias}
\sup_{c\in\calc,\,w\in\calw}
\left|R^{\UU}_{c,w}(f;\widehat\pi_+,\widehat\gamma)-R_{c,w}(f)\right|
\le2e_\pi+\frac{2e_\gamma}{|\widehat\Delta|}.
\end{equation}
If \(e_\gamma\le|\Delta|/2\), the right-hand side is at most
\(2e_\pi+4e_\gamma/|\Delta|\). If
\(|\Delta|,|\widehat\Delta|\ge\Delta_{\min}>0\), it is at most
\(2e_\pi+2e_\gamma/\Delta_{\min}\).
\end{proposition}
\begin{proof}
The coefficients in Eq.~\eqref{eq:nuisance-uu-coefficients} satisfy
\begin{equation}
\label{eq:nuisance-uu-system}
\begin{aligned}
(1-\widehat\gamma_1)\widehat\alpha_1(v)
+\widehat\gamma_2\widehat\alpha_2(v)&=\widehat\pi_+(1-v),\\
\widehat\gamma_1\widehat\alpha_1(v)
+(1-\widehat\gamma_2)\widehat\alpha_2(v)&=-\widehat\pi_-v.
\end{aligned}
\end{equation}
Indeed, the inverse of the coefficient matrix is
\[
\begin{pmatrix}
1-\widehat\gamma_1&\widehat\gamma_2\\
\widehat\gamma_1&1-\widehat\gamma_2
\end{pmatrix}^{-1}
=\frac1{\widehat\Delta}
\begin{pmatrix}
1-\widehat\gamma_2&-\widehat\gamma_2\\
-\widehat\gamma_1&1-\widehat\gamma_1
\end{pmatrix}.
\]
Multiplying this matrix by
\((\widehat\pi_+(1-v),-\widehat\pi_-v)^\top\) gives
Eq.~\eqref{eq:nuisance-uu-coefficients}.

Set \(d(v):=\delta_1\widehat\alpha_1(v)-\delta_2\widehat\alpha_2(v)\).
Substituting \(\gamma_k=\widehat\gamma_k-\delta_k\) into
Eq.~\eqref{eq:nuisance-uu-system} gives
\begin{align*}
(1-\gamma_1)\widehat\alpha_1(v)+\gamma_2\widehat\alpha_2(v)
&=(1-\widehat\gamma_1+\delta_1)\widehat\alpha_1(v)
+(\widehat\gamma_2-\delta_2)\widehat\alpha_2(v)\\
&=\widehat\pi_+(1-v)+d(v)
 =\pi_+(1-v)+\delta_\pi(1-v)+d(v),\\
\gamma_1\widehat\alpha_1(v)+(1-\gamma_2)\widehat\alpha_2(v)
&=(\widehat\gamma_1-\delta_1)\widehat\alpha_1(v)
+(1-\widehat\gamma_2+\delta_2)\widehat\alpha_2(v)\\
&=-\widehat\pi_-v-d(v)
 =-\pi_-v+\delta_\pi v-d(v),
\end{align*}
where the last equality uses \(\widehat\pi_-=\pi_--\delta_\pi\).
Expanding the expectations in Eq.~\eqref{eq:nuisance-uu-residual}
under the class-conditional distributions yields
\begin{align*}
R^{\UU}_{c,w}(f;\widehat\pi_+,\widehat\gamma)
={}&(1-\gamma_1)\Ex_{P_+}[c(X)w(f(X))\widehat\alpha_1(f(X))]
 +\gamma_1\Ex_{P_-}[c(X)w(f(X))\widehat\alpha_1(f(X))]\\
&+\gamma_2\Ex_{P_+}[c(X)w(f(X))\widehat\alpha_2(f(X))]
 +(1-\gamma_2)\Ex_{P_-}[c(X)w(f(X))\widehat\alpha_2(f(X))]\\
={}&\Ex_{P_+}\!\left[c(X)w(f(X))
\{\pi_+(1-f(X))+\delta_\pi(1-f(X))+d(f(X))\}\right]\\
&+\Ex_{P_-}\!\left[c(X)w(f(X))
\{-\pi_-f(X)+\delta_\pi f(X)-d(f(X))\}\right].
\end{align*}
Subtracting the true residual \(R_{c,w}(f)\) cancels the two corresponding terms, giving
\begin{equation}
\label{eq:nuisance-uu-difference}
\begin{aligned}
R^{\UU}_{c,w}(f;\widehat\pi_+,\widehat\gamma)-R_{c,w}(f)
={}&\Ex_{P_+}\!\left[c(X)w(f(X))
\{\delta_\pi(1-f(X))+d(f(X))\}\right]\\
&+\Ex_{P_-}\!\left[c(X)w(f(X))
\{\delta_\pi f(X)-d(f(X))\}\right].
\end{aligned}
\end{equation}
Apart from the leading minus sign, each numerator in
Eq.~\eqref{eq:nuisance-uu-coefficients} is a convex combination of two
values in \([0,1]\). For example, the first numerator combines
\((1-\widehat\gamma_2)\widehat\pi_+\) and
\(\widehat\gamma_2\widehat\pi_-\) with weights \(1-v\) and \(v\), respectively.
Therefore,
\[
|\widehat\alpha_k(v)|\le\frac1{|\widehat\Delta|}\quad(k=1,2),
\qquad
|d(v)|\le|\delta_1|\,|\widehat\alpha_1(v)|
+|\delta_2|\,|\widehat\alpha_2(v)|
\le\frac{e_\gamma}{|\widehat\Delta|}.
\]
Using these bounds, \(|c\,w(f)|\le1\), and \(0\le f\le1\) in
Eq.~\eqref{eq:nuisance-uu-difference} gives
\begin{align*}
\left|R^{\UU}_{c,w}(f;\widehat\pi_+,\widehat\gamma)-R_{c,w}(f)\right|
\le{}&e_\pi\{\Ex_{P_+}[1-f(X)]+\Ex_{P_-}[f(X)]\}\\
&+\Ex_{P_+}[|d(f(X))|]+\Ex_{P_-}[|d(f(X))|]
\le{} 2e_\pi+\frac{2e_\gamma}{|\widehat\Delta|}.
\end{align*}
The bound is independent of \((c,w)\), proving
Eq.~\eqref{eq:nuisance-uu-bias}. Finally,
\[
|\widehat\Delta-\Delta|=|\delta_1+\delta_2|\le e_\gamma,
\qquad |\widehat\Delta|\ge|\Delta|-e_\gamma.
\]
Thus, \(e_\gamma\le|\Delta|/2\) implies
\(|\widehat\Delta|\ge|\Delta|/2\).
Substituting the lower bounds on the denominator into
Eq.~\eqref{eq:nuisance-uu-bias} proves the remaining two bounds.
\end{proof}

\begin{corollary}[UU nuisance and sampling errors]
\label{cor:nuisance-uu-sampling}
Assume that \(\calc\) and \(\calw\) are nonempty and finite, and for
\(k=1,2\), let
\(\{X_i^{U_k}\}_{i=1}^{n_k}\overset{\mathrm{i.i.d.}}{\sim}P_{U_k}\).
Assume that \(f,\calc,\calw,\widehat\pi_+,\widehat\gamma_1,\widehat\gamma_2\)
are fixed independently of these samples, and define
\[
\widehat R^{\UU}_{c,w}(f;\widehat\pi_+,\widehat\gamma)
:=\sum_{k=1}^2\frac1{n_k}\sum_{i=1}^{n_k}
c(X_i^{U_k})w(f(X_i^{U_k}))\widehat\alpha_k(f(X_i^{U_k})).
\]
For \(\delta\in(0,1)\), with probability at least \(1-\delta\),
\begin{equation}
\label{eq:nuisance-uu-sampling}
\sup_{c\in\calc,\,w\in\calw}
\left|\widehat R^{\UU}_{c,w}(f;\widehat\pi_+,\widehat\gamma)-R_{c,w}(f)\right|
\le{} 2e_\pi+\frac{2e_\gamma}{|\widehat\Delta|}
 +\frac1{|\widehat\Delta|}\sum_{k=1}^2
\sqrt{\frac{2\log(4|\calc||\calw|/\delta)}{n_k}}.
\end{equation}
\end{corollary}
\begin{proof}
Set \(Z_{k,c,w}(x):=c(x)w(f(x))\widehat\alpha_k(f(x))\).
Then \(|Z_{k,c,w}|\le1/|\widehat\Delta|\). Hoeffding's inequality gives
\[
\Pr\!\left\{
\left|\frac1{n_k}\sum_{i=1}^{n_k}Z_{k,c,w}(X_i^{U_k})
-\Ex_{P_{U_k}}[Z_{k,c,w}(X)]\right|\ge t
\right\}
\le2\exp\!\left(-\frac{n_k|\widehat\Delta|^2t^2}{2}\right).
\]
Choose
\[
t=\frac1{|\widehat\Delta|}
\sqrt{\frac{2\log(4|\calc||\calw|/\delta)}{n_k}}
\]
and take a union bound over the two distributions and
\(|\calc||\calw|\) witnesses. Adding the resulting errors from the two
sources and the bias in Proposition~\ref{prop:nuisance-uu} proves
Eq.~\eqref{eq:nuisance-uu-sampling}.
\end{proof}
With the true parameters and \(n_1=n_2=n\),
\[
n\ge\frac{8\log(4|\calc||\calw|/\delta)}{\Delta^2\varepsilon^2}
\]
is sufficient to make the sampling error at most \(\varepsilon\).
Moreover, \(P_{U_1}-P_{U_2}=\Delta(P_+-P_-)\), so when \(\Delta=0\),
the two observation distributions coincide and this system cannot separate
\(P_+\) and \(P_-\).
\subsection{Pconf: estimated confidences and clipping}
\label{app:nuisance-pconf}
Let \(r(x):=P(Y=1\mid X=x)\), and assume that \(r(X)>0\) holds \(P_X\)-almost surely.
By Bayes' rule, \(\pi_+\,dP_+=r\,dP_X\). Hence,
\begin{equation}
\label{eq:nuisance-pconf-exact}
R_{c,w}(f)
 =\Ex_{P_X}[c(X)w(f(X))\{r(X)-f(X)\}]
 =\pi_+\Ex_{P_+}\!\left[c(X)w(f(X))
\left(1-\frac{f(X)}{r(X)}\right)\right].
\end{equation}
The first equality follows by conditioning on \(X\), and the second follows from the measure identity above. Moreover,
\[
 \pi_+\Ex_{P_+}\!\left[|c(X)w(f(X))|
\left|1-\frac{f(X)}{r(X)}\right|\right]
=\Ex_{P_X}[|c(X)w(f(X))|\,|r(X)-f(X)|]\le1,
\]
so the integrand in the last expectation in Eq.~\eqref{eq:nuisance-pconf-exact} is integrable.

\paragraph{Parameter estimation.}
Given an estimate \(\widehat\pi_+\in[0,1]\) of the class prior \(\pi_+\)
and an estimate \(\widehat r:\calx\to[0,1]\) of the confidence function \(r\), let
\(e_r:=\|\widehat r-r\|_{L^\infty(P_X)}\).
Assume that, for some \(C_r\in(0,1]\),
\[
r(X)\ge C_r,\qquad \widehat r(X)\ge C_r
\qquad P_X\text{-almost surely}.
\]
No clipping is applied here. Define the population residual using the estimates by
\[
R^{\Pconf}_{c,w}(f;\widehat\pi_+,\widehat r)
:=\widehat\pi_+\Ex_{P_+}\!\left[
c(X)w(f(X))\left(1-\frac{f(X)}{\widehat r(X)}\right)\right].
\]
\begin{proposition}\label{prop:nuisance-pconf}
Under the assumptions above,
\begin{equation}\label{eq:nuisance-pconf-bias}
\sup_{c\in\calc,\,w\in\calw}
\left|R^{\Pconf}_{c,w}(f;\widehat\pi_+,\widehat r)-R_{c,w}(f)\right|
\le\frac{e_\pi}{C_r}+\frac{\pi_+e_r}{C_r^2}.
\end{equation}
\end{proposition}
\begin{proof}
Add and subtract the population residual using \(\pi_+\) and \(\widehat r\).
By Eq.~\eqref{eq:nuisance-pconf-exact},
\begin{align*}
 R^{\Pconf}_{c,w}(f;\widehat\pi_+,\widehat r)-R_{c,w}(f)
&=(\widehat\pi_+-\pi_+)\Ex_{P_+}\!\left[
c(X)w(f(X))\left(1-\frac{f(X)}{\widehat r(X)}\right)\right]\\
&+\pi_+\Ex_{P_+}\!\left[
c(X)w(f(X))f(X)
\left(\frac1{r(X)}-\frac1{\widehat r(X)}\right)\right].
\end{align*}
For the first term, \(0\le f\le1\) and \(\widehat r\ge C_r\) imply
\[
0\le\frac{f(X)}{\widehat r(X)}\le C_r^{-1},
\qquad
1-C_r^{-1}\le1-\frac{f(X)}{\widehat r(X)}\le1.
\]
Comparing the absolute values of the two endpoints of this interval gives
\[
\left|1-\frac{f(X)}{\widehat r(X)}\right|
\le\max\{1,C_r^{-1}-1\}\le C_r^{-1}.
\]
The last inequality follows because \(0<C_r\le1\) implies
\(1\le C_r^{-1}\) and \(C_r^{-1}-1\le C_r^{-1}\).
For the second term,
\[
\left|\frac1{r(X)}-\frac1{\widehat r(X)}\right|
=\frac{|\widehat r(X)-r(X)|}{r(X)\widehat r(X)}
\le\frac{e_r}{C_r^2}.
\]
Therefore, using \(|c\,w(f)|\le1\) and \(0\le f\le1\), we obtain
\begin{align*}
 \left|R^{\Pconf}_{c,w}(f;\widehat\pi_+,\widehat r)-R_{c,w}(f)\right|
&\le e_\pi\Ex_{P_+}\!\left[\left|1-\frac{f(X)}{\widehat r(X)}\right|\right]
+\pi_+\Ex_{P_+}\!\left[f(X)
\left|\frac1{r(X)}-\frac1{\widehat r(X)}\right|\right]\\
&\le\frac{e_\pi}{C_r}+\frac{\pi_+e_r}{C_r^2}.
\end{align*}
Taking the supremum over \((c,w)\) proves the claim.
\end{proof}

\paragraph{Clipping bias.}
Here we use the true values of \(\pi_+\) and \(r\).
Fix \(\tau\in(0,1]\) and define
\[
R^{\Pconf,\tau}_{c,w}(f)
:=\pi_+\Ex_{P_+}\!\left[
c(X)w(f(X))\left(1-\frac{f(X)}{\tau\vee r(X)}\right)\right].
\]
In this part, we do not assume a positive uniform lower bound on \(r\).
We continue to assume that \(r>0\) holds \(P_X\)-almost surely.

\begin{proposition}[Clipping bias]\label{prop:nuisance-pconf-clipping}
The signed bias is
\begin{equation}
\label{eq:nuisance-pconf-clipping-identity}
 R^{\Pconf,\tau}_{c,w}(f)-R_{c,w}(f)
=\Ex_{P_X}\!\left[c(X)w(f(X))f(X)
\left(1-\frac{r(X)}{\tau}\right)\mathbf1\{r(X)<\tau\}\right].
\end{equation}
In particular,
\begin{equation}
\label{eq:nuisance-pconf-clipping-bias}
\begin{aligned}
\sup_{c\in\calc,\,w\in\calw}
|R^{\Pconf,\tau}_{c,w}(f)-R_{c,w}(f)|
&\le B_\tau(f)
 :=\Ex_{P_X}\!\left[f(X)\left(1-\frac{r(X)}{\tau}\right)
\mathbf1\{r(X)<\tau\}\right]\\
&\le P_X\{r(X)<\tau\}.
\end{aligned}
\end{equation}
Moreover, \(B_\tau(f)\) is nondecreasing in \(\tau\), and
\(B_\tau(f)\to0\) as \(\tau\downarrow0\).
\end{proposition}
\begin{proof}
First, subtracting Eq.~\eqref{eq:nuisance-pconf-exact} from
\(R^{\Pconf,\tau}_{c,w}(f)\) gives
\[
 R^{\Pconf,\tau}_{c,w}(f)-R_{c,w}(f)
=\pi_+\Ex_{P_+}\!\left[c(X)w(f(X))f(X)
\left(\frac1{r(X)}-\frac1{\tau\vee r(X)}\right)\right].
\]
On \(\{r\ge\tau\}\), we have \(\tau\vee r=r\), so the difference is zero;
on \(\{r<\tau\}\), we have \(\tau\vee r=\tau\).
Thus, only the integral over \(\{x:r(x)<\tau\}\) remains.
Substituting Bayes' rule \(\pi_+P_+(dx)=r(x)P_X(dx)\) yields
\begin{align*}
 R^{\Pconf,\tau}_{c,w}(f)-R_{c,w}(f)
&=\pi_+\int_{\{r<\tau\}}c(x)w(f(x))f(x)
\left(\frac1{r(x)}-\frac1\tau\right)P_+(dx)\\
&=\int_{\{r<\tau\}}c(x)w(f(x))f(x)r(x)
\left(\frac1{r(x)}-\frac1\tau\right)P_X(dx)\\
&=\Ex_{P_X}\!\left[c(X)w(f(X))f(X)
\left(1-\frac{r(X)}\tau\right)\mathbf1\{r(X)<\tau\}\right].
\end{align*}
This proves Eq.~\eqref{eq:nuisance-pconf-clipping-identity}.
Since \(1-r/\tau\ge0\) on \(\{r<\tau\}\), taking absolute values and using
\(|c\,w(f)|\le1\) gives
\[
|R^{\Pconf,\tau}_{c,w}(f)-R_{c,w}(f)|
 \le\Ex_{P_X}\!\left[f(X)\left(1-\frac{r(X)}\tau\right)
\mathbf1\{r(X)<\tau\}\right]
 =B_\tau(f).
\]
Furthermore, \(0\le f(1-r/\tau)\mathbf1\{r<\tau\}\le\mathbf1\{r<\tau\}\)
implies \(B_\tau(f)\le P_X\{r(X)<\tau\}\).

This bias bound has the following integral representation over regions with small confidence values.
For fixed \(u>0\), the integral \(\int_0^\tau\mathbf1\{u\le s\}\,ds\)
is zero if \(u\ge\tau\), and equals \(\tau-u\) if \(u<\tau\).
Therefore,
\[
\left(1-\frac u\tau\right)\mathbf1\{u<\tau\}
=\frac1\tau\int_0^\tau\mathbf1\{u\le s\}\,ds.
\]
Substituting \(u=r(X)\) and interchanging the expectation and integral,
which is valid because the integrand is nonnegative, gives
\begin{align*}
B_\tau(f)
&=\Ex_{P_X}\!\left[f(X)\frac1\tau
\int_0^\tau\mathbf1\{r(X)\le s\}\,ds\right]\\
&=\frac1\tau\int_0^\tau
\Ex_{P_X}[f(X)\mathbf1\{r(X)\le s\}]\,ds
 \le\frac1\tau\int_0^\tau P_X\{r(X)\le s\}\,ds
\le P_X\{r(X)<\tau\}.
\end{align*}
The last inequality uses \(\{r\le s\}\subset\{r<\tau\}\) for \(s<\tau\).

For every \(u>0\), the function
\((1-u/\tau)\mathbf1\{u<\tau\}\) is zero when \(\tau\le u\)
and increases with \(\tau\) when \(\tau>u\).
Multiplying by \(f\ge0\) and taking the expectation proves that
\(B_\tau(f)\) is nondecreasing.
Finally, the integrand defining \(B_\tau(f)\) converges to zero as
\(\tau\downarrow0\) and is bounded by one.
The claimed limit therefore follows from the dominated convergence theorem.
\end{proof}

\begin{corollary}[Clipping bias and sampling error]
\label{cor:nuisance-pconf-clipping-sampling}
Assume that \(\calc\) and \(\calw\) are finite.
Let \(f,\calc,\calw,\tau\) be fixed independently of the sample
\(X_1,\ldots,X_n\overset{\mathrm{i.i.d.}}\sim P_+\).
Define the empirical residual by
\[
\widehat R^{\Pconf,\tau}_{c,w}(f)
:=\frac{\pi_+}{n}\sum_{i=1}^n
c(X_i)w(f(X_i))\left(1-\frac{f(X_i)}{\tau\vee r(X_i)}\right).
\]
For \(\delta\in(0,1)\), with probability at least \(1-\delta\),
\begin{equation}\label{eq:nuisance-pconf-clipping-sampling}
\sup_{c\in\calc,\,w\in\calw}
|\widehat R^{\Pconf,\tau}_{c,w}(f)-R_{c,w}(f)|
\le B_\tau(f)+\frac{\pi_+}{\tau}
\sqrt{\frac{2\log(2|\calc||\calw|/\delta)}{n}}.
\end{equation}
\end{corollary}
\begin{proof}
Let
\[
Z_{\tau,c,w}(X)
:=\pi_+c(X)w(f(X))
\left(1-\frac{f(X)}{\tau\vee r(X)}\right).
\]
Since \(0\le f/(\tau\vee r)\le1/\tau\) and \(0<\tau\le1\),
\[
|Z_{\tau,c,w}(X)|
\le\pi_+\max\{1,\tau^{-1}-1\}\le\frac{\pi_+}{\tau}.
\]
The expectation of this random variable is \(R^{\Pconf,\tau}_{c,w}(f)\),
and its range is contained in an interval of length at most \(2\pi_+/\tau\).
Thus, Hoeffding's inequality gives
\[
\Pr\!\left\{
|\widehat R^{\Pconf,\tau}_{c,w}(f)-R^{\Pconf,\tau}_{c,w}(f)|\ge t
\right\}
\le2\exp\!\left(-\frac{n\tau^2t^2}{2\pi_+^2}\right).
\]
Set \(t=(\pi_+/\tau)\sqrt{2\log(2|\calc||\calw|/\delta)/n}\)
and sum the failure probabilities over the \(|\calc||\calw|\) pairs \((c,w)\).
On the resulting event of probability at least \(1-\delta\), the triangle inequality
and Proposition~\ref{prop:nuisance-pconf-clipping} imply
\begin{align*}
 \sup_{c,w}|\widehat R^{\Pconf,\tau}_{c,w}(f)-R_{c,w}(f)|
&\le\sup_{c,w}|R^{\Pconf,\tau}_{c,w}(f)-R_{c,w}(f)|
+\sup_{c,w}|\widehat R^{\Pconf,\tau}_{c,w}(f)-R^{\Pconf,\tau}_{c,w}(f)|\\
&\le B_\tau(f)+\frac{\pi_+}{\tau}
\sqrt{\frac{2\log(2|\calc||\calw|/\delta)}{n}}.
\end{align*}
\end{proof}
In Eq.~\eqref{eq:nuisance-pconf-clipping-sampling}, the bias bound is nondecreasing
in \(\tau\), whereas the Hoeffding term decreases in proportion to \(1/\tau\).
This is the tradeoff between clipping bias and sampling error.

\paragraph{Clipping estimated confidences.}
Given an estimate \(\widehat\pi_+\in[0,1]\) of the class prior \(\pi_+\)
and an estimate \(\widehat r:\calx\to[0,1]\) of the confidence function \(r\), let
\(e_r=\|\widehat r-r\|_{L^\infty(P_X)}\) and clip it as \(\tau\vee\widehat r\).
Here we do not assume positive uniform lower bounds on \(r\) or \(\widehat r\).
Define the population and empirical residuals, respectively, by
\begin{align*}
R^{\Pconf,\tau}_{c,w}(f;\widehat\pi_+,\widehat r)
&:=\widehat\pi_+\Ex_{P_+}\!\left[
c(X)w(f(X))\left(1-\frac{f(X)}{\tau\vee\widehat r(X)}\right)\right],\\
\widehat R^{\Pconf,\tau}_{c,w}(f;\widehat\pi_+,\widehat r)
&:=\frac{\widehat\pi_+}{n}\sum_{i=1}^n
c(X_i)w(f(X_i))\left(1-\frac{f(X_i)}{\tau\vee\widehat r(X_i)}\right).
\end{align*}
\begin{corollary}\label{cor:nuisance-pconf-combined}
The population bias is bounded by
\begin{equation}\label{eq:nuisance-pconf-combined-bias}
\sup_{c\in\calc,\,w\in\calw}
|R^{\Pconf,\tau}_{c,w}(f;\widehat\pi_+,\widehat r)-R_{c,w}(f)|
\le B_\tau(f)+\frac{e_\pi}{\tau}+\frac{\pi_+e_r}{\tau^2}.
\end{equation}
Furthermore, suppose that \(\calc\) and \(\calw\) are finite and
\(f,\calc,\calw,\tau,\widehat\pi_+,\widehat r\) are fixed independently of
\(X_1,\ldots,X_n\overset{\mathrm{i.i.d.}}\sim P_+\).
For \(\delta\in(0,1)\), with probability at least \(1-\delta\),
\begin{equation}
\label{eq:nuisance-pconf-combined}
 \sup_{c\in\calc,\,w\in\calw}
|\widehat R^{\Pconf,\tau}_{c,w}(f;\widehat\pi_+,\widehat r)-R_{c,w}(f)|
\le B_\tau(f)+\frac{e_\pi}{\tau}+\frac{\pi_+e_r}{\tau^2}
+\frac{\widehat\pi_+}{\tau}
\sqrt{\frac{2\log(2|\calc||\calw|/\delta)}{n}}.
\end{equation}
\end{corollary}
\begin{proof}
Since \(|\max\{u,\tau\}-\max\{v,\tau\}|\le|u-v|\),
\[
\left|\frac1{\tau\vee\widehat r}-\frac1{\tau\vee r}\right|
 =\frac{|(\tau\vee\widehat r)-(\tau\vee r)|}
{(\tau\vee\widehat r)(\tau\vee r)}
 \le\frac{|\widehat r-r|}{\tau^2}\le\frac{e_r}{\tau^2}
\qquad P_X\text{-almost surely}.
\]
Comparing with the residual \(R^{\Pconf,\tau}_{c,w}(f)\) obtained by clipping
the true confidence gives
\begin{align*}
 R^{\Pconf,\tau}_{c,w}(f;\widehat\pi_+,\widehat r)
-R^{\Pconf,\tau}_{c,w}(f)
&=(\widehat\pi_+-\pi_+)\Ex_{P_+}\!\left[
c(X)w(f(X))\left(1-\frac{f(X)}{\tau\vee\widehat r(X)}\right)\right]\\
&+\pi_+\Ex_{P_+}\!\left[c(X)w(f(X))f(X)
\left(\frac1{\tau\vee r(X)}-\frac1{\tau\vee\widehat r(X)}\right)\right].
\end{align*}
The absolute value of the first term is at most \(e_\pi/\tau\), and that of
the second is at most \(\pi_+e_r/\tau^2\).
Adding the bias in Proposition~\ref{prop:nuisance-pconf-clipping} gives
Eq.~\eqref{eq:nuisance-pconf-combined-bias}.
Each summand in the empirical average has absolute value at most
\(\widehat\pi_+/\tau\), so adding the same Hoeffding bound as in
Corollary~\ref{cor:nuisance-pconf-clipping-sampling} yields
Eq.~\eqref{eq:nuisance-pconf-combined}.
\end{proof}

\section{Proofs for Section~\ref{sec:algo}}
\label{app:proof-boosting}

\subsection{\texorpdfstring{Fresh-sample version of WLMC}{Fresh-sample version of WLMC}}
\label{app:fresh-sample-wlmc}
Algorithm~\ref{alg:wsl-mc-boosting-fresh} draws new independent weak samples at each search call, whereas Algorithm~\ref{alg:wsl-mc-boosting} in the main paper draws the samples once and reuses them throughout.

\begin{algorithm}[H]
\small
\caption{Fresh-sample weak-label multicalibration boost (WLMC)}
\label{alg:wsl-mc-boosting-fresh}
\begin{algorithmic}
\State \textbf{Input:} \(\varepsilon>0\), source-wise sample sizes \(\mathbf n\), step size \(\lambda\), initial predictor \(f_0\), classes \(\calc,\calw\), estimator \(\widehat R^{\WSL}\)
\State Set \(t=0\) and draw a fresh weak batch \(S_0\) with source-wise sizes \(\mathbf n\)
\While{\(\exists(\tau,c,w)\in\{\pm1\}\times\calc\times\calw\) such that \(\tau\widehat R^{\WSL}_{c,w}(f_t;S_t)>3\varepsilon/4\)}
    \State Choose \((\tau_t,c_t,w_t)\in\{\pm1\}\times\calc\times\calw\) satisfying the while condition
    \State \(\forall x\in\calx,\quad f_{t+1}(x)=\Pi_{[0,1]}\!\left[f_t(x)+\lambda\tau_t c_t(x)w_t(f_t(x))\right]\)
    \State \(t\gets t+1\)
    \State Draw a fresh weak batch \(S_t\) with source-wise sizes \(\mathbf n\), independently of the past
\EndWhile
\State \textbf{Output:} \(f_t\)
\end{algorithmic}
\end{algorithm}

\begin{theorem}[Formal PU guarantee under Rademacher complexity]
\label{thm:pu-wlmc-boost-formal-rad}
Let \(0<\varepsilon\le1\) and \(0<\delta<1\). Assume the PU setting of Proposition~\ref{prop:pu-gen} with known class prior \(\pi_+\), and fix \(f_0,\calc,\calw\) independently of the calibration samples. Suppose \(|c|\le1\), \(\|w\|_\infty\le1\), and \(\calw\) is finite. Assume further that the Rademacher complexities are bounded as
\[
\mathfrak R_{n_+}(\calc;P_+)
\le
C_R\sqrt{\frac{d_{\calc}\log n_+}{n_+}},
\qquad
\mathfrak R_{n_u}(\calc;P_X)
\le
C_R\sqrt{\frac{d_{\calc}\log n_u}{n_u}} .
\]
Let
\(D_0=\Ex[(f_0(X)-f^\star(X))^2]\) and
\(T_\varepsilon=\lceil 4D_0/\varepsilon^2\rceil\).  Run Algorithm~\ref{alg:wsl-mc-boosting-fresh} with
\(\lambda=\varepsilon/2\), threshold \(3\varepsilon/4\), and the PU estimator
\(\widehat R^{\PU}\).  At each search call, draw fresh independent batches
\(S_t^+\sim P_+^{n_+}\) and \(S_t^u\sim P_X^{n_u}\).  Let
\[
L=\log\frac{4|\calw|(T_\varepsilon+1)}{\delta}.
\]
There exists \(C>0\), depending only on \(C_R\), such that if
\[
n_+\ge \frac{C\pi_+^2}{\varepsilon^2}(d_{\calc}\log n_+ + L),
\qquad
n_u\ge \frac{C}{\varepsilon^2}(d_{\calc}\log n_u + L),
\]
then, with probability at least \(1-\delta\), Algorithm~\ref{alg:wsl-mc-boosting-fresh} stops after
\(\widehat T\le T_\varepsilon\) successful updates and returns a
\((\calc,\calw,\varepsilon)\)-multicalibrated predictor \(f_{\widehat T}\).
In particular, if \(n_+=n_u=n\), it suffices that
\[
n\ge
\frac{C}{\varepsilon^2}
\left[
d_{\calc}\log n+\log|\calw|+
\log\frac{4(T_\varepsilon+1)}{\delta}
\right].
\]
On the same event, the total source-wise budgets are
\(N_+\le (T_\varepsilon+1)n_+\) and
\(N_u\le (T_\varepsilon+1)n_u\).
\end{theorem}

\begin{proof}
The proof is similar to \citet{gopalan2022low}. However, we need to control the finite-sample estimation error.

For a fixed predictor \(f\), define the population residual moment
\[
R_{c,w}^{\WSL}(f)
=
\Ex[c(X)w(f(X))(Y-f(X))]
=
\Ex[c(X)w(f(X))(f^\star(X)-f(X))].
\]
In the PU setting, the same moment is represented as
\[
R_{c,w}^{\PU}(f)
=
\pi_+\Ex_{P_+}[c(X)w(f(X))]
-
\Ex_{P_X}[c(X)w(f(X))f(X)].
\]
Consider a population search rule which either returns
\((\tau,c,w)\in\{\pm1\}\times\calc\times\calw\) with
\(\tau R_{c,w}^{\WSL}(f)\ge \varepsilon/2\), or returns \(\perp\), in which
case \(\sup_{c,w}|R_{c,w}^{\WSL}(f)|\le\varepsilon\).

Let
\[
\Psi(f)=\Ex[(f(X)-f^\star(X))^2].
\]
If the population rule returns \((\tau_t,c_t,w_t)\), the update is
\[
f_{t+1}(x)
=
\Pi_{[0,1]}
\left[
f_t(x)+\frac{\varepsilon}{2}\tau_t c_t(x)w_t(f_t(x))
\right].
\]

Set
\[
\Delta_t(x):=\tau_t c_t(x)w_t(f_t(x)).
\]
Since \(f^\star(x)\in[0,1]\) and projection onto \([0,1]\) is non-expansive,
\[
(f_{t+1}(x)-f^\star(x))^2
\le
\left(f_t(x)+\frac{\varepsilon}{2}\Delta_t(x)-f^\star(x)\right)^2 .
\]
Taking expectations and expanding the square, we obtain
\[
\begin{aligned}
\Psi(f_{t+1})
&\le
\Ex\left[
\left(f_t(X)+\frac{\varepsilon}{2}\Delta_t(X)-f^\star(X)\right)^2
\right] \\
&=
\Psi(f_t)
+
\varepsilon\Ex[\Delta_t(X)(f_t(X)-f^\star(X))]
+
\frac{\varepsilon^2}{4}\Ex[\Delta_t(X)^2] \\
&=
\Psi(f_t)
-
\varepsilon\Ex[\Delta_t(X)(f^\star(X)-f_t(X))]
+
\frac{\varepsilon^2}{4}\Ex[\Delta_t(X)^2] \\
&=
\Psi(f_t)
-
\varepsilon\tau_t R_{c_t,w_t}^{\WSL}(f_t)
+
\frac{\varepsilon^2}{4}\Ex[\Delta_t(X)^2].
\end{aligned}
\]
Since \(|c_t|\le1\), \(\|w_t\|_\infty\le1\), and \(\tau_t\in\{\pm1\}\), we have
\(|\Delta_t|\le1\).  Moreover, the population witness satisfies
\[
\tau_t R_{c_t,w_t}^{\WSL}(f_t)\ge \varepsilon/2.
\]
Therefore,
\[
\Psi(f_{t+1})
\le
\Psi(f_t)
-
\varepsilon\cdot\frac{\varepsilon}{2}
+
\frac{\varepsilon^2}{4}
=
\Psi(f_t)-\frac{\varepsilon^2}{4}.
\]

Thus every successful population update decreases the potential by at least
\(\varepsilon^2/4\).  After \(M\) successful updates,
\[
\Psi(f_M)
\le
D_0-\frac{M\varepsilon^2}{4}.
\]
Since \(\Psi(f)\ge0\) for every predictor \(f\), we must have
\[
M\le \frac{4D_0}{\varepsilon^2}.
\]
Otherwise the potential would become negative, which is impossible.  Hence the
population procedure cannot keep finding valid update directions forever; after
at most \(4D_0/\varepsilon^2\) successful updates, it must fail to find another
valid witness and therefore stops.  In particular, the number of successful
population updates is at most \(T_\varepsilon=\lceil 4D_0/\varepsilon^2\rceil\).

Next, we control the finite-sample estimation error. At each reached search call \(t\), condition on the history before drawing the new batches. Then \(f_t\) is fixed, and \(S_t^+\) and \(S_t^u\) are independent of that history. For \(t=0,\ldots,T_\varepsilon\), let \(\mathcal E_t\) be the event that either call \(t\) is not reached or its residual estimation error satisfies
\[
\sup_{c\in\calc,w\in\calw}
\left|
\widehat R_{c,w}^{\PU}(f_t;S_t^+,S_t^u)
-
R_{c,w}^{\PU}(f_t)
\right|
\le
\varepsilon/4.
\]
On \(\mathcal E_t\), if the empirical search returns a witness with
\[
\tau_t\widehat R_{c_t,w_t}^{\PU}(f_t)>3\varepsilon/4,
\]
then
\[
\tau_t R_{c_t,w_t}^{\PU}(f_t)>\varepsilon/2.
\]
Hence the update is a valid population update and decreases \(\Psi\) by at
least \(\varepsilon^2/4\).  Conversely, if the empirical search stops, then
\[
\sup_{c,w}|\widehat R_{c,w}^{\PU}(f_t)|\le 3\varepsilon/4,
\]
and therefore \(\sup_{c,w}|R_{c,w}^{\PU}(f_t)|\le\varepsilon\).  Thus the
output is \((\calc,\calw,\varepsilon)\)-multicalibrated.

Apply the finite-\(\calw\) residual-process bound established in the proof of Proposition~\ref{prop:pu-gen} with
failure probability \(\delta/(T_\varepsilon+1)\).  Using the assumed
Rademacher bound, there exists \(C>0\), depending only on \(C_R\), such that
the conditional failure probability at every reached call is at most \(\delta/(T_\varepsilon+1)\). Since \(\mathcal E_t\) holds automatically after stopping, averaging over the past gives
\(\Pr(\mathcal E_t^c)\le \delta/(T_\varepsilon+1)\) whenever
\[
n_+\ge \frac{C\pi_+^2}{\varepsilon^2}(d_{\calc}\log n_+ + L),
\qquad
n_u\ge \frac{C}{\varepsilon^2}(d_{\calc}\log n_u + L),
\]
where \(L=\log(4|\calw|(T_\varepsilon+1)/\delta)\).

By a union bound over the first \(T_\varepsilon+1\) search calls,
\[
\Pr\left(\bigcap_{t=0}^{T_\varepsilon}\mathcal E_t\right)\ge 1-\delta.
\]
On this event, every empirical update at a reached call \(t\le T_\varepsilon\) is a valid population update.  Since more than \(T_\varepsilon\) successful updates are impossible,
Algorithm~\ref{alg:wsl-mc-boosting-fresh} stops within the first
\(T_\varepsilon+1\) search calls.  The sample budgets follow by multiplying
the per-call source-wise sample sizes by the number of search calls:
\[
N_+\le (T_\varepsilon+1)n_+,
\qquad
N_u\le (T_\varepsilon+1)n_u.
\]
This proves the theorem.
\end{proof}

\subsection{\texorpdfstring{Sample-reuse version of WLMC}{Sample-reuse version of WLMC}}
\label{app:proof-pu-wlmc-boost}
We prove the guarantee for Algorithm~\ref{alg:wsl-mc-boosting}, which draws the calibration samples once at the beginning and reuses them at every iteration. We assume that both \(\calc\) and \(\calw\) are finite, a stronger restriction than in the guarantee for Algorithm~\ref{alg:wsl-mc-boosting-fresh}, which allows an infinite \(\calc\) with controlled Rademacher complexity. The proof controls the residual estimation error uniformly over all predictors reachable within a deterministic number of updates.

\begin{theorem}[Formal PU guarantee for finite classes]
\label{thm:pu-wlmc-boost-formal-finite}
Let \(0<\varepsilon\le1\) and \(0<\delta<1\). Assume the PU setting of
Proposition~\ref{prop:pu-gen} with known class prior \(\pi_+\). Suppose
\(\calc\) and \(\calw\) are finite, and \(|c|\le1\), \(\|w\|_\infty\le1\).
Assume that \(f_0,\calc,\calw\) are fixed independently of the calibration samples.
Let \(D_0=\Ex[(f_0(X)-f^\star(X))^2]\) and
\(T_\varepsilon=\lceil 4D_0/\varepsilon^2\rceil\).  Run
Algorithm~\ref{alg:wsl-mc-boosting} with \(\lambda=\varepsilon/2\), threshold
\(3\varepsilon/4\), and the PU estimator \(\widehat R^{\PU}\), using independent
calibration samples \(S^+\sim P_+^{n_+}\) and \(S^u\sim P_X^{n_u}\)
drawn once and then reused. Let \(K=2|\calc||\calw|\). If
\[
n_+\ge \frac{128\pi_+^2}{\varepsilon^2}
\log\frac{2K^{T_\varepsilon+2}}{\delta},
\qquad
n_u\ge \frac{128}{\varepsilon^2}
\log\frac{2K^{T_\varepsilon+2}}{\delta},
\]
then, with probability at least \(1-\delta\), Algorithm~\ref{alg:wsl-mc-boosting} stops after
\(\widehat T\le T_\varepsilon\) successful updates and returns a
\((\calc,\calw,\varepsilon)\)-multicalibrated predictor \(f_{\widehat T}\).

In particular, if \(n_+=n_u=n\), it suffices that
\[
n\ge
\frac{128}{\varepsilon^2}\log\frac{2K^{T_\varepsilon+2}}{\delta}.
\]
The total source-wise budgets are
\(N_+=n_+\) and
\(N_u=n_u\); in the equal-size case, a sufficient budget is
\[
N_+=N_u
\;=\;
\calo\!\left(
\frac{(T_\varepsilon+2)\log K+\log(1/\delta)}
{\varepsilon^2}
\right).
\]
\end{theorem}

\begin{proof}
For \(S=(S^+,S^u)\), write \(\widehat R^{\PU}_{c,w}(f;S)\) for the estimator
in Eq.~\eqref{eq:Rhat-PU}. Define the update map and the classes of reachable predictors by
\[
U_{\tau,c,w}(f)(x)
:=\Pi_{[0,1]}\!\left[
f(x)+\frac{\varepsilon}{2}\tau c(x)w(f(x))
\right],
\]
\[
\begin{aligned}
\mathcal F^{(0)}&=\{f_0\},\\
\mathcal F^{(t+1)}
&=
\left\{
U_{\tau,c,w}(f):f\in\mathcal F^{(t)},\ 
(\tau,c,w)\in\{\pm1\}\times\calc\times\calw
\right\}.
\end{aligned}
\]
and let \(\mathcal F_{\le T_\varepsilon}=\bigcup_{t=0}^{T_\varepsilon}\mathcal F^{(t)}\).
These classes contain all signed update sequences, regardless of whether the
empirical search would choose them. Since \(f_0,\calc,\calw\), and the update
rule are fixed independently of \(S\), so is \(\mathcal F_{\le T_\varepsilon}\).
Every actual iterate \(f_t\), whenever it exists and \(t\le T_\varepsilon\),
belongs to this class. Moreover,
\[
N_{T_\varepsilon}:=|\mathcal F_{\le T_\varepsilon}|
\le\sum_{t=0}^{T_\varepsilon}K^t
\le K^{T_\varepsilon+1}.
\]

For a fixed triple \((f,c,w)\), both \(c(x)w(f(x))\) and
\(c(x)w(f(x))f(x)\) lie in \([-1,1]\). Define
\[
\mathcal E_{T_\varepsilon}:=
\left\{
\sup_{\substack{f\in\mathcal F_{\le T_\varepsilon}\\c\in\calc,\ w\in\calw}}
\left|
\widehat R^{\PU}_{c,w}(f;S)-R^{\PU}_{c,w}(f)
\right|
\le\frac{\varepsilon}{4}
\right\}.
\]
Bounding each source's contribution to the residual estimation error by
\(\varepsilon/8\), Hoeffding's inequality and a union bound give
\[
\Pr(\mathcal E_{T_\varepsilon}^{c})
\le2N_{T_\varepsilon}|\calc||\calw|
\left[
\exp\!\left(-\frac{n_+\varepsilon^2}{128\pi_+^2}\right)
+\exp\!\left(-\frac{n_u\varepsilon^2}{128}\right)
\right]
\le\delta.
\]
In particular, when \(n_+=n_u=n\),
\[
\Pr(\mathcal E_{T_\varepsilon}^{c})
\le4N_{T_\varepsilon}|\calc||\calw|
\exp\!\left(-\frac{n\varepsilon^2}{128}\right).
\]
This is one event for the fixed samples, and it controls the sample-dependent
iterates through their membership in \(\mathcal F_{\le T_\varepsilon}\).

Work on \(\mathcal E_{T_\varepsilon}\). Whenever an update is selected at \(t\le T_\varepsilon\),
the empirical threshold implies
\[
\tau_t R^{\PU}_{c_t,w_t}(f_t)
\ge
\tau_t\widehat R^{\PU}_{c_t,w_t}(f_t;S)-\frac{\varepsilon}{4}
>\frac{\varepsilon}{2}.
\]
Set \(\Psi(f)=\Ex[(f(X)-f^\star(X))^2]\) and
\(d_t(x)=\tau_t c_t(x)w_t(f_t(x))\). Since \(|d_t|\le1\),
\(f^\star(x)\in[0,1]\), and projection onto \([0,1]\) cannot increase the
distance to \(f^\star(x)\), expanding the square gives
\[
\begin{aligned}
\Psi(f_{t+1})
&\le
\Ex\!\left[
\left(f_t(X)+\frac{\varepsilon}{2}d_t(X)-f^\star(X)\right)^2
\right]\\
&=
\Psi(f_t)
-\varepsilon\tau_t R^{\PU}_{c_t,w_t}(f_t)
+\frac{\varepsilon^2}{4}\Ex[d_t(X)^2]\\
&\le
\Psi(f_t)-\frac{\varepsilon^2}{4}.
\end{aligned}
\]
Here we used
\(R^{\PU}_{c,w}(f)=\Ex[c(X)w(f(X))(f^\star(X)-f(X))]\).

If the procedure performed \(T_\varepsilon+1\) successful updates, the
preceding inequality would apply at \(t=0,\ldots,T_\varepsilon\), yielding
\[
0\le\Psi(f_{T_\varepsilon+1})
\le D_0-\frac{(T_\varepsilon+1)\varepsilon^2}{4}<0,
\]
a contradiction. Thus the empirical stopping condition is met after
\(\widehat T\le T_\varepsilon\) updates. At termination,
\(\sup_{c,w}|\widehat R^{\PU}_{c,w}(f_{\widehat T};S)|\le3\varepsilon/4\), and
\(f_{\widehat T}\in\mathcal F_{\le T_\varepsilon}\). The same event \(\mathcal E_{T_\varepsilon}\) therefore gives
\[
\MC_{\calc,\calw}(f_{\widehat T})
=\sup_{c,w}|R^{\PU}_{c,w}(f_{\widehat T})|
\le\frac{3\varepsilon}{4}+\frac{\varepsilon}{4}
=\varepsilon.
\]
In particular, the stopping and multicalibration conclusions are simultaneous:
\[
\Pr\!\left(
\bigcup_{t=0}^{T_\varepsilon}
\{\widehat T=t,\ \MC_{\calc,\calw}(f_t)\le\varepsilon\}
\right)\ge1-\delta.
\]
Expanding \(\log(2K^{T_\varepsilon+2}/\delta)=(T_\varepsilon+2)\log K+\log(2/\delta)\)
gives the stated sufficient total budget. The samples are drawn only once, so
\(N_+=n_+\) and \(N_u=n_u\).
\end{proof}

Neither \(D_0\) nor \(T_\varepsilon\) is an input to Algorithm~\ref{alg:wsl-mc-boosting}. Since \(D_0\le1\), one may replace \(T_\varepsilon\) by \(\lceil4/\varepsilon^2\rceil\) when choosing the sample sizes. If \(f_0,\calc,\calw\) are obtained from independent training data, the argument applies conditionally on them.

\paragraph{\texorpdfstring{VC classes and Cutoff calibration.}{VC classes and Cutoff calibration.}}
The finite-class assumption can be relaxed when the witness classes have finitely many evaluation patterns on every finite sample. For a class \(\mathcal A\), let \(\Pi_{\mathcal A}(m)\) be the maximum number of distinct evaluation vectors on \(m\) points in its domain. A finite class satisfies \(\Pi_{\mathcal A}(m)\le|\mathcal A|\); an indicator class of VC dimension \(d_{\mathcal A}\) satisfies \(\Pi_{\mathcal A}(m)\le(m+1)^{d_{\mathcal A}}\) by Sauer's lemma~\citep{Mohri}.

Fix \(m\) points and write \(q_m=2\Pi_{\calc}(m)\Pi_{\calw}(m)\). For each evaluation vector of \(f\), an update has at most \(q_m\) possible evaluation vectors: the sign, the evaluations of \(c\), and the evaluations of \(w\) at the current scores determine the next vector. Consequently, each of the two source classes
\[
\begin{aligned}
\mathcal G_+&=\{c\,w(f):f\in\mathcal F_{\le T_\varepsilon},\ c\in\calc,\ w\in\calw\},\\
\mathcal G_u&=\{c\,w(f)f:f\in\mathcal F_{\le T_\varepsilon},\ c\in\calc,\ w\in\calw\}
\end{aligned}
\]
has at most
\[
M_{T_\varepsilon}(m)
\le \Pi_{\calc}(m)\Pi_{\calw}(m)\sum_{t=0}^{T_\varepsilon}q_m^t
\le q_m^{T_\varepsilon+2}
\]
distinct evaluation vectors on these points. Multiplication by \(f\) in \(\mathcal G_u\) introduces no additional choices once its evaluation vector is fixed.

Apply symmetrization and the finite-vector Rademacher bound to these classes, followed by bounded-differences concentration~\citep{bartlett2002rademacher,Mohri}. For \(n_+=n_u=n\), with probability at least \(1-\delta\),
\[
\sup_{f,c,w}
|\widehat R^{\PU}_{c,w}(f;S)-R^{\PU}_{c,w}(f)|
\le C_0(1+\pi_+)\sqrt{\frac{(T_\varepsilon+2)\log\!\bigl(2\Pi_{\calc}(n)\Pi_{\calw}(n)\bigr)+\log(4/\delta)}{n}},
\]
where the supremum ranges over \(f\in\mathcal F_{\le T_\varepsilon}\), \(c\in\calc\), and \(w\in\calw\), and \(C_0>0\) is universal. Choosing this bound to be at most \(\varepsilon/4\) gives the same potential-decrease and stopping argument. In particular, for indicator classes with finite VC dimensions \(d_{\calc},d_{\calw}\), it suffices to take
\[
n\ge \frac{C}{\varepsilon^2}\left\{(T_\varepsilon+2)\bigl[1+(d_{\calc}+d_{\calw})\log(n+1)\bigr]+\log(4/\delta)\right\}.
\]
where \(C>0\) is universal. Since \(T_\varepsilon\le\lceil4/\varepsilon^2\rceil\), this yields a sufficient total sample budget per source of \(\widetilde O((1+d_{\calc}+d_{\calw})/\varepsilon^4+\log(1/\delta)/\varepsilon^2)\).

For example, \(\calw=\{v\mapsto\mathbf1\{v\in I\}:I\subseteq[0,1]\text{ is an interval}\}\) has VC dimension two. With \(\calc=\{1\}\), the corresponding criterion is Cutoff Calibration Error~\citep{rossellini2025cutoff}; a subgroup class gives its multicalibration counterpart. This also covers any finite bounded real-valued \(\calc\), using \(\Pi_{\calc}(n)\le|\calc|\), together with all interval indicators. General infinite real-valued classes require separate control of the reachable residual classes.

\subsection{Other weak-supervision models}
\label{app:other-wsl-boost-orders}
The same argument as in the finite-class PU case yields the promised
\(O(\varepsilon^{-4})\) total weak-sample order for the other weak-supervision
examples using fixed calibration samples. The only model-specific ingredient
is the uniform estimation error of the corrected residuals.

Assume the same bounds and sample independence of \(f_0,\calc,\calw\) as in
Theorem~\ref{thm:pu-wlmc-boost-formal-finite}, with finite
\(\calc,\calw\). Suppose a weak-supervision model has a nonempty finite source
index set \(J\), source laws \(Q_j\), and positive envelope constants \(B_j\).
Draw independent calibration samples \(S_j=\{X_i^j\}_{i=1}^{m_j}\sim Q_j^{m_j}\)
once and reuse them. Write the corrected population and empirical moments as
\[
R^{\WSL}_{c,w}(f)
=\sum_{j\in J}\Ex_{Q_j}[g_{j,f,c,w}(X)],
\qquad
\widehat R^{\WSL}_{c,w}(f;S)
=\sum_{j\in J}\frac1{m_j}\sum_{i=1}^{m_j}g_{j,f,c,w}(X_i^j),
\]
where \(|g_{j,f,c,w}(x)|\le B_j\) uniformly over \(f,c,w\).
The correction uses the exact parameters of the weak-observation model, so that
\[
R^{\WSL}_{c,w}(f)
=
\Ex[c(X)w(f(X))(Y-f(X))]
=
\Ex[c(X)w(f(X))(f^\star(X)-f(X))].
\]
Use the same reachable-class construction with
\[
T_\varepsilon=\lceil4/\varepsilon^2\rceil,
\qquad
K=2|\calc||\calw|,
\qquad
N_{T_\varepsilon}:=|\mathcal F_{\le T_\varepsilon}|\le K^{T_\varepsilon+1}.
\]
If
\[
\sum_{j\in J}B_j\sqrt{\frac{2\log(|J|K^{T_\varepsilon+2}/\delta)}{m_j}}
\le\frac{\varepsilon}{4},
\]
then, with probability at least \(1-\delta\), Algorithm~\ref{alg:wsl-mc-boosting}
with \(\lambda=\varepsilon/2\) stops after at most \(T_\varepsilon\) successful
updates and returns a \((\calc,\calw,\varepsilon)\)-multicalibrated predictor.
A sufficient source-wise condition is
\[
m_j\ge
\frac{32|J|^2B_j^2}{\varepsilon^2}
\log\frac{|J|K^{T_\varepsilon+2}}{\delta},
\qquad j\in J.
\]
Here \(m_j\) is the total number of observations from source \(Q_j\).

For UU, \(J=\{1,2\}\), \(Q_j=P_{U_j}\), and
\(g_{j,f,c,w}(x)=\alpha_j(f(x))c(x)w(f(x))\). Thus \(B_j=A_j\), with
\(A_j=\sup_{v\in[0,1]}|\alpha_j(v)|\) as in Proposition~\ref{prop:uu-gen}.
The class prior and mixture proportions are known, and \(\Delta\ne0\).
For Pconf, use the exact-confidence observation model of
Corollary~\ref{prop:pconf}, including \(P_-\ll P_+\), known \(\pi_+\), and
\(r(X)=P(Y=1\mid X)\). Then \(J=\{+\}\), \(Q_+=P_+\), and
\[
g_{+,f,c,w}(x)=\pi_+c(x)w(f(x))(1-f(x)/r(x)).
\]
Under \(r(x)\ge C_r>0\) almost surely under \(P_+\), the existing envelope
\(B_+=\pi_+(1+1/C_r)\) is valid.

\begin{proof}
For every fixed \((f,c,w)\) and source \(j\), Hoeffding's inequality gives
\[
\Pr\!\left(
\left|
\frac1{m_j}\sum_{i=1}^{m_j}g_{j,f,c,w}(X_i^j)
-\Ex_{Q_j}[g_{j,f,c,w}(X)]
\right|
>B_j\sqrt{\frac{2\log(|J|K^{T_\varepsilon+2}/\delta)}{m_j}}
\right)\le\frac{2\delta}{|J|K^{T_\varepsilon+2}}.
\]
A union bound over
\(J\times\mathcal F_{\le T_\varepsilon}\times\calc\times\calw\), which has
cardinality \(|J|N_{T_\varepsilon}|\calc||\calw|\le |J|K^{T_\varepsilon+2}/2\), and summation over
sources show that
\[
\Pr(\mathcal E_{\WSL})\ge1-\delta,
\qquad
\mathcal E_{\WSL}:=
\left\{
\sup_{\substack{f\in\mathcal F_{\le T_\varepsilon}\\c\in\calc,\ w\in\calw}}
\left|
\widehat R^{\WSL}_{c,w}(f;S)-R^{\WSL}_{c,w}(f)
\right|\le\frac{\varepsilon}{4}
\right\}.
\]
The reachable class contains every realized iterate up to the deterministic
horizon, so this event also controls their deviations.

On \(\mathcal E_{\WSL}\), any selected update at \(t\le T_\varepsilon\)
satisfies
\[
\tau_t R^{\WSL}_{c_t,w_t}(f_t)
\ge
\tau_t\widehat R^{\WSL}_{c_t,w_t}(f_t;S)-\frac{\varepsilon}{4}
>\frac{\varepsilon}{2}.
\]
The exact weak-residual rewrite therefore gives, by the same squared-potential
calculation as in the PU proof,
\[
\begin{aligned}
\Psi(f_{t+1})
&\le
\Psi(f_t)-\varepsilon\tau_t R^{\WSL}_{c_t,w_t}(f_t)
+\frac{\varepsilon^2}{4}
\Ex[(c_t(X)w_t(f_t(X)))^2]\\
&\le\Psi(f_t)-\frac{\varepsilon^2}{4}.
\end{aligned}
\]
If \(T_\varepsilon+1\) updates were performed, then
\[
0\le\Psi(f_{T_\varepsilon+1})
\le\Psi(f_0)-\frac{(T_\varepsilon+1)\varepsilon^2}{4}<0,
\]
because \(\Psi(f_0)\le1\), a contradiction. Thus the procedure stops with
\(\widehat T\le T_\varepsilon\). At termination, the empirical supremum is at most
\(3\varepsilon/4\), and the same event gives
\[
\MC_{\calc,\calw}(f_{\widehat T})
=\sup_{c,w}|R^{\WSL}_{c,w}(f_{\widehat T})|
\le
\sup_{c,w}|\widehat R^{\WSL}_{c,w}(f_{\widehat T};S)|
+\frac{\varepsilon}{4}
\le\varepsilon.
\]
Finally, \(\log(|J|K^{T_\varepsilon+2}/\delta)=(T_\varepsilon+2)\log K+\log(|J|/\delta)\).
For the fixed source count in UU and Pconf, a sufficient total budget is
\[
N_j=m_j
=\calo\!\left(
B_j^2\left[
\frac{\log(2|\calc||\calw|)}{\varepsilon^4}
+\frac{\log(1/\delta)}{\varepsilon^2}
\right]
\right).
\]
\end{proof}

Thus the population boosting decrease is identical after the weak residual
rewrite, and the model-specific envelopes determine how many fixed
calibration samples suffice to control every reachable predictor.

\subsection{Corrected weak objectives for Platt and temperature scaling}
\label{app:weak-postprocessing-objectives}
Let \(g_\theta\) denote either \(g_\beta\) or \(g_{a,b}\), and write
\[
\ell_\theta^+(x):=-\log g_\theta(x),
\qquad
\ell_\theta^-(x):=-\log(1-g_\theta(x)).
\]
The fully supervised calibration objective is \(\Ex[\ell_\theta^+(X)\mathbf 1\{Y=1\}+\ell_\theta^-(X)\mathbf 1\{Y=0\}]\).  The weak baselines in Section~\ref{sec:weak-postprocessing-baselines} minimize corrected or stabilized empirical versions of this objective.  For PU, the unbiased corrected objective is
\[
\widehat L_{\PU}(\theta)=
\frac{\pi_+}{n_+}\sum_{i=1}^{n_+}\{\ell_\theta^+(X_i^+)-\ell_\theta^-(X_i^+)\}
+\frac1{n_u}\sum_{j=1}^{n_u}\ell_\theta^-(X_j^u).
\]
The supplied implementation uses the non-negative PU stabilization of \citet{kiryo2017nnpu} for PU Platt and temperature scaling on both tabular and large-model data, optimizing
\[
\widehat L_{\mathrm{nnPU}}(\theta)=
\frac{\pi_+}{n_+}\sum_{i=1}^{n_+}\ell_\theta^+(X_i^+)
+
\max\left\{0,
\frac1{n_u}\sum_{j=1}^{n_u}\ell_\theta^-(X_j^u)
-
\frac{\pi_+}{n_+}\sum_{i=1}^{n_+}\ell_\theta^-(X_i^+)
\right\}.
\]
This prevents the empirical negative-risk component from becoming negative and mitigates finite-sample overfitting.  The stabilization is used only to fit the parametric PU scaling map; PU MC estimation and WLMC auditing use the corrected residual moments.  For UU/MCD with \(P_{U_k}=\theta_kP_++(1-\theta_k)P_-\) and \(\Delta=\theta_1-\theta_2\), the objective is obtained by substituting the log-loss vector into \(M_{\UU}^{-1}\):
\[
\widehat L_{\UU}(\theta)=
\frac1{n_1}\sum_{i=1}^{n_1}\frac{(1-\theta_2)\pi_+\ell_\theta^+(X_i^{U_1})-\theta_2\pi_-\ell_\theta^-(X_i^{U_1})}{\Delta}
+\frac1{n_2}\sum_{j=1}^{n_2}\frac{-(1-\theta_1)\pi_+\ell_\theta^+(X_j^{U_2})+\theta_1\pi_-\ell_\theta^-(X_j^{U_2})}{\Delta}.
\]
The UU implementation uses this objective without non-negative clipping; the non-negative UU variant combines the two sources within each class-specific risk component, clips each component below at zero, and then adds them.

For Pconf, Bayes' rule gives the corrected positive-sample objective
\[
\widehat L_{\Pconf}(\theta)=
\frac{\pi_+}{n}\sum_{i=1}^n
\left\{\ell_\theta^+(X_i)+\frac{1-r_i}{r_i}\ell_\theta^-(X_i)\right\}.
\]
In experiments, these objectives are optimized over the one-dimensional temperature parameter or the two-dimensional Platt parameters on a held-out weak calibration split, keeping the base predictor fixed.  For each fixed post-processing parameter, the unmodified corrected objectives follow from the WSL decontamination identities and source-wise empirical averaging. Non-negative risk correction is an optimization stabilization and is not an unbiased risk rewrite.

\section{Settings of Toy Data Experiments}\label{app:toy-data-settings}
The toy experiment in Figure~\ref{fig:toy-and-tabular-mc} uses a one-dimensional binary classification problem with a known data-generating distribution.  We draw $X\sim\mathrm{Unif}([0,1])$ and evaluate a fixed probabilistic predictor
\[
    f(x)=\operatorname{clip}(0.12+0.76x,10^{-6},1-10^{-6}).
\]
The true conditional label probability is
\[
    r(x)=\Pr(Y=1\mid X=x)=\operatorname{clip}\{f(x)+0.11\sin(2\pi x),0.02,0.98\},
\]
and labels follow $Y\mid X=x\sim\operatorname{Bernoulli}(r(x))$.  The prevalence is $\pi_+=\Ex[r(X)]$.  The evaluated finite class is formed by crossing eight equal-width subgroups $\calc=\{G_1,\ldots,G_8\}$ partitioning $\calx$ ($[0,1]$) with ten equal-width score bins $\calw=\{B_1,\ldots,B_{10}\}$ on $[0,1]$.  Thus each witness cell is
\[
    a_{g,b}(x;f):=\mathbf 1\{x\in G_g\}\mathbf 1\{f(x)\in B_b\}.
\]

For this diagnostic, the population target is the finite-class multicalibration moment
\[
    \MC_{\mathrm{toy}}(f)
    :=\max_{g\in[8],\,b\in[10]}
    \left|\Ex\left[a_{g,b}(X;f)(Y-f(X))\right]\right|
    =\max_{g\in[8],\,b\in[10]}
    \left|\Ex\left[a_{g,b}(X;f)(r(X)-f(X))\right]\right|.
\]
This is the unnormalized residual moment used in the multicalibration definition; it is not divided by the mass of the corresponding cell.  The population value is computed numerically using a grid of 400,000 points over $[0,1]$.  The plotted sample sizes are $128,256,512,1024,2048,4096,8192,16384,32768$, and $65536$.

For each sample size, independent Monte Carlo repetitions are used to construct PN, Pconf, PU, and UU estimates of the same target.  The PN estimate uses labeled marginal samples $(X_i,Y_i)$ and the empirical cell moments $n^{-1}\sum_i a_{g,b}(X_i;f)(Y_i-f(X_i))$.  The Pconf estimate uses positive samples $X_i^+\sim P(X\mid Y=1)$ with the exact confidence $r(X_i^+)$ and the corrected moments
\[
    \frac{\pi_+}{n}\sum_{i=1}^n
    a_{g,b}(X_i^+;f)\left(1-\frac{f(X_i^+)}{r(X_i^+)}\right).
\]
The PU estimate uses positive samples and marginal unlabeled samples $X_j^u\sim P_X$:
\[
    \frac{\pi_+}{n}\sum_{i=1}^n a_{g,b}(X_i^+;f)
    -
    \frac1n\sum_{j=1}^n a_{g,b}(X_j^u;f)f(X_j^u).
\]
The UU estimate uses two unlabeled mixtures $P_{U_k}=\rho_kP_++(1-\rho_k)P_-$ with $\rho_1=0.80$ and $\rho_2=0.20$, where $P_+=P(X\mid Y=1)$ and $P_-=P(X\mid Y=0)$, and applies the corresponding two-source decontamination weights to estimate the same cell moments.

For a method $q\in\{\mathrm{PN},\mathrm{Pconf},\mathrm{PU},\mathrm{UU}\}$, sample size $n$, and repetition $s$, let $\widehat\MC^{q,(s)}_n$ be the corresponding estimate.  The plotted point is the mean absolute estimation error
\[
    \frac1R\sum_{s=1}^{R}
    \left|\widehat\MC^{q,(s)}_n-\MC_{\mathrm{toy}}(f)\right|,
\]
where $R$ is the number of Monte Carlo repetitions at that sample size; in the plotted experiment, $R=10$.  The vertical error bar is one standard deviation across those absolute errors,
\[
    \left[
    \frac1R\sum_{s=1}^{R}
    \left(
    \left|\widehat\MC^{q,(s)}_n-\MC_{\mathrm{toy}}(f)\right|
    -
    \frac1R\sum_{t=1}^{R}
    \left|\widehat\MC^{q,(t)}_n-\MC_{\mathrm{toy}}(f)\right|
    \right)^2
    \right]^{1/2}.
\]
Thus the error bars show variability across Monte Carlo repetitions; they are not confidence intervals or standard errors.  The gray guide lines in the plot have slope $-1/2$ on the log--log scale and indicate the parametric finite-sample rate predicted by the concentration bounds.

\section{Dataset and Subgroup Descriptions}\label{app:dataset-subgroups}
This appendix describes the real-data evaluation setting used in Section~\ref{sec:experiments}.  We follow the finite-subgroup viewpoint common in practical multicalibration experiments: rather than searching over an abstract infinite class, we evaluate a fixed collection of interpretable groups defined by dataset metadata or tabular attributes.  The groups may overlap, so a point can belong to several evaluated subpopulations.

We first define the evaluation metrics, subgroup collections, and data splits.  Appendix~\ref{app:weak-data-preparation} then gives the authoritative split-first construction of the weak observation views used by those metrics.  Later appendices summarize the model and post-processing hyperparameters, reproducibility details, and detailed result tables.
\subsection{Calibration and multicalibration metrics}
MC is the unnormalized witness moment appearing in the theory, so small subgroup--bin cells can have small absolute moments even when their conditional calibration error is large. To compare subpopulations fairly, we also report MaxECE: it normalizes each subgroup's binned calibration error by the subgroup mass and then takes the worst subgroup. Following \citet{hansen2024multicalibration}, our numerical experiments report ECE and MaxECE in addition to MC. Let $f:\calx\to[0,1]$ be a frozen predictor, let $B_1,\ldots,B_K$ be score bins, and let $G_1,\ldots,G_m$ be the evaluated subgroup collection.  We also write $G_0=\calx$ for the whole population and
\[
    a_{g,b}(x;f):=\mathbf 1\{x\in G_g\}\mathbf 1\{f(x)\in B_b\}.
\]
The superscript $\mathrm{or}$ denotes oracle clean-label quantities; in the fully supervised PN setting, these coincide with the corresponding full-label estimates.
For a group--bin cell, define the following moment and group mass
\begin{equation}
\label{eq:app-oracle-cell-moment}
    m^{\rm or}_{g,b}(f)
    :=\Ex\!\big[a_{g,b}(X;f)(Y-f(X))\big],
    \qquad
    \mu_g:=\Pr(X\in G_g).
\end{equation}
Given a clean sample $S=\{(X_i,Y_i)\}_{i=1}^n$, the corresponding oracle empirical quantities are
\begin{equation}
\label{eq:app-oracle-cell-estimator}
    \widehat m^{\rm or}_{g,b}(f)
    :=\frac1n\sum_{i=1}^n a_{g,b}(X_i;f)(Y_i-f(X_i)),
    \qquad
    \widehat\mu^{\rm or}_{g}:=\frac1n\sum_{i=1}^n \mathbf 1\{X_i\in G_g\}.
\end{equation}
The oracle ECE, maxECE, and MC columns in the tables are computed as
\begin{align}
\label{eq:app-oracle-ece-maxece-mc}
    \widehat{\mathrm{ECE}}^{\rm or}(f)
    &:=\sum_{b=1}^K\left|\widehat m^{\rm or}_{0,b}(f)\right|,\\
    \widehat{\mathrm{maxECE}}^{\rm or}(f)
    &:=\max_{g\in[m]:\,\widehat\mu^{\rm or}_g>\mu_{\min}}
    \frac{1}{\widehat\mu^{\rm or}_g}\sum_{b=1}^K
    \left|\widehat m^{\rm or}_{g,b}(f)\right|,\\
    \widehat{\mathrm{MC}}^{\rm or}(f)
    &:=\max_{g\in[m],\,b\in[K]}
    \left|\widehat m^{\rm or}_{g,b}(f)\right|.
\end{align}
Here $\mu_{\min}$ is the minimum active-group mass used by the implementation; setting $\mu_{\min}=0$ includes all groups with positive empirical mass.  The ECE formula is the usual binned ECE because $\left|\widehat m^{\rm or}_{0,b}\right|$ equals the bin mass times the absolute difference between the average label and average prediction in bin $B_b$.  The maxECE formula applies the same binned ECE calculation conditionally within each subgroup and then takes the worst subgroup.  MC instead keeps the supremum-style unnormalized group--bin residual used by the witness definition.

\paragraph{Weak estimates of ECE, maxECE, and MC.}
Appendix~\ref{app:weak-data-preparation} specifies how the held-out validation and test splits are converted into Pconf, PU, and UU views.  Here we only define the corrected group--bin residuals that are plugged into the metrics and into the WLMC audit.  For weak-observation model $q\in\{\PU,\UU,\Pconf\}$, define $\widehat m^{q}_{g,b}(f)$ by applying the corrected witness estimator to $c(x)=\mathbf 1\{x\in G_g\}$ and $w(v)=\mathbf 1\{v\in B_b\}$.  Thus, for PU data $(S_+,S_u)$,
\begin{equation}
\label{eq:app-pu-cell-estimator}
\widehat m^{\PU}_{g,b}(f)
=\frac{\pi_+}{n_+}\sum_{i=1}^{n_+}a_{g,b}(X_i^+;f)
-\frac1{n_u}\sum_{j=1}^{n_u}a_{g,b}(X_j^u;f)f(X_j^u).
\end{equation}
For UU data with $U_k=\rho_kP_+ +(1-\rho_k)P_-$ and $\Delta=\rho_1-\rho_2\neq0$,
\begin{equation}
\label{eq:app-uu-cell-estimator}
\widehat m^{\UU}_{g,b}(f)
=\sum_{k=1}^2\frac1{n_k}\sum_{i=1}^{n_k}
\alpha_k(f(X_i^{U_k}))a_{g,b}(X_i^{U_k};f),
\end{equation}
where
\[
\alpha_1(v)=\frac{(1-\rho_2)\pi_+(1-v)+\rho_2\pi_-v}{\Delta},
\qquad
\alpha_2(v)=\frac{-(1-\rho_1)\pi_+(1-v)-\rho_1\pi_-v}{\Delta}.
\]
For Pconf data $\{(X_i,r_i)\}_{i=1}^n$ with $X_i\sim P_+$ and $r_i\approx\Pr(Y=1\mid X_i)$,
\begin{equation}
\label{eq:app-pconf-cell-estimator}
\widehat m^{\Pconf}_{g,b}(f)
=\frac{\pi_+}{n}\sum_{i=1}^n
    a_{g,b}(X_i;f)\left(1-\frac{f(X_i)}{r_i}\right),
\end{equation}
with $r_i$ replaced by $\max\{r_i,\tau\}$ if confidence clipping is enabled.

The WLMC audit uses exactly this corrected residual table: it evaluates the finite subgroup--bin family, finds the largest corrected signed residual, and applies the corresponding post-processing update.  Appendix~\ref{app:weak-data-preparation} supplements this definition by specifying which split-specific source indices, confidence values, and subgroup masks enter the Pconf, PU, and UU versions of the table.

The weak ECE, maxECE, and MC columns use the same aggregation rules as Eq.~\eqref{eq:app-oracle-ece-maxece-mc}, replacing oracle cell moments by the corrected weak moments above.  For group-wise maxECE, the denominator is the target covariate group mass computed from the full feature-only held-out evaluation split, denoted $\widehat\mu_g^{\rm eval}$ and defined in Eq.~\eqref{eq:app-eval-denominator-split-first}.  Hence, for $q\in\{\PU,\UU,\Pconf\}$,
\begin{align}
\label{eq:app-weak-ece-maxece}
    \widehat{\mathrm{ECE}}^{q}(f)
    &:=\sum_{b=1}^K\left|\widehat m^{q}_{0,b}(f)\right|,\\
    \widehat{\mathrm{maxECE}}^{q}(f)
    &:=\max_{g\in[m]:\,\widehat\mu_g^{\rm eval}>\mu_{\min}}
    \frac{1}{\widehat\mu_g^{\rm eval}}\sum_{b=1}^K
    \left|\widehat m^{q}_{g,b}(f)\right|,\\
    \widehat{\mathrm{MC}}^{q}(f)
    &:=\max_{g\in[m],\,b\in[K]}
    \left|\widehat m^{q}_{g,b}(f)\right|.
\end{align}
This section therefore fixes the metric definitions; Appendix~\ref{app:weak-data-preparation} is the source of truth for the split-first weak-data construction, the unbiasedness of $\widehat\mu_g^{\rm eval}$ in our semi-synthetic experiments, and the alternative Pconf/PU/UU denominators needed when no full target covariate evaluation pool is available.  Oracle estimates are used only for reporting and analysis, while weak estimates are used for weak validation and weak post-processing selection.

\subsection{Real-data datasets and subgroup collections}
Table~\ref{tab:realdata-benchmark-summary} summarizes the real-data benchmark setting.  For the finite subgroup family, we adopt the dataset-specific subgroup collections used by the clean-label multicalibration benchmark of \citet{hansen2024multicalibration}; the subgroup family is not tuned separately for our weak-supervision experiments.  Each loader defines a finite, interpretable collection of protected-attribute groups, identity-mention groups, facial-attribute groups, and simple intersections.  The same code-level group indicators are used for oracle metrics, weak estimates, validation-based method selection, WLMC audits, and the CDF/scatter diagnostics in Appendix~\ref{app:large-model-heatmap-supp}.  All MC evaluations use 10 score bins, so the number of evaluated group--bin cells is $10m$, where $m$ is the number of groups reported in Table~\ref{tab:realdata-benchmark-summary}.  Table~\ref{tab:realdata-benchmark-summary} gives short descriptions of the loader-defined subgroup collections; the exact code-level definitions are provided in the reproducibility code.

\begin{table*}[t]
\centering
\small
\caption{Real-data benchmark summary.}
\label{tab:realdata-benchmark-summary}
\resizebox{\textwidth}{!}{%
\begin{tabular}{llrrp{6.8cm}}
\toprule
Dataset & Prediction target & Total $n$ & Groups $m$ & Subgroup description \\
\midrule
ACSIncome~\citep{ding2021retiring} & income $>50$K & 195,665 & 10 & Groups based on demographic ACS attributes such as race, sex, age, marital status, and simple intersections. \\
CreditDefault~\citep{default_of_credit_card_clients_350} & default on credit-card debt & 30,000 & 15 & Groups based on sex, education, marital status, age buckets, and related intersections after preprocessing. \\
HMDA~\citep{cooper2023variance} & mortgage approval & 114,185 & 13 & Groups based on applicant and co-applicant ethnicity, sex, and race in the processed HMDA loan-application data. \\
MEPS~\citep{sharma2021fair} & high healthcare utilization & 11,079 & 14 & Groups based on age, race, region, poverty-related categories, and related health-survey metadata. \\
CelebA & \texttt{Blond\_Hair} & 202,599 & 18 & Binary facial-attribute groups derived from CelebA metadata, including sex, age, hair, skin, and face-related attributes. \\
CivilComments & toxicity & 447,998 & 13 & Identity-mention groups for toxic-comment prediction, including gender, LGBTQ, religion, and race-related groups and their complements. \\
\bottomrule
\end{tabular}%
}
\end{table*}

\subsection{Data splits and effective sample sizes}
\label{app:data-splits}
The test split is created first with seed $42$.  The validation split is then created from the remaining data using the run seed.  Post-processing uses a separate correction split carved out of the training partition with seed $50$.  The tabular data uses correction fraction $0.4$ within the training partition, giving roughly $0.36n$ base-training examples and $0.24n$ correction examples.  CelebA and CivilComments use correction fraction $0.2$, giving roughly $0.48n$ base-training examples and $0.12n$ correction examples.  In all cases, validation and test use roughly $0.20n$ each.  These split sizes refer to the underlying fully labeled data before labels are hidden; Appendix~\ref{app:weak-data-preparation} explains how PN, Pconf, PU, and UU views are generated separately within each split.

\begin{table*}[t]
\centering
\small
\caption{Effective split sizes for the real-data experiments.  Values are rounded because the exact counts depend on integer rounding in the split routine.}
\label{tab:realdata-split-summary}
\begin{tabular}{lrrrr}
\toprule
Dataset & base train & post-processing correction & validation & test \\
\midrule
ACSIncome & $\approx 70{,}400$ & $\approx 47{,}000$ & $\approx 39{,}100$ & $\approx 39{,}100$ \\
CreditDefault & $\approx 10{,}800$ & $\approx 7{,}200$ & $\approx 6{,}000$ & $\approx 6{,}000$ \\
HMDA & $\approx 41{,}100$ & $\approx 27{,}400$ & $\approx 22{,}800$ & $\approx 22{,}800$ \\
MEPS & $\approx 4{,}000$ & $\approx 2{,}700$ & $\approx 2{,}200$ & $\approx 2{,}200$ \\
CelebA & $\approx 97{,}200$ & $\approx 24{,}300$ & $\approx 40{,}500$ & $\approx 40{,}500$ \\
CivilComments & $\approx 215{,}000$ & $\approx 53{,}800$ & $\approx 89{,}600$ & $\approx 89{,}600$ \\
\bottomrule
\end{tabular}
\end{table*}

\paragraph{Tabular preprocessing.}
The tabular preprocessing follows the repository dataloaders. ACSIncome uses the folktables ACSIncome setup with the binary income label. CreditDefault drops the ID column, bins several continuous financial variables into terciles, and one-hot encodes categorical columns. HMDA uses the processed Texas loan-application files, maps action codes to a binary approval label, drops denial-reason and constant columns, and one-hot encodes low-cardinality categorical variables. MEPS uses the provided processed feature/label files with cleaned column names. Standardization, when used by a model, is fit only on the training split.

\subsection{Random seeds and reported variability}\label{app:seeds-variability}
The toy convergence plot and result tables summarize repeated runs, whereas the tabular scatter plots display individual runs. The appendix reports variability for aggregated results.  The toy convergence experiments use $10$ random seeds.  The tabular real-data experiments use $10$ random seeds, varying the train/validation split and the randomized components of the base predictor when applicable.  The CelebA and CivilComments large-model experiments use $5$ random seeds.  Oracle test metrics are used for final reporting and descriptive comparisons, not for fitting or validation-based model selection.

\section{Data preparation under weak supervision}\label{app:weak-data-preparation}
This appendix is the source of truth for the experimental weak-data construction.  Appendix~\ref{app:dataset-subgroups} defines the finite subgroup--bin metrics and the corrected weak MC residuals used by WLMC; the present appendix complements that definition by specifying exactly when the data are split, how Pconf, PU, and UU views are generated inside each split, and why the maxECE denominator used in the experiments estimates the target covariate marginal.

All weak views in the real-data experiments are semi-synthetic views of the same underlying labeled data.  This protocol is important for two reasons.  First, it lets us evaluate weak-observation methods against oracle clean-label metrics.  Second, it makes the evaluation population invariant across PN, Pconf, PU, and UU: only the observation model changes, not the held-out covariate population on which calibration is evaluated.

\paragraph{Order of splitting and weak-view generation.}
For each dataset and random seed, we first split the fully labeled dataset into disjoint supervised splits.  The training portion is further divided into a base-training split and a post-processing correction split; the remaining held-out data are used for validation and test evaluation.  We write these four splits as
\[
D_{\rm tr},\qquad D_{\rm cal},\qquad D_{\rm val},\qquad D_{\rm test}.
\]
The weak observations are generated \emph{after} this split.  In particular, for each
\(S\in\{D_{\rm tr},D_{\rm cal},D_{\rm val},D_{\rm test}\}\), the Pconf, PU, and UU views are constructed separately using only examples inside that split.  We never first convert the full dataset into a weak dataset and then split it.  Thus, the weak training view is generated from \(D_{\rm tr}\), the weak post-processing view from \(D_{\rm cal}\), the weak validation view from \(D_{\rm val}\), and the weak test view from \(D_{\rm test}\).

For a generic split \(S=\{(X_i,Y_i)\}_{i\in I_S}\), let
\[
I_S^+ := \{i\in I_S:Y_i=1\},\qquad
I_S^- := \{i\in I_S:Y_i=0\},\qquad
\widehat\pi_{+,S}:=|I_S^+|/|I_S|.
\]
The hidden labels in \(S\) are used by the experimenter only to simulate the weak observation source and to compute oracle reference metrics.  The weak base learner, weak validation criterion, and weak post-processing algorithm receive only the corresponding weak view described below.

\paragraph{PN.}
The PN view observes the clean labeled pairs in the corresponding split,
\[
D_S^{\mathrm{PN}}=\{(X_i,Y_i):i\in I_S\}.
\]
This is the usual fully supervised view.  In the PN setting, the weak and oracle ECE formulas coincide when evaluated on the same scores and bins.

\paragraph{Pconf.}
The Pconf view keeps only positive examples from the split and attaches a positive-confidence value
\(r_i\approx \Pr(Y=1\mid X_i)\).  Concretely,
\[
D_S^{\Pconf}=\{(X_i,r_i):i\in I_S^+\}.
\]
In the implementation, \(r_i\) is obtained from a teacher score; the default setting uses this teacher score without additional perturbation.  The class prior \(\widehat\pi_{+,S}\) is also recorded for the Pconf correction.  Thus, the Pconf learner or post-processor sees positive examples, confidence values, and the supplied prior, but not negative labels in that split. These teacher-based confidences are estimates, so the exact-confidence Pconf guarantee does not directly apply to the real-data runs; clipping does not remove confidence-estimation error.

\paragraph{PU.}
The PU view uses a case-control / SCAR two-sample construction inside the split.  The positive bag is sampled from the positive pool \(I_S^+\), while the unlabeled bag is sampled from the full split \(I_S\).  With default rates \(\lambda_P=0.5\) and \(\lambda_U=1.0\), the implementation draws
\[
n_P=\operatorname{round}(\lambda_P |I_S^+|),
\qquad
n_U=\operatorname{round}(\lambda_U |I_S|)
\]
indices.  The positive bag is
\[
S_P=\{X_i^+: i=1,\ldots,n_P\},\qquad X_i^+\sim P_+,
\]
and the unlabeled bag is
\[
S_U=\{X_j^u: j=1,\ldots,n_U\},\qquad X_j^u\sim P_X.
\]
Semi-synthetically, these samples are drawn from \(I_S^+\) and \(I_S\), respectively.  By default, both draws are with replacement and independent.  Therefore the same underlying example may appear more than once, a positive example may also appear in the unlabeled bag, and the two bags should be viewed as empirical draws from two source distributions rather than as a partition of the split.  The class prior \(\widehat\pi_{+,S}\) is computed from the underlying split and used in the PU correction.

\paragraph{UU.}
The UU view constructs two unlabeled mixture sources inside the split.  In code notation,
\[
U_1=(1-\gamma_1)P_+ + \gamma_1P_-,
\qquad
U_2=\gamma_2P_+ +(1-\gamma_2)P_-.
\]
The default values are \(\gamma_1=\gamma_2=0.2\), so the two sources have target positive fractions \(0.8\) and \(0.2\), respectively.  With default source rates \(\lambda_1=\lambda_2=1.0\), the implementation draws
\[
n_1=\operatorname{round}(\lambda_1 |I_S|),
\qquad
n_2=\operatorname{round}(\lambda_2 |I_S|)
\]
examples for the two sources.  Source 1 is formed by drawing approximately
\(\operatorname{round}((1-\gamma_1)n_1)\) positive examples and the remaining examples from the negative pool; source 2 is formed by drawing approximately
\(\operatorname{round}(\gamma_2 n_2)\) positive examples and the remaining examples from the negative pool.  The draws are with replacement and then shuffled.  The two sources are independently generated from the same split, and \(\widehat\pi_{+,S}\) is recorded for the UU decontamination coefficients.

\paragraph{Training, post-processing, and evaluation usage.}
The base predictor is trained on the view of \(D_{\rm tr}\) specified by the experimental scenario: PN, Pconf, PU, or UU.  Post-hoc recalibration is fit on the corresponding view of \(D_{\rm cal}\).  Weak validation quantities are computed from the weak view of \(D_{\rm val}\), and final weak estimates are computed from the weak view of \(D_{\rm test}\).  Oracle validation or test labels are used only for oracle reporting and analysis, not for weak post-processing.

\paragraph{Evaluation denominator for maxECE.}
This paragraph expands the denominator convention used in Eq.~\eqref{eq:app-weak-ece-maxece}.  For the reported weak ECE, maxECE, and MC estimates, the signed group--bin residual moment is estimated using the regime-specific weak correction, as in Eqs.~\eqref{eq:app-pu-cell-estimator}--\eqref{eq:app-pconf-cell-estimator}.  However, the group denominator in maxECE is computed from the invariant held-out covariate pool of the same evaluation split.  For an evaluation split \(E\in\{D_{\rm val},D_{\rm test}\}\), define
\[
E^X:=\{X_i^{\rm eval}:i=1,\ldots,n_{\rm eval}\}.
\]
Because weak views are generated after the split, this covariate pool is a held-out sample from the target marginal \(P_X=\pi_+P_+ +(1-\pi_+)P_-\), containing both \(Y=1\) and \(Y=0\) examples according to the target class mixture.  Thus, for every group \(G_g\), the implementation uses
\begin{equation}
\label{eq:app-eval-denominator-split-first}
\widehat\mu_g^{\rm eval}
:=\frac{1}{n_{\rm eval}}\sum_{i=1}^{n_{\rm eval}}
\mathbf 1\{X_i^{\rm eval}\in G_g\}.
\end{equation}
This estimator uses no labels.  Under the random split model, it is an unbiased estimator of the target group mass,
\[
\mathbb E[\widehat\mu_g^{\rm eval}]=P_X(X\in G_g).
\]
Consequently, in our semi-synthetic experimental protocol, the same denominator is valid for PN, Pconf, PU, and UU.  The weak-observation correction is applied to the numerator residual moments, not to the definition of the evaluation population.  In particular, for \(q\in\{\Pconf,\PU,\UU\}\), the reported group-wise weak maxECE has the form
\[
\widehat{\mathrm{maxECE}}^q(f)
=
\max_{g:\widehat\mu_g^{\rm eval}>\mu_{\min}}
\frac{1}{\widehat\mu_g^{\rm eval}}
\sum_{b=1}^K |\widehat m^q_{g,b}(f)|,
\]
where \(\widehat m^q_{g,b}(f)\) is the corrected weak cell residual.  The denominator \(\widehat\mu_g^{\rm eval}\) itself is unbiased for \(P_X(G_g)\); the ratio defining maxECE is a plug-in finite-sample estimate of the corresponding normalized calibration quantity.

\paragraph{If no full evaluation covariate pool is available.}
The denominator choice above is specific to our controlled semi-synthetic experiments, where the full feature-only validation and test pools are retained even when labels are hidden from the weak estimator. Here we discuss the case where, in a genuinely weak-only evaluation setting, one may not have access to such an invariant held-out covariate pool.  Then the group mass \(P_X(G_g)\) must be estimated from the available weak source using the observation-model-specific marginal correction.  For PU, the unlabeled bag is already drawn from \(P_X\), so one may use
\[
\widehat\mu_g^{\PU}=\frac1{n_U}\sum_{j=1}^{n_U}\mathbf 1\{X_j^u\in G_g\}.
\]
For Pconf, if only positive-confidence data \((X_i^+,r_i)\) are available, then
\[
P_X(G_g)=\pi_+\mathbb E_{X\sim P_+}
\left[\frac{\mathbf 1\{X\in G_g\}}{r(X)}\right],
\]
so the corresponding plug-in estimator is
\[
\widehat\mu_g^{\Pconf}
=\frac{\pi_+}{n_+}\sum_{i=1}^{n_+}
\frac{\mathbf 1\{X_i^+\in G_g\}}{r_i},
\]
with the same confidence clipping convention as in the residual estimator if clipping is used.  For UU, let
\[
\widehat q_{k,g}:=\frac1{n_k}\sum_{i=1}^{n_k}\mathbf 1\{X_i^{U_k}\in G_g\},
\qquad k=1,2,
\]
and let \(\delta=1-\gamma_1-\gamma_2\neq0\).  Since
\(U_1=(1-\gamma_1)P_+ +\gamma_1P_-\) and
\(U_2=\gamma_2P_+ +(1-\gamma_2)P_-\), the target marginal group mass is estimated by
\[
\widehat\mu_g^{\UU}
=
\frac{\pi_+-\gamma_2}{\delta}\widehat q_{1,g}
+
\frac{(1-\pi_+)-\gamma_1}{\delta}\widehat q_{2,g}.
\]
These alternative denominators are not used in our reported semi-synthetic experiments because Eq.~\eqref{eq:app-eval-denominator-split-first} is available and directly estimates the target evaluation population.  They would be the appropriate denominators in deployments where the only accessible evaluation data are the Pconf or UU weak sources themselves.

\paragraph{Validation and model selection.}
We distinguish two uses of validation.  The first is base-predictor selection, such as choosing among fixed candidate predictors or selecting a neural model by validation performance.  In the PN setting, the validation view is the clean labeled split, and the validation score is ordinary accuracy:
\[
\widehat{\mathrm{Acc}}_{\rm val}(f)
=\frac{1}{|I_{\rm val}|}\sum_{i\in I_{\rm val}}
\mathbf 1\{\mathbf 1[f(X_i)\ge 1/2]=Y_i\}.
\]
PN base candidates are selected by maximizing this quantity.  This is the only case in which clean validation labels are used to select a base predictor.

For weakly trained base predictors, the clean labels in \(D_{\rm val}\) are hidden from the selection rule.  We instead construct the same weak view of \(D_{\rm val}\) as in the corresponding training regime and select the candidate with the smallest weak validation loss.  Let \(s_f(x)\) denote the logit of candidate \(f\), and define
\(\ell_+(s)=\log(1+e^{-s})\) and \(\ell_-(s)=\log(1+e^s)\).  For Pconf, using validation positive-confidence examples \((X_i^+,r_i)\), the criterion is
\[
\widehat L_{\rm val}^{\Pconf}(f)
=\frac{\widehat\pi_{+,{\rm val}}}{n_+}
\sum_{i=1}^{n_+}
\left\{
\ell_+(s_f(X_i^+))+
\frac{1-r_i}{r_i}\ell_-(s_f(X_i^+))
\right\},
\]
with the same clipping convention for \(r_i\) as in the Pconf estimator.  For PU, using the validation positive bag \(S_P^{\rm val}\) and unlabeled bag \(S_U^{\rm val}\), the tabular validation criterion is the corrected PU validation risk
\[
\widehat L_{\rm val}^{\PU}(f)
=\frac{\widehat\pi_{+,{\rm val}}}{n_P}
\sum_{i=1}^{n_P}\ell_+(s_f(X_i^+))
+
\frac{1}{n_U}\sum_{j=1}^{n_U}\ell_-(s_f(X_j^u))
-\frac{\widehat\pi_{+,{\rm val}}}{n_P}
\sum_{i=1}^{n_P}\ell_-(s_f(X_i^+)).
\]
Thus the positive bag contributes the positive-risk term and the positive correction term, while the unlabeled bag estimates the marginal negative-loss term.

For UU, using the two validation mixture sources \(U_1^{\rm val}\) and \(U_2^{\rm val}\), we use the contamination-matrix inverse.  Let \(\widehat{\mathbb E}_{1,{\rm val}}\) and \(\widehat{\mathbb E}_{2,{\rm val}}\) denote empirical averages over these two sources and set \(\delta=1-\gamma_1-\gamma_2\).  The corrected validation loss is
\[
\begin{aligned}
\widehat L_{\rm val}^{\UU}(f)
={}&\widehat\pi_{+,{\rm val}}
\left\{
\frac{1-\gamma_2}{\delta}\widehat{\mathbb E}_{1,{\rm val}}[\ell_+(s_f(X))]
-\frac{\gamma_1}{\delta}\widehat{\mathbb E}_{2,{\rm val}}[\ell_+(s_f(X))]
\right\} \\
&+(1-\widehat\pi_{+,{\rm val}})
\left\{
-\frac{\gamma_2}{\delta}\widehat{\mathbb E}_{1,{\rm val}}[\ell_-(s_f(X))]
+\frac{1-\gamma_1}{\delta}\widehat{\mathbb E}_{2,{\rm val}}[\ell_-(s_f(X))]
\right\}.
\end{aligned}
\]
For the non-negative UU variant, the same two-source decomposition is used with the non-negative correction applied to the partial risks, matching the training objective.  The hidden labels in \(D_{\rm val}\) are used only by the experimenter to generate these weak validation views and to compute oracle reference diagnostics; they are not used by the weak selection rule.

The second use of validation is post-processing selection after the base predictor has been frozen.  All post-processing candidates are fit on \(D_{\rm cal}\), not on validation or test.  They are then selected on \(D_{\rm val}\) by the supervision-appropriate calibration criterion for the reported table or figure.  Clean-label PN selection computes this criterion from the PN validation labels, while weak selection computes the corresponding weak estimate from \(D_{\rm val}^{\Pconf}\), \(D_{\rm val}^{\PU}\), or \(D_{\rm val}^{\UU}\).  For MC-focused post-processing comparisons, for example, weak post-processing methods are selected by the weak validation estimate of MC violation, whereas PN/oracle diagnostics use the clean-label validation analogue.  Accuracy is not used to select post-processing methods; it is reported on the held-out test split as auxiliary predictive-performance information.  Test labels are not used until the final evaluation.

\paragraph{Use of nnPU and ordinary PU objectives.}
For reference, the non-negative PU principle of \citet{kiryo2017nnpu} clips the empirical negative-risk component when that stabilized objective is used.  Let \(s_\theta(x)\in\mathbb R\) denote the model logit.  With logistic losses \(\ell_+(s)=\log(1+e^{-s})\) and \(\ell_-(s)=\log(1+e^{s})\), the empirical PU decomposition is
\begin{align}
\label{eq:app-nnpu-components}
\widehat R_+^{\PU}(\theta)
&:=\pi_+\widehat\Ex_{P_+}[\ell_+(s_\theta(X))],\\
\widehat R_-^{\PU}(\theta)
&:=\widehat\Ex_{P_X}[\ell_-(s_\theta(X))]
  -\pi_+\widehat\Ex_{P_+}[\ell_-(s_\theta(X))].
\end{align}
The nnPU training objective clips the empirical negative-risk component,
\begin{equation}
\label{eq:app-nnpu-objective}
\widehat R_{\rm nnPU}(\theta)
=\widehat R_+^{\PU}(\theta)+\max\{0,\widehat R_-^{\PU}(\theta)\}.
\end{equation}
The supplied implementation uses this clipped nnPU objective for PU base learning in CelebA and CivilComments and for PU Platt and temperature post-processing on both tabular and large-model data. Tabular base learning uses the corresponding supervised-loss or fitting-criterion rewrite. Weak MC/WLMC and calibration evaluation use the corrected residual estimators without applying this non-negative risk correction.

\section{Real-data model and post-processing hyperparameters}\label{app:calib-hparams}
This section records the base predictors, fixed hyperparameters, and post-processing hyperparameters used in the real-data experiments.  The model families and tabular hyperparameter settings follow the clean-label multicalibration benchmark of \citet{hansen2024multicalibration}.  Our weak-supervision experiments keep the same base architecture within each dataset and replace only the supervision source and objective as described in Appendix~\ref{app:weak-data-preparation}. PN base models use oracle validation accuracy for base-model selection.  Weakly trained base models use the weak validation objectives described in Appendix~\ref{app:weak-data-preparation}; no clean validation labels are used to select weak candidates.  All reported post-processing methods freeze the selected base predictor and are fit only on the correction split.

\subsection{Tabular base predictors}\label{app:tabular-model-details}
The tabular data uses the same finite feature representations and subgroup metadata as in Appendix~\ref{app:dataset-subgroups}.  The reported tabular experiments use six base-predictor families: decision trees, random forests, logistic regression, support vector machines, naive Bayes, and MLP.  We follow the same benchmark candidate grids used in the clean-label benchmark and use the selected settings for correction fraction $0.4$; the weak experiments do not run a new hyperparameter search.  For tabular weak baselines, the base predictor is obtained by minimizing the regime-specific empirical objective obtained from the WSL invariant-estimator rewrite of the corresponding supervised loss or fitting criterion.

For decision trees, the selected model controls tree depth and the minimum number of samples required to split an internal node.  For random forests, we use 100-tree ensembles with the selected tree-depth and split-size controls for each dataset.  The random-forest predictor is also used as the Pconf teacher in the tabular construction.

For logistic regression, we use the benchmark-selected amount of regularization.  For support vector machines, we use the benchmark-selected margin-regularization setting with standardized features.  Naive Bayes is treated as a hyperparameter-free baseline.  Here MLP denotes a fully connected feed-forward binary classifier with ReLU hidden layers and two output logits, trained with cross-entropy and the Adam optimizer on standardized features.  At correction fraction $0.4$, the MLP configurations are as follows: ACSIncome uses architecture $10$--$128$--BN--$256$--BN--$128$--$2$, batch size $64$, learning rate $10^{-3}$, no weight decay, and $50$ epochs; CreditDefault uses $118$--$128$--$256$--$128$--$2$, batch size $128$, learning rate $10^{-2}$, no weight decay, and $5$ epochs; HMDA uses $89$--$100$--$2$, batch size $128$, learning rate $10^{-3}$, weight decay $10^{-5}$, and $30$ epochs; and MEPS uses $139$--$100$--$2$, batch size $16$, learning rate $10^{-4}$, weight decay $10^{-5}$, and $50$ epochs.  Here BN denotes batch normalization.

\subsection{Image and text base predictors}\label{app:large-model-details}
The large-model experiments follow the vision/language setup of the reference benchmark: transformer backbones are fine-tuned from pretrained weights, whereas the ResNet image companion model is trained from scratch as a task-specific classifier.  For each named backbone, the preprocessing or tokenizer, backbone size, classifier head, and training grid are fixed across seeds and weak-observation models, so differences across reported result rows are attributable to the supervision and post-processing data rather than to architecture changes.

The large-model base predictors are as follows.  CelebA-ViT predicts the binary facial attribute \textbf{\texttt{Blond\_Hair}} using a pretrained ViT backbone with a binary classification head; the ViT preprocessing is fixed across PN, PU, and UU rows, and PU-trained bases use nnPU.  CelebA-ImageResNet predicts the same \textbf{\texttt{Blond\_Hair}} target with a ResNet-50 binary classifier trained as the appendix companion image model.  CivilComments-BERT predicts toxicity with a pretrained BERT-based encoder and a binary classification head, using the same tokenizer and classification head across PN and weak-supervision regimes.

\paragraph{CelebA-ViT.}
The CelebA task is binary classification of the facial attribute \textbf{\texttt{Blond\_Hair}}.  The ViT setting uses a pretrained ViT backbone with a binary classification head.  Images are resized and normalized according to the ViT preprocessing pipeline, and the model is fine-tuned for 10 epochs with batch size 64, Adam, learning rate \(10^{-4}\), and weight decay \(10^{-2}\).  The same ViT architecture and image preprocessing are used for PN, PU, and UU rows; weakly trained rows differ only in the weak training objective and observed data view.

\paragraph{CelebA-ImageResNet.}
The ImageResNet companion experiment uses a ResNet-50 binary classifier for the same \textbf{\texttt{Blond\_Hair}} target.  The model is trained for 50 epochs with batch size 64, SGD, momentum 0.9, and learning rate \(10^{-3}\).  Table~\ref{tab:celeba-resnet-main} reports this ResNet experiment in the same format as Tables~\ref{tab:celeba-vit-main} and~\ref{tab:civil-bert-main}.

\paragraph{CivilComments-BERT.}
The CivilComments task is a binary text classification problem: given the text of an online comment, the model predicts whether the comment is toxic or
non-toxic.  Following the CivilComments setup used in \citet{hansen2024multicalibration}, the identity annotations are not prediction targets; they are used to define evaluation groups for group-wise calibration and multicalibration metrics. Text is tokenized with the corresponding
pretrained tokenizer, truncated or padded to maximum length 300, and fine-tuned for 10 epochs with batch size 16, Adam, learning rate \(10^{-5}\), and weight decay \(10^{-2}\).  The same tokenizer, maximum length, and classification head are reused across PN and weak-supervision regimes.

\subsection{Post-processing hyperparameters and WLMC implementation}\label{app:postproc-hparams}
The post-processing methods use a frozen base predictor.  The non-WLMC post-processing optimizers are as follows, while Table~\ref{tab:wlmc-hparams} gives the WLMC audit/update parameters.

For temperature scaling, the fitted map is a single scalar temperature applied to the logit score.  For tabular rows, the objective is clean PN log-loss for PN, corrected weak log-loss for Pconf and UU, and the non-negative PU objective in Eq.~\eqref{eq:app-nnpu-objective} for PU, as used in the supplied implementation.  We optimize this scalar with L-BFGS-B; tabular weak-temperature rows use at most 200 iterations, and the large-model exported rows use the same fixed post-processing routine used for the corresponding reported result tables.

For Platt scaling, the fitted map is the affine logit transformation $a\,\mathrm{logit}(f(x))+b$.  For tabular rows, the objective is clean PN log-loss for PN, corrected weak log-loss for Pconf and UU, and the non-negative PU objective in Eq.~\eqref{eq:app-nnpu-objective} for PU, as used in the supplied implementation.  We optimize the two affine parameters with L-BFGS-B and use a maximum of 250 iterations in the tabular exported grid.

\paragraph{Finite-basis WLMC.}
All real-data WLMC runs use the same finite group--bin witness family as the reported MC metric.  The witness family is the Cartesian product of the adopted dataset subgroups and uniform score-bin indicators.  Because both $\mathcal C$ and $\mathcal W$ are finite in all main-text experiments, the audit maximizes over $\mathcal C\times\mathcal W$ exactly.  Consequently, the audit step is implemented by exhaustive enumeration over all subgroup--score-bin cells rather than by an external learning oracle.  In each boosting round, WLMC audits all cells under the requested supervision regime (PN, PU, UU, or Pconf), selects the cell with the largest absolute signed violation, and applies an additive correction to the scores in that cell.  If the selected cell is $A_t$ and the signed violation is $v_t$, the update is
\[
f_{t+1}(x)
=
\mathrm{clip}_{[0,1]}\left(
    f_t(x)+\eta\,\mathrm{sign}(v_t)\mathbf 1\{x\in A_t\}
\right).
\]
The procedure stops after $T$ rounds or when the largest absolute violation is below the threshold $\alpha$.  Cells with empirical active mass below the minimum active mass are ignored.  For Pconf audits, confidences are lower-clipped at $r_{\min}$ when clipping is enabled.

\begin{table}[t]
\centering
\small
\caption{WLMC hyperparameters used in the real-data experiments.  All settings use exhaustive auditing over the finite subgroup--score-bin witness family.}
\label{tab:wlmc-hparams}
\begin{tabular}{llll}
\toprule
Parameter & Tabular & CelebA-ViT & CivilComments-BERT \\
\midrule
Number of bins & $10$ & $10$ & $10$ \\
Step size $\eta$ & $0.05$ & $0.04$ & $0.05$ \\
Number of rounds $T$ & $50$ & $50$ & $50$ \\
Threshold $\alpha$ & $0.005$ & $0.005$ & $0.005$ \\
Minimum active mass & $0.01$ & $0.01$ & $0.01$ \\
Positive-confidence $r_{\min}$ & $10^{-3}$ & -- & -- \\
\bottomrule
\end{tabular}
\end{table}

CelebA-ImageResNet uses the same finite CelebA subgroup family and 10-bin witness construction.  Table~\ref{tab:celeba-resnet-main} is a companion to the CelebA-ViT results in Table~\ref{tab:celeba-vit-main} and now reports test accuracy as well as calibration metrics.

\section{Reproducibility capsule}\label{app:reproducibility-capsule}
All experiments were run on a workstation with two NVIDIA A6000 Ada GPUs and an Intel(R) Xeon(R) w5-3435X CPU running at 3.10 GHz.  The released reproduction code records the exact Python environment, package versions, command-line arguments, random seeds, and data-processing scripts used to generate the figures and tables.  In this appendix, we therefore specify the experimental protocol, model families, and hardware, while deferring exact package-version strings to the reproducibility package.

The tabular experiments run within approximately 5 hours on CPU/GPU, while the CelebA and CivilComments large-model experiments require approximately 72 hours on GPU. Peak GPU memory usage is approximately 40 GB for both CelebA-ViT and for CivilComments-BERT.

The experiment design follows the clean-label multicalibration benchmark of \citet{hansen2024multicalibration} for dataset preprocessing, finite subgroup construction, tabular base-model families, and large-model training grids; our changes are the weak-observation views, corrected weak objectives, and weak-label post-processing/evaluation rules described above.  The tabular experiments use six base-model families, including MLP.  All reported post-processing methods are fit on the correction split, selected on the validation split using the supervision-appropriate validation criterion, and evaluated once on the held-out test split.

\section{Results on Tabular Datasets}\label{app:tabular-fixed}
The main text reports diagnostic plots.  Here we show the fixed-hyperparameter table format used for detailed inspection.  We keep only the representative WLMC implementation setting with step size $0.05$ and update cap $50$ so that each supervision family appears once, rather than repeating many nearly identical rows for different WLMC budgets.

\paragraph{Coverage of tabular base predictors.}
Tables~\ref{tab:tabular-acsincome-decisiontree-main}--\ref{tab:tabular-meps-mlp-main} report the tabular test results for the four datasets used in the main text---ACSIncome, CreditDefault, HMDA, and MEPS---and the six base-predictor families: DecisionTree, LogisticRegression, NaiveBayes, RandomForest, SVM, and MLP. Rows are grouped by $(\mathrm{base},\mathrm{post})$; the oracle columns use clean test labels for reporting only, the weak columns use the corresponding corrected weak estimates, accuracy is reported against clean test labels, and boldface marks the single best mean within each displayed block.

The tabular scatter diagnostics use the same ten seeds, four datasets, and six base-predictor families.  Therefore the WLMC-only and best-of-three post-processing scatters each have $10\times4\times6=240$ points in the $(\mathrm{PN},\mathrm{PN})$ panel and $10\times4\times6\times3=720$ points in each of the $(\mathrm{PN},\mathrm{weak})$ and $(\mathrm{weak},\mathrm{weak})$ panels, where the additional factor of $3$ is the number of weak-observation types \(\{\PU,\Pconf,\UU\}\).  The oracle-versus-weak scatter that displays all three post-processing families has $10\times4\times6\times3\times3=2160$ points per panel and metric after including the additional factor of $3$ for WLMC, temperature scaling, and Platt scaling.

The WLMC-only scatter and the tabular tables report fixed correction methods separately, while Figure~\ref{fig:tabular-posthoc-mc-bestofthree} uses clean validation labels to select among WLMC, temperature scaling, and Platt scaling using oracle validation MC.

Figures~\ref{fig:tabular-mc-correlation-appendix} and~\ref{fig:tabular-ece-correlation-appendix} provide the oracle-versus-weak MC and ECE plots with correlation coefficients.

\paragraph{Tabular result summary.}
The tabular results show that the proposed WLMC is particularly stable for the MC-oriented quantities: it often reduces subgroup-wise multicalibration error and maxECE without relying on clean labels at post-processing time, and the oracle-versus-weak scatter plots indicate that the weak estimates track the corresponding oracle quantities well enough to guide model selection.  Platt scaling is frequently competitive for ECE because it can adjust both slope and intercept of the score distribution, but its effect on the worst subgroup-bin violations is less uniform.  Temperature scaling is the most fragile of the three post-processing families: a single global scale can improve average calibration in some easy cases, but it often fails to correct subgroup- or bin-specific errors and can even worsen ECE or MC when the base scores are already misaligned in a non-uniform way. WLMC directly targets the WSL-estimated multicalibration violations, and therefore often continues to improve MC and maxECE even when the one-dimensional score transformations are insufficient.

\begin{figure*}[t]
    \centering
    \includegraphics[width=\dimexpr\textwidth-0.8pt\relax]{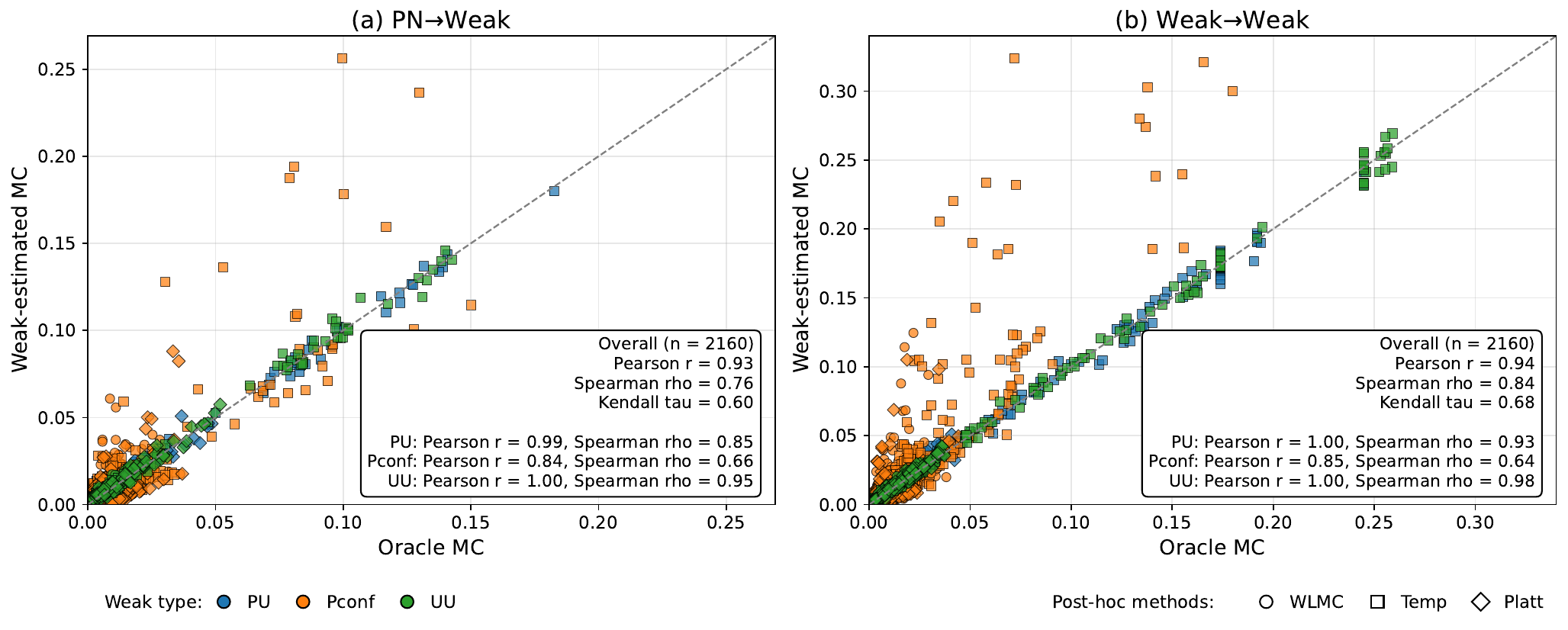}
    \caption{Oracle versus weak-estimated MC on tabular data, with correlation coefficients. The dashed diagonal denotes exact agreement.}
    \label{fig:tabular-mc-correlation-appendix}
\end{figure*}

\begin{figure*}[t]
    \centering
    \includegraphics[width=\dimexpr\textwidth-0.8pt\relax]{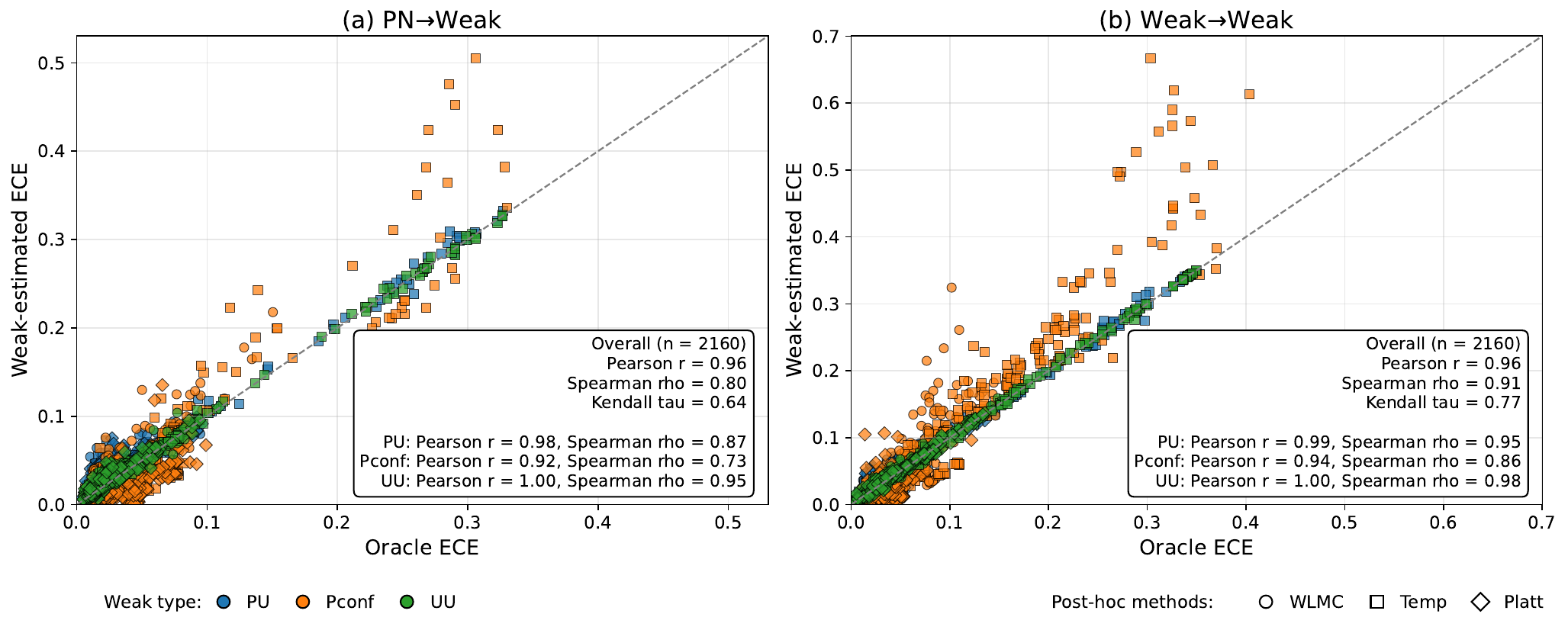}
    \caption{Oracle versus weak-estimated ECE on tabular data, with correlation coefficients. The dashed diagonal denotes exact agreement.}
    \label{fig:tabular-ece-correlation-appendix}
\end{figure*}

\begin{figure*}[t]
    \centering
    \includegraphics[width=\dimexpr0.98\textwidth-0.8pt\relax]{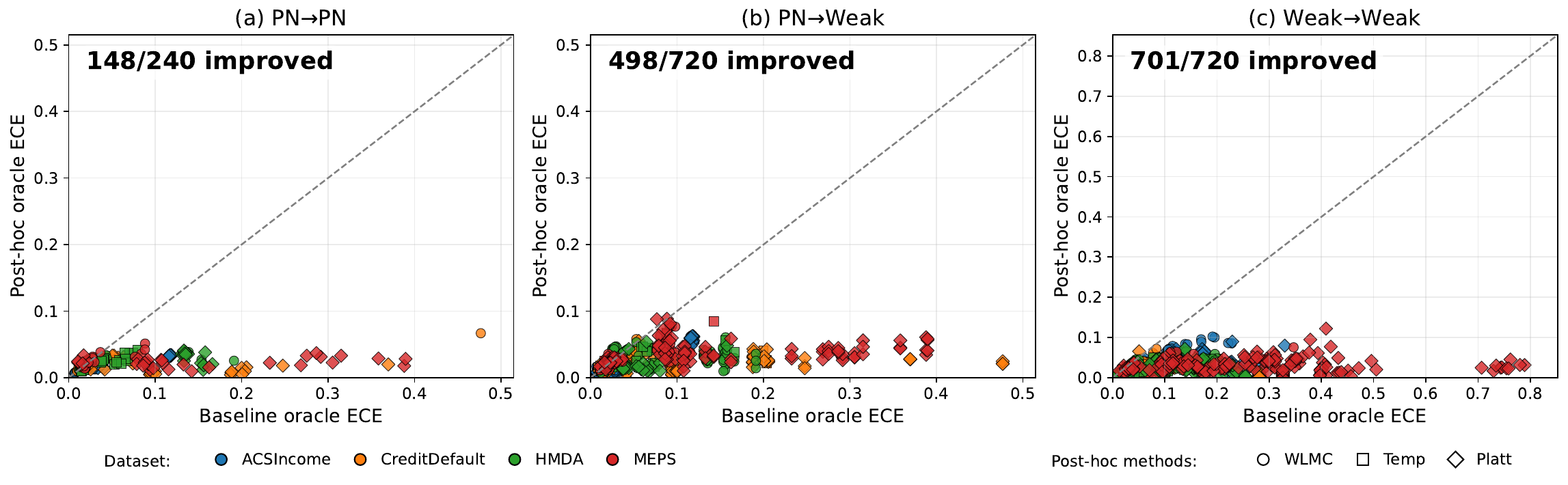}
    \caption{Oracle test ECE before and after correction by a method selected from WLMC, temperature scaling, and Platt scaling using the smallest oracle validation ECE.}
    \label{fig:tabular-posthoc-ece-bestofthree-appendix}
\end{figure*}

\begin{table*}[t]
\centering
\scriptsize
\setlength{\tabcolsep}{3.2pt}
\renewcommand{\arraystretch}{0.92}
\caption{Tabular test results on ACSIncome with DecisionTree as the base predictor. Calibration metrics are reported as $10^2\!\times$ mean $\pm$ standard deviation over ten seeds; accuracy is reported in percent. WLMC uses step size $0.05$ and at most $50$ updates. Bold marks the single best mean within each $(\mathrm{base},\mathrm{post})$ block; variance is not used for bolding.}
\label{tab:tabular-acsincome-decisiontree-main}
\resizebox{\textwidth}{!}{%
\begin{tabular}{llccccccc}
\toprule
\multirow{2}{*}{(Base,Post)} & \multirow{2}{*}{Post-processing} & \multicolumn{3}{c}{oracle estimate ($\times 10^{-2}$)} & \multicolumn{3}{c}{weak estimate ($\times 10^{-2}$)} & \multirow{2}{*}{Acc. (\%)} \\
\cmidrule(lr){3-5}\cmidrule(lr){6-8}
& & ECE & maxECE & MC & ECE & maxECE & MC & \\
\midrule
\multirow{4}{*}{(PN,PN)} & ERM & \textbf{2.48$\pm$0.15} & 7.42$\pm$0.68 & 0.37$\pm$0.02 & 2.48$\pm$0.15 & 7.43$\pm$0.68 & 0.37$\pm$0.02 & 80.07$\pm$0.10 \\
 & WLMC & 2.48$\pm$0.15 & 7.42$\pm$0.68 & \textbf{0.37$\pm$0.02} & \textbf{2.48$\pm$0.15} & 7.43$\pm$0.68 & \textbf{0.37$\pm$0.02} & \textbf{80.13$\pm$0.10} \\
 & Temp & 4.35$\pm$0.33 & 9.18$\pm$1.12 & 0.50$\pm$0.07 & 4.35$\pm$0.33 & 9.18$\pm$1.12 & 0.50$\pm$0.07 & 80.13$\pm$0.10 \\
 & Platt & 2.68$\pm$0.13 & \textbf{7.11$\pm$1.02} & 0.70$\pm$0.04 & 2.68$\pm$0.13 & \textbf{7.10$\pm$1.02} & 0.70$\pm$0.04 & 80.13$\pm$0.11 \\
\midrule
\multirow{3}{*}{(PN,PU)} & WLMC & \textbf{2.46$\pm$0.17} & \textbf{7.38$\pm$0.71} & \textbf{0.37$\pm$0.02} & \textbf{2.72$\pm$0.48} & \textbf{15.34$\pm$2.88} & \textbf{0.41$\pm$0.10} & \textbf{80.14$\pm$0.11} \\
 & Temp & 4.00$\pm$0.33 & 9.10$\pm$1.25 & 0.47$\pm$0.07 & 4.28$\pm$0.74 & 17.96$\pm$4.25 & 0.50$\pm$0.09 & 80.13$\pm$0.10 \\
 & Platt & 4.07$\pm$0.40 & 8.49$\pm$0.90 & 0.46$\pm$0.09 & 4.28$\pm$0.75 & 17.66$\pm$4.52 & 0.50$\pm$0.12 & 80.13$\pm$0.10 \\
\midrule
\multirow{3}{*}{(PN,Pconf)} & WLMC & \textbf{2.39$\pm$0.32} & \textbf{7.40$\pm$0.69} & \textbf{0.37$\pm$0.02} & \textbf{2.93$\pm$0.24} & \textbf{10.65$\pm$6.97} & \textbf{0.37$\pm$0.15} & \textbf{80.13$\pm$0.10} \\
 & Temp & 3.86$\pm$0.68 & 8.73$\pm$0.80 & 0.43$\pm$0.10 & 3.58$\pm$0.64 & 20.36$\pm$13.99 & 0.55$\pm$0.17 & 80.13$\pm$0.10 \\
 & Platt & 3.94$\pm$0.73 & 8.20$\pm$0.95 & 0.46$\pm$0.11 & 3.68$\pm$0.72 & 19.55$\pm$13.32 & 0.53$\pm$0.16 & 80.11$\pm$0.12 \\
\midrule
\multirow{3}{*}{(PN,UU)} & WLMC & \textbf{2.48$\pm$0.15} & \textbf{7.42$\pm$0.68} & \textbf{0.37$\pm$0.02} & \textbf{2.61$\pm$0.25} & \textbf{11.57$\pm$1.85} & \textbf{0.43$\pm$0.04} & 80.13$\pm$0.10 \\
 & Temp & 4.17$\pm$0.56 & 9.26$\pm$1.06 & 0.49$\pm$0.09 & 4.19$\pm$0.64 & 13.99$\pm$2.95 & 0.49$\pm$0.07 & 80.13$\pm$0.10 \\
 & Platt & 4.25$\pm$0.61 & 8.66$\pm$0.98 & 0.48$\pm$0.11 & 4.30$\pm$0.71 & 13.13$\pm$3.38 & 0.49$\pm$0.09 & \textbf{80.13$\pm$0.10} \\
\midrule
\multirow{4}{*}{(PU,PU)} & ERM & 5.60$\pm$0.28 & 12.86$\pm$1.09 & 1.01$\pm$0.13 & 5.77$\pm$0.61 & 17.29$\pm$3.01 & 1.01$\pm$0.21 & 78.41$\pm$0.25 \\
 & WLMC & \textbf{4.17$\pm$0.47} & \textbf{11.28$\pm$0.89} & \textbf{0.48$\pm$0.12} & \textbf{4.29$\pm$0.43} & \textbf{16.75$\pm$4.76} & \textbf{0.53$\pm$0.11} & 78.43$\pm$0.25 \\
 & Temp & 13.45$\pm$0.74 & 14.79$\pm$0.95 & 1.62$\pm$0.21 & 13.47$\pm$0.69 & 18.18$\pm$2.23 & 1.68$\pm$0.21 & \textbf{78.44$\pm$0.26} \\
 & Platt & 13.56$\pm$0.72 & 15.44$\pm$0.86 & 1.75$\pm$0.26 & 13.62$\pm$0.88 & 18.41$\pm$2.27 & 1.74$\pm$0.25 & 76.11$\pm$1.31 \\
\midrule
\multirow{4}{*}{(Pconf,Pconf)} & ERM & 15.48$\pm$0.16 & 22.69$\pm$0.54 & 2.25$\pm$0.24 & 18.18$\pm$1.11 & 108.83$\pm$48.36 & 3.82$\pm$0.77 & 76.34$\pm$0.28 \\
 & WLMC & 7.41$\pm$0.52 & 13.32$\pm$0.81 & 0.61$\pm$0.12 & 8.09$\pm$0.86 & \textbf{23.26$\pm$16.37} & \textbf{0.67$\pm$0.19} & 78.20$\pm$0.33 \\
 & Temp & 15.89$\pm$0.17 & 22.28$\pm$0.60 & 1.77$\pm$0.16 & 18.28$\pm$0.93 & 87.43$\pm$43.70 & 3.22$\pm$0.72 & 76.34$\pm$0.28 \\
 & Platt & \textbf{2.47$\pm$0.24} & \textbf{6.22$\pm$0.62} & \textbf{0.35$\pm$0.05} & \textbf{3.15$\pm$0.46} & 23.87$\pm$16.35 & 0.78$\pm$0.22 & \textbf{80.17$\pm$0.19} \\
\midrule
\multirow{4}{*}{(UU,UU)} & ERM & 3.59$\pm$0.16 & 9.54$\pm$0.70 & 0.58$\pm$0.08 & 3.65$\pm$0.29 & 11.54$\pm$1.54 & 0.57$\pm$0.09 & \textbf{79.39$\pm$0.20} \\
 & WLMC & \textbf{3.05$\pm$0.60} & \textbf{9.10$\pm$0.68} & \textbf{0.42$\pm$0.07} & \textbf{3.13$\pm$0.48} & \textbf{11.50$\pm$1.36} & \textbf{0.42$\pm$0.10} & 79.39$\pm$0.20 \\
 & Temp & 10.59$\pm$1.61 & 12.10$\pm$1.38 & 1.23$\pm$0.27 & 10.66$\pm$1.79 & 16.31$\pm$2.76 & 1.26$\pm$0.29 & 79.39$\pm$0.20 \\
 & Platt & 10.53$\pm$1.55 & 12.42$\pm$1.31 & 1.15$\pm$0.19 & 10.63$\pm$1.71 & 15.47$\pm$2.45 & 1.17$\pm$0.25 & 79.27$\pm$0.30 \\
\bottomrule
\end{tabular}%
}
\end{table*}

\begin{table*}[t]
\centering
\scriptsize
\setlength{\tabcolsep}{3.2pt}
\renewcommand{\arraystretch}{0.92}
\caption{Tabular test results on ACSIncome with LogisticRegression as the base predictor. Calibration metrics are reported as $10^2\!\times$ mean $\pm$ standard deviation over ten seeds; accuracy is reported in percent. WLMC uses step size $0.05$ and at most $50$ updates. Bold marks the single best mean within each $(\mathrm{base},\mathrm{post})$ block; variance is not used for bolding.}
\label{tab:tabular-acsincome-logisticregression-main}
\resizebox{\textwidth}{!}{%
\begin{tabular}{llccccccc}
\toprule
\multirow{2}{*}{(Base,Post)} & \multirow{2}{*}{Post-processing} & \multicolumn{3}{c}{oracle estimate ($\times 10^{-2}$)} & \multicolumn{3}{c}{weak estimate ($\times 10^{-2}$)} & \multirow{2}{*}{Acc. (\%)} \\
\cmidrule(lr){3-5}\cmidrule(lr){6-8}
& & ECE & maxECE & MC & ECE & maxECE & MC & \\
\midrule
\multirow{4}{*}{(PN,PN)} & ERM & \textbf{0.93$\pm$0.30} & \textbf{7.60$\pm$1.32} & \textbf{0.31$\pm$0.04} & \textbf{0.93$\pm$0.30} & \textbf{7.60$\pm$1.33} & \textbf{0.31$\pm$0.04} & \textbf{77.02$\pm$0.90} \\
 & WLMC & 0.93$\pm$0.30 & 7.60$\pm$1.32 & 0.31$\pm$0.04 & 0.93$\pm$0.30 & 7.60$\pm$1.33 & 0.31$\pm$0.04 & 77.02$\pm$0.90 \\
 & Temp & 2.65$\pm$0.47 & 8.81$\pm$1.37 & 0.41$\pm$0.06 & 2.65$\pm$0.47 & 8.81$\pm$1.36 & 0.41$\pm$0.06 & 77.02$\pm$0.90 \\
 & Platt & 2.35$\pm$0.11 & 7.97$\pm$1.01 & 0.52$\pm$0.09 & 2.35$\pm$0.11 & 7.97$\pm$1.02 & 0.52$\pm$0.09 & 77.00$\pm$0.92 \\
\midrule
\multirow{3}{*}{(PN,PU)} & WLMC & \textbf{0.92$\pm$0.27} & 7.64$\pm$1.30 & \textbf{0.30$\pm$0.03} & \textbf{1.56$\pm$0.34} & 15.74$\pm$3.10 & \textbf{0.40$\pm$0.10} & \textbf{77.02$\pm$0.90} \\
 & Temp & 0.98$\pm$0.23 & 7.79$\pm$1.15 & 0.31$\pm$0.03 & 1.63$\pm$0.31 & 16.24$\pm$3.90 & 0.42$\pm$0.11 & 77.02$\pm$0.90 \\
 & Platt & 0.93$\pm$0.23 & \textbf{7.20$\pm$0.68} & 0.35$\pm$0.05 & 1.62$\pm$0.33 & \textbf{14.97$\pm$3.38} & 0.45$\pm$0.14 & 77.00$\pm$0.91 \\
\midrule
\multirow{3}{*}{(PN,Pconf)} & WLMC & 1.45$\pm$0.49 & \textbf{7.24$\pm$0.88} & 0.51$\pm$0.19 & \textbf{1.89$\pm$0.34} & \textbf{22.17$\pm$10.19} & \textbf{0.49$\pm$0.17} & \textbf{77.02$\pm$0.90} \\
 & Temp & \textbf{1.36$\pm$0.45} & 7.37$\pm$1.18 & \textbf{0.29$\pm$0.04} & 1.99$\pm$0.26 & 26.44$\pm$13.81 & 0.50$\pm$0.30 & 77.02$\pm$0.90 \\
 & Platt & 2.28$\pm$0.62 & 7.28$\pm$1.57 & 0.45$\pm$0.08 & 2.18$\pm$0.51 & 24.30$\pm$12.50 & 0.58$\pm$0.20 & 76.90$\pm$0.94 \\
\midrule
\multirow{3}{*}{(PN,UU)} & WLMC & \textbf{0.93$\pm$0.30} & 7.62$\pm$1.34 & \textbf{0.31$\pm$0.04} & 1.19$\pm$0.29 & 11.83$\pm$2.74 & \textbf{0.33$\pm$0.02} & \textbf{77.02$\pm$0.90} \\
 & Temp & 1.00$\pm$0.31 & 7.62$\pm$1.29 & 0.32$\pm$0.04 & 1.23$\pm$0.33 & 11.80$\pm$2.44 & 0.33$\pm$0.02 & 77.02$\pm$0.90 \\
 & Platt & 0.94$\pm$0.32 & \textbf{7.13$\pm$0.85} & 0.36$\pm$0.04 & \textbf{1.16$\pm$0.37} & \textbf{11.39$\pm$2.12} & 0.38$\pm$0.05 & 77.01$\pm$0.91 \\
\midrule
\multirow{4}{*}{(PU,PU)} & ERM & 4.52$\pm$1.31 & 14.59$\pm$1.08 & 1.28$\pm$0.22 & 4.75$\pm$1.19 & 20.16$\pm$2.48 & 1.30$\pm$0.23 & 73.63$\pm$0.72 \\
 & WLMC & \textbf{2.28$\pm$0.55} & \textbf{10.42$\pm$1.57} & \textbf{0.48$\pm$0.08} & \textbf{2.58$\pm$0.50} & \textbf{17.50$\pm$4.29} & \textbf{0.54$\pm$0.07} & \textbf{73.78$\pm$0.54} \\
 & Temp & 4.32$\pm$1.45 & 14.22$\pm$1.40 & 1.17$\pm$0.26 & 4.61$\pm$1.44 & 19.72$\pm$2.32 & 1.21$\pm$0.24 & 73.63$\pm$0.72 \\
 & Platt & 3.82$\pm$1.15 & 12.30$\pm$0.85 & 1.13$\pm$0.30 & 4.14$\pm$1.14 & 18.05$\pm$2.38 & 1.17$\pm$0.29 & 73.69$\pm$0.67 \\
\midrule
\multirow{4}{*}{(Pconf,Pconf)} & ERM & 22.68$\pm$0.21 & 35.18$\pm$0.36 & 4.32$\pm$0.06 & 27.39$\pm$2.26 & 235.04$\pm$83.69 & 6.20$\pm$0.94 & 56.20$\pm$0.61 \\
 & WLMC & 6.24$\pm$0.78 & 14.04$\pm$1.11 & 0.69$\pm$0.17 & 7.56$\pm$1.01 & \textbf{34.59$\pm$24.13} & \textbf{1.19$\pm$0.81} & 72.91$\pm$0.67 \\
 & Temp & 19.84$\pm$1.09 & 30.81$\pm$0.47 & 7.22$\pm$0.38 & 23.74$\pm$2.26 & 232.29$\pm$83.48 & 9.83$\pm$1.82 & 56.20$\pm$0.61 \\
 & Platt & \textbf{3.04$\pm$0.57} & \textbf{9.63$\pm$0.85} & \textbf{0.62$\pm$0.07} & \textbf{3.53$\pm$0.64} & 64.39$\pm$30.06 & 1.45$\pm$0.50 & \textbf{74.60$\pm$0.41} \\
\midrule
\multirow{4}{*}{(UU,UU)} & ERM & 3.48$\pm$0.97 & 13.90$\pm$0.82 & 1.10$\pm$0.19 & 3.52$\pm$1.04 & 17.35$\pm$2.38 & 1.10$\pm$0.20 & 74.16$\pm$0.49 \\
 & WLMC & \textbf{1.94$\pm$0.42} & \textbf{10.46$\pm$0.89} & \textbf{0.48$\pm$0.07} & \textbf{2.10$\pm$0.45} & \textbf{13.90$\pm$2.54} & \textbf{0.48$\pm$0.06} & \textbf{74.19$\pm$0.39} \\
 & Temp & 3.21$\pm$1.04 & 13.14$\pm$1.01 & 0.96$\pm$0.20 & 3.25$\pm$1.10 & 17.04$\pm$1.93 & 0.97$\pm$0.21 & 74.16$\pm$0.49 \\
 & Platt & 2.92$\pm$0.84 & 11.75$\pm$0.47 & 0.91$\pm$0.21 & 2.95$\pm$0.90 & 15.54$\pm$1.89 & 0.91$\pm$0.21 & 74.15$\pm$0.46 \\
\bottomrule
\end{tabular}%
}
\end{table*}

\begin{table*}[t]
\centering
\scriptsize
\setlength{\tabcolsep}{3.2pt}
\renewcommand{\arraystretch}{0.92}
\caption{Tabular test results on ACSIncome with NaiveBayes as the base predictor. Calibration metrics are reported as $10^2\!\times$ mean $\pm$ standard deviation over ten seeds; accuracy is reported in percent. WLMC uses step size $0.05$ and at most $50$ updates. Bold marks the single best mean within each $(\mathrm{base},\mathrm{post})$ block; variance is not used for bolding.}
\label{tab:tabular-acsincome-naivebayes-main}
\resizebox{\textwidth}{!}{%
\begin{tabular}{llccccccc}
\toprule
\multirow{2}{*}{(Base,Post)} & \multirow{2}{*}{Post-processing} & \multicolumn{3}{c}{oracle estimate ($\times 10^{-2}$)} & \multicolumn{3}{c}{weak estimate ($\times 10^{-2}$)} & \multirow{2}{*}{Acc. (\%)} \\
\cmidrule(lr){3-5}\cmidrule(lr){6-8}
& & ECE & maxECE & MC & ECE & maxECE & MC & \\
\midrule
\multirow{4}{*}{(PN,PN)} & ERM & 11.73$\pm$0.10 & 16.62$\pm$0.18 & 1.90$\pm$0.01 & 11.73$\pm$0.10 & 16.62$\pm$0.18 & 1.90$\pm$0.01 & 75.34$\pm$0.05 \\
 & WLMC & 6.31$\pm$0.21 & 12.17$\pm$0.17 & \textbf{0.46$\pm$0.04} & 6.31$\pm$0.21 & 12.17$\pm$0.16 & \textbf{0.46$\pm$0.04} & \textbf{76.67$\pm$0.05} \\
 & Temp & 7.76$\pm$0.13 & 15.03$\pm$0.18 & 1.05$\pm$0.03 & 7.76$\pm$0.13 & 15.03$\pm$0.18 & 1.05$\pm$0.03 & 75.34$\pm$0.05 \\
 & Platt & \textbf{3.32$\pm$0.07} & \textbf{10.02$\pm$0.48} & 0.63$\pm$0.01 & \textbf{3.32$\pm$0.07} & \textbf{10.02$\pm$0.48} & 0.63$\pm$0.01 & 76.52$\pm$0.05 \\
\midrule
\multirow{3}{*}{(PN,PU)} & WLMC & 6.37$\pm$0.57 & 12.56$\pm$0.53 & \textbf{0.51$\pm$0.17} & \textbf{6.25$\pm$0.53} & 19.64$\pm$3.07 & \textbf{0.58$\pm$0.24} & 76.55$\pm$0.34 \\
 & Temp & 7.69$\pm$0.21 & 14.66$\pm$0.31 & 0.97$\pm$0.03 & 7.84$\pm$0.70 & 20.14$\pm$3.10 & 1.04$\pm$0.08 & 75.34$\pm$0.05 \\
 & Platt & \textbf{6.14$\pm$0.18} & \textbf{10.38$\pm$0.31} & 0.73$\pm$0.02 & 6.35$\pm$0.35 & \textbf{16.24$\pm$2.52} & 0.76$\pm$0.05 & \textbf{76.65$\pm$0.04} \\
\midrule
\multirow{3}{*}{(PN,Pconf)} & WLMC & 7.18$\pm$0.46 & 11.34$\pm$0.71 & 1.02$\pm$0.31 & 7.61$\pm$0.39 & \textbf{39.25$\pm$11.25} & \textbf{0.80$\pm$0.20} & 76.52$\pm$0.31 \\
 & Temp & 7.40$\pm$0.19 & 14.44$\pm$0.28 & \textbf{0.92$\pm$0.03} & 8.38$\pm$0.38 & 84.16$\pm$30.68 & 1.41$\pm$0.28 & 75.34$\pm$0.05 \\
 & Platt & \textbf{5.48$\pm$0.38} & \textbf{9.94$\pm$0.38} & 0.95$\pm$0.13 & \textbf{5.79$\pm$0.51} & 60.05$\pm$21.74 & 1.01$\pm$0.12 & \textbf{76.80$\pm$0.04} \\
\midrule
\multirow{3}{*}{(PN,UU)} & WLMC & 6.26$\pm$0.21 & 12.19$\pm$0.29 & \textbf{0.45$\pm$0.06} & 6.23$\pm$0.21 & 15.94$\pm$1.91 & \textbf{0.47$\pm$0.06} & 76.62$\pm$0.10 \\
 & Temp & 7.72$\pm$0.14 & 14.73$\pm$0.21 & 0.98$\pm$0.02 & 7.64$\pm$0.29 & 18.38$\pm$1.95 & 1.00$\pm$0.06 & 75.34$\pm$0.05 \\
 & Platt & \textbf{6.19$\pm$0.13} & \textbf{10.23$\pm$0.30} & 0.72$\pm$0.02 & \textbf{6.13$\pm$0.28} & \textbf{14.41$\pm$1.86} & 0.75$\pm$0.04 & \textbf{76.66$\pm$0.05} \\
\midrule
\multirow{4}{*}{(PU,PU)} & ERM & 11.83$\pm$0.33 & 16.60$\pm$0.58 & 1.91$\pm$0.02 & 11.66$\pm$0.41 & 20.98$\pm$2.63 & 1.84$\pm$0.06 & 75.31$\pm$0.14 \\
 & WLMC & 6.53$\pm$0.62 & 12.52$\pm$0.57 & \textbf{0.50$\pm$0.19} & \textbf{6.41$\pm$0.74} & 18.82$\pm$2.89 & \textbf{0.59$\pm$0.24} & 76.48$\pm$0.40 \\
 & Temp & 7.89$\pm$0.43 & 14.64$\pm$0.49 & 0.99$\pm$0.06 & 8.04$\pm$0.73 & 20.00$\pm$2.98 & 1.06$\pm$0.08 & 75.31$\pm$0.14 \\
 & Platt & \textbf{6.36$\pm$0.37} & \textbf{10.37$\pm$0.34} & 0.77$\pm$0.05 & 6.55$\pm$0.60 & \textbf{15.74$\pm$2.38} & 0.79$\pm$0.09 & \textbf{76.66$\pm$0.08} \\
\midrule
\multirow{4}{*}{(Pconf,Pconf)} & ERM & 17.32$\pm$0.25 & 27.74$\pm$0.43 & 2.40$\pm$0.07 & 18.17$\pm$0.96 & 147.43$\pm$49.19 & 2.67$\pm$0.28 & 68.76$\pm$0.22 \\
 & WLMC & 8.50$\pm$0.53 & 13.25$\pm$1.08 & 1.21$\pm$0.32 & 8.47$\pm$0.70 & \textbf{47.96$\pm$16.96} & \textbf{1.12$\pm$0.22} & 74.75$\pm$0.60 \\
 & Temp & 13.44$\pm$0.19 & 24.02$\pm$0.33 & 2.44$\pm$0.06 & 14.61$\pm$1.38 & 183.88$\pm$65.84 & 3.00$\pm$0.72 & 68.76$\pm$0.22 \\
 & Platt & \textbf{6.41$\pm$0.36} & \textbf{9.84$\pm$0.40} & \textbf{0.82$\pm$0.06} & \textbf{6.16$\pm$0.52} & 83.53$\pm$29.66 & 1.25$\pm$0.42 & \textbf{76.33$\pm$0.11} \\
\midrule
\multirow{4}{*}{(UU,UU)} & ERM & 11.72$\pm$0.15 & 16.67$\pm$0.33 & 1.90$\pm$0.02 & 11.68$\pm$0.19 & 18.99$\pm$1.70 & 1.90$\pm$0.10 & 75.29$\pm$0.08 \\
 & WLMC & \textbf{6.16$\pm$0.31} & 12.08$\pm$0.49 & \textbf{0.42$\pm$0.08} & 6.12$\pm$0.27 & 15.72$\pm$2.02 & \textbf{0.44$\pm$0.08} & 76.63$\pm$0.13 \\
 & Temp & 7.73$\pm$0.21 & 14.80$\pm$0.27 & 0.98$\pm$0.03 & 7.64$\pm$0.34 & 18.40$\pm$1.93 & 1.00$\pm$0.06 & 75.29$\pm$0.08 \\
 & Platt & 6.17$\pm$0.21 & \textbf{10.29$\pm$0.27} & 0.74$\pm$0.03 & \textbf{6.09$\pm$0.42} & \textbf{14.41$\pm$2.03} & 0.77$\pm$0.05 & \textbf{76.64$\pm$0.06} \\
\bottomrule
\end{tabular}%
}
\end{table*}

\begin{table*}[t]
\centering
\scriptsize
\setlength{\tabcolsep}{3.2pt}
\renewcommand{\arraystretch}{0.92}
\caption{Tabular test results on ACSIncome with RandomForest as the base predictor. Calibration metrics are reported as $10^2\!\times$ mean $\pm$ standard deviation over ten seeds; accuracy is reported in percent. WLMC uses step size $0.05$ and at most $50$ updates. Bold marks the single best mean within each $(\mathrm{base},\mathrm{post})$ block; variance is not used for bolding.}
\label{tab:tabular-acsincome-randomforest-main}
\resizebox{\textwidth}{!}{%
\begin{tabular}{llccccccc}
\toprule
\multirow{2}{*}{(Base,Post)} & \multirow{2}{*}{Post-processing} & \multicolumn{3}{c}{oracle estimate ($\times 10^{-2}$)} & \multicolumn{3}{c}{weak estimate ($\times 10^{-2}$)} & \multirow{2}{*}{Acc. (\%)} \\
\cmidrule(lr){3-5}\cmidrule(lr){6-8}
& & ECE & maxECE & MC & ECE & maxECE & MC & \\
\midrule
\multirow{4}{*}{(PN,PN)} & ERM & \textbf{0.96$\pm$0.10} & \textbf{5.39$\pm$0.66} & \textbf{0.16$\pm$0.02} & \textbf{0.96$\pm$0.10} & 5.39$\pm$0.66 & \textbf{0.16$\pm$0.02} & \textbf{81.70$\pm$0.08} \\
 & WLMC & 0.96$\pm$0.10 & 5.39$\pm$0.66 & 0.16$\pm$0.02 & 0.96$\pm$0.10 & \textbf{5.39$\pm$0.66} & 0.16$\pm$0.02 & 81.70$\pm$0.08 \\
 & Temp & 2.64$\pm$0.09 & 6.34$\pm$0.67 & 0.32$\pm$0.03 & 2.64$\pm$0.09 & 6.34$\pm$0.67 & 0.32$\pm$0.03 & 81.70$\pm$0.08 \\
 & Platt & 2.77$\pm$0.15 & 5.88$\pm$0.37 & 0.54$\pm$0.03 & 2.77$\pm$0.15 & 5.88$\pm$0.37 & 0.54$\pm$0.03 & 81.69$\pm$0.07 \\
\midrule
\multirow{3}{*}{(PN,PU)} & WLMC & 0.96$\pm$0.10 & \textbf{5.39$\pm$0.66} & 0.16$\pm$0.02 & 1.56$\pm$0.28 & \textbf{13.88$\pm$2.94} & \textbf{0.25$\pm$0.05} & \textbf{81.70$\pm$0.08} \\
 & Temp & \textbf{0.87$\pm$0.18} & 5.56$\pm$1.00 & 0.13$\pm$0.04 & 1.53$\pm$0.35 & 14.79$\pm$3.04 & 0.27$\pm$0.06 & 81.70$\pm$0.08 \\
 & Platt & 0.88$\pm$0.18 & 5.53$\pm$0.95 & \textbf{0.13$\pm$0.04} & \textbf{1.50$\pm$0.36} & 14.76$\pm$3.00 & 0.26$\pm$0.06 & 81.70$\pm$0.10 \\
\midrule
\multirow{3}{*}{(PN,Pconf)} & WLMC & 0.96$\pm$0.10 & 5.39$\pm$0.66 & 0.16$\pm$0.02 & 0.26$\pm$0.05 & 1.74$\pm$0.82 & 0.06$\pm$0.04 & 81.70$\pm$0.08 \\
 & Temp & \textbf{0.96$\pm$0.15} & 5.40$\pm$0.67 & 0.16$\pm$0.03 & 0.26$\pm$0.05 & 1.73$\pm$0.82 & 0.06$\pm$0.04 & 81.70$\pm$0.08 \\
 & Platt & 0.97$\pm$0.13 & \textbf{5.28$\pm$0.68} & \textbf{0.16$\pm$0.03} & \textbf{0.19$\pm$0.04} & \textbf{1.65$\pm$0.77} & \textbf{0.05$\pm$0.03} & \textbf{81.70$\pm$0.08} \\
\midrule
\multirow{3}{*}{(PN,UU)} & WLMC & 0.96$\pm$0.10 & \textbf{5.39$\pm$0.66} & 0.16$\pm$0.02 & 1.19$\pm$0.26 & \textbf{10.22$\pm$2.26} & 0.20$\pm$0.03 & \textbf{81.70$\pm$0.08} \\
 & Temp & \textbf{0.74$\pm$0.13} & 5.51$\pm$0.85 & \textbf{0.13$\pm$0.02} & \textbf{1.02$\pm$0.30} & 10.52$\pm$1.74 & 0.17$\pm$0.03 & 81.70$\pm$0.08 \\
 & Platt & 0.79$\pm$0.15 & 5.50$\pm$0.98 & 0.13$\pm$0.03 & 1.12$\pm$0.33 & 10.59$\pm$1.72 & \textbf{0.16$\pm$0.04} & 81.69$\pm$0.10 \\
\midrule
\multirow{4}{*}{(PU,PU)} & ERM & 20.17$\pm$5.56 & 30.07$\pm$5.87 & 5.95$\pm$4.25 & 20.23$\pm$5.57 & 33.60$\pm$6.94 & 5.99$\pm$4.15 & 73.34$\pm$3.89 \\
 & WLMC & \textbf{9.09$\pm$0.67} & 17.67$\pm$3.11 & \textbf{0.53$\pm$0.08} & \textbf{9.19$\pm$0.85} & 21.61$\pm$2.48 & \textbf{0.62$\pm$0.16} & \textbf{73.78$\pm$3.29} \\
 & Temp & 17.04$\pm$1.98 & 25.54$\pm$3.23 & 6.22$\pm$1.26 & 16.92$\pm$1.86 & 29.18$\pm$3.86 & 6.24$\pm$1.24 & 73.34$\pm$3.89 \\
 & Platt & 13.35$\pm$3.29 & \textbf{17.60$\pm$2.59} & 3.00$\pm$1.00 & 13.24$\pm$3.27 & \textbf{21.33$\pm$3.76} & 2.99$\pm$0.97 & 68.44$\pm$4.07 \\
\midrule
\multirow{4}{*}{(Pconf,Pconf)} & ERM & 17.19$\pm$0.17 & 24.86$\pm$0.52 & 4.01$\pm$0.15 & 19.83$\pm$1.60 & 154.68$\pm$60.84 & 5.65$\pm$0.94 & 77.38$\pm$0.16 \\
 & WLMC & 6.73$\pm$0.58 & 12.17$\pm$1.33 & 0.61$\pm$0.14 & 7.24$\pm$0.70 & \textbf{16.64$\pm$6.70} & \textbf{0.82$\pm$0.29} & 79.44$\pm$0.34 \\
 & Temp & 17.02$\pm$0.20 & 23.64$\pm$0.30 & 2.10$\pm$0.18 & 19.39$\pm$1.38 & 104.28$\pm$46.53 & 3.21$\pm$0.50 & 77.38$\pm$0.16 \\
 & Platt & \textbf{4.14$\pm$0.16} & \textbf{8.59$\pm$0.57} & \textbf{0.48$\pm$0.10} & \textbf{4.17$\pm$0.43} & 28.59$\pm$14.99 & 0.97$\pm$0.19 & \textbf{80.63$\pm$0.11} \\
\midrule
\multirow{4}{*}{(UU,UU)} & ERM & 10.99$\pm$0.52 & 17.06$\pm$1.02 & 2.24$\pm$0.18 & 11.07$\pm$0.60 & 20.74$\pm$1.72 & 2.30$\pm$0.24 & 78.18$\pm$0.38 \\
 & WLMC & \textbf{6.98$\pm$0.42} & \textbf{13.28$\pm$1.12} & \textbf{0.49$\pm$0.05} & \textbf{7.06$\pm$0.48} & \textbf{17.11$\pm$1.62} & \textbf{0.51$\pm$0.08} & \textbf{78.40$\pm$0.38} \\
 & Temp & 15.48$\pm$0.42 & 21.53$\pm$1.38 & 3.04$\pm$0.16 & 15.37$\pm$0.52 & 25.39$\pm$3.75 & 2.99$\pm$0.15 & 78.18$\pm$0.38 \\
 & Platt & 12.68$\pm$0.62 & 14.83$\pm$1.02 & 1.90$\pm$0.09 & 12.69$\pm$0.68 & 17.17$\pm$2.95 & 1.90$\pm$0.12 & 72.10$\pm$0.64 \\
\bottomrule
\end{tabular}%
}
\end{table*}

\begin{table*}[t]
\centering
\scriptsize
\setlength{\tabcolsep}{3.2pt}
\renewcommand{\arraystretch}{0.92}
\caption{Tabular test results on ACSIncome with SVM as the base predictor. Calibration metrics are reported as $10^2\!\times$ mean $\pm$ standard deviation over ten seeds; accuracy is reported in percent. WLMC uses step size $0.05$ and at most $50$ updates. Bold marks the single best mean within each $(\mathrm{base},\mathrm{post})$ block; variance is not used for bolding.}
\label{tab:tabular-acsincome-svm-main}
\resizebox{\textwidth}{!}{%
\begin{tabular}{llccccccc}
\toprule
\multirow{2}{*}{(Base,Post)} & \multirow{2}{*}{Post-processing} & \multicolumn{3}{c}{oracle estimate ($\times 10^{-2}$)} & \multicolumn{3}{c}{weak estimate ($\times 10^{-2}$)} & \multirow{2}{*}{Acc. (\%)} \\
\cmidrule(lr){3-5}\cmidrule(lr){6-8}
& & ECE & maxECE & MC & ECE & maxECE & MC & \\
\midrule
\multirow{4}{*}{(PN,PN)} & ERM & 3.12$\pm$0.43 & 8.16$\pm$0.57 & 0.50$\pm$0.05 & 3.12$\pm$0.43 & 8.15$\pm$0.57 & 0.50$\pm$0.05 & 78.09$\pm$0.19 \\
 & WLMC & 2.74$\pm$0.33 & 7.59$\pm$0.44 & \textbf{0.44$\pm$0.05} & 2.75$\pm$0.33 & 7.59$\pm$0.44 & \textbf{0.44$\pm$0.05} & 78.09$\pm$0.19 \\
 & Temp & 10.89$\pm$9.50 & 14.83$\pm$9.85 & 2.82$\pm$4.04 & 10.89$\pm$9.50 & 14.83$\pm$9.85 & 2.82$\pm$4.04 & 66.88$\pm$23.67 \\
 & Platt & \textbf{1.62$\pm$0.11} & \textbf{7.13$\pm$0.43} & 0.48$\pm$0.03 & \textbf{1.61$\pm$0.11} & \textbf{7.13$\pm$0.43} & 0.48$\pm$0.03 & \textbf{78.17$\pm$0.20} \\
\midrule
\multirow{3}{*}{(PN,PU)} & WLMC & 2.62$\pm$0.40 & 7.66$\pm$0.51 & 0.42$\pm$0.08 & 2.98$\pm$0.70 & 15.35$\pm$3.21 & 0.49$\pm$0.09 & 78.09$\pm$0.19 \\
 & Temp & 1.63$\pm$0.50 & 7.24$\pm$0.80 & 0.31$\pm$0.07 & 1.98$\pm$0.36 & 15.81$\pm$3.18 & \textbf{0.39$\pm$0.11} & 78.09$\pm$0.19 \\
 & Platt & \textbf{1.25$\pm$0.30} & \textbf{6.31$\pm$0.54} & \textbf{0.29$\pm$0.04} & \textbf{1.74$\pm$0.37} & \textbf{15.26$\pm$3.28} & 0.39$\pm$0.10 & \textbf{78.16$\pm$0.21} \\
\midrule
\multirow{3}{*}{(PN,Pconf)} & WLMC & 2.86$\pm$0.38 & 7.35$\pm$0.62 & 0.59$\pm$0.22 & 3.01$\pm$0.73 & \textbf{13.30$\pm$8.20} & 0.63$\pm$0.13 & 78.09$\pm$0.19 \\
 & Temp & \textbf{1.58$\pm$0.56} & 7.28$\pm$0.56 & \textbf{0.25$\pm$0.07} & 2.34$\pm$0.91 & 24.01$\pm$11.34 & 0.62$\pm$0.29 & 78.09$\pm$0.19 \\
 & Platt & 2.05$\pm$0.62 & \textbf{6.14$\pm$0.82} & 0.31$\pm$0.04 & \textbf{1.82$\pm$0.58} & 21.20$\pm$8.97 & \textbf{0.58$\pm$0.27} & \textbf{78.14$\pm$0.18} \\
\midrule
\multirow{3}{*}{(PN,UU)} & WLMC & 2.73$\pm$0.41 & 7.62$\pm$0.52 & 0.44$\pm$0.05 & 2.71$\pm$0.40 & 12.24$\pm$2.94 & 0.45$\pm$0.05 & 78.09$\pm$0.19 \\
 & Temp & 1.69$\pm$0.52 & 7.42$\pm$0.73 & 0.33$\pm$0.05 & 1.82$\pm$0.57 & 11.49$\pm$2.17 & 0.37$\pm$0.07 & 78.09$\pm$0.19 \\
 & Platt & \textbf{1.40$\pm$0.28} & \textbf{6.38$\pm$0.44} & \textbf{0.31$\pm$0.03} & \textbf{1.58$\pm$0.33} & \textbf{10.77$\pm$2.16} & \textbf{0.33$\pm$0.06} & \textbf{78.16$\pm$0.20} \\
\midrule
\multirow{4}{*}{(PU,PU)} & ERM & 3.84$\pm$1.21 & 8.03$\pm$1.12 & 0.58$\pm$0.17 & 4.02$\pm$1.43 & 15.06$\pm$3.55 & 0.64$\pm$0.21 & 78.00$\pm$0.39 \\
 & WLMC & 3.18$\pm$0.66 & 7.30$\pm$1.09 & 0.46$\pm$0.05 & 3.38$\pm$0.99 & \textbf{14.33$\pm$2.75} & 0.51$\pm$0.08 & 78.09$\pm$0.24 \\
 & Temp & 2.06$\pm$0.98 & 7.53$\pm$1.46 & 0.39$\pm$0.20 & 2.45$\pm$0.94 & 15.12$\pm$2.48 & 0.46$\pm$0.22 & 78.00$\pm$0.39 \\
 & Platt & \textbf{1.54$\pm$0.42} & \textbf{6.74$\pm$0.71} & \textbf{0.32$\pm$0.11} & \textbf{1.98$\pm$0.44} & 14.99$\pm$2.72 & \textbf{0.39$\pm$0.13} & \textbf{78.15$\pm$0.24} \\
\midrule
\multirow{4}{*}{(Pconf,Pconf)} & ERM & 23.95$\pm$2.03 & 34.61$\pm$3.55 & 5.35$\pm$1.91 & 27.64$\pm$3.42 & 221.51$\pm$74.60 & 6.32$\pm$1.82 & 55.22$\pm$6.60 \\
 & WLMC & 6.78$\pm$0.75 & 12.34$\pm$0.99 & 0.69$\pm$0.21 & 7.26$\pm$1.00 & \textbf{23.98$\pm$8.58} & 0.82$\pm$0.33 & 73.06$\pm$1.27 \\
 & Temp & 18.16$\pm$5.11 & 29.49$\pm$1.65 & 6.37$\pm$1.69 & 21.73$\pm$4.84 & 218.11$\pm$81.02 & 8.33$\pm$2.93 & 55.22$\pm$6.60 \\
 & Platt & \textbf{2.51$\pm$0.79} & \textbf{8.07$\pm$1.77} & \textbf{0.40$\pm$0.11} & \textbf{2.16$\pm$0.70} & 25.06$\pm$9.00 & \textbf{0.55$\pm$0.13} & \textbf{77.18$\pm$0.87} \\
\midrule
\multirow{4}{*}{(UU,UU)} & ERM & 4.16$\pm$0.83 & 7.95$\pm$0.82 & 0.57$\pm$0.08 & 4.21$\pm$0.77 & 12.80$\pm$2.16 & 0.57$\pm$0.11 & \textbf{78.31$\pm$0.11} \\
 & WLMC & 3.69$\pm$0.51 & 7.44$\pm$0.66 & 0.50$\pm$0.07 & 3.76$\pm$0.46 & 12.14$\pm$1.95 & 0.50$\pm$0.08 & 78.31$\pm$0.11 \\
 & Temp & 1.77$\pm$0.29 & 6.58$\pm$0.89 & 0.36$\pm$0.06 & 1.91$\pm$0.29 & 10.87$\pm$1.85 & 0.37$\pm$0.06 & 78.31$\pm$0.11 \\
 & Platt & \textbf{1.56$\pm$0.22} & \textbf{6.44$\pm$0.77} & \textbf{0.34$\pm$0.07} & \textbf{1.68$\pm$0.25} & \textbf{10.80$\pm$1.91} & \textbf{0.35$\pm$0.06} & 78.30$\pm$0.11 \\
\bottomrule
\end{tabular}%
}
\end{table*}

\begin{table*}[t]
\centering
\scriptsize
\setlength{\tabcolsep}{3.2pt}
\renewcommand{\arraystretch}{0.92}
\caption{Tabular test results on ACSIncome with MLP as the base predictor. Calibration metrics are reported as $10^2\!\times$ mean $\pm$ standard deviation over ten seeds; accuracy is reported in percent. WLMC uses step size $0.05$ and at most $50$ updates. Bold marks the single best mean within each $(\mathrm{base},\mathrm{post})$ block; variance is not used for bolding.}
\label{tab:tabular-acsincome-mlp-main}
\resizebox{\textwidth}{!}{%
\begin{tabular}{llccccccc}
\toprule
\multirow{2}{*}{(Base,Post)} & \multirow{2}{*}{Post-processing} & \multicolumn{3}{c}{oracle estimate ($\times 10^{-2}$)} & \multicolumn{3}{c}{weak estimate ($\times 10^{-2}$)} & \multirow{2}{*}{Acc. (\%)} \\
\cmidrule(lr){3-5}\cmidrule(lr){6-8}
& & ECE & maxECE & MC & ECE & maxECE & MC & \\
\midrule
\multirow{4}{*}{(PN,PN)} & ERM & \textbf{1.30$\pm$0.41} & \textbf{6.31$\pm$1.09} & 0.18$\pm$0.06 & 1.30$\pm$0.41 & 6.31$\pm$1.09 & 0.18$\pm$0.06 & 80.74$\pm$0.12 \\
 & WLMC & 1.30$\pm$0.41 & 6.31$\pm$1.09 & \textbf{0.18$\pm$0.06} & \textbf{1.30$\pm$0.41} & \textbf{6.31$\pm$1.09} & \textbf{0.18$\pm$0.06} & 80.74$\pm$0.12 \\
 & Temp & 2.86$\pm$0.97 & 6.82$\pm$1.04 & 0.37$\pm$0.07 & 2.86$\pm$0.97 & 6.82$\pm$1.04 & 0.37$\pm$0.07 & 80.74$\pm$0.12 \\
 & Platt & 2.95$\pm$0.37 & 6.95$\pm$0.74 & 0.58$\pm$0.09 & 2.95$\pm$0.37 & 6.95$\pm$0.74 & 0.58$\pm$0.09 & \textbf{80.74$\pm$0.13} \\
\midrule
\multirow{3}{*}{(PN,PU)} & WLMC & 1.30$\pm$0.41 & \textbf{6.31$\pm$1.09} & 0.18$\pm$0.06 & \textbf{1.55$\pm$0.31} & \textbf{14.67$\pm$3.98} & \textbf{0.24$\pm$0.05} & 80.74$\pm$0.12 \\
 & Temp & 1.10$\pm$0.22 & 6.37$\pm$1.12 & 0.17$\pm$0.04 & 1.61$\pm$0.36 & 15.56$\pm$4.11 & 0.25$\pm$0.08 & 80.74$\pm$0.12 \\
 & Platt & \textbf{0.91$\pm$0.09} & 6.51$\pm$1.04 & \textbf{0.14$\pm$0.04} & 1.56$\pm$0.41 & 15.60$\pm$4.04 & 0.25$\pm$0.07 & \textbf{80.75$\pm$0.13} \\
\midrule
\multirow{3}{*}{(PN,Pconf)} & WLMC & 1.34$\pm$0.43 & 6.37$\pm$1.02 & 0.21$\pm$0.10 & 1.43$\pm$0.61 & 5.86$\pm$2.78 & 0.24$\pm$0.10 & 80.74$\pm$0.12 \\
 & Temp & \textbf{1.11$\pm$0.21} & 6.32$\pm$1.13 & 0.19$\pm$0.04 & 1.25$\pm$0.53 & 5.75$\pm$2.50 & 0.22$\pm$0.05 & 80.74$\pm$0.12 \\
 & Platt & 1.40$\pm$0.13 & \textbf{5.94$\pm$0.95} & \textbf{0.16$\pm$0.03} & \textbf{0.77$\pm$0.31} & \textbf{4.77$\pm$2.46} & \textbf{0.20$\pm$0.05} & \textbf{80.74$\pm$0.15} \\
\midrule
\multirow{3}{*}{(PN,UU)} & WLMC & 1.30$\pm$0.41 & 6.31$\pm$1.09 & 0.18$\pm$0.06 & 1.45$\pm$0.33 & 11.40$\pm$2.31 & 0.22$\pm$0.07 & 80.74$\pm$0.12 \\
 & Temp & 1.07$\pm$0.18 & \textbf{6.26$\pm$0.92} & 0.16$\pm$0.05 & 1.26$\pm$0.25 & \textbf{11.04$\pm$1.89} & 0.22$\pm$0.08 & 80.74$\pm$0.12 \\
 & Platt & \textbf{0.85$\pm$0.17} & 6.51$\pm$0.92 & \textbf{0.14$\pm$0.03} & \textbf{1.22$\pm$0.24} & 11.30$\pm$1.97 & \textbf{0.20$\pm$0.06} & \textbf{80.74$\pm$0.13} \\
\midrule
\multirow{4}{*}{(PU,PU)} & ERM & 4.86$\pm$3.51 & 11.08$\pm$3.37 & 0.88$\pm$1.05 & 5.16$\pm$3.34 & 18.27$\pm$3.97 & 0.95$\pm$0.98 & 77.90$\pm$1.01 \\
 & WLMC & \textbf{3.44$\pm$1.47} & 9.78$\pm$1.94 & \textbf{0.40$\pm$0.10} & \textbf{3.79$\pm$1.45} & 16.71$\pm$3.73 & \textbf{0.47$\pm$0.11} & \textbf{77.94$\pm$0.96} \\
 & Temp & 4.96$\pm$3.44 & 11.83$\pm$4.14 & 0.70$\pm$0.64 & 5.12$\pm$3.46 & 18.30$\pm$4.03 & 0.74$\pm$0.65 & 77.90$\pm$1.01 \\
 & Platt & 3.88$\pm$2.35 & \textbf{9.48$\pm$1.79} & 0.45$\pm$0.31 & 4.16$\pm$2.41 & \textbf{16.60$\pm$2.83} & 0.50$\pm$0.32 & 77.39$\pm$2.01 \\
\midrule
\multirow{4}{*}{(Pconf,Pconf)} & ERM & 12.22$\pm$1.76 & 20.03$\pm$2.10 & 2.98$\pm$0.76 & 14.97$\pm$2.45 & 121.09$\pm$47.31 & 4.27$\pm$1.04 & 78.50$\pm$0.42 \\
 & WLMC & \textbf{4.43$\pm$0.69} & 10.31$\pm$1.11 & \textbf{0.56$\pm$0.13} & 5.65$\pm$1.12 & \textbf{21.16$\pm$10.07} & \textbf{1.05$\pm$0.60} & \textbf{78.83$\pm$0.31} \\
 & Temp & 11.49$\pm$2.27 & 17.71$\pm$2.40 & 1.59$\pm$0.45 & 13.92$\pm$2.87 & 76.42$\pm$29.50 & 2.40$\pm$0.62 & 78.50$\pm$0.42 \\
 & Platt & 4.74$\pm$1.02 & \textbf{7.99$\pm$1.12} & 0.65$\pm$0.28 & \textbf{5.02$\pm$1.12} & 35.22$\pm$17.24 & 1.23$\pm$0.37 & 78.48$\pm$0.43 \\
\midrule
\multirow{4}{*}{(UU,UU)} & ERM & 3.10$\pm$1.59 & 8.47$\pm$2.06 & 0.44$\pm$0.25 & 3.20$\pm$1.59 & 12.79$\pm$3.22 & 0.46$\pm$0.25 & \textbf{78.85$\pm$0.45} \\
 & WLMC & 2.80$\pm$0.96 & 8.12$\pm$1.25 & 0.37$\pm$0.10 & 2.90$\pm$0.98 & 12.37$\pm$2.44 & 0.39$\pm$0.09 & 78.85$\pm$0.45 \\
 & Temp & 2.26$\pm$0.87 & 8.41$\pm$2.00 & 0.33$\pm$0.15 & 2.36$\pm$0.92 & 12.90$\pm$2.81 & 0.36$\pm$0.18 & 78.85$\pm$0.45 \\
 & Platt & \textbf{1.68$\pm$0.71} & \textbf{7.62$\pm$1.22} & \textbf{0.28$\pm$0.09} & \textbf{1.84$\pm$0.84} & \textbf{11.99$\pm$2.48} & \textbf{0.29$\pm$0.11} & 78.77$\pm$0.50 \\
\bottomrule
\end{tabular}%
}
\end{table*}

\begin{table*}[t]
\centering
\scriptsize
\setlength{\tabcolsep}{3.2pt}
\renewcommand{\arraystretch}{0.92}
\caption{Tabular test results on CreditDefault with DecisionTree as the base predictor. Calibration metrics are reported as $10^2\!\times$ mean $\pm$ standard deviation over ten seeds; accuracy is reported in percent. WLMC uses step size $0.05$ and at most $50$ updates. Bold marks the single best mean within each $(\mathrm{base},\mathrm{post})$ block; variance is not used for bolding.}
\label{tab:tabular-creditdefault-decisiontree-main}
\resizebox{\textwidth}{!}{%
\begin{tabular}{llccccccc}
\toprule
\multirow{2}{*}{(Base,Post)} & \multirow{2}{*}{Post-processing} & \multicolumn{3}{c}{oracle estimate ($\times 10^{-2}$)} & \multicolumn{3}{c}{weak estimate ($\times 10^{-2}$)} & \multirow{2}{*}{Acc. (\%)} \\
\cmidrule(lr){3-5}\cmidrule(lr){6-8}
& & ECE & maxECE & MC & ECE & maxECE & MC & \\
\midrule
\multirow{4}{*}{(PN,PN)} & ERM & 5.42$\pm$0.48 & 16.38$\pm$3.61 & 1.13$\pm$0.17 & 5.43$\pm$0.48 & 16.38$\pm$3.61 & 1.13$\pm$0.17 & \textbf{80.55$\pm$0.32} \\
 & WLMC & 3.47$\pm$0.35 & 14.98$\pm$4.39 & \textbf{0.52$\pm$0.10} & 3.47$\pm$0.34 & 14.97$\pm$4.39 & \textbf{0.52$\pm$0.10} & 80.54$\pm$0.24 \\
 & Temp & 12.20$\pm$0.61 & 17.03$\pm$2.18 & 3.34$\pm$0.70 & 12.20$\pm$0.61 & 17.03$\pm$2.18 & 3.34$\pm$0.70 & 80.54$\pm$0.24 \\
 & Platt & \textbf{2.88$\pm$0.44} & \textbf{8.70$\pm$2.12} & 0.97$\pm$0.21 & \textbf{2.88$\pm$0.44} & \textbf{8.70$\pm$2.11} & 0.97$\pm$0.21 & 80.43$\pm$0.24 \\
\midrule
\multirow{3}{*}{(PN,PU)} & WLMC & \textbf{3.54$\pm$0.92} & 14.81$\pm$4.62 & \textbf{0.71$\pm$0.36} & \textbf{4.01$\pm$1.05} & 24.96$\pm$9.06 & \textbf{0.83$\pm$0.34} & \textbf{80.54$\pm$0.24} \\
 & Temp & 22.46$\pm$2.56 & 24.82$\pm$2.02 & 8.18$\pm$1.40 & 22.68$\pm$2.54 & 32.37$\pm$7.30 & 8.14$\pm$1.37 & 80.54$\pm$0.24 \\
 & Platt & 5.72$\pm$2.53 & \textbf{9.92$\pm$2.83} & 1.98$\pm$0.78 & 5.48$\pm$2.64 & \textbf{14.31$\pm$5.41} & 1.92$\pm$0.78 & 78.58$\pm$0.42 \\
\midrule
\multirow{3}{*}{(PN,Pconf)} & WLMC & \textbf{4.50$\pm$0.80} & 16.42$\pm$5.03 & \textbf{1.08$\pm$0.57} & 4.29$\pm$0.57 & 7.59$\pm$0.93 & \textbf{0.62$\pm$0.26} & \textbf{80.54$\pm$0.24} \\
 & Temp & 24.55$\pm$1.20 & 26.80$\pm$1.03 & 8.74$\pm$2.90 & 21.67$\pm$1.03 & 28.42$\pm$1.54 & 7.54$\pm$1.84 & 80.54$\pm$0.24 \\
 & Platt & 4.69$\pm$1.13 & \textbf{7.46$\pm$1.32} & 2.12$\pm$0.65 & \textbf{3.22$\pm$0.70} & \textbf{6.36$\pm$1.22} & 1.08$\pm$0.30 & 78.74$\pm$0.36 \\
\midrule
\multirow{3}{*}{(PN,UU)} & WLMC & \textbf{3.64$\pm$0.33} & 14.11$\pm$3.35 & \textbf{0.61$\pm$0.12} & \textbf{4.11$\pm$0.50} & 18.43$\pm$5.65 & \textbf{0.79$\pm$0.12} & \textbf{80.54$\pm$0.24} \\
 & Temp & 22.11$\pm$1.80 & 24.38$\pm$1.44 & 8.20$\pm$0.97 & 22.15$\pm$1.70 & 28.77$\pm$4.50 & 8.28$\pm$0.83 & 80.54$\pm$0.24 \\
 & Platt & 5.60$\pm$2.22 & \textbf{9.11$\pm$2.81} & 1.82$\pm$0.77 & 5.47$\pm$2.08 & \textbf{11.37$\pm$3.28} & 1.82$\pm$0.75 & 78.74$\pm$0.36 \\
\midrule
\multirow{4}{*}{(PU,PU)} & ERM & 8.75$\pm$1.26 & 17.29$\pm$3.92 & 2.05$\pm$0.26 & 9.22$\pm$1.45 & 25.99$\pm$6.70 & 2.21$\pm$0.43 & 79.28$\pm$0.57 \\
 & WLMC & \textbf{4.82$\pm$0.89} & 13.96$\pm$4.04 & \textbf{0.84$\pm$0.36} & \textbf{5.11$\pm$1.07} & 25.06$\pm$7.03 & \textbf{0.99$\pm$0.29} & \textbf{79.31$\pm$0.58} \\
 & Temp & 25.10$\pm$1.16 & 28.22$\pm$1.86 & 10.43$\pm$2.61 & 25.52$\pm$1.35 & 37.21$\pm$9.06 & 10.40$\pm$2.75 & 79.29$\pm$0.57 \\
 & Platt & 6.56$\pm$1.20 & \textbf{10.78$\pm$1.63} & 2.31$\pm$0.27 & 6.84$\pm$1.36 & \textbf{18.55$\pm$5.85} & 2.46$\pm$0.45 & 79.23$\pm$0.62 \\
\midrule
\multirow{4}{*}{(Pconf,Pconf)} & ERM & 4.39$\pm$0.31 & \textbf{10.15$\pm$1.69} & 1.25$\pm$0.12 & 3.88$\pm$0.22 & 5.04$\pm$0.47 & 0.70$\pm$0.07 & \textbf{81.53$\pm$0.20} \\
 & WLMC & 3.17$\pm$0.32 & 10.50$\pm$2.65 & 1.01$\pm$0.32 & 2.89$\pm$0.23 & 4.22$\pm$0.44 & 0.43$\pm$0.09 & 81.53$\pm$0.20 \\
 & Temp & 3.22$\pm$0.33 & 13.12$\pm$1.96 & 0.94$\pm$0.18 & 3.04$\pm$0.18 & 3.88$\pm$0.34 & 0.42$\pm$0.03 & 81.53$\pm$0.20 \\
 & Platt & \textbf{2.34$\pm$0.33} & 11.35$\pm$2.36 & \textbf{0.89$\pm$0.14} & \textbf{0.73$\pm$0.17} & \textbf{2.22$\pm$0.49} & \textbf{0.23$\pm$0.05} & 81.44$\pm$0.25 \\
\midrule
\multirow{4}{*}{(UU,UU)} & ERM & 4.84$\pm$0.91 & 14.76$\pm$3.34 & 1.28$\pm$0.26 & 5.43$\pm$0.95 & 19.57$\pm$5.39 & 1.34$\pm$0.26 & \textbf{80.41$\pm$0.67} \\
 & WLMC & \textbf{3.35$\pm$0.64} & 13.82$\pm$4.09 & \textbf{0.65$\pm$0.16} & \textbf{3.66$\pm$0.68} & 19.29$\pm$5.39 & \textbf{0.75$\pm$0.28} & 80.41$\pm$0.67 \\
 & Temp & 17.98$\pm$2.85 & 22.21$\pm$2.78 & 5.56$\pm$1.20 & 18.01$\pm$2.92 & 26.04$\pm$5.70 & 5.70$\pm$1.40 & 80.41$\pm$0.67 \\
 & Platt & 8.31$\pm$1.07 & \textbf{11.85$\pm$1.82} & 2.59$\pm$0.30 & 8.15$\pm$1.23 & \textbf{13.84$\pm$2.81} & 2.52$\pm$0.24 & 79.04$\pm$0.39 \\
\bottomrule
\end{tabular}%
}
\end{table*}

\begin{table*}[t]
\centering
\scriptsize
\setlength{\tabcolsep}{3.2pt}
\renewcommand{\arraystretch}{0.92}
\caption{Tabular test results on CreditDefault with LogisticRegression as the base predictor. Calibration metrics are reported as $10^2\!\times$ mean $\pm$ standard deviation over ten seeds; accuracy is reported in percent. WLMC uses step size $0.05$ and at most $50$ updates. Bold marks the single best mean within each $(\mathrm{base},\mathrm{post})$ block; variance is not used for bolding.}
\label{tab:tabular-creditdefault-logisticregression-main}
\resizebox{\textwidth}{!}{%
\begin{tabular}{llccccccc}
\toprule
\multirow{2}{*}{(Base,Post)} & \multirow{2}{*}{Post-processing} & \multicolumn{3}{c}{oracle estimate ($\times 10^{-2}$)} & \multicolumn{3}{c}{weak estimate ($\times 10^{-2}$)} & \multirow{2}{*}{Acc. (\%)} \\
\cmidrule(lr){3-5}\cmidrule(lr){6-8}
& & ECE & maxECE & MC & ECE & maxECE & MC & \\
\midrule
\multirow{4}{*}{(PN,PN)} & ERM & \textbf{1.23$\pm$0.22} & 11.52$\pm$3.41 & \textbf{0.23$\pm$0.04} & \textbf{1.23$\pm$0.22} & 11.52$\pm$3.41 & \textbf{0.23$\pm$0.04} & \textbf{81.84$\pm$0.09} \\
 & WLMC & 1.26$\pm$0.23 & 11.44$\pm$3.50 & 0.24$\pm$0.05 & 1.27$\pm$0.23 & 11.44$\pm$3.49 & 0.24$\pm$0.05 & 81.84$\pm$0.09 \\
 & Temp & 3.38$\pm$0.37 & 9.60$\pm$2.69 & 1.05$\pm$0.16 & 3.38$\pm$0.37 & 9.60$\pm$2.69 & 1.05$\pm$0.16 & 81.84$\pm$0.09 \\
 & Platt & 2.26$\pm$0.24 & \textbf{7.64$\pm$0.81} & 0.48$\pm$0.16 & 2.26$\pm$0.24 & \textbf{7.64$\pm$0.81} & 0.48$\pm$0.16 & 81.81$\pm$0.11 \\
\midrule
\multirow{3}{*}{(PN,PU)} & WLMC & \textbf{1.33$\pm$0.29} & 11.88$\pm$3.00 & 0.30$\pm$0.14 & 2.61$\pm$0.91 & 24.01$\pm$7.35 & 0.63$\pm$0.12 & \textbf{81.86$\pm$0.11} \\
 & Temp & 1.34$\pm$0.29 & \textbf{10.84$\pm$3.08} & 0.30$\pm$0.11 & 2.44$\pm$0.69 & \textbf{23.43$\pm$7.98} & \textbf{0.53$\pm$0.11} & 81.84$\pm$0.09 \\
 & Platt & 1.54$\pm$0.42 & 11.05$\pm$3.92 & \textbf{0.29$\pm$0.10} & \textbf{2.36$\pm$0.84} & 24.99$\pm$10.19 & 0.54$\pm$0.09 & 81.82$\pm$0.17 \\
\midrule
\multirow{3}{*}{(PN,Pconf)} & WLMC & \textbf{1.30$\pm$0.29} & 11.44$\pm$3.49 & \textbf{0.27$\pm$0.14} & 2.02$\pm$0.24 & 3.80$\pm$0.36 & 0.30$\pm$0.07 & \textbf{81.84$\pm$0.09} \\
 & Temp & 2.42$\pm$0.35 & 9.30$\pm$2.14 & 0.67$\pm$0.13 & 2.17$\pm$0.16 & 4.07$\pm$0.40 & 0.26$\pm$0.05 & 81.84$\pm$0.09 \\
 & Platt & 2.74$\pm$0.39 & \textbf{7.86$\pm$1.60} & 0.63$\pm$0.11 & \textbf{0.95$\pm$0.15} & \textbf{2.85$\pm$0.53} & \textbf{0.19$\pm$0.03} & 81.60$\pm$0.17 \\
\midrule
\multirow{3}{*}{(PN,UU)} & WLMC & 1.40$\pm$0.30 & 11.61$\pm$3.66 & 0.32$\pm$0.16 & 1.82$\pm$0.77 & 21.03$\pm$4.20 & 0.43$\pm$0.11 & \textbf{81.84$\pm$0.09} \\
 & Temp & \textbf{1.36$\pm$0.28} & 10.80$\pm$2.11 & \textbf{0.28$\pm$0.07} & 1.69$\pm$0.62 & 20.48$\pm$4.40 & 0.37$\pm$0.06 & 81.84$\pm$0.09 \\
 & Platt & 1.44$\pm$0.35 & \textbf{10.38$\pm$2.51} & 0.28$\pm$0.07 & \textbf{1.63$\pm$0.61} & \textbf{19.95$\pm$5.29} & \textbf{0.34$\pm$0.08} & 81.83$\pm$0.10 \\
\midrule
\multirow{4}{*}{(PU,PU)} & ERM & 6.78$\pm$0.45 & \textbf{9.97$\pm$1.30} & 2.30$\pm$0.15 & 7.05$\pm$0.64 & 19.30$\pm$5.37 & 2.26$\pm$0.30 & 78.12$\pm$0.00 \\
 & WLMC & 2.76$\pm$0.72 & 10.62$\pm$3.46 & 0.64$\pm$0.19 & 3.27$\pm$0.77 & 18.82$\pm$5.38 & 0.86$\pm$0.29 & 79.17$\pm$0.49 \\
 & Temp & 7.51$\pm$0.43 & 10.41$\pm$0.69 & 1.78$\pm$0.17 & 7.72$\pm$1.31 & \textbf{17.81$\pm$4.50} & 1.92$\pm$0.40 & 78.12$\pm$0.00 \\
 & Platt & \textbf{1.66$\pm$0.28} & 12.90$\pm$2.42 & \textbf{0.39$\pm$0.09} & \textbf{2.68$\pm$0.95} & 25.80$\pm$5.33 & \textbf{0.65$\pm$0.09} & \textbf{80.62$\pm$0.24} \\
\midrule
\multirow{4}{*}{(Pconf,Pconf)} & ERM & 11.46$\pm$0.32 & 13.18$\pm$0.77 & 5.51$\pm$0.20 & 9.31$\pm$0.41 & 13.23$\pm$1.08 & 4.01$\pm$0.23 & 78.13$\pm$0.03 \\
 & WLMC & 5.03$\pm$0.75 & 9.66$\pm$0.94 & \textbf{0.80$\pm$0.43} & 3.51$\pm$0.53 & 10.91$\pm$0.90 & 0.54$\pm$0.16 & 78.75$\pm$0.34 \\
 & Temp & 6.68$\pm$0.48 & \textbf{9.39$\pm$0.89} & 1.92$\pm$0.09 & 5.63$\pm$0.33 & 8.10$\pm$0.42 & 1.10$\pm$0.06 & 78.13$\pm$0.03 \\
 & Platt & \textbf{2.75$\pm$0.34} & 14.00$\pm$1.27 & 0.85$\pm$0.15 & \textbf{1.46$\pm$0.15} & \textbf{5.18$\pm$0.40} & \textbf{0.32$\pm$0.04} & \textbf{80.44$\pm$0.07} \\
\midrule
\multirow{4}{*}{(UU,UU)} & ERM & 6.66$\pm$0.19 & 9.76$\pm$1.12 & 2.32$\pm$0.12 & 6.41$\pm$0.33 & 13.00$\pm$2.08 & 2.23$\pm$0.24 & 78.12$\pm$0.00 \\
 & WLMC & 2.98$\pm$0.51 & \textbf{9.41$\pm$2.99} & 0.60$\pm$0.22 & 2.98$\pm$0.47 & 14.80$\pm$3.79 & 0.66$\pm$0.11 & 79.10$\pm$0.33 \\
 & Temp & 7.63$\pm$0.37 & 10.31$\pm$0.83 & 1.80$\pm$0.17 & 7.29$\pm$0.57 & \textbf{12.80$\pm$1.59} & 1.75$\pm$0.27 & 78.12$\pm$0.00 \\
 & Platt & \textbf{1.64$\pm$0.33} & 12.90$\pm$1.63 & \textbf{0.34$\pm$0.04} & \textbf{2.10$\pm$0.56} & 20.37$\pm$4.42 & \textbf{0.41$\pm$0.07} & \textbf{80.58$\pm$0.17} \\
\bottomrule
\end{tabular}%
}
\end{table*}

\begin{table*}[t]
\centering
\scriptsize
\setlength{\tabcolsep}{3.2pt}
\renewcommand{\arraystretch}{0.92}
\caption{Tabular test results on CreditDefault with NaiveBayes as the base predictor. Calibration metrics are reported as $10^2\!\times$ mean $\pm$ standard deviation over ten seeds; accuracy is reported in percent. WLMC uses step size $0.05$ and at most $50$ updates. Bold marks the single best mean within each $(\mathrm{base},\mathrm{post})$ block; variance is not used for bolding.}
\label{tab:tabular-creditdefault-naivebayes-main}
\resizebox{\textwidth}{!}{%
\begin{tabular}{llccccccc}
\toprule
\multirow{2}{*}{(Base,Post)} & \multirow{2}{*}{Post-processing} & \multicolumn{3}{c}{oracle estimate ($\times 10^{-2}$)} & \multicolumn{3}{c}{weak estimate ($\times 10^{-2}$)} & \multirow{2}{*}{Acc. (\%)} \\
\cmidrule(lr){3-5}\cmidrule(lr){6-8}
& & ECE & maxECE & MC & ECE & maxECE & MC & \\
\midrule
\multirow{4}{*}{(PN,PN)} & ERM & 24.69$\pm$9.78 & 33.00$\pm$11.66 & 10.83$\pm$3.93 & 24.69$\pm$9.78 & 32.99$\pm$11.66 & 10.83$\pm$3.93 & 73.80$\pm$11.27 \\
 & WLMC & 6.31$\pm$1.29 & 13.99$\pm$1.25 & 1.06$\pm$0.41 & 6.31$\pm$1.29 & 13.99$\pm$1.24 & 1.06$\pm$0.41 & 78.72$\pm$2.02 \\
 & Temp & 23.64$\pm$8.93 & 31.81$\pm$10.64 & 9.24$\pm$2.14 & 23.65$\pm$8.93 & 31.81$\pm$10.64 & 9.24$\pm$2.14 & 73.80$\pm$11.27 \\
 & Platt & \textbf{1.98$\pm$2.66} & \textbf{7.21$\pm$2.08} & \textbf{0.90$\pm$0.68} & \textbf{1.97$\pm$2.66} & \textbf{7.21$\pm$2.08} & \textbf{0.90$\pm$0.68} & \textbf{79.43$\pm$1.04} \\
\midrule
\multirow{3}{*}{(PN,PU)} & WLMC & 6.51$\pm$1.05 & 15.59$\pm$1.96 & 1.26$\pm$0.50 & 6.62$\pm$1.05 & 20.75$\pm$6.81 & 1.55$\pm$0.56 & 78.60$\pm$2.24 \\
 & Temp & 11.33$\pm$12.02 & 16.00$\pm$11.01 & 3.33$\pm$3.65 & 11.71$\pm$11.80 & 25.18$\pm$10.55 & 3.40$\pm$3.48 & 73.80$\pm$11.27 \\
 & Platt & \textbf{2.53$\pm$0.48} & \textbf{8.10$\pm$1.50} & \textbf{0.78$\pm$0.30} & \textbf{2.85$\pm$0.61} & \textbf{18.01$\pm$5.58} & \textbf{1.00$\pm$0.34} & \textbf{79.97$\pm$0.91} \\
\midrule
\multirow{3}{*}{(PN,Pconf)} & WLMC & 5.92$\pm$1.64 & 13.30$\pm$2.76 & \textbf{1.25$\pm$0.86} & 6.02$\pm$1.73 & 11.51$\pm$2.32 & 0.94$\pm$0.49 & 78.55$\pm$2.21 \\
 & Temp & 12.59$\pm$11.09 & 16.08$\pm$10.56 & 4.52$\pm$3.72 & 11.04$\pm$10.39 & 15.44$\pm$11.73 & 3.28$\pm$3.35 & 73.80$\pm$11.27 \\
 & Platt & \textbf{3.26$\pm$0.79} & \textbf{7.50$\pm$2.24} & 1.39$\pm$0.56 & \textbf{2.10$\pm$0.70} & \textbf{5.56$\pm$2.08} & \textbf{0.62$\pm$0.19} & \textbf{79.62$\pm$1.14} \\
\midrule
\multirow{3}{*}{(PN,UU)} & WLMC & 6.67$\pm$1.49 & 14.04$\pm$1.62 & 1.06$\pm$0.27 & 6.79$\pm$1.61 & 17.56$\pm$3.16 & 1.12$\pm$0.29 & 78.60$\pm$2.24 \\
 & Temp & 11.50$\pm$11.88 & 16.02$\pm$10.87 & 3.38$\pm$3.57 & 11.53$\pm$11.82 & 19.84$\pm$10.51 & 3.51$\pm$3.81 & 73.80$\pm$11.27 \\
 & Platt & \textbf{2.51$\pm$0.64} & \textbf{8.08$\pm$1.72} & \textbf{0.83$\pm$0.38} & \textbf{2.37$\pm$0.66} & \textbf{11.25$\pm$2.24} & \textbf{0.83$\pm$0.39} & \textbf{79.96$\pm$0.92} \\
\midrule
\multirow{4}{*}{(PU,PU)} & ERM & 28.12$\pm$0.00 & 31.41$\pm$0.00 & 17.40$\pm$0.00 & 28.12$\pm$0.00 & 34.49$\pm$6.79 & 17.14$\pm$0.80 & 21.88$\pm$0.00 \\
 & WLMC & 5.08$\pm$1.06 & 12.68$\pm$2.04 & 1.32$\pm$0.57 & 5.09$\pm$1.09 & 16.74$\pm$3.98 & 1.30$\pm$0.47 & \textbf{78.12$\pm$0.00} \\
 & Temp & 28.12$\pm$0.00 & 31.41$\pm$0.00 & 17.40$\pm$0.00 & 28.12$\pm$0.00 & 34.49$\pm$6.79 & 17.14$\pm$0.80 & 78.12$\pm$0.00 \\
 & Platt & \textbf{0.45$\pm$0.39} & \textbf{5.87$\pm$0.48} & \textbf{0.88$\pm$0.22} & \textbf{0.45$\pm$0.39} & \textbf{9.05$\pm$4.51} & \textbf{1.14$\pm$0.32} & 78.12$\pm$0.00 \\
\midrule
\multirow{4}{*}{(Pconf,Pconf)} & ERM & 20.03$\pm$0.25 & 27.03$\pm$0.53 & 9.96$\pm$0.22 & 20.07$\pm$0.17 & 25.79$\pm$0.48 & 9.96$\pm$0.22 & \textbf{79.73$\pm$0.21} \\
 & WLMC & 5.49$\pm$1.09 & 11.64$\pm$1.38 & \textbf{1.53$\pm$0.82} & 4.48$\pm$0.82 & 10.74$\pm$0.60 & \textbf{0.80$\pm$0.57} & 79.73$\pm$0.21 \\
 & Temp & 4.95$\pm$0.85 & 9.17$\pm$1.53 & 2.34$\pm$0.39 & 3.55$\pm$0.67 & \textbf{7.67$\pm$1.83} & 1.18$\pm$0.24 & 79.73$\pm$0.21 \\
 & Platt & \textbf{4.20$\pm$0.44} & \textbf{8.44$\pm$1.50} & 1.99$\pm$0.19 & \textbf{2.87$\pm$0.56} & 7.87$\pm$2.00 & 0.87$\pm$0.13 & 79.41$\pm$0.28 \\
\midrule
\multirow{4}{*}{(UU,UU)} & ERM & 28.12$\pm$0.00 & 31.41$\pm$0.00 & 17.40$\pm$0.00 & 28.12$\pm$0.00 & 33.67$\pm$3.99 & 17.51$\pm$0.39 & 21.88$\pm$0.00 \\
 & WLMC & 4.77$\pm$1.33 & 12.68$\pm$2.13 & 1.17$\pm$0.65 & 4.67$\pm$1.27 & 16.70$\pm$4.54 & 1.22$\pm$0.60 & \textbf{78.12$\pm$0.00} \\
 & Temp & 28.12$\pm$0.00 & 31.41$\pm$0.00 & 17.40$\pm$0.00 & 28.12$\pm$0.00 & 33.67$\pm$3.99 & 17.51$\pm$0.39 & 78.12$\pm$0.00 \\
 & Platt & \textbf{0.45$\pm$0.39} & \textbf{5.87$\pm$0.48} & \textbf{0.88$\pm$0.22} & \textbf{0.45$\pm$0.39} & \textbf{6.89$\pm$1.71} & \textbf{1.07$\pm$0.29} & 78.12$\pm$0.00 \\
\bottomrule
\end{tabular}%
}
\end{table*}

\begin{table*}[t]
\centering
\scriptsize
\setlength{\tabcolsep}{3.2pt}
\renewcommand{\arraystretch}{0.92}
\caption{Tabular test results on CreditDefault with RandomForest as the base predictor. Calibration metrics are reported as $10^2\!\times$ mean $\pm$ standard deviation over ten seeds; accuracy is reported in percent. WLMC uses step size $0.05$ and at most $50$ updates. Bold marks the single best mean within each $(\mathrm{base},\mathrm{post})$ block; variance is not used for bolding.}
\label{tab:tabular-creditdefault-randomforest-main}
\resizebox{\textwidth}{!}{%
\begin{tabular}{llccccccc}
\toprule
\multirow{2}{*}{(Base,Post)} & \multirow{2}{*}{Post-processing} & \multicolumn{3}{c}{oracle estimate ($\times 10^{-2}$)} & \multicolumn{3}{c}{weak estimate ($\times 10^{-2}$)} & \multirow{2}{*}{Acc. (\%)} \\
\cmidrule(lr){3-5}\cmidrule(lr){6-8}
& & ECE & maxECE & MC & ECE & maxECE & MC & \\
\midrule
\multirow{4}{*}{(PN,PN)} & ERM & 1.76$\pm$0.27 & 11.91$\pm$3.74 & 0.43$\pm$0.10 & 1.76$\pm$0.27 & 11.92$\pm$3.74 & 0.43$\pm$0.10 & 81.81$\pm$0.18 \\
 & WLMC & \textbf{1.55$\pm$0.32} & 11.74$\pm$3.24 & \textbf{0.37$\pm$0.09} & \textbf{1.55$\pm$0.32} & 11.74$\pm$3.24 & \textbf{0.37$\pm$0.09} & 81.81$\pm$0.18 \\
 & Temp & 6.85$\pm$9.57 & 10.44$\pm$9.26 & 3.03$\pm$5.81 & 6.85$\pm$9.57 & 10.44$\pm$9.27 & 3.03$\pm$5.81 & 75.45$\pm$20.11 \\
 & Platt & 1.96$\pm$0.21 & \textbf{8.80$\pm$2.30} & 0.46$\pm$0.14 & 1.95$\pm$0.21 & \textbf{8.80$\pm$2.29} & 0.46$\pm$0.14 & \textbf{81.90$\pm$0.16} \\
\midrule
\multirow{3}{*}{(PN,PU)} & WLMC & 1.50$\pm$0.51 & \textbf{11.78$\pm$3.82} & 0.46$\pm$0.28 & \textbf{2.73$\pm$0.78} & \textbf{23.51$\pm$5.31} & 0.73$\pm$0.20 & 81.81$\pm$0.22 \\
 & Temp & 1.52$\pm$0.46 & 13.11$\pm$3.52 & \textbf{0.30$\pm$0.06} & 2.74$\pm$0.71 & 26.59$\pm$8.62 & 0.71$\pm$0.28 & 81.81$\pm$0.18 \\
 & Platt & \textbf{1.44$\pm$0.54} & 13.87$\pm$3.90 & 0.31$\pm$0.09 & 3.02$\pm$0.76 & 28.97$\pm$8.48 & \textbf{0.70$\pm$0.23} & \textbf{81.88$\pm$0.10} \\
\midrule
\multirow{3}{*}{(PN,Pconf)} & WLMC & \textbf{1.76$\pm$0.27} & 11.91$\pm$3.74 & \textbf{0.43$\pm$0.10} & 0.23$\pm$0.05 & \textbf{0.83$\pm$0.12} & 0.05$\pm$0.01 & \textbf{81.81$\pm$0.18} \\
 & Temp & 1.80$\pm$0.24 & 11.88$\pm$3.33 & 0.43$\pm$0.09 & 0.22$\pm$0.04 & 0.84$\pm$0.13 & 0.05$\pm$0.01 & 81.81$\pm$0.18 \\
 & Platt & 1.84$\pm$0.27 & \textbf{11.88$\pm$3.28} & 0.44$\pm$0.10 & \textbf{0.19$\pm$0.06} & 0.85$\pm$0.16 & \textbf{0.05$\pm$0.01} & 81.80$\pm$0.17 \\
\midrule
\multirow{3}{*}{(PN,UU)} & WLMC & 1.53$\pm$0.35 & \textbf{12.40$\pm$4.01} & 0.36$\pm$0.09 & 2.06$\pm$0.50 & \textbf{19.65$\pm$6.47} & 0.53$\pm$0.13 & 81.81$\pm$0.18 \\
 & Temp & 1.35$\pm$0.20 & 12.86$\pm$2.64 & 0.29$\pm$0.06 & \textbf{1.80$\pm$0.30} & 20.49$\pm$4.02 & 0.48$\pm$0.15 & 81.81$\pm$0.18 \\
 & Platt & \textbf{1.24$\pm$0.35} & 13.91$\pm$3.19 & \textbf{0.29$\pm$0.06} & 2.03$\pm$0.55 & 19.67$\pm$4.34 & \textbf{0.48$\pm$0.14} & \textbf{81.89$\pm$0.16} \\
\midrule
\multirow{4}{*}{(PU,PU)} & ERM & 7.58$\pm$3.95 & 20.26$\pm$5.49 & 2.65$\pm$2.72 & 8.00$\pm$4.04 & 30.08$\pm$7.16 & 2.57$\pm$2.44 & 79.00$\pm$2.56 \\
 & WLMC & \textbf{4.08$\pm$0.74} & 17.74$\pm$3.61 & \textbf{0.59$\pm$0.14} & 4.66$\pm$1.23 & 28.84$\pm$7.17 & \textbf{0.92$\pm$0.24} & \textbf{79.67$\pm$1.56} \\
 & Temp & 18.92$\pm$10.00 & 23.83$\pm$8.34 & 8.11$\pm$6.78 & 19.29$\pm$9.97 & 35.80$\pm$8.46 & 8.13$\pm$6.64 & 79.00$\pm$2.56 \\
 & Platt & 4.60$\pm$2.61 & \textbf{10.47$\pm$3.28} & 1.38$\pm$0.70 & \textbf{4.65$\pm$2.84} & \textbf{22.70$\pm$9.76} & 1.45$\pm$0.89 & 78.77$\pm$1.79 \\
\midrule
\multirow{4}{*}{(Pconf,Pconf)} & ERM & 8.04$\pm$0.26 & 15.99$\pm$2.23 & 2.36$\pm$0.20 & 6.36$\pm$0.40 & 8.37$\pm$0.82 & 1.89$\pm$0.19 & \textbf{81.33$\pm$0.21} \\
 & WLMC & 4.11$\pm$0.79 & 17.89$\pm$2.73 & \textbf{0.84$\pm$0.19} & 3.07$\pm$0.42 & 5.41$\pm$0.78 & 0.58$\pm$0.15 & 81.33$\pm$0.21 \\
 & Temp & 4.15$\pm$0.36 & 12.69$\pm$1.81 & 1.17$\pm$0.10 & 3.27$\pm$0.30 & 4.79$\pm$0.77 & 0.53$\pm$0.08 & 81.33$\pm$0.21 \\
 & Platt & \textbf{3.59$\pm$0.39} & \textbf{11.94$\pm$1.96} & 0.90$\pm$0.10 & \textbf{1.94$\pm$0.23} & \textbf{3.89$\pm$0.41} & \textbf{0.39$\pm$0.09} & 80.96$\pm$0.23 \\
\midrule
\multirow{4}{*}{(UU,UU)} & ERM & 8.16$\pm$0.55 & 17.93$\pm$4.04 & 1.63$\pm$0.27 & 8.48$\pm$0.64 & 25.92$\pm$7.62 & 1.83$\pm$0.35 & 79.34$\pm$0.45 \\
 & WLMC & \textbf{4.90$\pm$1.19} & 15.74$\pm$3.24 & \textbf{0.71$\pm$0.16} & \textbf{5.29$\pm$1.32} & 22.23$\pm$4.85 & \textbf{0.83$\pm$0.18} & \textbf{79.52$\pm$0.38} \\
 & Temp & 21.67$\pm$2.34 & 24.56$\pm$2.15 & 6.98$\pm$1.55 & 21.91$\pm$2.44 & 31.12$\pm$7.75 & 7.09$\pm$1.56 & 79.34$\pm$0.45 \\
 & Platt & 8.16$\pm$0.73 & \textbf{10.51$\pm$1.07} & 2.25$\pm$0.32 & 8.21$\pm$0.96 & \textbf{14.10$\pm$3.41} & 2.26$\pm$0.40 & 79.05$\pm$0.28 \\
\bottomrule
\end{tabular}%
}
\end{table*}

\begin{table*}[t]
\centering
\scriptsize
\setlength{\tabcolsep}{3.2pt}
\renewcommand{\arraystretch}{0.92}
\caption{Tabular test results on CreditDefault with SVM as the base predictor. Calibration metrics are reported as $10^2\!\times$ mean $\pm$ standard deviation over ten seeds; accuracy is reported in percent. WLMC uses step size $0.05$ and at most $50$ updates. Bold marks the single best mean within each $(\mathrm{base},\mathrm{post})$ block; variance is not used for bolding.}
\label{tab:tabular-creditdefault-svm-main}
\resizebox{\textwidth}{!}{%
\begin{tabular}{llccccccc}
\toprule
\multirow{2}{*}{(Base,Post)} & \multirow{2}{*}{Post-processing} & \multicolumn{3}{c}{oracle estimate ($\times 10^{-2}$)} & \multicolumn{3}{c}{weak estimate ($\times 10^{-2}$)} & \multirow{2}{*}{Acc. (\%)} \\
\cmidrule(lr){3-5}\cmidrule(lr){6-8}
& & ECE & maxECE & MC & ECE & maxECE & MC & \\
\midrule
\multirow{4}{*}{(PN,PN)} & ERM & 9.64$\pm$0.25 & 11.92$\pm$0.49 & 5.81$\pm$0.09 & 9.64$\pm$0.25 & 11.92$\pm$0.49 & 5.81$\pm$0.09 & 81.88$\pm$0.05 \\
 & WLMC & 2.85$\pm$0.77 & 7.21$\pm$0.90 & 0.86$\pm$0.43 & 2.85$\pm$0.77 & 7.21$\pm$0.90 & 0.86$\pm$0.43 & 81.88$\pm$0.05 \\
 & Temp & 15.34$\pm$0.08 & 17.14$\pm$0.19 & 9.37$\pm$0.08 & 15.33$\pm$0.08 & 17.14$\pm$0.19 & 9.37$\pm$0.08 & 81.88$\pm$0.05 \\
 & Platt & \textbf{0.76$\pm$0.31} & \textbf{5.33$\pm$0.70} & \textbf{0.61$\pm$0.23} & \textbf{0.75$\pm$0.31} & \textbf{5.34$\pm$0.70} & \textbf{0.61$\pm$0.23} & \textbf{81.88$\pm$0.05} \\
\midrule
\multirow{3}{*}{(PN,PU)} & WLMC & 3.92$\pm$1.16 & 8.51$\pm$1.24 & 1.56$\pm$0.71 & 4.30$\pm$1.25 & 17.91$\pm$7.51 & 1.64$\pm$0.70 & 81.88$\pm$0.05 \\
 & Temp & 3.64$\pm$1.70 & 7.93$\pm$1.63 & 1.47$\pm$0.71 & 3.87$\pm$2.22 & 19.58$\pm$8.73 & 1.59$\pm$0.77 & 81.88$\pm$0.05 \\
 & Platt & \textbf{1.45$\pm$0.66} & \textbf{5.82$\pm$1.33} & \textbf{0.68$\pm$0.24} & \textbf{1.56$\pm$0.53} & \textbf{16.90$\pm$10.11} & \textbf{0.82$\pm$0.15} & \textbf{81.88$\pm$0.05} \\
\midrule
\multirow{3}{*}{(PN,Pconf)} & WLMC & \textbf{3.10$\pm$0.83} & 7.79$\pm$0.91 & \textbf{1.07$\pm$0.75} & 3.42$\pm$0.56 & 8.32$\pm$1.40 & 0.74$\pm$0.44 & 81.88$\pm$0.05 \\
 & Temp & 5.20$\pm$0.52 & 8.40$\pm$0.59 & 2.59$\pm$0.60 & 4.94$\pm$0.37 & 7.81$\pm$0.77 & 1.56$\pm$0.33 & 81.88$\pm$0.05 \\
 & Platt & 3.23$\pm$1.54 & \textbf{5.81$\pm$1.49} & 1.58$\pm$0.55 & \textbf{1.28$\pm$1.19} & \textbf{4.59$\pm$1.31} & \textbf{0.61$\pm$0.31} & \textbf{81.88$\pm$0.05} \\
\midrule
\multirow{3}{*}{(PN,UU)} & WLMC & 3.16$\pm$1.11 & 7.72$\pm$1.23 & 1.17$\pm$0.69 & 3.42$\pm$1.20 & 11.30$\pm$3.89 & 1.23$\pm$0.70 & 81.88$\pm$0.05 \\
 & Temp & 4.01$\pm$1.56 & 8.12$\pm$2.29 & 1.65$\pm$0.58 & 4.13$\pm$1.70 & 11.85$\pm$5.98 & 1.70$\pm$0.63 & 81.88$\pm$0.05 \\
 & Platt & \textbf{0.99$\pm$0.39} & \textbf{5.78$\pm$1.31} & \textbf{0.66$\pm$0.26} & \textbf{1.17$\pm$0.45} & \textbf{9.46$\pm$3.66} & \textbf{0.74$\pm$0.32} & \textbf{81.88$\pm$0.05} \\
\midrule
\multirow{4}{*}{(PU,PU)} & ERM & 9.86$\pm$0.36 & 13.71$\pm$2.31 & 5.67$\pm$0.21 & 10.25$\pm$0.77 & 21.78$\pm$8.91 & 5.65$\pm$0.38 & \textbf{81.91$\pm$0.05} \\
 & WLMC & 3.43$\pm$1.10 & 9.56$\pm$2.22 & \textbf{1.01$\pm$0.67} & 3.83$\pm$1.15 & 21.17$\pm$7.89 & 1.21$\pm$0.63 & 81.91$\pm$0.05 \\
 & Temp & 4.23$\pm$1.29 & 9.28$\pm$1.52 & 1.57$\pm$0.56 & 4.45$\pm$1.26 & 22.97$\pm$8.38 & 1.80$\pm$0.46 & 81.91$\pm$0.05 \\
 & Platt & \textbf{2.99$\pm$2.12} & \textbf{7.10$\pm$1.40} & 1.06$\pm$0.72 & \textbf{3.18$\pm$2.14} & \textbf{18.29$\pm$8.97} & \textbf{1.19$\pm$0.51} & 81.20$\pm$1.49 \\
\midrule
\multirow{4}{*}{(Pconf,Pconf)} & ERM & 10.29$\pm$0.35 & 13.62$\pm$1.42 & 5.94$\pm$0.15 & 9.41$\pm$0.57 & 13.88$\pm$1.22 & 4.46$\pm$0.34 & \textbf{81.93$\pm$0.03} \\
 & WLMC & 3.94$\pm$0.97 & 9.84$\pm$1.73 & \textbf{1.37$\pm$0.74} & 3.69$\pm$0.73 & 10.39$\pm$1.19 & \textbf{0.75$\pm$0.54} & 81.93$\pm$0.03 \\
 & Temp & 4.96$\pm$0.55 & 8.71$\pm$2.29 & 1.84$\pm$0.34 & 4.78$\pm$0.64 & 8.78$\pm$1.17 & 1.21$\pm$0.26 & 81.93$\pm$0.03 \\
 & Platt & \textbf{3.90$\pm$0.91} & \textbf{6.22$\pm$0.90} & 1.78$\pm$0.32 & \textbf{2.27$\pm$0.83} & \textbf{5.36$\pm$1.26} & 0.75$\pm$0.19 & 81.92$\pm$0.02 \\
\midrule
\multirow{4}{*}{(UU,UU)} & ERM & 9.78$\pm$0.31 & 12.37$\pm$0.80 & 5.87$\pm$0.09 & 9.96$\pm$0.32 & 15.03$\pm$3.14 & 5.91$\pm$0.27 & \textbf{81.91$\pm$0.04} \\
 & WLMC & 3.43$\pm$1.28 & 8.42$\pm$1.31 & 1.28$\pm$0.71 & 3.80$\pm$1.32 & 12.17$\pm$3.67 & 1.40$\pm$0.69 & 81.91$\pm$0.04 \\
 & Temp & 4.13$\pm$1.35 & 8.75$\pm$2.83 & 1.77$\pm$0.58 & 4.14$\pm$1.43 & 12.36$\pm$4.54 & 1.81$\pm$0.60 & 81.91$\pm$0.04 \\
 & Platt & \textbf{1.16$\pm$0.61} & \textbf{5.62$\pm$1.00} & \textbf{0.75$\pm$0.27} & \textbf{1.34$\pm$0.52} & \textbf{9.58$\pm$3.17} & \textbf{0.82$\pm$0.32} & 81.89$\pm$0.04 \\
\bottomrule
\end{tabular}%
}
\end{table*}

\begin{table*}[t]
\centering
\scriptsize
\setlength{\tabcolsep}{3.2pt}
\renewcommand{\arraystretch}{0.92}
\caption{Tabular test results on CreditDefault with MLP as the base predictor. Calibration metrics are reported as $10^2\!\times$ mean $\pm$ standard deviation over ten seeds; accuracy is reported in percent. WLMC uses step size $0.05$ and at most $50$ updates. Bold marks the single best mean within each $(\mathrm{base},\mathrm{post})$ block; variance is not used for bolding.}
\label{tab:tabular-creditdefault-mlp-main}
\resizebox{\textwidth}{!}{%
\begin{tabular}{llccccccc}
\toprule
\multirow{2}{*}{(Base,Post)} & \multirow{2}{*}{Post-processing} & \multicolumn{3}{c}{oracle estimate ($\times 10^{-2}$)} & \multicolumn{3}{c}{weak estimate ($\times 10^{-2}$)} & \multirow{2}{*}{Acc. (\%)} \\
\cmidrule(lr){3-5}\cmidrule(lr){6-8}
& & ECE & maxECE & MC & ECE & maxECE & MC & \\
\midrule
\multirow{4}{*}{(PN,PN)} & ERM & 2.86$\pm$1.77 & 10.61$\pm$1.86 & 0.71$\pm$0.62 & 2.86$\pm$1.77 & 10.61$\pm$1.86 & 0.71$\pm$0.62 & 81.91$\pm$0.15 \\
 & WLMC & 2.07$\pm$0.48 & 10.58$\pm$1.82 & \textbf{0.39$\pm$0.12} & 2.07$\pm$0.48 & 10.58$\pm$1.82 & \textbf{0.39$\pm$0.12} & 81.91$\pm$0.15 \\
 & Temp & 5.35$\pm$3.71 & 10.80$\pm$3.65 & 1.54$\pm$1.24 & 5.35$\pm$3.72 & 10.80$\pm$3.65 & 1.54$\pm$1.24 & 81.91$\pm$0.15 \\
 & Platt & \textbf{1.84$\pm$0.59} & \textbf{9.09$\pm$2.66} & 0.44$\pm$0.18 & \textbf{1.84$\pm$0.59} & \textbf{9.09$\pm$2.66} & 0.44$\pm$0.18 & \textbf{82.04$\pm$0.09} \\
\midrule
\multirow{3}{*}{(PN,PU)} & WLMC & 2.07$\pm$0.53 & \textbf{10.62$\pm$2.21} & 0.47$\pm$0.20 & 2.86$\pm$1.11 & \textbf{22.77$\pm$7.74} & \textbf{0.60$\pm$0.13} & 81.91$\pm$0.15 \\
 & Temp & 1.90$\pm$0.61 & 11.33$\pm$2.98 & 0.39$\pm$0.12 & 2.87$\pm$1.08 & 26.59$\pm$6.58 & 0.61$\pm$0.14 & 81.91$\pm$0.15 \\
 & Platt & \textbf{1.87$\pm$0.67} & 11.46$\pm$3.54 & \textbf{0.37$\pm$0.13} & \textbf{2.64$\pm$0.95} & 27.30$\pm$6.77 & 0.64$\pm$0.17 & \textbf{82.00$\pm$0.18} \\
\midrule
\multirow{3}{*}{(PN,Pconf)} & WLMC & \textbf{2.09$\pm$0.57} & 9.95$\pm$1.56 & \textbf{0.52$\pm$0.22} & 2.07$\pm$0.68 & 4.15$\pm$1.37 & 0.36$\pm$0.10 & \textbf{81.91$\pm$0.15} \\
 & Temp & 2.44$\pm$0.54 & 10.15$\pm$3.40 & 0.67$\pm$0.27 & 2.11$\pm$0.74 & 4.29$\pm$1.14 & 0.34$\pm$0.13 & 81.91$\pm$0.15 \\
 & Platt & 2.62$\pm$0.60 & \textbf{7.85$\pm$2.82} & 0.60$\pm$0.20 & \textbf{1.20$\pm$0.38} & \textbf{3.32$\pm$0.65} & \textbf{0.26$\pm$0.06} & 81.78$\pm$0.19 \\
\midrule
\multirow{3}{*}{(PN,UU)} & WLMC & 2.24$\pm$0.32 & 10.54$\pm$1.82 & 0.50$\pm$0.12 & 2.69$\pm$0.40 & 18.52$\pm$6.38 & 0.60$\pm$0.15 & 81.91$\pm$0.15 \\
 & Temp & 1.84$\pm$0.49 & 10.47$\pm$2.93 & 0.40$\pm$0.10 & 2.22$\pm$0.62 & \textbf{18.52$\pm$6.76} & 0.48$\pm$0.13 & 81.91$\pm$0.15 \\
 & Platt & \textbf{1.65$\pm$0.63} & \textbf{10.22$\pm$3.85} & \textbf{0.37$\pm$0.09} & \textbf{2.01$\pm$0.64} & 18.54$\pm$5.37 & \textbf{0.45$\pm$0.12} & \textbf{82.01$\pm$0.10} \\
\midrule
\multirow{4}{*}{(PU,PU)} & ERM & 3.45$\pm$1.64 & 11.58$\pm$3.09 & 0.91$\pm$0.50 & 4.12$\pm$1.36 & 25.62$\pm$7.93 & 1.09$\pm$0.38 & \textbf{81.32$\pm$0.20} \\
 & WLMC & 2.18$\pm$0.82 & 12.79$\pm$3.65 & 0.59$\pm$0.44 & 2.92$\pm$0.85 & 25.95$\pm$10.09 & 0.77$\pm$0.41 & 81.32$\pm$0.20 \\
 & Temp & 2.44$\pm$1.12 & 10.97$\pm$3.07 & 0.52$\pm$0.19 & 3.28$\pm$1.15 & 23.10$\pm$8.13 & 0.73$\pm$0.23 & 81.32$\pm$0.20 \\
 & Platt & \textbf{2.12$\pm$0.70} & \textbf{10.25$\pm$1.99} & \textbf{0.51$\pm$0.19} & \textbf{2.81$\pm$0.97} & \textbf{20.73$\pm$7.37} & \textbf{0.71$\pm$0.21} & 81.29$\pm$0.32 \\
\midrule
\multirow{4}{*}{(Pconf,Pconf)} & ERM & 4.76$\pm$2.78 & 11.13$\pm$2.55 & 1.55$\pm$1.17 & 4.35$\pm$2.60 & 5.94$\pm$3.04 & 1.08$\pm$0.93 & \textbf{81.68$\pm$0.19} \\
 & WLMC & 3.09$\pm$0.89 & \textbf{10.67$\pm$2.33} & 0.83$\pm$0.31 & 2.99$\pm$0.94 & 4.57$\pm$1.39 & 0.45$\pm$0.09 & 81.68$\pm$0.19 \\
 & Temp & 3.06$\pm$0.99 & 12.92$\pm$3.43 & 0.89$\pm$0.29 & 2.96$\pm$1.05 & 4.15$\pm$1.24 & 0.48$\pm$0.19 & 81.68$\pm$0.19 \\
 & Platt & \textbf{2.63$\pm$0.45} & 13.31$\pm$3.83 & \textbf{0.82$\pm$0.21} & \textbf{1.20$\pm$0.45} & \textbf{2.59$\pm$0.56} & \textbf{0.31$\pm$0.09} & 81.57$\pm$0.23 \\
\midrule
\multirow{4}{*}{(UU,UU)} & ERM & 3.76$\pm$1.89 & 11.89$\pm$2.95 & 0.90$\pm$0.40 & 3.91$\pm$2.21 & 18.58$\pm$3.05 & 0.99$\pm$0.45 & 81.54$\pm$0.28 \\
 & WLMC & 2.41$\pm$0.75 & 12.21$\pm$3.05 & 0.48$\pm$0.19 & 2.66$\pm$0.94 & 17.85$\pm$2.68 & \textbf{0.53$\pm$0.18} & 81.54$\pm$0.28 \\
 & Temp & 2.58$\pm$1.02 & 11.81$\pm$3.79 & 0.54$\pm$0.19 & 2.88$\pm$1.21 & \textbf{17.56$\pm$3.44} & 0.67$\pm$0.24 & 81.54$\pm$0.28 \\
 & Platt & \textbf{2.24$\pm$0.57} & \textbf{11.74$\pm$4.43} & \textbf{0.47$\pm$0.15} & \textbf{2.43$\pm$0.59} & 19.78$\pm$5.17 & 0.54$\pm$0.14 & \textbf{81.56$\pm$0.22} \\
\bottomrule
\end{tabular}%
}
\end{table*}

\begin{table*}[t]
\centering
\scriptsize
\setlength{\tabcolsep}{3.2pt}
\renewcommand{\arraystretch}{0.92}
\caption{Tabular test results on HMDA with DecisionTree as the base predictor. Calibration metrics are reported as $10^2\!\times$ mean $\pm$ standard deviation over ten seeds; accuracy is reported in percent. WLMC uses step size $0.05$ and at most $50$ updates. Bold marks the single best mean within each $(\mathrm{base},\mathrm{post})$ block; variance is not used for bolding.}
\label{tab:tabular-hmda-decisiontree-main}
\resizebox{\textwidth}{!}{%
\begin{tabular}{llccccccc}
\toprule
\multirow{2}{*}{(Base,Post)} & \multirow{2}{*}{Post-processing} & \multicolumn{3}{c}{oracle estimate ($\times 10^{-2}$)} & \multicolumn{3}{c}{weak estimate ($\times 10^{-2}$)} & \multirow{2}{*}{Acc. (\%)} \\
\cmidrule(lr){3-5}\cmidrule(lr){6-8}
& & ECE & maxECE & MC & ECE & maxECE & MC & \\
\midrule
\multirow{4}{*}{(PN,PN)} & ERM & 2.73$\pm$0.17 & 7.04$\pm$1.35 & 0.44$\pm$0.05 & 2.73$\pm$0.17 & 7.04$\pm$1.35 & 0.44$\pm$0.05 & 86.02$\pm$0.25 \\
 & WLMC & \textbf{2.67$\pm$0.19} & \textbf{6.99$\pm$1.24} & 0.41$\pm$0.08 & \textbf{2.67$\pm$0.19} & \textbf{6.99$\pm$1.24} & 0.41$\pm$0.08 & \textbf{86.03$\pm$0.24} \\
 & Temp & 2.81$\pm$0.26 & 8.19$\pm$0.80 & \textbf{0.40$\pm$0.08} & 2.81$\pm$0.26 & 8.19$\pm$0.80 & \textbf{0.40$\pm$0.08} & 86.03$\pm$0.24 \\
 & Platt & 4.67$\pm$0.28 & 7.97$\pm$1.51 & 0.98$\pm$0.09 & 4.67$\pm$0.29 & 7.97$\pm$1.51 & 0.98$\pm$0.09 & 85.75$\pm$0.29 \\
\midrule
\multirow{3}{*}{(PN,PU)} & WLMC & \textbf{2.63$\pm$0.22} & \textbf{6.94$\pm$1.33} & \textbf{0.46$\pm$0.12} & \textbf{2.74$\pm$0.47} & \textbf{15.45$\pm$2.67} & \textbf{0.69$\pm$0.31} & \textbf{86.03$\pm$0.24} \\
 & Temp & 3.64$\pm$1.42 & 8.79$\pm$1.68 & 0.65$\pm$0.29 & 3.75$\pm$1.34 & 17.97$\pm$2.79 & 0.75$\pm$0.20 & 86.03$\pm$0.24 \\
 & Platt & 3.56$\pm$1.36 & 8.56$\pm$1.45 & 0.65$\pm$0.29 & 3.67$\pm$1.24 & 17.46$\pm$3.15 & 0.76$\pm$0.21 & 86.00$\pm$0.21 \\
\midrule
\multirow{3}{*}{(PN,Pconf)} & WLMC & \textbf{4.38$\pm$0.17} & \textbf{8.29$\pm$1.55} & 1.70$\pm$0.04 & 3.92$\pm$0.09 & 6.45$\pm$0.42 & \textbf{0.64$\pm$0.05} & \textbf{86.03$\pm$0.24} \\
 & Temp & 7.56$\pm$0.19 & 11.56$\pm$1.40 & \textbf{1.38$\pm$0.08} & 3.46$\pm$0.20 & \textbf{4.74$\pm$0.40} & 0.98$\pm$0.18 & 86.03$\pm$0.24 \\
 & Platt & 7.05$\pm$0.39 & 11.90$\pm$0.65 & 1.60$\pm$0.08 & \textbf{3.24$\pm$0.24} & 4.91$\pm$0.30 & 0.79$\pm$0.07 & 85.31$\pm$0.46 \\
\midrule
\multirow{3}{*}{(PN,UU)} & WLMC & \textbf{2.64$\pm$0.16} & \textbf{6.95$\pm$1.27} & \textbf{0.40$\pm$0.07} & \textbf{2.70$\pm$0.34} & \textbf{9.86$\pm$2.74} & \textbf{0.43$\pm$0.09} & \textbf{86.03$\pm$0.24} \\
 & Temp & 4.24$\pm$0.40 & 9.60$\pm$0.66 & 0.76$\pm$0.15 & 4.19$\pm$0.46 & 12.08$\pm$2.81 & 0.81$\pm$0.14 & 86.03$\pm$0.24 \\
 & Platt & 4.11$\pm$0.45 & 8.98$\pm$0.44 & 0.80$\pm$0.13 & 4.05$\pm$0.47 & 12.19$\pm$2.96 & 0.80$\pm$0.12 & 86.02$\pm$0.22 \\
\midrule
\multirow{4}{*}{(PU,PU)} & ERM & 4.36$\pm$0.52 & 9.23$\pm$2.72 & 0.77$\pm$0.14 & 4.45$\pm$1.17 & \textbf{17.77$\pm$5.58} & 0.96$\pm$0.37 & 84.50$\pm$0.46 \\
 & WLMC & \textbf{3.92$\pm$0.65} & \textbf{8.87$\pm$2.80} & \textbf{0.72$\pm$0.12} & \textbf{4.03$\pm$0.72} & 17.84$\pm$5.52 & \textbf{0.86$\pm$0.34} & \textbf{84.53$\pm$0.46} \\
 & Temp & 5.39$\pm$2.12 & 13.91$\pm$3.41 & 1.59$\pm$0.98 & 5.34$\pm$2.37 & 19.58$\pm$4.38 & 1.69$\pm$1.13 & 84.51$\pm$0.47 \\
 & Platt & 5.28$\pm$2.28 & 14.41$\pm$3.76 & 1.60$\pm$0.82 & 5.39$\pm$2.59 & 20.01$\pm$5.34 & 1.68$\pm$0.97 & 84.19$\pm$0.57 \\
\midrule
\multirow{4}{*}{(Pconf,Pconf)} & ERM & 8.61$\pm$0.22 & 18.53$\pm$0.84 & 1.24$\pm$0.07 & 5.34$\pm$0.11 & 7.60$\pm$0.31 & 0.70$\pm$0.03 & 82.01$\pm$0.31 \\
 & WLMC & 7.45$\pm$0.23 & 16.86$\pm$0.68 & \textbf{0.96$\pm$0.07} & 4.34$\pm$0.19 & 6.71$\pm$0.18 & 0.47$\pm$0.03 & 82.99$\pm$0.28 \\
 & Temp & 9.11$\pm$0.26 & 19.03$\pm$0.90 & 1.37$\pm$0.07 & 4.56$\pm$0.09 & 6.70$\pm$0.31 & 0.74$\pm$0.03 & 82.01$\pm$0.31 \\
 & Platt & \textbf{4.43$\pm$0.19} & \textbf{12.80$\pm$0.95} & 1.00$\pm$0.07 & \textbf{0.36$\pm$0.07} & \textbf{1.87$\pm$0.31} & \textbf{0.10$\pm$0.02} & \textbf{86.02$\pm$0.14} \\
\midrule
\multirow{4}{*}{(UU,UU)} & ERM & 2.67$\pm$0.30 & 8.51$\pm$0.79 & 0.62$\pm$0.18 & 2.76$\pm$0.42 & 11.42$\pm$2.33 & 0.67$\pm$0.25 & \textbf{85.60$\pm$0.33} \\
 & WLMC & \textbf{2.36$\pm$0.53} & \textbf{8.28$\pm$0.85} & \textbf{0.42$\pm$0.08} & \textbf{2.48$\pm$0.59} & \textbf{11.39$\pm$2.64} & \textbf{0.46$\pm$0.12} & 85.60$\pm$0.33 \\
 & Temp & 5.52$\pm$0.86 & 10.41$\pm$1.82 & 0.92$\pm$0.19 & 5.50$\pm$0.95 & 13.31$\pm$3.25 & 0.97$\pm$0.19 & 85.60$\pm$0.33 \\
 & Platt & 5.56$\pm$0.93 & 11.04$\pm$1.86 & 0.96$\pm$0.22 & 5.56$\pm$1.03 & 14.09$\pm$3.39 & 1.01$\pm$0.26 & 85.25$\pm$0.68 \\
\bottomrule
\end{tabular}%
}
\end{table*}

\begin{table*}[t]
\centering
\scriptsize
\setlength{\tabcolsep}{3.2pt}
\renewcommand{\arraystretch}{0.92}
\caption{Tabular test results on HMDA with LogisticRegression as the base predictor. Calibration metrics are reported as $10^2\!\times$ mean $\pm$ standard deviation over ten seeds; accuracy is reported in percent. WLMC uses step size $0.05$ and at most $50$ updates. Bold marks the single best mean within each $(\mathrm{base},\mathrm{post})$ block; variance is not used for bolding.}
\label{tab:tabular-hmda-logisticregression-main}
\resizebox{\textwidth}{!}{%
\begin{tabular}{llccccccc}
\toprule
\multirow{2}{*}{(Base,Post)} & \multirow{2}{*}{Post-processing} & \multicolumn{3}{c}{oracle estimate ($\times 10^{-2}$)} & \multicolumn{3}{c}{weak estimate ($\times 10^{-2}$)} & \multirow{2}{*}{Acc. (\%)} \\
\cmidrule(lr){3-5}\cmidrule(lr){6-8}
& & ECE & maxECE & MC & ECE & maxECE & MC & \\
\midrule
\multirow{4}{*}{(PN,PN)} & ERM & 1.50$\pm$0.19 & \textbf{10.06$\pm$0.21} & \textbf{0.26$\pm$0.03} & 1.49$\pm$0.19 & \textbf{10.06$\pm$0.22} & \textbf{0.26$\pm$0.03} & \textbf{82.68$\pm$0.11} \\
 & WLMC & 1.50$\pm$0.19 & 10.06$\pm$0.21 & 0.26$\pm$0.03 & 1.50$\pm$0.19 & 10.06$\pm$0.22 & 0.26$\pm$0.03 & 82.68$\pm$0.11 \\
 & Temp & \textbf{1.10$\pm$0.08} & 10.17$\pm$0.18 & 0.31$\pm$0.02 & \textbf{1.10$\pm$0.08} & 10.17$\pm$0.18 & 0.31$\pm$0.02 & 82.68$\pm$0.11 \\
 & Platt & 5.13$\pm$0.17 & 13.40$\pm$0.54 & 0.79$\pm$0.06 & 5.13$\pm$0.17 & 13.40$\pm$0.55 & 0.79$\pm$0.06 & 82.61$\pm$0.13 \\
\midrule
\multirow{3}{*}{(PN,PU)} & WLMC & \textbf{1.48$\pm$0.20} & \textbf{10.06$\pm$0.26} & \textbf{0.26$\pm$0.03} & \textbf{2.30$\pm$0.52} & \textbf{17.22$\pm$2.15} & \textbf{0.72$\pm$0.19} & 82.68$\pm$0.11 \\
 & Temp & 2.00$\pm$0.43 & 10.23$\pm$0.29 & 0.40$\pm$0.16 & 2.89$\pm$0.67 & 17.55$\pm$2.81 & 0.89$\pm$0.39 & 82.68$\pm$0.11 \\
 & Platt & 2.44$\pm$0.92 & 10.40$\pm$0.70 & 0.42$\pm$0.14 & 3.22$\pm$0.89 & 17.80$\pm$2.51 & 0.76$\pm$0.30 & \textbf{82.68$\pm$0.13} \\
\midrule
\multirow{3}{*}{(PN,Pconf)} & WLMC & \textbf{1.49$\pm$0.18} & \textbf{10.06$\pm$0.22} & \textbf{0.26$\pm$0.03} & \textbf{3.98$\pm$0.17} & 6.38$\pm$0.29 & \textbf{1.03$\pm$0.06} & \textbf{82.68$\pm$0.11} \\
 & Temp & 5.98$\pm$0.17 & 12.24$\pm$0.24 & 1.23$\pm$0.06 & 4.75$\pm$0.17 & \textbf{6.00$\pm$0.21} & 1.05$\pm$0.05 & 82.68$\pm$0.11 \\
 & Platt & 7.44$\pm$0.22 & 14.81$\pm$0.28 & 1.44$\pm$0.02 & 4.76$\pm$0.15 & 6.39$\pm$0.28 & 1.19$\pm$0.02 & 80.79$\pm$0.30 \\
\midrule
\multirow{3}{*}{(PN,UU)} & WLMC & \textbf{1.50$\pm$0.19} & 10.06$\pm$0.21 & \textbf{0.26$\pm$0.03} & \textbf{1.50$\pm$0.26} & \textbf{11.52$\pm$2.09} & \textbf{0.32$\pm$0.05} & 82.68$\pm$0.11 \\
 & Temp & 1.54$\pm$0.40 & 10.08$\pm$0.19 & 0.28$\pm$0.04 & 1.57$\pm$0.46 & 11.69$\pm$2.37 & 0.34$\pm$0.05 & 82.68$\pm$0.11 \\
 & Platt & 1.57$\pm$0.37 & \textbf{10.05$\pm$0.30} & 0.28$\pm$0.04 & 1.57$\pm$0.43 & 11.62$\pm$2.02 & 0.35$\pm$0.08 & \textbf{82.68$\pm$0.13} \\
\midrule
\multirow{4}{*}{(PU,PU)} & ERM & 3.00$\pm$1.61 & 14.96$\pm$1.84 & 1.88$\pm$0.39 & 3.19$\pm$1.46 & 22.16$\pm$5.94 & 1.91$\pm$0.51 & 75.41$\pm$0.04 \\
 & WLMC & \textbf{1.76$\pm$0.74} & \textbf{14.23$\pm$1.08} & \textbf{0.71$\pm$0.15} & \textbf{2.20$\pm$0.69} & 23.48$\pm$5.66 & \textbf{0.99$\pm$0.35} & \textbf{75.43$\pm$0.09} \\
 & Temp & 2.93$\pm$1.53 & 14.66$\pm$1.82 & 1.85$\pm$0.36 & 3.13$\pm$1.40 & 22.03$\pm$5.67 & 1.87$\pm$0.51 & 75.41$\pm$0.04 \\
 & Platt & 2.85$\pm$1.60 & 15.02$\pm$1.32 & 1.86$\pm$0.43 & 3.01$\pm$1.53 & \textbf{21.79$\pm$5.08} & 1.89$\pm$0.61 & 75.40$\pm$0.06 \\
\midrule
\multirow{4}{*}{(Pconf,Pconf)} & ERM & 10.51$\pm$0.15 & 26.46$\pm$0.98 & 4.32$\pm$0.58 & 7.78$\pm$0.20 & 13.52$\pm$0.41 & 3.40$\pm$0.59 & \textbf{75.44$\pm$0.00} \\
 & WLMC & 3.61$\pm$1.21 & 18.44$\pm$1.61 & \textbf{1.22$\pm$0.41} & 1.65$\pm$0.53 & 6.57$\pm$0.80 & \textbf{0.81$\pm$0.54} & 75.44$\pm$0.00 \\
 & Temp & 2.79$\pm$0.81 & 18.25$\pm$0.78 & 1.49$\pm$0.10 & \textbf{1.11$\pm$1.02} & 6.16$\pm$0.32 & 1.65$\pm$0.41 & 75.44$\pm$0.00 \\
 & Platt & \textbf{2.44$\pm$0.44} & \textbf{17.76$\pm$1.53} & 1.29$\pm$0.27 & 1.37$\pm$1.46 & \textbf{5.94$\pm$0.35} & 1.70$\pm$0.35 & 75.44$\pm$0.00 \\
\midrule
\multirow{4}{*}{(UU,UU)} & ERM & 2.53$\pm$1.03 & 14.14$\pm$1.25 & 1.76$\pm$0.29 & 2.65$\pm$0.95 & 13.54$\pm$1.94 & 1.81$\pm$0.32 & 75.40$\pm$0.04 \\
 & WLMC & \textbf{1.63$\pm$0.58} & 14.36$\pm$0.97 & \textbf{0.56$\pm$0.07} & \textbf{1.64$\pm$0.64} & \textbf{13.51$\pm$1.91} & \textbf{0.58$\pm$0.12} & \textbf{75.40$\pm$0.04} \\
 & Temp & 2.55$\pm$0.97 & 14.41$\pm$1.60 & 1.79$\pm$0.38 & 2.63$\pm$0.90 & 14.46$\pm$2.27 & 1.84$\pm$0.38 & 75.40$\pm$0.04 \\
 & Platt & 2.55$\pm$1.03 & \textbf{14.05$\pm$1.46} & 1.80$\pm$0.36 & 2.66$\pm$1.00 & 13.85$\pm$2.27 & 1.86$\pm$0.38 & 75.40$\pm$0.04 \\
\bottomrule
\end{tabular}%
}
\end{table*}

\begin{table*}[t]
\centering
\scriptsize
\setlength{\tabcolsep}{3.2pt}
\renewcommand{\arraystretch}{0.92}
\caption{Tabular test results on HMDA with NaiveBayes as the base predictor. Calibration metrics are reported as $10^2\!\times$ mean $\pm$ standard deviation over ten seeds; accuracy is reported in percent. WLMC uses step size $0.05$ and at most $50$ updates. Bold marks the single best mean within each $(\mathrm{base},\mathrm{post})$ block; variance is not used for bolding.}
\label{tab:tabular-hmda-naivebayes-main}
\resizebox{\textwidth}{!}{%
\begin{tabular}{llccccccc}
\toprule
\multirow{2}{*}{(Base,Post)} & \multirow{2}{*}{Post-processing} & \multicolumn{3}{c}{oracle estimate ($\times 10^{-2}$)} & \multicolumn{3}{c}{weak estimate ($\times 10^{-2}$)} & \multirow{2}{*}{Acc. (\%)} \\
\cmidrule(lr){3-5}\cmidrule(lr){6-8}
& & ECE & maxECE & MC & ECE & maxECE & MC & \\
\midrule
\multirow{4}{*}{(PN,PN)} & ERM & 13.44$\pm$0.20 & 19.83$\pm$0.60 & 4.19$\pm$0.08 & 13.44$\pm$0.20 & 19.83$\pm$0.59 & 4.19$\pm$0.08 & 80.80$\pm$0.21 \\
 & WLMC & 4.95$\pm$0.59 & 12.36$\pm$1.12 & 1.07$\pm$0.15 & 4.95$\pm$0.59 & 12.36$\pm$1.12 & 1.07$\pm$0.15 & 80.80$\pm$0.21 \\
 & Temp & 9.47$\pm$0.20 & 17.90$\pm$0.58 & 2.03$\pm$0.05 & 9.47$\pm$0.20 & 17.90$\pm$0.58 & 2.03$\pm$0.05 & 80.80$\pm$0.21 \\
 & Platt & \textbf{3.72$\pm$0.19} & \textbf{11.27$\pm$0.25} & \textbf{0.98$\pm$0.07} & \textbf{3.72$\pm$0.19} & \textbf{11.27$\pm$0.25} & \textbf{0.98$\pm$0.07} & \textbf{81.70$\pm$0.18} \\
\midrule
\multirow{3}{*}{(PN,PU)} & WLMC & 3.47$\pm$0.91 & \textbf{10.09$\pm$1.24} & 0.80$\pm$0.30 & 3.83$\pm$1.13 & \textbf{16.76$\pm$3.42} & 1.02$\pm$0.47 & 80.83$\pm$0.22 \\
 & Temp & 7.58$\pm$0.33 & 14.09$\pm$0.50 & 0.81$\pm$0.04 & 7.66$\pm$0.42 & 19.34$\pm$1.58 & 1.08$\pm$0.29 & 80.80$\pm$0.21 \\
 & Platt & \textbf{3.16$\pm$0.17} & 13.11$\pm$0.54 & \textbf{0.67$\pm$0.07} & \textbf{3.75$\pm$0.53} & 20.01$\pm$2.77 & \textbf{0.97$\pm$0.31} & \textbf{81.64$\pm$0.19} \\
\midrule
\multirow{3}{*}{(PN,Pconf)} & WLMC & \textbf{3.15$\pm$0.40} & \textbf{11.47$\pm$0.92} & \textbf{0.91$\pm$0.12} & 2.07$\pm$0.31 & 7.63$\pm$0.78 & 0.47$\pm$0.06 & 80.79$\pm$0.39 \\
 & Temp & 8.25$\pm$0.26 & 13.23$\pm$0.37 & 1.28$\pm$0.03 & 8.54$\pm$0.29 & 14.56$\pm$0.36 & 0.75$\pm$0.07 & 80.80$\pm$0.21 \\
 & Platt & 4.58$\pm$0.21 & 14.71$\pm$0.30 & 0.95$\pm$0.05 & \textbf{1.37$\pm$0.17} & \textbf{5.73$\pm$0.30} & \textbf{0.36$\pm$0.03} & \textbf{80.81$\pm$0.18} \\
\midrule
\multirow{3}{*}{(PN,UU)} & WLMC & 4.97$\pm$0.75 & \textbf{11.99$\pm$1.40} & 1.08$\pm$0.14 & 5.04$\pm$0.78 & \textbf{12.98$\pm$1.35} & 1.13$\pm$0.21 & 80.80$\pm$0.21 \\
 & Temp & 7.59$\pm$0.32 & 14.05$\pm$0.39 & 0.79$\pm$0.03 & 7.62$\pm$0.32 & 14.39$\pm$1.59 & 0.81$\pm$0.07 & 80.80$\pm$0.21 \\
 & Platt & \textbf{2.97$\pm$0.18} & 12.91$\pm$0.42 & \textbf{0.66$\pm$0.06} & \textbf{3.10$\pm$0.29} & 14.09$\pm$1.69 & \textbf{0.69$\pm$0.08} & \textbf{81.65$\pm$0.19} \\
\midrule
\multirow{4}{*}{(PU,PU)} & ERM & 19.68$\pm$5.01 & 29.76$\pm$2.78 & 9.86$\pm$8.04 & 19.74$\pm$5.02 & 31.53$\pm$3.24 & 9.99$\pm$8.05 & 76.11$\pm$1.08 \\
 & WLMC & 2.24$\pm$0.77 & \textbf{12.47$\pm$2.79} & \textbf{0.79$\pm$0.23} & 2.45$\pm$0.94 & \textbf{19.56$\pm$3.91} & \textbf{1.05$\pm$0.39} & 76.57$\pm$1.11 \\
 & Temp & 15.90$\pm$8.31 & 23.53$\pm$7.37 & 8.34$\pm$9.35 & 15.97$\pm$8.24 & 25.86$\pm$7.28 & 8.46$\pm$9.36 & 55.76$\pm$26.87 \\
 & Platt & \textbf{1.29$\pm$1.09} & 14.66$\pm$1.66 & 1.46$\pm$1.21 & \textbf{1.73$\pm$1.38} & 19.85$\pm$4.26 & 1.70$\pm$1.14 & \textbf{77.68$\pm$2.05} \\
\midrule
\multirow{4}{*}{(Pconf,Pconf)} & ERM & 15.62$\pm$0.18 & 25.92$\pm$0.74 & 3.11$\pm$0.12 & 14.97$\pm$0.15 & 25.29$\pm$0.62 & 2.76$\pm$0.07 & 75.19$\pm$0.14 \\
 & WLMC & 4.50$\pm$0.49 & \textbf{15.12$\pm$1.04} & 0.52$\pm$0.05 & 4.38$\pm$0.43 & 13.03$\pm$3.26 & 0.56$\pm$0.04 & 77.08$\pm$0.20 \\
 & Temp & 9.54$\pm$0.17 & 15.26$\pm$0.53 & 1.41$\pm$0.05 & 10.88$\pm$0.16 & 17.93$\pm$0.50 & 1.51$\pm$0.04 & 75.19$\pm$0.14 \\
 & Platt & \textbf{2.78$\pm$0.22} & 16.11$\pm$0.27 & \textbf{0.51$\pm$0.06} & \textbf{0.95$\pm$0.11} & \textbf{5.36$\pm$0.52} & \textbf{0.22$\pm$0.03} & \textbf{78.10$\pm$0.13} \\
\midrule
\multirow{4}{*}{(UU,UU)} & ERM & 14.93$\pm$3.73 & 20.64$\pm$4.52 & 5.92$\pm$4.68 & 15.04$\pm$3.72 & 22.34$\pm$3.90 & 6.04$\pm$4.70 & 80.12$\pm$1.77 \\
 & WLMC & 4.52$\pm$1.39 & \textbf{11.83$\pm$2.70} & 1.10$\pm$0.19 & 4.64$\pm$1.37 & \textbf{12.40$\pm$1.80} & 1.14$\pm$0.17 & 80.12$\pm$1.78 \\
 & Temp & 9.28$\pm$5.71 & 14.70$\pm$6.23 & 2.73$\pm$5.79 & 9.41$\pm$5.66 & 15.40$\pm$5.52 & 2.78$\pm$5.81 & 75.04$\pm$17.75 \\
 & Platt & \textbf{2.91$\pm$1.20} & 12.95$\pm$1.57 & \textbf{0.87$\pm$0.75} & \textbf{3.01$\pm$1.28} & 13.99$\pm$1.88 & \textbf{0.91$\pm$0.77} & \textbf{80.99$\pm$1.98} \\
\bottomrule
\end{tabular}%
}
\end{table*}

\begin{table*}[t]
\centering
\scriptsize
\setlength{\tabcolsep}{3.2pt}
\renewcommand{\arraystretch}{0.92}
\caption{Tabular test results on HMDA with RandomForest as the base predictor. Calibration metrics are reported as $10^2\!\times$ mean $\pm$ standard deviation over ten seeds; accuracy is reported in percent. WLMC uses step size $0.05$ and at most $50$ updates. Bold marks the single best mean within each $(\mathrm{base},\mathrm{post})$ block; variance is not used for bolding.}
\label{tab:tabular-hmda-randomforest-main}
\resizebox{\textwidth}{!}{%
\begin{tabular}{llccccccc}
\toprule
\multirow{2}{*}{(Base,Post)} & \multirow{2}{*}{Post-processing} & \multicolumn{3}{c}{oracle estimate ($\times 10^{-2}$)} & \multicolumn{3}{c}{weak estimate ($\times 10^{-2}$)} & \multirow{2}{*}{Acc. (\%)} \\
\cmidrule(lr){3-5}\cmidrule(lr){6-8}
& & ECE & maxECE & MC & ECE & maxECE & MC & \\
\midrule
\multirow{4}{*}{(PN,PN)} & ERM & 3.97$\pm$0.25 & 10.21$\pm$0.97 & 0.88$\pm$0.07 & 3.97$\pm$0.25 & 10.21$\pm$0.97 & 0.88$\pm$0.07 & 86.60$\pm$0.10 \\
 & WLMC & \textbf{3.09$\pm$0.17} & 9.38$\pm$0.85 & \textbf{0.39$\pm$0.04} & \textbf{3.09$\pm$0.17} & 9.38$\pm$0.85 & \textbf{0.39$\pm$0.04} & 86.60$\pm$0.10 \\
 & Temp & 7.70$\pm$14.11 & 13.35$\pm$13.00 & 2.36$\pm$5.40 & 7.70$\pm$14.11 & 13.36$\pm$13.00 & 2.36$\pm$5.40 & 79.28$\pm$23.14 \\
 & Platt & 3.42$\pm$0.17 & \textbf{6.67$\pm$0.82} & 0.53$\pm$0.07 & 3.42$\pm$0.16 & \textbf{6.67$\pm$0.82} & 0.53$\pm$0.07 & \textbf{86.62$\pm$0.10} \\
\midrule
\multirow{3}{*}{(PN,PU)} & WLMC & 3.11$\pm$0.46 & 9.25$\pm$0.99 & 0.48$\pm$0.20 & 3.71$\pm$0.80 & 20.54$\pm$4.43 & 0.82$\pm$0.33 & 86.60$\pm$0.10 \\
 & Temp & 2.46$\pm$0.67 & 6.66$\pm$1.09 & \textbf{0.34$\pm$0.16} & 3.06$\pm$0.53 & \textbf{18.34$\pm$2.54} & \textbf{0.69$\pm$0.23} & 86.60$\pm$0.10 \\
 & Platt & \textbf{2.06$\pm$0.80} & \textbf{6.00$\pm$1.15} & 0.37$\pm$0.14 & \textbf{2.82$\pm$0.62} & 18.93$\pm$2.26 & 0.70$\pm$0.28 & \textbf{86.67$\pm$0.09} \\
\midrule
\multirow{3}{*}{(PN,Pconf)} & WLMC & \textbf{3.97$\pm$0.25} & \textbf{10.21$\pm$0.97} & \textbf{0.88$\pm$0.07} & 0.41$\pm$0.18 & 1.07$\pm$0.16 & 0.15$\pm$0.09 & 86.60$\pm$0.10 \\
 & Temp & 4.42$\pm$0.18 & 10.61$\pm$0.76 & 0.99$\pm$0.04 & 0.24$\pm$0.09 & 0.93$\pm$0.18 & 0.05$\pm$0.02 & 86.60$\pm$0.10 \\
 & Platt & 4.46$\pm$0.20 & 10.76$\pm$0.83 & 0.99$\pm$0.04 & \textbf{0.23$\pm$0.08} & \textbf{0.93$\pm$0.17} & \textbf{0.05$\pm$0.02} & \textbf{86.60$\pm$0.10} \\
\midrule
\multirow{3}{*}{(PN,UU)} & WLMC & 2.97$\pm$0.26 & 9.20$\pm$0.78 & 0.37$\pm$0.02 & 2.99$\pm$0.34 & 11.22$\pm$2.08 & 0.42$\pm$0.07 & 86.60$\pm$0.10 \\
 & Temp & 1.98$\pm$0.19 & 6.65$\pm$0.37 & \textbf{0.22$\pm$0.03} & 2.14$\pm$0.32 & \textbf{10.33$\pm$0.71} & 0.31$\pm$0.07 & 86.60$\pm$0.10 \\
 & Platt & \textbf{1.49$\pm$0.13} & \textbf{6.12$\pm$0.76} & 0.23$\pm$0.04 & \textbf{1.78$\pm$0.31} & 10.96$\pm$1.62 & \textbf{0.31$\pm$0.10} & \textbf{86.67$\pm$0.10} \\
\midrule
\multirow{4}{*}{(PU,PU)} & ERM & 17.60$\pm$2.91 & 31.69$\pm$5.65 & 6.18$\pm$2.14 & 17.31$\pm$2.84 & 35.62$\pm$6.59 & 6.06$\pm$2.24 & 76.98$\pm$1.95 \\
 & WLMC & \textbf{4.98$\pm$1.61} & 19.12$\pm$3.14 & \textbf{0.90$\pm$0.31} & \textbf{4.86$\pm$1.70} & 25.52$\pm$5.59 & \textbf{1.04$\pm$0.29} & \textbf{77.55$\pm$1.98} \\
 & Temp & 11.97$\pm$4.78 & 17.91$\pm$2.55 & 3.52$\pm$0.72 & 12.00$\pm$4.73 & 22.08$\pm$5.24 & 3.46$\pm$0.74 & 76.98$\pm$1.95 \\
 & Platt & 5.90$\pm$2.76 & \textbf{16.28$\pm$1.91} & 2.29$\pm$0.76 & 6.01$\pm$2.82 & \textbf{19.61$\pm$6.00} & 2.49$\pm$1.04 & 76.33$\pm$3.99 \\
\midrule
\multirow{4}{*}{(Pconf,Pconf)} & ERM & 12.06$\pm$0.13 & 22.07$\pm$0.56 & 2.20$\pm$0.10 & 7.00$\pm$0.11 & 9.60$\pm$0.20 & 1.22$\pm$0.04 & 77.62$\pm$0.22 \\
 & WLMC & 8.09$\pm$0.27 & 17.13$\pm$0.72 & 0.93$\pm$0.09 & 3.41$\pm$0.25 & 8.48$\pm$0.37 & 0.53$\pm$0.09 & 81.83$\pm$0.42 \\
 & Temp & 10.79$\pm$0.17 & 22.38$\pm$1.11 & 1.91$\pm$0.06 & 6.07$\pm$0.09 & 8.40$\pm$0.22 & 1.17$\pm$0.07 & 77.62$\pm$0.22 \\
 & Platt & \textbf{5.18$\pm$0.11} & \textbf{14.22$\pm$0.91} & \textbf{0.70$\pm$0.04} & \textbf{1.19$\pm$0.03} & \textbf{2.66$\pm$0.19} & \textbf{0.16$\pm$0.01} & \textbf{84.98$\pm$0.16} \\
\midrule
\multirow{4}{*}{(UU,UU)} & ERM & 19.80$\pm$2.61 & 24.91$\pm$3.62 & 6.87$\pm$2.09 & 19.82$\pm$2.73 & 29.13$\pm$5.46 & 6.88$\pm$2.22 & 79.62$\pm$2.07 \\
 & WLMC & \textbf{3.99$\pm$0.67} & 10.51$\pm$0.50 & \textbf{0.51$\pm$0.06} & \textbf{4.08$\pm$0.65} & 11.65$\pm$1.46 & \textbf{0.55$\pm$0.10} & 80.36$\pm$1.28 \\
 & Temp & 28.64$\pm$0.50 & 33.61$\pm$1.14 & 10.78$\pm$2.10 & 28.65$\pm$0.52 & 33.34$\pm$1.72 & 10.77$\pm$1.98 & 79.62$\pm$2.07 \\
 & Platt & 7.17$\pm$1.18 & \textbf{9.09$\pm$1.41} & 2.02$\pm$0.37 & 7.23$\pm$1.04 & \textbf{9.99$\pm$1.85} & 2.03$\pm$0.34 & \textbf{81.44$\pm$2.66} \\
\bottomrule
\end{tabular}%
}
\end{table*}

\begin{table*}[t]
\centering
\scriptsize
\setlength{\tabcolsep}{3.2pt}
\renewcommand{\arraystretch}{0.92}
\caption{Tabular test results on HMDA with SVM as the base predictor. Calibration metrics are reported as $10^2\!\times$ mean $\pm$ standard deviation over ten seeds; accuracy is reported in percent. WLMC uses step size $0.05$ and at most $50$ updates. Bold marks the single best mean within each $(\mathrm{base},\mathrm{post})$ block; variance is not used for bolding.}
\label{tab:tabular-hmda-svm-main}
\resizebox{\textwidth}{!}{%
\begin{tabular}{llccccccc}
\toprule
\multirow{2}{*}{(Base,Post)} & \multirow{2}{*}{Post-processing} & \multicolumn{3}{c}{oracle estimate ($\times 10^{-2}$)} & \multicolumn{3}{c}{weak estimate ($\times 10^{-2}$)} & \multirow{2}{*}{Acc. (\%)} \\
\cmidrule(lr){3-5}\cmidrule(lr){6-8}
& & ECE & maxECE & MC & ECE & maxECE & MC & \\
\midrule
\multirow{4}{*}{(PN,PN)} & ERM & 6.51$\pm$0.85 & 19.54$\pm$6.54 & 1.06$\pm$0.32 & 6.51$\pm$0.85 & 19.53$\pm$6.54 & 1.06$\pm$0.32 & 84.31$\pm$0.71 \\
 & WLMC & 4.78$\pm$0.83 & 16.36$\pm$7.17 & 0.70$\pm$0.37 & 4.78$\pm$0.83 & 16.36$\pm$7.17 & 0.70$\pm$0.37 & \textbf{84.33$\pm$0.72} \\
 & Temp & \textbf{2.96$\pm$0.79} & 17.17$\pm$6.58 & \textbf{0.41$\pm$0.09} & \textbf{2.96$\pm$0.79} & 17.15$\pm$6.61 & \textbf{0.41$\pm$0.09} & 84.31$\pm$0.71 \\
 & Platt & 6.16$\pm$1.59 & \textbf{14.26$\pm$8.66} & 1.35$\pm$0.39 & 6.16$\pm$1.59 & \textbf{14.26$\pm$8.66} & 1.35$\pm$0.39 & 84.14$\pm$0.78 \\
\midrule
\multirow{3}{*}{(PN,PU)} & WLMC & 4.67$\pm$1.25 & 16.53$\pm$7.39 & 0.64$\pm$0.13 & 4.78$\pm$1.18 & \textbf{21.99$\pm$8.67} & 0.82$\pm$0.24 & 84.34$\pm$0.71 \\
 & Temp & 2.11$\pm$0.47 & 15.88$\pm$6.80 & 0.39$\pm$0.10 & 2.85$\pm$0.58 & 22.64$\pm$6.91 & 0.74$\pm$0.25 & 84.31$\pm$0.71 \\
 & Platt & \textbf{1.88$\pm$0.50} & \textbf{15.54$\pm$6.93} & \textbf{0.34$\pm$0.06} & \textbf{2.54$\pm$0.65} & 22.60$\pm$7.16 & \textbf{0.71$\pm$0.26} & \textbf{84.46$\pm$0.71} \\
\midrule
\multirow{3}{*}{(PN,Pconf)} & WLMC & 5.53$\pm$0.90 & 14.80$\pm$6.78 & 1.49$\pm$0.37 & 4.65$\pm$0.68 & 12.53$\pm$1.34 & 0.65$\pm$0.30 & \textbf{84.33$\pm$0.69} \\
 & Temp & \textbf{4.61$\pm$0.73} & \textbf{12.60$\pm$5.02} & 1.34$\pm$0.11 & 1.76$\pm$0.49 & 8.79$\pm$3.02 & 0.59$\pm$0.12 & 84.31$\pm$0.71 \\
 & Platt & 4.91$\pm$0.34 & 13.31$\pm$4.71 & \textbf{1.30$\pm$0.09} & \textbf{1.42$\pm$0.39} & \textbf{8.48$\pm$2.82} & \textbf{0.44$\pm$0.14} & 84.12$\pm$0.86 \\
\midrule
\multirow{3}{*}{(PN,UU)} & WLMC & 4.79$\pm$0.86 & 15.73$\pm$7.06 & 0.80$\pm$0.46 & 4.95$\pm$0.88 & 16.03$\pm$4.93 & 0.85$\pm$0.44 & 84.32$\pm$0.70 \\
 & Temp & 1.66$\pm$0.62 & 15.15$\pm$6.50 & 0.38$\pm$0.12 & 1.78$\pm$0.60 & 15.93$\pm$3.69 & 0.43$\pm$0.14 & 84.31$\pm$0.71 \\
 & Platt & \textbf{1.34$\pm$0.31} & \textbf{14.85$\pm$6.40} & \textbf{0.31$\pm$0.08} & \textbf{1.40$\pm$0.29} & \textbf{15.45$\pm$3.92} & \textbf{0.37$\pm$0.09} & \textbf{84.45$\pm$0.73} \\
\midrule
\multirow{4}{*}{(PU,PU)} & ERM & 12.80$\pm$1.65 & 25.97$\pm$8.18 & 3.56$\pm$0.79 & 12.46$\pm$1.95 & 29.21$\pm$9.20 & 3.52$\pm$0.93 & 82.12$\pm$1.39 \\
 & WLMC & 4.13$\pm$1.02 & \textbf{17.42$\pm$7.81} & 0.93$\pm$0.39 & 4.28$\pm$1.26 & \textbf{22.02$\pm$7.83} & 1.23$\pm$0.48 & \textbf{82.14$\pm$1.39} \\
 & Temp & 4.16$\pm$0.80 & 19.84$\pm$8.78 & 0.78$\pm$0.30 & 4.66$\pm$1.08 & 25.54$\pm$9.12 & 1.15$\pm$0.42 & 82.12$\pm$1.39 \\
 & Platt & \textbf{3.82$\pm$0.92} & 19.61$\pm$8.35 & \textbf{0.72$\pm$0.28} & \textbf{4.21$\pm$1.15} & 24.84$\pm$8.91 & \textbf{1.10$\pm$0.45} & 81.99$\pm$1.51 \\
\midrule
\multirow{4}{*}{(Pconf,Pconf)} & ERM & 9.03$\pm$2.11 & 31.61$\pm$6.96 & 2.28$\pm$1.08 & 8.07$\pm$2.32 & 19.12$\pm$6.47 & 2.16$\pm$1.17 & 76.69$\pm$0.70 \\
 & WLMC & 4.71$\pm$0.64 & 27.91$\pm$7.18 & \textbf{0.81$\pm$0.23} & 3.36$\pm$0.43 & 18.23$\pm$7.48 & \textbf{0.53$\pm$0.06} & \textbf{76.72$\pm$0.68} \\
 & Temp & \textbf{3.48$\pm$1.53} & 26.45$\pm$7.49 & 0.96$\pm$0.38 & 3.19$\pm$1.02 & 16.45$\pm$7.04 & 0.85$\pm$0.34 & 76.69$\pm$0.70 \\
 & Platt & 4.33$\pm$1.18 & \textbf{26.38$\pm$7.23} & 1.05$\pm$0.47 & \textbf{2.04$\pm$1.11} & \textbf{14.72$\pm$5.72} & 0.84$\pm$0.44 & 76.53$\pm$0.84 \\
\midrule
\multirow{4}{*}{(UU,UU)} & ERM & 2.43$\pm$0.85 & 9.67$\pm$3.90 & 0.42$\pm$0.12 & 2.73$\pm$0.88 & 14.06$\pm$3.96 & 0.44$\pm$0.13 & 85.48$\pm$0.58 \\
 & WLMC & 2.33$\pm$0.78 & 9.52$\pm$3.85 & 0.37$\pm$0.09 & 2.63$\pm$0.83 & 13.80$\pm$3.47 & 0.39$\pm$0.11 & 85.48$\pm$0.58 \\
 & Temp & 1.79$\pm$0.68 & 9.61$\pm$2.64 & \textbf{0.33$\pm$0.06} & 1.85$\pm$0.66 & 13.06$\pm$2.97 & \textbf{0.35$\pm$0.04} & 85.48$\pm$0.58 \\
 & Platt & \textbf{1.67$\pm$0.66} & \textbf{9.22$\pm$2.68} & 0.37$\pm$0.12 & \textbf{1.71$\pm$0.62} & \textbf{12.98$\pm$2.81} & 0.37$\pm$0.10 & \textbf{85.52$\pm$0.58} \\
\bottomrule
\end{tabular}%
}
\end{table*}

\begin{table*}[t]
\centering
\scriptsize
\setlength{\tabcolsep}{3.2pt}
\renewcommand{\arraystretch}{0.92}
\caption{Tabular test results on HMDA with MLP as the base predictor. Calibration metrics are reported as $10^2\!\times$ mean $\pm$ standard deviation over ten seeds; accuracy is reported in percent. WLMC uses step size $0.05$ and at most $50$ updates. Bold marks the single best mean within each $(\mathrm{base},\mathrm{post})$ block; variance is not used for bolding.}
\label{tab:tabular-hmda-mlp-main}
\resizebox{\textwidth}{!}{%
\begin{tabular}{llccccccc}
\toprule
\multirow{2}{*}{(Base,Post)} & \multirow{2}{*}{Post-processing} & \multicolumn{3}{c}{oracle estimate ($\times 10^{-2}$)} & \multicolumn{3}{c}{weak estimate ($\times 10^{-2}$)} & \multirow{2}{*}{Acc. (\%)} \\
\cmidrule(lr){3-5}\cmidrule(lr){6-8}
& & ECE & maxECE & MC & ECE & maxECE & MC & \\
\midrule
\multirow{4}{*}{(PN,PN)} & ERM & 15.60$\pm$1.56 & 19.81$\pm$2.70 & 6.30$\pm$1.30 & 15.60$\pm$1.56 & 19.81$\pm$2.70 & 6.30$\pm$1.30 & 80.99$\pm$1.02 \\
 & WLMC & 3.27$\pm$0.98 & 10.61$\pm$1.42 & \textbf{0.63$\pm$0.23} & 3.27$\pm$0.97 & 10.61$\pm$1.42 & \textbf{0.63$\pm$0.23} & 81.00$\pm$1.00 \\
 & Temp & 12.78$\pm$1.95 & 16.78$\pm$3.02 & 4.64$\pm$1.19 & 12.78$\pm$1.95 & 16.78$\pm$3.02 & 4.64$\pm$1.19 & 80.99$\pm$1.02 \\
 & Platt & \textbf{2.91$\pm$1.15} & \textbf{8.73$\pm$1.73} & 0.67$\pm$0.35 & \textbf{2.91$\pm$1.15} & \textbf{8.73$\pm$1.73} & 0.67$\pm$0.35 & \textbf{81.04$\pm$0.80} \\
\midrule
\multirow{3}{*}{(PN,PU)} & WLMC & \textbf{3.91$\pm$1.51} & 10.23$\pm$2.52 & 1.25$\pm$0.71 & \textbf{4.27$\pm$1.67} & \textbf{17.36$\pm$6.76} & 1.44$\pm$0.86 & 81.00$\pm$1.00 \\
 & Temp & 6.12$\pm$2.02 & 10.84$\pm$2.13 & 1.40$\pm$0.55 & 6.26$\pm$2.11 & 20.85$\pm$4.11 & 1.58$\pm$0.52 & 80.99$\pm$1.02 \\
 & Platt & 4.53$\pm$0.93 & \textbf{9.69$\pm$1.49} & \textbf{1.02$\pm$0.33} & 4.71$\pm$0.99 & 20.71$\pm$4.25 & \textbf{1.27$\pm$0.38} & \textbf{81.32$\pm$0.66} \\
\midrule
\multirow{3}{*}{(PN,Pconf)} & WLMC & \textbf{3.08$\pm$1.61} & \textbf{10.59$\pm$2.03} & \textbf{0.83$\pm$0.56} & 3.45$\pm$1.46 & 9.36$\pm$3.36 & 0.87$\pm$0.50 & 80.95$\pm$1.15 \\
 & Temp & 6.23$\pm$2.34 & 11.09$\pm$2.46 & 1.82$\pm$0.59 & 5.17$\pm$2.95 & 8.29$\pm$3.63 & 1.15$\pm$0.61 & 80.99$\pm$1.02 \\
 & Platt & 5.65$\pm$0.79 & 12.09$\pm$1.30 & 1.51$\pm$0.24 & \textbf{3.05$\pm$0.82} & \textbf{5.09$\pm$0.71} & \textbf{0.80$\pm$0.24} & \textbf{81.10$\pm$0.63} \\
\midrule
\multirow{3}{*}{(PN,UU)} & WLMC & \textbf{3.49$\pm$0.76} & 10.90$\pm$1.76 & \textbf{0.71$\pm$0.30} & \textbf{3.63$\pm$0.82} & 12.29$\pm$2.08 & \textbf{0.74$\pm$0.25} & 81.00$\pm$1.00 \\
 & Temp & 6.11$\pm$1.95 & 10.92$\pm$1.94 & 1.35$\pm$0.51 & 6.19$\pm$1.99 & 14.35$\pm$2.92 & 1.34$\pm$0.52 & 80.99$\pm$1.02 \\
 & Platt & 4.52$\pm$0.93 & \textbf{9.48$\pm$1.58} & 0.95$\pm$0.33 & 4.74$\pm$0.99 & \textbf{11.86$\pm$1.53} & 0.95$\pm$0.38 & \textbf{81.33$\pm$0.67} \\
\midrule
\multirow{4}{*}{(PU,PU)} & ERM & 17.29$\pm$6.78 & 26.77$\pm$10.09 & 8.16$\pm$4.40 & 17.36$\pm$6.67 & 29.50$\pm$8.80 & 8.30$\pm$4.39 & 77.16$\pm$1.39 \\
 & WLMC & 3.69$\pm$1.97 & 13.77$\pm$2.46 & \textbf{1.09$\pm$0.64} & 4.31$\pm$2.14 & \textbf{19.39$\pm$4.94} & \textbf{1.24$\pm$0.73} & 77.17$\pm$1.39 \\
 & Temp & 2.56$\pm$1.48 & \textbf{12.18$\pm$2.49} & 1.43$\pm$0.98 & 2.97$\pm$1.28 & 19.59$\pm$4.91 & 1.71$\pm$1.21 & 77.16$\pm$1.39 \\
 & Platt & \textbf{1.94$\pm$1.19} & 12.57$\pm$2.37 & 1.25$\pm$0.91 & \textbf{2.40$\pm$1.01} & 19.73$\pm$4.29 & 1.59$\pm$1.17 & \textbf{77.20$\pm$1.45} \\
\midrule
\multirow{4}{*}{(Pconf,Pconf)} & ERM & 18.40$\pm$2.62 & 27.25$\pm$4.92 & 8.89$\pm$2.59 & 16.53$\pm$2.07 & 19.43$\pm$3.21 & 8.20$\pm$2.29 & 77.79$\pm$1.08 \\
 & WLMC & 5.80$\pm$1.70 & 15.92$\pm$3.28 & 1.67$\pm$0.99 & 4.46$\pm$1.07 & 10.43$\pm$2.81 & 1.26$\pm$0.61 & 77.78$\pm$1.07 \\
 & Temp & 5.64$\pm$1.87 & 14.53$\pm$3.20 & 1.12$\pm$0.27 & 3.89$\pm$1.89 & 5.65$\pm$2.11 & 0.79$\pm$0.38 & 77.79$\pm$1.08 \\
 & Platt & \textbf{4.08$\pm$0.94} & \textbf{13.81$\pm$1.59} & \textbf{0.89$\pm$0.25} & \textbf{2.08$\pm$0.90} & \textbf{4.28$\pm$0.78} & \textbf{0.60$\pm$0.38} & \textbf{78.41$\pm$0.98} \\
\midrule
\multirow{4}{*}{(UU,UU)} & ERM & 7.68$\pm$4.35 & 13.05$\pm$4.56 & 2.01$\pm$1.97 & 7.82$\pm$4.31 & 14.75$\pm$4.27 & 2.12$\pm$2.07 & 81.90$\pm$1.01 \\
 & WLMC & 3.85$\pm$1.26 & 9.43$\pm$1.92 & 0.58$\pm$0.33 & 4.01$\pm$1.20 & 11.29$\pm$3.10 & 0.64$\pm$0.33 & 82.15$\pm$1.10 \\
 & Temp & 4.79$\pm$2.65 & 10.78$\pm$2.40 & 0.61$\pm$0.27 & 4.90$\pm$2.60 & 12.74$\pm$3.23 & 0.63$\pm$0.24 & 81.90$\pm$1.01 \\
 & Platt & \textbf{2.14$\pm$0.66} & \textbf{8.12$\pm$1.58} & \textbf{0.37$\pm$0.11} & \textbf{2.25$\pm$0.71} & \textbf{11.06$\pm$1.55} & \textbf{0.39$\pm$0.11} & \textbf{82.45$\pm$1.03} \\
\bottomrule
\end{tabular}%
}
\end{table*}

\begin{table*}[t]
\centering
\scriptsize
\setlength{\tabcolsep}{3.2pt}
\renewcommand{\arraystretch}{0.92}
\caption{Tabular test results on MEPS with DecisionTree as the base predictor. Calibration metrics are reported as $10^2\!\times$ mean $\pm$ standard deviation over ten seeds; accuracy is reported in percent. WLMC uses step size $0.05$ and at most $50$ updates. Bold marks the single best mean within each $(\mathrm{base},\mathrm{post})$ block; variance is not used for bolding.}
\label{tab:tabular-meps-decisiontree-main}
\resizebox{\textwidth}{!}{%
\begin{tabular}{llccccccc}
\toprule
\multirow{2}{*}{(Base,Post)} & \multirow{2}{*}{Post-processing} & \multicolumn{3}{c}{oracle estimate ($\times 10^{-2}$)} & \multicolumn{3}{c}{weak estimate ($\times 10^{-2}$)} & \multirow{2}{*}{Acc. (\%)} \\
\cmidrule(lr){3-5}\cmidrule(lr){6-8}
& & ECE & maxECE & MC & ECE & maxECE & MC & \\
\midrule
\multirow{4}{*}{(PN,PN)} & ERM & 8.74$\pm$0.72 & 30.22$\pm$2.60 & 1.82$\pm$0.22 & 8.74$\pm$0.72 & 30.21$\pm$2.60 & 1.82$\pm$0.22 & 83.82$\pm$0.20 \\
 & WLMC & 4.25$\pm$1.30 & \textbf{21.75$\pm$3.23} & \textbf{1.08$\pm$0.26} & 4.25$\pm$1.30 & \textbf{21.75$\pm$3.23} & \textbf{1.08$\pm$0.26} & \textbf{84.12$\pm$0.58} \\
 & Temp & 13.51$\pm$1.60 & 25.18$\pm$2.54 & 4.14$\pm$0.85 & 13.51$\pm$1.60 & 25.18$\pm$2.54 & 4.14$\pm$0.85 & 83.84$\pm$0.32 \\
 & Platt & \textbf{2.95$\pm$1.69} & 23.28$\pm$1.06 & 2.07$\pm$0.52 & \textbf{2.95$\pm$1.69} & 23.28$\pm$1.06 & 2.07$\pm$0.52 & 83.60$\pm$0.21 \\
\midrule
\multirow{3}{*}{(PN,PU)} & WLMC & \textbf{4.58$\pm$1.08} & \textbf{23.46$\pm$3.14} & \textbf{1.17$\pm$0.39} & \textbf{6.00$\pm$1.36} & \textbf{27.94$\pm$9.41} & \textbf{1.54$\pm$0.38} & \textbf{84.04$\pm$0.77} \\
 & Temp & 26.95$\pm$2.15 & 32.63$\pm$2.72 & 13.45$\pm$1.89 & 27.53$\pm$2.92 & 35.46$\pm$5.43 & 13.41$\pm$1.83 & 83.84$\pm$0.32 \\
 & Platt & 6.98$\pm$2.35 & 32.45$\pm$1.42 & 4.00$\pm$0.65 & 7.09$\pm$2.33 & 31.06$\pm$6.70 & 3.85$\pm$0.81 & 82.65$\pm$0.09 \\
\midrule
\multirow{3}{*}{(PN,Pconf)} & WLMC & \textbf{7.16$\pm$1.53} & \textbf{26.23$\pm$5.82} & \textbf{1.45$\pm$0.37} & 8.10$\pm$2.28 & \textbf{29.51$\pm$5.64} & \textbf{2.42$\pm$1.35} & \textbf{83.84$\pm$0.32} \\
 & Temp & 17.48$\pm$6.69 & 26.52$\pm$3.63 & 6.32$\pm$3.87 & 25.40$\pm$11.44 & 54.26$\pm$30.97 & 11.72$\pm$8.04 & 83.84$\pm$0.32 \\
 & Platt & 7.42$\pm$1.91 & 36.12$\pm$3.02 & 2.43$\pm$0.76 & \textbf{7.31$\pm$3.11} & 34.65$\pm$3.34 & 4.40$\pm$2.48 & 82.82$\pm$0.42 \\
\midrule
\multirow{3}{*}{(PN,UU)} & WLMC & \textbf{4.96$\pm$1.76} & \textbf{20.21$\pm$2.37} & \textbf{1.26$\pm$0.46} & \textbf{5.49$\pm$1.96} & \textbf{24.84$\pm$3.88} & \textbf{1.39$\pm$0.46} & \textbf{84.23$\pm$0.50} \\
 & Temp & 26.44$\pm$1.40 & 31.73$\pm$1.80 & 12.76$\pm$1.41 & 26.56$\pm$1.11 & 34.29$\pm$4.01 & 12.74$\pm$1.38 & 83.84$\pm$0.32 \\
 & Platt & 7.12$\pm$2.65 & 31.80$\pm$2.14 & 4.03$\pm$0.76 & 6.97$\pm$3.00 & 33.58$\pm$4.22 & 4.08$\pm$1.00 & 82.69$\pm$0.23 \\
\midrule
\multirow{4}{*}{(PU,PU)} & ERM & 11.36$\pm$0.95 & 35.81$\pm$3.37 & 2.57$\pm$0.18 & 12.16$\pm$2.23 & 38.99$\pm$7.46 & 3.00$\pm$0.68 & 82.99$\pm$0.67 \\
 & WLMC & \textbf{5.29$\pm$1.55} & \textbf{19.70$\pm$5.44} & \textbf{1.36$\pm$0.59} & 6.49$\pm$1.34 & \textbf{24.41$\pm$7.78} & \textbf{1.61$\pm$0.63} & \textbf{83.34$\pm$1.01} \\
 & Temp & 28.06$\pm$1.69 & 34.86$\pm$2.64 & 14.98$\pm$3.00 & 28.15$\pm$2.42 & 37.31$\pm$2.44 & 14.91$\pm$2.99 & 82.99$\pm$0.66 \\
 & Platt & 5.68$\pm$1.91 & 32.66$\pm$2.99 & 3.73$\pm$0.50 & \textbf{5.19$\pm$2.76} & 30.20$\pm$6.88 & 3.58$\pm$0.76 & 82.62$\pm$0.00 \\
\midrule
\multirow{4}{*}{(Pconf,Pconf)} & ERM & 29.98$\pm$0.74 & 35.98$\pm$0.50 & 8.24$\pm$1.56 & 44.39$\pm$9.73 & 93.49$\pm$37.45 & 16.77$\pm$7.04 & 75.20$\pm$2.69 \\
 & WLMC & 10.11$\pm$1.97 & 17.69$\pm$3.31 & 1.34$\pm$0.43 & 13.87$\pm$3.27 & 26.06$\pm$11.19 & 2.78$\pm$1.68 & 82.07$\pm$1.39 \\
 & Temp & 25.41$\pm$2.78 & 34.94$\pm$1.91 & 4.96$\pm$1.39 & 34.41$\pm$10.48 & 65.10$\pm$29.22 & 9.03$\pm$6.88 & 75.20$\pm$2.69 \\
 & Platt & \textbf{3.96$\pm$1.31} & \textbf{16.22$\pm$1.92} & \textbf{0.80$\pm$0.41} & \textbf{4.29$\pm$0.96} & \textbf{14.61$\pm$2.72} & \textbf{1.56$\pm$1.20} & \textbf{85.22$\pm$0.41} \\
\midrule
\multirow{4}{*}{(UU,UU)} & ERM & 9.30$\pm$1.22 & 30.52$\pm$3.18 & 3.05$\pm$0.64 & 9.39$\pm$1.03 & 30.52$\pm$5.24 & 2.95$\pm$0.52 & 82.82$\pm$0.89 \\
 & WLMC & 4.03$\pm$0.96 & \textbf{20.56$\pm$3.29} & \textbf{0.90$\pm$0.20} & 4.45$\pm$0.82 & \textbf{25.18$\pm$4.91} & \textbf{1.00$\pm$0.18} & \textbf{84.03$\pm$0.50} \\
 & Temp & 27.42$\pm$2.43 & 33.32$\pm$2.84 & 12.64$\pm$2.89 & 27.47$\pm$2.27 & 36.03$\pm$3.83 & 12.71$\pm$3.13 & 82.82$\pm$0.89 \\
 & Platt & \textbf{2.08$\pm$0.91} & 28.05$\pm$1.83 & 2.51$\pm$0.42 & \textbf{2.30$\pm$0.83} & 30.12$\pm$4.71 & 2.62$\pm$0.52 & 82.90$\pm$0.59 \\
\bottomrule
\end{tabular}%
}
\end{table*}

\begin{table*}[t]
\centering
\scriptsize
\setlength{\tabcolsep}{3.2pt}
\renewcommand{\arraystretch}{0.92}
\caption{Tabular test results on MEPS with LogisticRegression as the base predictor. Calibration metrics are reported as $10^2\!\times$ mean $\pm$ standard deviation over ten seeds; accuracy is reported in percent. WLMC uses step size $0.05$ and at most $50$ updates. Bold marks the single best mean within each $(\mathrm{base},\mathrm{post})$ block; variance is not used for bolding.}
\label{tab:tabular-meps-logisticregression-main}
\resizebox{\textwidth}{!}{%
\begin{tabular}{llccccccc}
\toprule
\multirow{2}{*}{(Base,Post)} & \multirow{2}{*}{Post-processing} & \multicolumn{3}{c}{oracle estimate ($\times 10^{-2}$)} & \multicolumn{3}{c}{weak estimate ($\times 10^{-2}$)} & \multirow{2}{*}{Acc. (\%)} \\
\cmidrule(lr){3-5}\cmidrule(lr){6-8}
& & ECE & maxECE & MC & ECE & maxECE & MC & \\
\midrule
\multirow{4}{*}{(PN,PN)} & ERM & \textbf{2.04$\pm$0.27} & 9.90$\pm$1.26 & \textbf{0.54$\pm$0.12} & \textbf{2.04$\pm$0.27} & 9.90$\pm$1.26 & \textbf{0.54$\pm$0.12} & \textbf{86.28$\pm$0.26} \\
 & WLMC & 2.09$\pm$0.33 & 9.90$\pm$1.26 & 0.60$\pm$0.24 & 2.09$\pm$0.33 & 9.90$\pm$1.26 & 0.60$\pm$0.24 & 86.28$\pm$0.26 \\
 & Temp & 3.43$\pm$0.38 & \textbf{9.65$\pm$1.14} & 1.31$\pm$0.09 & 3.43$\pm$0.37 & \textbf{9.65$\pm$1.14} & 1.31$\pm$0.09 & 86.28$\pm$0.26 \\
 & Platt & 4.61$\pm$0.37 & 12.62$\pm$1.21 & 1.46$\pm$0.21 & 4.61$\pm$0.37 & 12.62$\pm$1.21 & 1.46$\pm$0.21 & 86.11$\pm$0.32 \\
\midrule
\multirow{3}{*}{(PN,PU)} & WLMC & \textbf{2.15$\pm$0.43} & \textbf{9.72$\pm$1.31} & \textbf{0.69$\pm$0.18} & \textbf{4.48$\pm$1.33} & 24.15$\pm$4.66 & \textbf{0.93$\pm$0.24} & 86.17$\pm$0.47 \\
 & Temp & 2.32$\pm$0.45 & 9.74$\pm$1.52 & 0.74$\pm$0.23 & 4.71$\pm$1.34 & 23.67$\pm$4.93 & 0.97$\pm$0.22 & \textbf{86.28$\pm$0.26} \\
 & Platt & 2.70$\pm$0.62 & 11.61$\pm$1.58 & 0.77$\pm$0.24 & 4.89$\pm$1.69 & \textbf{22.15$\pm$4.77} & 0.98$\pm$0.28 & 86.12$\pm$0.28 \\
\midrule
\multirow{3}{*}{(PN,Pconf)} & WLMC & \textbf{2.56$\pm$0.48} & \textbf{9.93$\pm$1.18} & 0.98$\pm$0.19 & 4.14$\pm$1.30 & 11.67$\pm$3.86 & 1.87$\pm$1.13 & \textbf{86.28$\pm$0.26} \\
 & Temp & 3.24$\pm$0.99 & 11.29$\pm$2.14 & \textbf{0.62$\pm$0.19} & \textbf{3.36$\pm$1.19} & 11.64$\pm$1.74 & \textbf{1.11$\pm$0.68} & 86.28$\pm$0.26 \\
 & Platt & 4.28$\pm$0.86 & 15.29$\pm$1.83 & 0.74$\pm$0.14 & 3.37$\pm$0.68 & \textbf{11.59$\pm$1.29} & 1.17$\pm$0.64 & 85.86$\pm$0.36 \\
\midrule
\multirow{3}{*}{(PN,UU)} & WLMC & \textbf{2.06$\pm$0.44} & 10.16$\pm$1.47 & \textbf{0.55$\pm$0.14} & \textbf{3.15$\pm$0.43} & 19.65$\pm$3.05 & \textbf{0.72$\pm$0.17} & \textbf{86.29$\pm$0.26} \\
 & Temp & 2.22$\pm$0.34 & \textbf{9.89$\pm$1.33} & 0.69$\pm$0.21 & 3.27$\pm$0.58 & \textbf{18.64$\pm$3.84} & 0.82$\pm$0.22 & 86.28$\pm$0.26 \\
 & Platt & 2.48$\pm$0.57 & 11.40$\pm$1.82 & 0.73$\pm$0.25 & 3.48$\pm$0.79 & 19.92$\pm$3.98 & 0.84$\pm$0.25 & 86.10$\pm$0.28 \\
\midrule
\multirow{4}{*}{(PU,PU)} & ERM & 3.61$\pm$0.91 & 12.89$\pm$2.84 & 1.19$\pm$0.35 & 3.64$\pm$0.82 & 16.50$\pm$4.63 & 1.06$\pm$0.27 & 84.27$\pm$0.42 \\
 & WLMC & 3.14$\pm$0.81 & 12.82$\pm$2.44 & \textbf{0.81$\pm$0.30} & 3.34$\pm$1.41 & \textbf{16.18$\pm$4.24} & 1.02$\pm$0.23 & 84.34$\pm$0.40 \\
 & Temp & 3.31$\pm$0.48 & 13.34$\pm$3.14 & 0.92$\pm$0.16 & 3.51$\pm$0.54 & 16.21$\pm$3.77 & 1.00$\pm$0.24 & 84.27$\pm$0.42 \\
 & Platt & \textbf{2.86$\pm$0.46} & \textbf{11.42$\pm$1.99} & 0.90$\pm$0.18 & \textbf{3.08$\pm$0.72} & 16.69$\pm$4.07 & \textbf{0.99$\pm$0.21} & \textbf{84.57$\pm$0.25} \\
\midrule
\multirow{4}{*}{(Pconf,Pconf)} & ERM & 35.46$\pm$1.56 & 42.60$\pm$1.13 & 10.20$\pm$2.31 & 51.25$\pm$12.24 & 106.75$\pm$50.48 & 18.77$\pm$9.25 & 56.74$\pm$7.59 \\
 & WLMC & 8.85$\pm$1.95 & 15.76$\pm$3.35 & 1.93$\pm$0.97 & 15.12$\pm$7.00 & 32.20$\pm$15.29 & 4.39$\pm$3.25 & 81.20$\pm$2.30 \\
 & Temp & 35.31$\pm$2.30 & 47.26$\pm$6.19 & 9.19$\pm$4.74 & 48.76$\pm$9.34 & 98.72$\pm$40.56 & 18.07$\pm$8.39 & 56.74$\pm$7.59 \\
 & Platt & \textbf{4.44$\pm$1.24} & \textbf{12.27$\pm$3.26} & \textbf{1.09$\pm$0.26} & \textbf{3.71$\pm$1.15} & \textbf{8.69$\pm$1.94} & \textbf{1.34$\pm$0.42} & \textbf{84.43$\pm$0.24} \\
\midrule
\multirow{4}{*}{(UU,UU)} & ERM & 3.25$\pm$0.48 & 11.42$\pm$2.15 & 1.10$\pm$0.18 & 3.71$\pm$0.47 & \textbf{15.58$\pm$3.33} & 1.10$\pm$0.15 & 84.48$\pm$0.46 \\
 & WLMC & \textbf{2.40$\pm$0.62} & 12.04$\pm$2.29 & \textbf{0.72$\pm$0.20} & \textbf{2.95$\pm$0.73} & 16.09$\pm$3.48 & \textbf{0.76$\pm$0.16} & 84.36$\pm$0.50 \\
 & Temp & 2.86$\pm$0.41 & 11.72$\pm$2.40 & 0.93$\pm$0.19 & 3.40$\pm$0.62 & 15.90$\pm$3.49 & 0.98$\pm$0.15 & 84.48$\pm$0.46 \\
 & Platt & 2.73$\pm$0.75 & \textbf{10.39$\pm$2.95} & 0.93$\pm$0.20 & 3.30$\pm$0.98 & 16.68$\pm$2.19 & 0.97$\pm$0.17 & \textbf{84.57$\pm$0.50} \\
\bottomrule
\end{tabular}%
}
\end{table*}

\begin{table*}[t]
\centering
\scriptsize
\setlength{\tabcolsep}{3.2pt}
\renewcommand{\arraystretch}{0.92}
\caption{Tabular test results on MEPS with NaiveBayes as the base predictor. Calibration metrics are reported as $10^2\!\times$ mean $\pm$ standard deviation over ten seeds; accuracy is reported in percent. WLMC uses step size $0.05$ and at most $50$ updates. Bold marks the single best mean within each $(\mathrm{base},\mathrm{post})$ block; variance is not used for bolding.}
\label{tab:tabular-meps-naivebayes-main}
\resizebox{\textwidth}{!}{%
\begin{tabular}{llccccccc}
\toprule
\multirow{2}{*}{(Base,Post)} & \multirow{2}{*}{Post-processing} & \multicolumn{3}{c}{oracle estimate ($\times 10^{-2}$)} & \multicolumn{3}{c}{weak estimate ($\times 10^{-2}$)} & \multirow{2}{*}{Acc. (\%)} \\
\cmidrule(lr){3-5}\cmidrule(lr){6-8}
& & ECE & maxECE & MC & ECE & maxECE & MC & \\
\midrule
\multirow{4}{*}{(PN,PN)} & ERM & 31.11$\pm$5.23 & 57.76$\pm$5.71 & 15.07$\pm$2.96 & 31.11$\pm$5.23 & 57.76$\pm$5.71 & 15.07$\pm$2.96 & 68.10$\pm$5.28 \\
 & WLMC & 5.33$\pm$1.74 & \textbf{13.99$\pm$3.94} & \textbf{1.17$\pm$0.27} & 5.33$\pm$1.74 & \textbf{13.99$\pm$3.94} & \textbf{1.17$\pm$0.27} & 81.94$\pm$1.27 \\
 & Temp & 30.53$\pm$5.21 & 56.96$\pm$5.79 & 13.57$\pm$2.32 & 30.53$\pm$5.21 & 56.96$\pm$5.79 & 13.57$\pm$2.32 & 68.10$\pm$5.28 \\
 & Platt & \textbf{2.72$\pm$0.66} & 22.97$\pm$2.78 & 2.05$\pm$0.27 & \textbf{2.72$\pm$0.66} & 22.97$\pm$2.78 & 2.05$\pm$0.27 & \textbf{82.62$\pm$0.00} \\
\midrule
\multirow{3}{*}{(PN,PU)} & WLMC & 6.12$\pm$1.67 & \textbf{16.83$\pm$3.84} & \textbf{1.44$\pm$0.33} & 6.91$\pm$1.57 & 20.23$\pm$4.71 & 1.78$\pm$0.48 & 82.33$\pm$1.87 \\
 & Temp & 30.26$\pm$2.18 & 40.43$\pm$2.83 & 8.91$\pm$0.82 & 30.40$\pm$2.23 & 40.57$\pm$2.77 & 8.79$\pm$0.98 & 68.10$\pm$5.28 \\
 & Platt & \textbf{3.90$\pm$0.84} & 20.80$\pm$2.22 & 1.82$\pm$0.19 & \textbf{3.59$\pm$0.98} & \textbf{19.26$\pm$5.99} & \textbf{1.74$\pm$0.40} & \textbf{82.62$\pm$0.00} \\
\midrule
\multirow{3}{*}{(PN,Pconf)} & WLMC & 8.77$\pm$2.68 & \textbf{18.32$\pm$4.76} & 2.17$\pm$0.61 & 11.97$\pm$2.52 & 25.51$\pm$7.69 & 2.95$\pm$0.67 & 80.78$\pm$1.74 \\
 & Temp & 28.82$\pm$3.65 & 45.80$\pm$5.24 & 8.55$\pm$0.87 & 37.69$\pm$6.99 & 67.57$\pm$14.65 & 11.95$\pm$4.79 & 68.10$\pm$5.28 \\
 & Platt & \textbf{4.09$\pm$0.81} & 20.23$\pm$3.05 & \textbf{1.77$\pm$0.28} & \textbf{2.57$\pm$1.21} & \textbf{11.38$\pm$3.64} & \textbf{1.16$\pm$0.26} & \textbf{82.62$\pm$0.00} \\
\midrule
\multirow{3}{*}{(PN,UU)} & WLMC & 6.02$\pm$1.98 & \textbf{17.67$\pm$3.61} & \textbf{1.26$\pm$0.19} & 6.24$\pm$1.64 & \textbf{18.94$\pm$5.74} & \textbf{1.52$\pm$0.28} & 82.24$\pm$1.49 \\
 & Temp & 30.18$\pm$2.29 & 40.22$\pm$2.37 & 9.03$\pm$1.00 & 30.05$\pm$2.42 & 40.56$\pm$3.83 & 9.26$\pm$0.86 & 68.10$\pm$5.28 \\
 & Platt & \textbf{4.23$\pm$1.22} & 21.24$\pm$3.07 & 1.85$\pm$0.27 & \textbf{4.10$\pm$1.43} & 23.52$\pm$5.77 & 2.08$\pm$0.39 & \textbf{82.62$\pm$0.00} \\
\midrule
\multirow{4}{*}{(PU,PU)} & ERM & 36.49$\pm$9.63 & 60.50$\pm$16.80 & 22.00$\pm$6.90 & 36.97$\pm$10.26 & 59.37$\pm$17.15 & 22.13$\pm$6.64 & 48.21$\pm$23.16 \\
 & WLMC & 5.68$\pm$2.27 & \textbf{16.94$\pm$4.31} & \textbf{1.41$\pm$0.54} & 6.43$\pm$2.85 & \textbf{18.40$\pm$3.90} & \textbf{1.70$\pm$0.57} & 81.84$\pm$2.13 \\
 & Temp & 31.33$\pm$3.59 & 41.14$\pm$4.05 & 15.19$\pm$6.69 & 31.36$\pm$3.57 & 41.75$\pm$3.67 & 14.93$\pm$6.52 & 67.78$\pm$13.73 \\
 & Platt & \textbf{2.32$\pm$1.42} & 27.70$\pm$7.74 & 2.52$\pm$0.79 & \textbf{2.33$\pm$1.44} & 26.20$\pm$9.69 & 2.48$\pm$0.98 & \textbf{82.62$\pm$0.00} \\
\midrule
\multirow{4}{*}{(Pconf,Pconf)} & ERM & 20.82$\pm$1.05 & 42.37$\pm$2.57 & 8.72$\pm$0.57 & 23.58$\pm$1.37 & 62.92$\pm$4.23 & 8.28$\pm$1.04 & 77.79$\pm$0.93 \\
 & WLMC & 9.27$\pm$3.62 & 19.34$\pm$7.23 & 2.06$\pm$0.78 & 11.70$\pm$4.14 & 28.57$\pm$12.77 & 2.48$\pm$0.66 & 82.28$\pm$1.80 \\
 & Temp & 20.42$\pm$1.86 & 36.92$\pm$2.87 & 4.32$\pm$0.61 & 25.47$\pm$4.45 & 50.95$\pm$6.27 & 5.63$\pm$2.26 & 77.79$\pm$0.93 \\
 & Platt & \textbf{3.01$\pm$0.86} & \textbf{14.12$\pm$1.88} & \textbf{1.03$\pm$0.14} & \textbf{2.40$\pm$0.78} & \textbf{6.83$\pm$1.91} & \textbf{1.12$\pm$0.45} & \textbf{83.11$\pm$0.55} \\
\midrule
\multirow{4}{*}{(UU,UU)} & ERM & 32.62$\pm$0.00 & 44.72$\pm$0.00 & 24.47$\pm$0.00 & 32.62$\pm$0.00 & 45.73$\pm$2.29 & 24.33$\pm$0.85 & 17.38$\pm$0.00 \\
 & WLMC & 3.03$\pm$1.67 & \textbf{16.25$\pm$6.12} & \textbf{1.17$\pm$0.75} & 3.59$\pm$1.55 & \textbf{20.07$\pm$4.50} & \textbf{1.38$\pm$0.78} & \textbf{82.78$\pm$0.40} \\
 & Temp & 32.62$\pm$0.00 & 44.72$\pm$0.00 & 24.47$\pm$0.00 & 32.62$\pm$0.00 & 45.73$\pm$2.29 & 24.33$\pm$0.85 & 82.62$\pm$0.00 \\
 & Platt & \textbf{0.81$\pm$0.58} & 36.93$\pm$0.58 & 3.51$\pm$0.13 & \textbf{0.81$\pm$0.58} & 38.47$\pm$3.22 & 3.77$\pm$0.26 & 82.62$\pm$0.00 \\
\bottomrule
\end{tabular}%
}
\end{table*}

\begin{table*}[t]
\centering
\scriptsize
\setlength{\tabcolsep}{3.2pt}
\renewcommand{\arraystretch}{0.92}
\caption{Tabular test results on MEPS with RandomForest as the base predictor. Calibration metrics are reported as $10^2\!\times$ mean $\pm$ standard deviation over ten seeds; accuracy is reported in percent. WLMC uses step size $0.05$ and at most $50$ updates. Bold marks the single best mean within each $(\mathrm{base},\mathrm{post})$ block; variance is not used for bolding.}
\label{tab:tabular-meps-randomforest-main}
\resizebox{\textwidth}{!}{%
\begin{tabular}{llccccccc}
\toprule
\multirow{2}{*}{(Base,Post)} & \multirow{2}{*}{Post-processing} & \multicolumn{3}{c}{oracle estimate ($\times 10^{-2}$)} & \multicolumn{3}{c}{weak estimate ($\times 10^{-2}$)} & \multirow{2}{*}{Acc. (\%)} \\
\cmidrule(lr){3-5}\cmidrule(lr){6-8}
& & ECE & maxECE & MC & ECE & maxECE & MC & \\
\midrule
\multirow{4}{*}{(PN,PN)} & ERM & 1.76$\pm$0.41 & \textbf{8.94$\pm$1.39} & \textbf{0.41$\pm$0.07} & 1.76$\pm$0.41 & \textbf{8.94$\pm$1.39} & \textbf{0.41$\pm$0.07} & \textbf{86.32$\pm$0.27} \\
 & WLMC & \textbf{1.74$\pm$0.39} & 9.05$\pm$1.46 & 0.45$\pm$0.18 & \textbf{1.74$\pm$0.39} & 9.06$\pm$1.46 & 0.45$\pm$0.18 & 86.32$\pm$0.27 \\
 & Temp & 3.82$\pm$0.50 & 9.44$\pm$1.16 & 0.88$\pm$0.25 & 3.82$\pm$0.50 & 9.44$\pm$1.16 & 0.88$\pm$0.25 & 86.32$\pm$0.27 \\
 & Platt & 2.63$\pm$0.47 & 10.62$\pm$0.97 & 0.83$\pm$0.23 & 2.63$\pm$0.47 & 10.62$\pm$0.97 & 0.83$\pm$0.23 & 86.27$\pm$0.20 \\
\midrule
\multirow{3}{*}{(PN,PU)} & WLMC & \textbf{2.14$\pm$0.54} & 10.18$\pm$2.36 & 0.59$\pm$0.24 & \textbf{4.01$\pm$1.21} & 21.99$\pm$4.17 & 1.06$\pm$0.23 & 86.26$\pm$0.33 \\
 & Temp & 2.18$\pm$0.50 & \textbf{9.16$\pm$1.44} & 0.54$\pm$0.08 & 4.47$\pm$1.03 & \textbf{19.83$\pm$3.25} & 0.88$\pm$0.21 & \textbf{86.32$\pm$0.27} \\
 & Platt & 2.68$\pm$0.86 & 11.55$\pm$3.05 & \textbf{0.52$\pm$0.12} & 4.34$\pm$1.25 & 21.86$\pm$4.31 & \textbf{0.86$\pm$0.19} & 85.98$\pm$0.56 \\
\midrule
\multirow{3}{*}{(PN,Pconf)} & WLMC & 2.03$\pm$0.58 & \textbf{9.01$\pm$1.33} & 0.62$\pm$0.29 & 1.80$\pm$1.42 & 6.58$\pm$5.99 & 1.28$\pm$1.37 & 86.32$\pm$0.27 \\
 & Temp & 1.77$\pm$0.31 & 9.19$\pm$1.29 & 0.39$\pm$0.09 & 0.82$\pm$0.22 & 2.46$\pm$0.52 & \textbf{0.21$\pm$0.05} & 86.32$\pm$0.27 \\
 & Platt & \textbf{1.76$\pm$0.27} & 9.43$\pm$1.58 & \textbf{0.37$\pm$0.07} & \textbf{0.80$\pm$0.22} & \textbf{2.29$\pm$0.63} & 0.21$\pm$0.05 & \textbf{86.34$\pm$0.24} \\
\midrule
\multirow{3}{*}{(PN,UU)} & WLMC & \textbf{1.98$\pm$0.79} & \textbf{9.22$\pm$1.77} & 0.51$\pm$0.19 & \textbf{2.97$\pm$1.04} & \textbf{15.99$\pm$3.32} & 0.70$\pm$0.22 & 86.20$\pm$0.27 \\
 & Temp & 2.13$\pm$0.49 & 9.27$\pm$1.62 & 0.54$\pm$0.19 & 3.09$\pm$0.73 & 16.85$\pm$3.80 & 0.73$\pm$0.22 & \textbf{86.32$\pm$0.27} \\
 & Platt & 2.35$\pm$1.00 & 10.54$\pm$2.79 & \textbf{0.46$\pm$0.13} & 3.21$\pm$0.72 & 16.84$\pm$3.24 & \textbf{0.66$\pm$0.19} & 85.90$\pm$0.56 \\
\midrule
\multirow{4}{*}{(PU,PU)} & ERM & 17.85$\pm$13.52 & 32.36$\pm$15.42 & 7.20$\pm$10.27 & 18.39$\pm$13.62 & 37.24$\pm$13.87 & 7.53$\pm$10.10 & 77.30$\pm$11.63 \\
 & WLMC & 4.67$\pm$1.19 & \textbf{15.77$\pm$4.64} & \textbf{1.20$\pm$0.57} & 5.78$\pm$1.33 & \textbf{20.74$\pm$1.93} & \textbf{1.40$\pm$0.65} & 83.15$\pm$2.30 \\
 & Temp & 26.07$\pm$4.98 & 34.63$\pm$6.24 & 9.03$\pm$4.05 & 26.44$\pm$5.10 & 37.34$\pm$5.87 & 8.90$\pm$3.95 & 77.30$\pm$11.63 \\
 & Platt & \textbf{3.70$\pm$1.85} & 26.55$\pm$5.57 & 2.14$\pm$0.94 & \textbf{4.08$\pm$2.43} & 25.43$\pm$5.29 & 2.15$\pm$0.87 & \textbf{83.65$\pm$1.14} \\
\midrule
\multirow{4}{*}{(Pconf,Pconf)} & ERM & 30.75$\pm$0.49 & 36.49$\pm$0.58 & 12.71$\pm$1.38 & 47.04$\pm$11.39 & 106.32$\pm$48.48 & 22.33$\pm$10.83 & 73.96$\pm$1.04 \\
 & WLMC & 7.69$\pm$2.14 & \textbf{15.28$\pm$3.32} & 1.49$\pm$0.61 & 12.68$\pm$5.18 & 25.52$\pm$12.37 & 3.51$\pm$2.20 & 83.02$\pm$1.83 \\
 & Temp & 24.44$\pm$2.25 & 35.58$\pm$2.75 & 3.11$\pm$0.71 & 33.47$\pm$9.33 & 67.57$\pm$26.99 & 9.24$\pm$7.22 & 73.96$\pm$1.04 \\
 & Platt & \textbf{4.08$\pm$0.89} & 15.41$\pm$2.41 & \textbf{0.72$\pm$0.27} & \textbf{3.65$\pm$1.18} & \textbf{12.34$\pm$3.68} & \textbf{1.49$\pm$1.19} & \textbf{85.27$\pm$0.19} \\
\midrule
\multirow{4}{*}{(UU,UU)} & ERM & 40.75$\pm$1.79 & 55.44$\pm$1.80 & 15.54$\pm$2.32 & 40.78$\pm$1.60 & 58.13$\pm$5.38 & 15.51$\pm$2.43 & 59.20$\pm$1.85 \\
 & WLMC & 7.82$\pm$1.01 & \textbf{20.22$\pm$3.13} & \textbf{1.69$\pm$0.31} & 8.06$\pm$1.31 & \textbf{24.09$\pm$5.03} & \textbf{1.80$\pm$0.33} & 79.61$\pm$1.28 \\
 & Temp & 33.66$\pm$0.14 & 44.46$\pm$0.28 & 15.80$\pm$0.64 & 33.66$\pm$0.12 & 45.55$\pm$2.32 & 15.74$\pm$0.84 & 59.20$\pm$1.85 \\
 & Platt & \textbf{1.54$\pm$1.01} & 23.11$\pm$1.93 & 1.89$\pm$0.22 & \textbf{1.56$\pm$0.88} & 24.89$\pm$5.61 & 2.13$\pm$0.27 & \textbf{82.62$\pm$0.00} \\
\bottomrule
\end{tabular}%
}
\end{table*}

\begin{table*}[t]
\centering
\scriptsize
\setlength{\tabcolsep}{3.2pt}
\renewcommand{\arraystretch}{0.92}
\caption{Tabular test results on MEPS with SVM as the base predictor. Calibration metrics are reported as $10^2\!\times$ mean $\pm$ standard deviation over ten seeds; accuracy is reported in percent. WLMC uses step size $0.05$ and at most $50$ updates. Bold marks the single best mean within each $(\mathrm{base},\mathrm{post})$ block; variance is not used for bolding.}
\label{tab:tabular-meps-svm-main}
\resizebox{\textwidth}{!}{%
\begin{tabular}{llccccccc}
\toprule
\multirow{2}{*}{(Base,Post)} & \multirow{2}{*}{Post-processing} & \multicolumn{3}{c}{oracle estimate ($\times 10^{-2}$)} & \multicolumn{3}{c}{weak estimate ($\times 10^{-2}$)} & \multirow{2}{*}{Acc. (\%)} \\
\cmidrule(lr){3-5}\cmidrule(lr){6-8}
& & ECE & maxECE & MC & ECE & maxECE & MC & \\
\midrule
\multirow{4}{*}{(PN,PN)} & ERM & 11.38$\pm$2.53 & 28.32$\pm$7.50 & 2.55$\pm$1.19 & 11.38$\pm$2.53 & 28.32$\pm$7.50 & 2.55$\pm$1.19 & 83.07$\pm$2.17 \\
 & WLMC & 5.02$\pm$1.40 & \textbf{18.45$\pm$4.73} & \textbf{0.93$\pm$0.33} & 5.02$\pm$1.40 & \textbf{18.45$\pm$4.73} & \textbf{0.93$\pm$0.33} & 84.30$\pm$1.19 \\
 & Temp & 8.59$\pm$3.17 & 23.37$\pm$8.09 & 1.49$\pm$0.84 & 8.59$\pm$3.17 & 23.37$\pm$8.09 & 1.49$\pm$0.84 & 83.07$\pm$2.17 \\
 & Platt & \textbf{2.03$\pm$0.66} & 21.15$\pm$6.43 & 1.45$\pm$0.81 & \textbf{2.03$\pm$0.66} & 21.15$\pm$6.43 & 1.45$\pm$0.81 & \textbf{84.52$\pm$1.27} \\
\midrule
\multirow{3}{*}{(PN,PU)} & WLMC & 5.26$\pm$1.04 & \textbf{19.51$\pm$5.52} & \textbf{1.21$\pm$0.74} & 6.08$\pm$1.61 & \textbf{24.44$\pm$6.45} & \textbf{1.48$\pm$0.79} & \textbf{84.24$\pm$1.40} \\
 & Temp & 10.47$\pm$2.71 & 22.09$\pm$4.76 & 2.07$\pm$0.53 & 11.06$\pm$2.91 & 27.96$\pm$3.66 & 2.13$\pm$0.71 & 83.07$\pm$2.17 \\
 & Platt & \textbf{3.28$\pm$1.09} & 23.34$\pm$6.51 & 1.58$\pm$0.84 & \textbf{4.43$\pm$1.25} & 27.82$\pm$8.14 & 1.66$\pm$0.78 & 84.21$\pm$0.97 \\
\midrule
\multirow{3}{*}{(PN,Pconf)} & WLMC & 8.18$\pm$3.49 & \textbf{21.96$\pm$7.45} & 1.58$\pm$0.64 & 10.63$\pm$5.37 & 29.59$\pm$6.93 & 2.12$\pm$0.44 & 83.48$\pm$2.06 \\
 & Temp & 9.02$\pm$2.96 & 22.63$\pm$7.02 & 1.41$\pm$0.54 & 12.65$\pm$3.86 & 29.66$\pm$5.27 & 2.11$\pm$0.73 & 83.07$\pm$2.17 \\
 & Platt & \textbf{5.46$\pm$2.17} & 26.88$\pm$8.34 & \textbf{1.36$\pm$0.70} & \textbf{4.76$\pm$1.91} & \textbf{23.15$\pm$8.54} & \textbf{1.66$\pm$0.92} & \textbf{84.27$\pm$0.90} \\
\midrule
\multirow{3}{*}{(PN,UU)} & WLMC & 4.72$\pm$1.10 & \textbf{18.26$\pm$4.43} & \textbf{1.02$\pm$0.30} & 5.53$\pm$1.47 & \textbf{26.02$\pm$6.24} & \textbf{1.21$\pm$0.45} & 84.18$\pm$1.25 \\
 & Temp & 10.34$\pm$2.15 & 21.62$\pm$5.01 & 1.98$\pm$0.54 & 10.59$\pm$2.21 & 26.53$\pm$5.52 & 2.06$\pm$0.58 & 83.07$\pm$2.17 \\
 & Platt & \textbf{3.06$\pm$1.01} & 23.19$\pm$6.66 & 1.57$\pm$0.86 & \textbf{3.92$\pm$1.40} & 28.17$\pm$6.29 & 1.68$\pm$0.95 & \textbf{84.31$\pm$0.87} \\
\midrule
\multirow{4}{*}{(PU,PU)} & ERM & 17.56$\pm$1.48 & 39.98$\pm$6.64 & 5.19$\pm$0.80 & 18.19$\pm$1.62 & 42.43$\pm$7.52 & 5.37$\pm$1.40 & 81.12$\pm$1.69 \\
 & WLMC & 5.81$\pm$1.23 & \textbf{22.03$\pm$7.65} & \textbf{1.31$\pm$0.34} & 6.62$\pm$1.56 & 27.49$\pm$8.48 & \textbf{1.63$\pm$0.40} & 83.23$\pm$0.81 \\
 & Temp & 9.98$\pm$3.07 & 23.43$\pm$5.31 & 3.05$\pm$1.11 & 9.89$\pm$3.33 & 28.11$\pm$5.30 & 3.18$\pm$1.34 & 81.12$\pm$1.69 \\
 & Platt & \textbf{2.16$\pm$0.96} & 26.85$\pm$6.02 & 1.73$\pm$0.76 & \textbf{3.00$\pm$1.74} & \textbf{26.34$\pm$7.80} & 1.90$\pm$0.72 & \textbf{83.60$\pm$1.11} \\
\midrule
\multirow{4}{*}{(Pconf,Pconf)} & ERM & 39.12$\pm$5.48 & 65.96$\pm$13.65 & 18.71$\pm$5.16 & 54.88$\pm$23.52 & 126.15$\pm$87.86 & 31.30$\pm$20.86 & 57.80$\pm$4.85 \\
 & WLMC & 11.19$\pm$3.00 & \textbf{25.81$\pm$5.27} & \textbf{2.32$\pm$0.93} & 14.45$\pm$3.84 & 34.82$\pm$9.94 & \textbf{4.52$\pm$2.67} & 80.04$\pm$2.37 \\
 & Temp & 31.92$\pm$2.58 & 46.20$\pm$3.36 & 10.90$\pm$4.96 & 46.46$\pm$13.23 & 100.96$\pm$49.07 & 19.75$\pm$11.57 & 57.80$\pm$4.85 \\
 & Platt & \textbf{5.71$\pm$3.23} & 35.03$\pm$8.14 & 2.58$\pm$1.08 & \textbf{7.20$\pm$3.15} & \textbf{34.76$\pm$8.53} & 4.75$\pm$3.33 & \textbf{82.80$\pm$0.56} \\
\midrule
\multirow{4}{*}{(UU,UU)} & ERM & 13.54$\pm$1.78 & 33.62$\pm$5.97 & 3.60$\pm$0.53 & 13.74$\pm$1.77 & 37.62$\pm$8.18 & 3.66$\pm$0.54 & 82.27$\pm$1.70 \\
 & WLMC & 4.97$\pm$1.09 & \textbf{19.39$\pm$3.13} & \textbf{1.00$\pm$0.38} & 5.51$\pm$1.06 & \textbf{25.86$\pm$4.93} & \textbf{1.14$\pm$0.39} & 83.31$\pm$1.61 \\
 & Temp & 10.78$\pm$2.62 & 21.44$\pm$3.97 & 2.24$\pm$0.66 & 10.92$\pm$2.45 & 25.95$\pm$4.96 & 2.31$\pm$0.70 & 82.27$\pm$1.70 \\
 & Platt & \textbf{3.31$\pm$0.72} & 26.31$\pm$6.21 & 1.87$\pm$0.81 & \textbf{3.55$\pm$0.95} & 30.04$\pm$6.70 & 1.94$\pm$0.88 & \textbf{83.53$\pm$0.98} \\
\bottomrule
\end{tabular}%
}
\end{table*}

\begin{table*}[t]
\centering
\scriptsize
\setlength{\tabcolsep}{3.2pt}
\renewcommand{\arraystretch}{0.92}
\caption{Tabular test results on MEPS with MLP as the base predictor. Calibration metrics are reported as $10^2\!\times$ mean $\pm$ standard deviation over ten seeds; accuracy is reported in percent. WLMC uses step size $0.05$ and at most $50$ updates. Bold marks the single best mean within each $(\mathrm{base},\mathrm{post})$ block; variance is not used for bolding.}
\label{tab:tabular-meps-mlp-main}
\resizebox{\textwidth}{!}{%
\begin{tabular}{llccccccc}
\toprule
\multirow{2}{*}{(Base,Post)} & \multirow{2}{*}{Post-processing} & \multicolumn{3}{c}{oracle estimate ($\times 10^{-2}$)} & \multicolumn{3}{c}{weak estimate ($\times 10^{-2}$)} & \multirow{2}{*}{Acc. (\%)} \\
\cmidrule(lr){3-5}\cmidrule(lr){6-8}
& & ECE & maxECE & MC & ECE & maxECE & MC & \\
\midrule
\multirow{4}{*}{(PN,PN)} & ERM & \textbf{2.46$\pm$0.59} & 10.02$\pm$1.21 & \textbf{0.74$\pm$0.30} & \textbf{2.46$\pm$0.59} & 10.02$\pm$1.21 & \textbf{0.74$\pm$0.30} & \textbf{86.50$\pm$0.35} \\
 & WLMC & 2.51$\pm$0.70 & 10.14$\pm$1.24 & 0.85$\pm$0.40 & 2.51$\pm$0.70 & 10.14$\pm$1.24 & 0.85$\pm$0.40 & 86.50$\pm$0.35 \\
 & Temp & 3.77$\pm$0.65 & \textbf{9.94$\pm$0.75} & 1.33$\pm$0.21 & 3.77$\pm$0.65 & \textbf{9.94$\pm$0.76} & 1.33$\pm$0.21 & 86.50$\pm$0.35 \\
 & Platt & 4.27$\pm$0.46 & 11.29$\pm$1.25 & 1.35$\pm$0.26 & 4.27$\pm$0.46 & 11.29$\pm$1.25 & 1.35$\pm$0.26 & 86.41$\pm$0.26 \\
\midrule
\multirow{3}{*}{(PN,PU)} & WLMC & \textbf{2.32$\pm$0.71} & 9.71$\pm$1.48 & 0.80$\pm$0.46 & 4.60$\pm$1.09 & 21.20$\pm$3.64 & 0.95$\pm$0.30 & 86.50$\pm$0.35 \\
 & Temp & 2.57$\pm$0.51 & \textbf{9.56$\pm$1.27} & \textbf{0.68$\pm$0.23} & 4.65$\pm$1.14 & 20.98$\pm$3.95 & \textbf{0.92$\pm$0.19} & \textbf{86.50$\pm$0.35} \\
 & Platt & 2.69$\pm$0.44 & 10.21$\pm$2.10 & 0.69$\pm$0.19 & \textbf{4.59$\pm$1.11} & \textbf{20.33$\pm$4.48} & 0.93$\pm$0.22 & 86.38$\pm$0.28 \\
\midrule
\multirow{3}{*}{(PN,Pconf)} & WLMC & \textbf{2.46$\pm$0.59} & \textbf{9.89$\pm$1.45} & 0.89$\pm$0.27 & 4.29$\pm$1.58 & 11.71$\pm$6.49 & 2.28$\pm$1.60 & \textbf{86.50$\pm$0.35} \\
 & Temp & 3.17$\pm$0.79 & 10.84$\pm$1.73 & \textbf{0.70$\pm$0.26} & 3.34$\pm$0.80 & 9.53$\pm$2.24 & \textbf{1.16$\pm$0.62} & 86.50$\pm$0.35 \\
 & Platt & 3.84$\pm$0.73 & 12.47$\pm$2.14 & 0.79$\pm$0.24 & \textbf{3.27$\pm$0.89} & \textbf{8.91$\pm$2.08} & 1.18$\pm$0.61 & 86.32$\pm$0.32 \\
\midrule
\multirow{3}{*}{(PN,UU)} & WLMC & 2.70$\pm$0.77 & 10.42$\pm$1.74 & 0.92$\pm$0.36 & 3.65$\pm$0.63 & \textbf{16.33$\pm$2.94} & 0.95$\pm$0.34 & \textbf{86.51$\pm$0.33} \\
 & Temp & \textbf{2.54$\pm$0.57} & \textbf{9.90$\pm$1.47} & \textbf{0.67$\pm$0.25} & \textbf{3.39$\pm$0.71} & 17.31$\pm$2.81 & \textbf{0.76$\pm$0.18} & 86.50$\pm$0.35 \\
 & Platt & 2.66$\pm$0.62 & 10.36$\pm$2.09 & 0.67$\pm$0.25 & 3.51$\pm$0.73 & 16.85$\pm$3.61 & 0.78$\pm$0.20 & 86.39$\pm$0.27 \\
\midrule
\multirow{4}{*}{(PU,PU)} & ERM & 2.90$\pm$0.84 & 11.14$\pm$2.70 & 0.88$\pm$0.26 & 4.58$\pm$1.17 & 19.96$\pm$3.28 & 1.03$\pm$0.22 & \textbf{85.71$\pm$0.17} \\
 & WLMC & 2.98$\pm$1.03 & 10.94$\pm$2.98 & 1.03$\pm$0.48 & 4.27$\pm$1.21 & \textbf{18.79$\pm$2.96} & 1.29$\pm$0.49 & 85.48$\pm$0.62 \\
 & Temp & 2.45$\pm$0.92 & \textbf{10.74$\pm$2.87} & 0.83$\pm$0.36 & 4.24$\pm$0.93 & 20.37$\pm$3.74 & 1.05$\pm$0.31 & 85.71$\pm$0.17 \\
 & Platt & \textbf{2.28$\pm$0.54} & 10.96$\pm$1.54 & \textbf{0.81$\pm$0.29} & \textbf{3.96$\pm$1.26} & 18.82$\pm$4.93 & \textbf{1.00$\pm$0.21} & 85.33$\pm$0.28 \\
\midrule
\multirow{4}{*}{(Pconf,Pconf)} & ERM & 26.70$\pm$1.07 & 32.90$\pm$1.35 & 12.30$\pm$0.82 & 41.24$\pm$10.11 & 94.39$\pm$43.41 & 21.88$\pm$9.92 & 83.32$\pm$0.90 \\
 & WLMC & 7.67$\pm$1.27 & \textbf{14.28$\pm$1.66} & 1.24$\pm$0.45 & 11.75$\pm$5.07 & 22.65$\pm$10.46 & 2.80$\pm$2.22 & 84.48$\pm$0.81 \\
 & Temp & 11.38$\pm$2.87 & 23.14$\pm$4.65 & 2.13$\pm$0.48 & 16.89$\pm$4.87 & 41.35$\pm$8.25 & 4.86$\pm$3.42 & 83.32$\pm$0.90 \\
 & Platt & \textbf{4.54$\pm$1.31} & 15.19$\pm$2.68 & \textbf{0.94$\pm$0.34} & \textbf{4.86$\pm$1.16} & \textbf{14.51$\pm$3.21} & \textbf{1.87$\pm$1.10} & \textbf{85.96$\pm$0.22} \\
\midrule
\multirow{4}{*}{(UU,UU)} & ERM & 75.05$\pm$2.42 & 88.14$\pm$0.93 & 36.51$\pm$3.21 & 74.89$\pm$2.53 & 89.36$\pm$3.35 & 35.66$\pm$4.08 & 17.61$\pm$0.63 \\
 & WLMC & 12.34$\pm$2.17 & \textbf{25.67$\pm$4.63} & 2.78$\pm$0.42 & 12.49$\pm$2.23 & \textbf{27.38$\pm$5.07} & 2.91$\pm$0.54 & 75.54$\pm$2.95 \\
 & Temp & 34.44$\pm$0.31 & 45.09$\pm$0.13 & 25.47$\pm$0.38 & 34.43$\pm$0.32 & 46.38$\pm$2.11 & 25.29$\pm$1.01 & 17.61$\pm$0.63 \\
 & Platt & \textbf{2.75$\pm$0.76} & 27.10$\pm$3.86 & \textbf{2.08$\pm$0.43} & \textbf{3.07$\pm$0.89} & 29.15$\pm$4.89 & \textbf{2.31$\pm$0.43} & \textbf{82.62$\pm$0.00} \\
\bottomrule
\end{tabular}%
}
\end{table*}

\clearpage

\section{Large-model experimental supplements}
\label{app:large-model-supp}

This section provides compact test summaries for the CelebA-ViT image model, the CelebA-ImageResNet companion image model, and the CivilComments-BERT text model.  The protocol, metric, weak-data construction, and hyperparameter definitions are given in Appendices~\ref{app:dataset-subgroups}--\ref{app:calib-hparams}; Tables~\ref{tab:celeba-vit-main}, \ref{tab:celeba-resnet-main}, and~\ref{tab:civil-bert-main} focus on test-split results.

\paragraph{Large-model result summary.}
The large-model results show that weak post-processing remains effective for both image and text classifiers, but the relative ranking of the post-processing methods is more mixed than in the tabular experiments.  For CelebA-ViT and CelebA-ImageResNet, WLMC often gives strong improvements on subgroup-sensitive calibration metrics such as MC and maxECE, while Platt scaling and temperature scaling are also competitive in several settings.  For CivilComments-BERT, the global post-processing baselines are particularly strong: Platt scaling and temperature scaling frequently achieve large gains and can match or outperform WLMC on some reported metrics.  Thus, these experiments show that the proposed WSL framework supports effective weak post-processing across modern image and text models, rather than as showing uniform dominance of WLMC over simpler calibration methods.  Overall, WLMC is most useful for subgroup-sensitive correction, while Platt and temperature scaling remain strong baselines, especially for the CivilComments-BERT task.

\subsection{CelebA-ViT}
\label{app:large-model-compact-tables}
Table~\ref{tab:celeba-vit-main} reports the test-split values underlying the CelebA-ViT entries in Figure~\ref{fig:large-model-heatmap}.

\begin{table*}[t]
\centering
\scriptsize
\setlength{\tabcolsep}{3.2pt}
\caption{CelebA-ViT test results. Calibration metrics are reported as $10^2\!\times$ mean $\pm$ standard deviation; accuracy is reported in percent. Bold marks the single best mean within each $(\mathrm{base},\mathrm{post})$ block; variance is not used for bolding.}
\label{tab:celeba-vit-main}
\resizebox{\textwidth}{!}{%
\begin{tabular}{llccccccc}
\toprule
\multirow{2}{*}{(Base,Post)} & \multirow{2}{*}{Post-processing} & \multicolumn{3}{c}{oracle estimate ($\times 10^{-2}$)} & \multicolumn{3}{c}{weak estimate ($\times 10^{-2}$)} & \multirow{2}{*}{Acc. (\%)} \\
\cmidrule(lr){3-5}\cmidrule(lr){6-8}
& & ECE & maxECE & MC & ECE & maxECE & MC & \\
\midrule
\multirow{4}{*}{(PN,PN)} & ERM & 1.68$\pm$0.60 & 6.05$\pm$0.40 & 0.60$\pm$0.14 & 1.68$\pm$0.60 & 6.05$\pm$0.40 & 0.60$\pm$0.14 & 92.11$\pm$0.27 \\
 & WLMC & \textbf{1.62$\pm$0.68} & \textbf{5.81$\pm$0.48} & \textbf{0.45$\pm$0.11} & \textbf{1.62$\pm$0.68} & \textbf{5.81$\pm$0.48} & \textbf{0.45$\pm$0.11} & 92.11$\pm$0.27 \\
 & Temp & 2.72$\pm$0.23 & 6.75$\pm$0.68 & 0.90$\pm$0.04 & 2.72$\pm$0.23 & 6.75$\pm$0.68 & 0.90$\pm$0.04 & 92.11$\pm$0.27 \\
 & Platt & 2.00$\pm$0.22 & 5.81$\pm$0.39 & 1.00$\pm$0.13 & 2.00$\pm$0.22 & 5.81$\pm$0.39 & 1.00$\pm$0.13 & \textbf{92.17$\pm$0.21} \\
\midrule
\multirow{4}{*}{(PN,PU)} & ERM & 1.68$\pm$0.60 & 6.05$\pm$0.40 & 0.60$\pm$0.14 & 1.86$\pm$0.42 & 7.46$\pm$0.89 & 0.66$\pm$0.18 & 92.11$\pm$0.27 \\
 & WLMC & 1.65$\pm$0.64 & \textbf{5.69$\pm$0.50} & 0.51$\pm$0.07 & 1.83$\pm$0.46 & \textbf{7.09$\pm$0.88} & 0.56$\pm$0.06 & 92.11$\pm$0.27 \\
 & Temp & 1.53$\pm$0.51 & 5.94$\pm$0.53 & 0.54$\pm$0.04 & 1.66$\pm$0.40 & 7.45$\pm$1.02 & 0.56$\pm$0.07 & 92.11$\pm$0.27 \\
 & Platt & \textbf{0.61$\pm$0.19} & 5.81$\pm$0.45 & \textbf{0.49$\pm$0.05} & \textbf{0.80$\pm$0.33} & 7.46$\pm$0.73 & \textbf{0.49$\pm$0.05} & \textbf{92.18$\pm$0.24} \\
\midrule
\multirow{4}{*}{(PN,UU)} & ERM & 1.68$\pm$0.60 & 6.05$\pm$0.40 & 0.60$\pm$0.14 & 1.75$\pm$0.53 & 6.53$\pm$0.98 & 0.62$\pm$0.14 & 92.11$\pm$0.27 \\
 & WLMC & 1.62$\pm$0.68 & \textbf{5.81$\pm$0.48} & \textbf{0.45$\pm$0.11} & 1.67$\pm$0.59 & 6.29$\pm$0.77 & \textbf{0.47$\pm$0.10} & 92.11$\pm$0.27 \\
 & Temp & 1.49$\pm$0.46 & 5.90$\pm$0.45 & 0.52$\pm$0.06 & 1.52$\pm$0.41 & 6.06$\pm$1.32 & 0.54$\pm$0.06 & 92.11$\pm$0.27 \\
 & Platt & \textbf{0.66$\pm$0.20} & 5.90$\pm$0.49 & 0.47$\pm$0.02 & \textbf{0.78$\pm$0.15} & \textbf{5.95$\pm$0.11} & 0.49$\pm$0.02 & \textbf{92.17$\pm$0.24} \\
\midrule
\multirow{4}{*}{(PU,PU)} & ERM & 1.37$\pm$0.49 & 7.37$\pm$1.13 & 0.66$\pm$0.13 & 1.54$\pm$0.70 & 8.05$\pm$0.90 & 0.71$\pm$0.11 & 90.97$\pm$0.79 \\
 & WLMC & 1.23$\pm$0.11 & 7.14$\pm$1.03 & \textbf{0.53$\pm$0.04} & 1.39$\pm$0.28 & \textbf{7.94$\pm$1.01} & \textbf{0.55$\pm$0.06} & \textbf{91.04$\pm$0.65} \\
 & Temp & 1.29$\pm$0.36 & \textbf{7.12$\pm$1.03} & 0.65$\pm$0.14 & 1.42$\pm$0.47 & 8.41$\pm$1.15 & 0.68$\pm$0.14 & 90.97$\pm$0.79 \\
 & Platt & \textbf{1.05$\pm$0.28} & 7.27$\pm$1.05 & 0.62$\pm$0.12 & \textbf{1.20$\pm$0.34} & 7.98$\pm$1.06 & 0.65$\pm$0.11 & 90.95$\pm$0.75 \\
\midrule
\multirow{4}{*}{(UU,UU)} & ERM & 1.39$\pm$0.41 & 7.14$\pm$0.72 & 0.52$\pm$0.07 & 1.41$\pm$0.38 & 7.77$\pm$0.88 & 0.52$\pm$0.08 & 91.49$\pm$0.28 \\
 & WLMC & 1.38$\pm$0.43 & 6.91$\pm$1.10 & \textbf{0.47$\pm$0.11} & 1.41$\pm$0.37 & 7.81$\pm$0.73 & \textbf{0.47$\pm$0.12} & 91.49$\pm$0.28 \\
 & Temp & 1.32$\pm$0.32 & 6.97$\pm$0.95 & 0.49$\pm$0.07 & 1.37$\pm$0.33 & 7.68$\pm$0.90 & 0.51$\pm$0.06 & 91.49$\pm$0.28 \\
 & Platt & \textbf{0.65$\pm$0.17} & \textbf{5.99$\pm$0.45} & 0.53$\pm$0.09 & \textbf{0.78$\pm$0.24} & \textbf{7.41$\pm$1.18} & 0.53$\pm$0.06 & \textbf{91.65$\pm$0.34} \\
\bottomrule
\end{tabular}%
}
\end{table*}

\subsection{CelebA-ImageResNet}
Table~\ref{tab:celeba-resnet-main} reports the compact test-split results for the companion ResNet image model.

\subsection{CivilComments-BERT}
Table~\ref{tab:civil-bert-main} reports the test-split values underlying the CivilComments-BERT entries in Figure~\ref{fig:large-model-heatmap}.

\begin{table*}[t]
\centering
\scriptsize
\setlength{\tabcolsep}{3.2pt}
\renewcommand{\arraystretch}{0.93}
\caption{CelebA-ImageResNet test results. Calibration metrics are reported as $10^2\!\times$ mean $\pm$ standard deviation over five seeds; accuracy is reported in percent. Bold marks the best mean within each $(\mathrm{base},\mathrm{post})$ block; ties are all bolded.}
\label{tab:celeba-resnet-main}
\resizebox{\textwidth}{!}{%
\begin{tabular}{llccccccc}
\toprule
\multirow{2}{*}{(Base,Post)} & \multirow{2}{*}{Post-processing} & \multicolumn{3}{c}{oracle estimate ($\times 10^{-2}$)} & \multicolumn{3}{c}{weak estimate ($\times 10^{-2}$)} & \multirow{2}{*}{Acc. (\%)} \\
\cmidrule(lr){3-5}\cmidrule(lr){6-8}
& & ECE & maxECE & MC & ECE & maxECE & MC & \\
\midrule
\multirow{4}{*}{$\left(\mathrm{PN},\mathrm{PN}\right)$} & ERM & \textbf{1.01$\pm$0.49} & \textbf{3.40$\pm$1.00} & \textbf{0.33$\pm$0.08} & \textbf{1.01$\pm$0.49} & \textbf{3.40$\pm$1.00} & \textbf{0.33$\pm$0.08} & 94.38$\pm$0.15 \\
 & WLMC & \textbf{1.01$\pm$0.49} & \textbf{3.40$\pm$1.00} & \textbf{0.33$\pm$0.08} & \textbf{1.01$\pm$0.49} & \textbf{3.40$\pm$1.00} & \textbf{0.33$\pm$0.08} & 94.38$\pm$0.15 \\
 & Temp & 1.54$\pm$0.69 & 4.80$\pm$1.02 & 0.43$\pm$0.07 & 1.54$\pm$0.69 & 4.80$\pm$1.02 & 0.43$\pm$0.07 & 94.38$\pm$0.15 \\
 & Platt & 1.48$\pm$0.16 & 3.46$\pm$0.49 & 0.60$\pm$0.09 & 1.48$\pm$0.16 & 3.46$\pm$0.49 & 0.60$\pm$0.09 & \textbf{94.42$\pm$0.05} \\
\midrule
\multirow{4}{*}{$\left(\mathrm{PN},\mathrm{PU}\right)$} & ERM & 1.01$\pm$0.49 & \textbf{3.40$\pm$1.00} & 0.33$\pm$0.08 & 1.16$\pm$0.37 & 4.94$\pm$1.32 & 0.40$\pm$0.09 & 94.38$\pm$0.15 \\
 & WLMC & 1.01$\pm$0.49 & \textbf{3.40$\pm$1.00} & 0.33$\pm$0.08 & 1.16$\pm$0.37 & 4.94$\pm$1.32 & 0.40$\pm$0.09 & 94.38$\pm$0.15 \\
 & Temp & 0.90$\pm$0.74 & 3.75$\pm$1.11 & 0.28$\pm$0.06 & 1.10$\pm$0.55 & \textbf{4.90$\pm$1.20} & 0.32$\pm$0.07 & 94.38$\pm$0.15 \\
 & Platt & \textbf{0.43$\pm$0.09} & 3.44$\pm$0.29 & \textbf{0.25$\pm$0.03} & \textbf{0.80$\pm$0.07} & 4.96$\pm$1.07 & \textbf{0.29$\pm$0.07} & \textbf{94.42$\pm$0.07} \\
\midrule
\multirow{4}{*}{$\left(\mathrm{PN},\mathrm{UU}\right)$} & ERM & 1.01$\pm$0.49 & \textbf{3.40$\pm$1.00} & 0.33$\pm$0.08 & 1.12$\pm$0.47 & \textbf{3.74$\pm$0.83} & 0.36$\pm$0.09 & 94.38$\pm$0.15 \\
 & WLMC & 1.01$\pm$0.49 & \textbf{3.40$\pm$1.00} & 0.33$\pm$0.08 & 1.12$\pm$0.47 & \textbf{3.74$\pm$0.83} & 0.36$\pm$0.09 & 94.38$\pm$0.15 \\
 & Temp & 0.87$\pm$0.77 & 3.87$\pm$1.08 & 0.26$\pm$0.08 & 1.07$\pm$0.67 & 4.29$\pm$1.15 & 0.27$\pm$0.08 & 94.38$\pm$0.15 \\
 & Platt & \textbf{0.39$\pm$0.17} & 3.52$\pm$0.43 & \textbf{0.24$\pm$0.06} & \textbf{0.69$\pm$0.06} & 3.78$\pm$0.66 & \textbf{0.25$\pm$0.07} & \textbf{94.41$\pm$0.06} \\
\midrule
\multirow{4}{*}{$\left(\mathrm{PU},\mathrm{PU}\right)$} & ERM & 1.33$\pm$0.51 & \textbf{5.21$\pm$0.31} & 0.47$\pm$0.18 & 1.47$\pm$0.52 & 6.38$\pm$0.91 & 0.49$\pm$0.14 & \textbf{93.23$\pm$0.16} \\
 & WLMC & 1.22$\pm$0.43 & 5.23$\pm$0.35 & 0.40$\pm$0.10 & 1.35$\pm$0.45 & 6.30$\pm$0.71 & 0.42$\pm$0.11 & \textbf{93.23$\pm$0.16} \\
 & Temp & 1.05$\pm$0.43 & 5.48$\pm$0.41 & 0.37$\pm$0.06 & 1.14$\pm$0.49 & \textbf{5.84$\pm$0.75} & 0.39$\pm$0.15 & \textbf{93.23$\pm$0.16} \\
 & Platt & \textbf{0.81$\pm$0.18} & 6.15$\pm$1.00 & \textbf{0.31$\pm$0.05} & \textbf{0.93$\pm$0.22} & 6.28$\pm$0.92 & \textbf{0.31$\pm$0.07} & 93.22$\pm$0.18 \\
\midrule
\multirow{4}{*}{$\left(\mathrm{UU},\mathrm{UU}\right)$} & ERM & 1.46$\pm$0.68 & 4.97$\pm$1.04 & 0.56$\pm$0.22 & 1.51$\pm$0.56 & \textbf{5.29$\pm$1.01} & 0.50$\pm$0.17 & \textbf{93.50$\pm$0.24} \\
 & WLMC & 1.29$\pm$0.48 & \textbf{4.82$\pm$0.89} & 0.39$\pm$0.08 & 1.35$\pm$0.39 & 5.53$\pm$1.23 & 0.38$\pm$0.08 & \textbf{93.50$\pm$0.24} \\
 & Temp & 1.11$\pm$0.81 & 5.83$\pm$1.36 & \textbf{0.35$\pm$0.07} & 1.21$\pm$0.67 & 6.33$\pm$1.55 & \textbf{0.36$\pm$0.07} & \textbf{93.50$\pm$0.24} \\
 & Platt & \textbf{0.82$\pm$0.34} & 5.33$\pm$0.66 & \textbf{0.35$\pm$0.06} & \textbf{1.06$\pm$0.38} & 5.89$\pm$0.85 & \textbf{0.36$\pm$0.07} & \textbf{93.50$\pm$0.23} \\
\bottomrule
\end{tabular}%
}
\vspace{0.45em}
\caption{CivilComments-BERT test results. Calibration metrics are reported as $10^2\!\times$ mean $\pm$ standard deviation; accuracy is reported in percent. Bold marks the single best reported mean within each $(\mathrm{base},\mathrm{post})$ block.}
\label{tab:civil-bert-main}
\resizebox{\textwidth}{!}{%
\begin{tabular}{llccccccc}
\toprule
\multirow{2}{*}{(Base,Post)} & \multirow{2}{*}{Post-processing} & \multicolumn{3}{c}{oracle estimate ($\times 10^{-2}$)} & \multicolumn{3}{c}{weak estimate ($\times 10^{-2}$)} & \multirow{2}{*}{Acc. (\%)} \\
\cmidrule(lr){3-5}\cmidrule(lr){6-8}
& & ECE & maxECE & MC & ECE & maxECE & MC & \\
\midrule
\multirow{4}{*}{(PN,PN)} & ERM & 1.64$\pm$0.92 & 9.80$\pm$9.82 & 0.87$\pm$0.92 & 1.64$\pm$0.92 & 9.79$\pm$9.82 & 0.87$\pm$0.92 & \textbf{90.96$\pm$1.30} \\
 & WLMC & 1.35$\pm$0.56 & 9.16$\pm$8.94 & \textbf{0.61$\pm$0.43} & 1.35$\pm$0.56 & 9.16$\pm$8.94 & \textbf{0.61$\pm$0.43} & 90.96$\pm$1.30 \\
 & Temp & 2.48$\pm$0.77 & \textbf{8.64$\pm$7.63} & 1.52$\pm$1.24 & 2.48$\pm$0.77 & \textbf{8.64$\pm$7.62} & 1.52$\pm$1.24 & 90.96$\pm$1.30 \\
 & Platt & \textbf{1.31$\pm$0.88} & 10.83$\pm$7.74 & 0.65$\pm$0.35 & \textbf{1.31$\pm$0.88} & 10.83$\pm$7.74 & 0.65$\pm$0.35 & 90.94$\pm$1.29 \\
\midrule
\multirow{4}{*}{(PN,PU)} & ERM & 1.64$\pm$0.92 & 9.80$\pm$9.82 & 0.87$\pm$0.92 & 1.64$\pm$0.92 & \textbf{9.79$\pm$9.82} & 0.87$\pm$0.92 & \textbf{90.96$\pm$1.30} \\
 & WLMC & 1.27$\pm$0.59 & 9.26$\pm$8.93 & 0.58$\pm$0.45 & 1.35$\pm$0.59 & 11.72$\pm$7.91 & 0.61$\pm$0.39 & 90.96$\pm$1.30 \\
 & Temp & 0.72$\pm$0.57 & 8.76$\pm$8.90 & 0.40$\pm$0.46 & 0.79$\pm$0.57 & 11.42$\pm$7.89 & 0.44$\pm$0.46 & 90.96$\pm$1.30 \\
 & Platt & \textbf{0.59$\pm$0.48} & \textbf{8.10$\pm$9.34} & \textbf{0.37$\pm$0.47} & \textbf{0.75$\pm$0.47} & 10.74$\pm$8.20 & \textbf{0.43$\pm$0.47} & 90.89$\pm$1.26 \\
\midrule
\multirow{4}{*}{(PN,UU)} & ERM & 1.64$\pm$0.92 & 9.80$\pm$9.82 & 0.87$\pm$0.92 & 1.64$\pm$0.92 & 9.79$\pm$9.82 & 0.87$\pm$0.92 & \textbf{90.96$\pm$1.30} \\
 & WLMC & 1.35$\pm$0.55 & 9.19$\pm$8.93 & 0.61$\pm$0.43 & 1.38$\pm$0.57 & 10.04$\pm$9.06 & 0.60$\pm$0.42 & 90.96$\pm$1.30 \\
 & Temp & 0.71$\pm$0.57 & 8.83$\pm$8.87 & 0.41$\pm$0.46 & 0.74$\pm$0.60 & 10.22$\pm$8.80 & 0.43$\pm$0.47 & 90.96$\pm$1.30 \\
 & Platt & \textbf{0.57$\pm$0.47} & \textbf{8.02$\pm$9.44} & \textbf{0.37$\pm$0.47} & \textbf{0.61$\pm$0.48} & \textbf{9.17$\pm$9.35} & \textbf{0.39$\pm$0.49} & 90.88$\pm$1.26 \\
\midrule
\multirow{4}{*}{(PU,PU)} & ERM & 2.79$\pm$1.78 & 10.78$\pm$5.42 & 1.24$\pm$1.05 & 2.79$\pm$1.81 & 11.56$\pm$4.00 & 1.21$\pm$1.06 & 90.81$\pm$0.09 \\
 & WLMC & 2.27$\pm$1.29 & 9.26$\pm$3.69 & 0.70$\pm$0.48 & 2.24$\pm$1.34 & 10.01$\pm$2.01 & 0.65$\pm$0.46 & 90.81$\pm$0.09 \\
 & Temp & 1.62$\pm$0.80 & 8.78$\pm$3.93 & 0.47$\pm$0.26 & 1.61$\pm$0.86 & 10.01$\pm$2.60 & 0.47$\pm$0.26 & 90.81$\pm$0.09 \\
 & Platt & \textbf{1.06$\pm$0.16} & \textbf{5.88$\pm$1.66} & \textbf{0.30$\pm$0.10} & \textbf{1.00$\pm$0.24} & \textbf{8.31$\pm$1.63} & \textbf{0.32$\pm$0.10} & \textbf{90.90$\pm$0.11} \\
\midrule
\multirow{4}{*}{(UU,UU)} & ERM & 1.56$\pm$0.53 & 6.21$\pm$1.12 & 0.54$\pm$0.21 & 1.57$\pm$0.48 & 7.14$\pm$0.80 & 0.55$\pm$0.21 & \textbf{91.22$\pm$0.16} \\
 & WLMC & 1.46$\pm$0.49 & 6.00$\pm$1.22 & 0.47$\pm$0.12 & 1.47$\pm$0.44 & 6.94$\pm$0.91 & 0.47$\pm$0.12 & 91.22$\pm$0.16 \\
 & Temp & 1.22$\pm$0.66 & 5.69$\pm$1.16 & 0.27$\pm$0.14 & 1.22$\pm$0.62 & \textbf{6.68$\pm$0.94} & 0.28$\pm$0.15 & 91.22$\pm$0.16 \\
 & Platt & \textbf{1.08$\pm$0.55} & \textbf{5.52$\pm$0.82} & \textbf{0.25$\pm$0.11} & \textbf{1.09$\pm$0.55} & 6.70$\pm$0.76 & \textbf{0.25$\pm$0.11} & 91.20$\pm$0.13 \\
\bottomrule
\end{tabular}%
}
\renewcommand{\arraystretch}{1.0}
\end{table*}

\subsection{Supplementary notes on the large-model improvement heatmap}
\label{app:large-model-heatmap-supp}
Figure~\ref{fig:app-large-model-heatmap} reproduces the two-panel improvement heatmap used in the main-text large-model comparison.  The panels use the same settings as the CelebA-ViT and CivilComments-BERT compact result tables, Tables~\ref{tab:celeba-vit-main} and~\ref{tab:civil-bert-main}.  For each dataset, setting, and metric $m\in\{\mathrm{ECE},\mathrm{MC}\}$, each cell in Figure~\ref{fig:app-large-model-heatmap} displays the best reported test mean among WLMC, temperature scaling, and Platt scaling for that metric. This descriptive comparison is distinct from validation-based selection of a fitted predictor.  The first line in each cell is the signed improvement
\[
\Delta m = m_{\mathrm{ERM}} - m_{\mathrm{post}},
\]
so positive values correspond to a reduction in calibration error.  The remaining lines in each cell of Figure~\ref{fig:app-large-model-heatmap} list the selected method, the metric transition $m_{\mathrm{ERM}}\to m_{\mathrm{post}}$, and the corresponding accuracy transition $\mathrm{Acc}_{\mathrm{ERM}}\to\mathrm{Acc}_{\mathrm{post}}$ for the same selected method.  Accuracy is therefore reported as contextual information for the calibration-improving method rather than as a separate heatmap row.

The colors summarize signed improvement after row-wise normalization over both panels: the ECE row is normalized by the largest absolute ECE improvement across CelebA and CivilComments, and the MC row is normalized analogously.  This makes the colors comparable between CelebA and CivilComments within the same metric row, while the exact numerical entries should be used when comparing ECE against MC.  A near-white cell means little change relative to the largest change in that row, green indicates improvement, and red indicates degradation.

\begin{figure*}[t]
    \centering
    \includegraphics[width=0.98\textwidth]{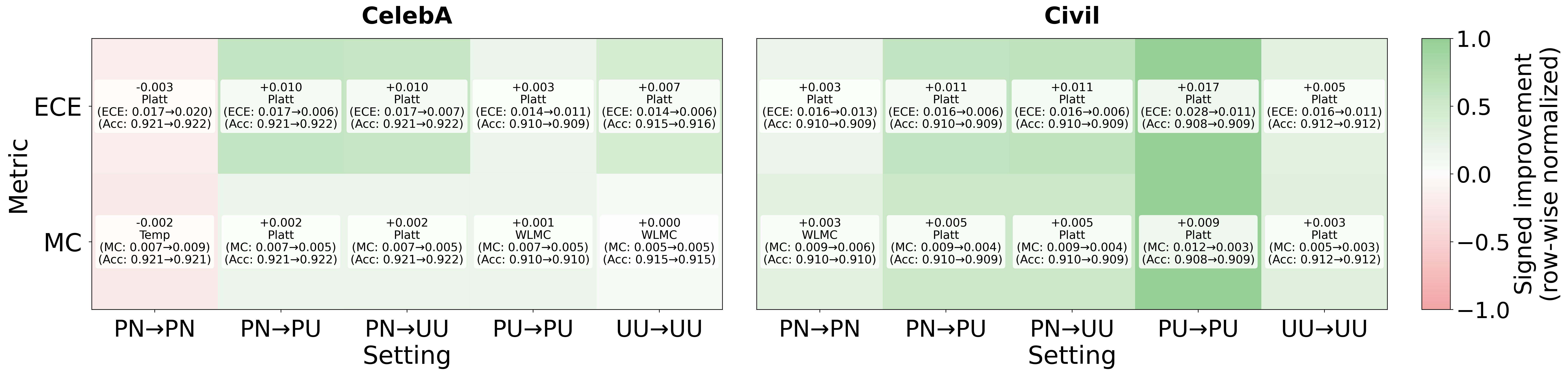}
    \caption{Supplementary view of the large-model improvement heatmap.  For each ECE or MC cell, the displayed method has the best reported test mean among WLMC, temperature scaling, and Platt scaling for that metric.  The cell reports the signed improvement, the selected method, the calibration-metric before/after values, and the accuracy before/after values.  Colors are row-wise normalized over both panels.}
    \label{fig:app-large-model-heatmap}
\end{figure*}

Figures~\ref{fig:celeba-cdf}--\ref{fig:celeba-mc-scatter} provide supplementary CelebA-ViT diagnostics. The MC scatter uses the additional setting $\eta=0.02$ and update cap $100$, as indicated by its plotted run label; it is not the $\eta=0.04$, cap-$50$ setting in Table~\ref{tab:wlmc-hparams}.  Per-group ECE CDFs show how post-processing shifts the group-wise ECE distribution, while the MC before/after scatter tracks changes in the supremum-style group--bin violation.

\paragraph{Subgroup-family diagnostics.}
These plots also partially address the sensitivity of the conclusions to the evaluated subgroup family.  They are not a full sweep over alternative families: all panels use the fixed finite CelebA subgroup collection described in Appendix~\ref{app:dataset-subgroups} and the same ten prediction bins as the main experiments.  However, the CDF view is useful because it shows whether post-processing improves calibration broadly across the evaluated subgroups or only changes a small number of extreme cases.  A leftward shift of the per-group ECE CDF indicates improvement throughout the subgroup family, while a change only in the upper tail would suggest that the method is mainly affecting the worst groups.  The MC scatter complements this distributional view by tracking the supremum-type group--bin violation.  Thus, Figures~\ref{fig:celeba-cdf}--\ref{fig:celeba-mc-scatter} should be read as a subgroup-family diagnostic for the current interpretable group collection.  A stronger robustness check would additionally sweep the number of prediction bins or add and remove subgroup intersections; the current CDF/scatter analysis provides evidence over the reported family but does not replace that alternative-family sensitivity study.

\begin{figure*}[!htbp]
    \centering
    \includegraphics[width=0.98\textwidth]{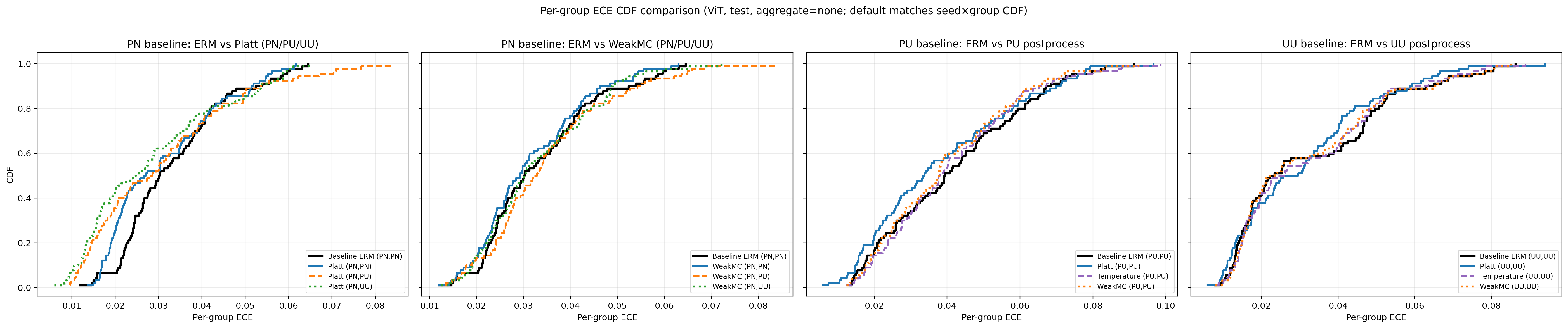}
    \caption{Per-group ECE CDFs on CelebA-ViT (test split). A leftward shift indicates improvement of the per-group ECE distribution. Platt scaling is especially effective for ECE-like metrics, whereas WLMC is more conservative in ECE terms.}
    \label{fig:celeba-cdf}
\end{figure*}

\begin{figure*}[!htbp]
    \centering
    \includegraphics[width=0.98\textwidth]{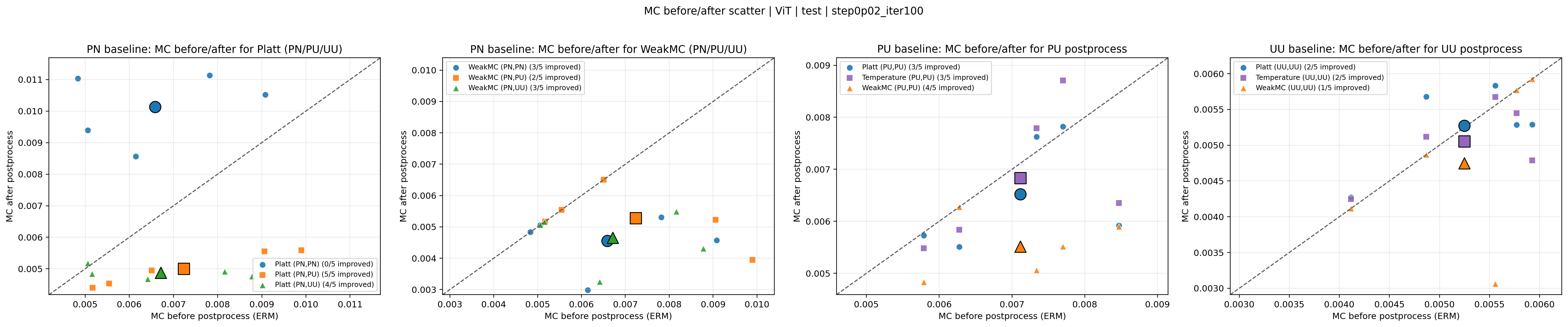}
    \caption{MC before/after scatter plots on CelebA-ViT (test split; $\eta=0.02$, update cap $100$). Points below the diagonal correspond to reduced MC after post-processing correction. WLMC is more directly aligned with MC reduction, while Platt and temperature can show mixed seed-level behavior depending on the regime.}
    \label{fig:celeba-mc-scatter}
\end{figure*}

\end{document}